\documentclass[twoside,11pt]{article}

\usepackage[preprint,nohyperref]{jmlr2e}

\usepackage[T1]{fontenc}
\usepackage[utf8]{inputenc}
\usepackage{microtype}
\usepackage{amsmath,amssymb,mathtools,bm,mathrsfs}
\usepackage{graphicx}
\usepackage{xcolor}
\usepackage{booktabs,array,multirow,tabularx,longtable,makecell}
\newcolumntype{Y}[1]{>{\raggedright\arraybackslash}p{#1}}
\newcolumntype{C}[1]{>{\centering\arraybackslash}p{#1}}
\usepackage{tikz}
\usetikzlibrary{arrows.meta,positioning,fit,backgrounds,calc,shapes.geometric,matrix}
\usepackage{enumitem}
\usepackage{siunitx}
\usepackage{pifont}
\usepackage{rotating}
\usepackage{afterpage}
\usepackage{placeins}
\usepackage{float}
\usepackage{xurl}
\usepackage{lastpage}
\usepackage{comment}
\usepackage[hidelinks]{hyperref}
\usepackage[capitalise,noabbrev]{cleveref}
\allowdisplaybreaks

\hypersetup{
  pdftitle={When Can Depth Replace Precision? A Resource Theory of Quantized Neural Computation},
  pdfauthor={Mojtaba Soltanalian},
  pdfsubject={Structural floors, finite arithmetic, learned codebooks, route-changing networks, and certification for quantized machine learning},
  pdfkeywords={quantized neural networks, depth--precision tradeoffs, residual networks, attention, reachability}}

\definecolor{DeepBlue}{HTML}{173A5E}
\definecolor{MidBlue}{HTML}{2F5D8C}
\definecolor{WarmRed}{HTML}{B6484A}
\definecolor{WarmOrange}{HTML}{D97732}
\definecolor{SeaGreen}{HTML}{3D8B74}
\definecolor{SoftGray}{HTML}{F2F3F4}
\definecolor{Gold}{HTML}{C49A3A}
\definecolor{SoftBlue}{HTML}{EAF2F8}
\definecolor{SoftGreen}{HTML}{EAF6F1}
\definecolor{SoftOrange}{HTML}{FDF2E9}
\definecolor{Purple}{HTML}{6C4A8E}

\newcommand{\glancebox}[3]{%
  \noindent\begingroup
  \fcolorbox{#1}{#2}{%
    \begin{minipage}{\dimexpr\linewidth-2\fboxsep-2\fboxrule\relax}
    #3
    \end{minipage}}%
  \endgroup}

\newtheorem{assumption}[theorem]{Assumption}
\renewenvironment{proof}[1][Proof]{\par\noindent\textbf{#1.}\ }{\hfill\BlackBox\par\medskip}
\crefname{theorem}{theorem}{theorems}
\Crefname{theorem}{Theorem}{Theorems}
\crefname{assumption}{assumption}{assumptions}
\Crefname{assumption}{Assumption}{Assumptions}
\crefname{lemma}{lemma}{lemmas}
\Crefname{lemma}{Lemma}{Lemmas}
\crefname{proposition}{proposition}{propositions}
\Crefname{proposition}{Proposition}{Propositions}
\crefname{corollary}{corollary}{corollaries}
\Crefname{corollary}{Corollary}{Corollaries}
\crefname{definition}{definition}{definitions}
\Crefname{definition}{Definition}{Definitions}
\crefname{remark}{remark}{remarks}
\Crefname{remark}{Remark}{Remarks}

\newcommand{\R}{\mathbb{R}}

\newcommand{\Z}{\mathcal{Z}}

\newcommand{\Rel}{\mathrm{rel}}
\newcommand{\dist}{\operatorname{dist}}
\newcommand{\conv}{\operatorname{conv}}

\newcommand{\Lip}{\operatorname{Lip}}
\newcommand{\soft}{\mathcal{S}}
\newcommand{\norm}[1]{\left\lVert #1\right\rVert}
\newcommand{\abs}[1]{\left\lvert #1\right\rvert}
\newcommand{\inner}[2]{\left\langle #1,#2\right\rangle}
\newcommand{\eps}{\varepsilon}
\newcommand{\one}{\mathbf{1}}
\newcommand{\RR}{\mathscr{R}}

\newcommand{\softmax}{\operatorname{softmax}}
\newcommand{\TopK}{\operatorname{TopK}}

\newcommand{\dH}{d_{\mathrm H}}

\jmlrheading{0}{2026}{1--\pageref{LastPage}}{7/26}{}{}{Mojtaba Soltanalian}
\ShortHeadings{A Resource Theory of Quantized Neural Computation}{Soltanalian}
\firstpageno{1}

\title{When Can Depth Replace Precision?\\
A Resource Theory of Quantized Neural Computation}
\author{\name Mojtaba Soltanalian \email msol@uic.edu \\
       \addr Department of Electrical and Computer Engineering\\
       University of Illinois Chicago\\
       851 S. Morgan St., MC 154\\
       Chicago, IL 60607, USA}
\editor{}

\begin{document}
\maketitle

\begin{abstract}
When can additional low-bit residual computation replace missing numerical
precision for a fixed input--output map? We model a quantized residual system
over a fixed horizon as a pure schedule selecting fields from a declared
low-bit operation library, and use relaxed controls to characterize its
infinite-depth limit. The distance from the target to the closed relaxed
reachable set is the exact structural floor: no increase in depth can remove
it for that library. Pure schedules approach the relaxed class at rate
$O(D^{-1})$ under bounded-variation time dependence and
$O(D^{-\vartheta}+D^{-1})$ under $\vartheta$-H\"older dependence. Execution
arithmetic can reverse this conclusion: full-state write-back introduces a
$D\rho_z$ penalty and can freeze residual updates, whereas increment error
feedback replaces this growth by a bounded carry term and obeys an exact
common-lattice conservation law. A fixed-teacher converse makes this rate sharp: for coherent depth-$L$ first-order
high-precision comparators, accuracy matching requires $D=\Theta(L)$. Learned codebooks add a metadata
resource, while state-dependent routing introduces hybrid event conditions.
Verified primal and dual bounds yield feasible, impossible, or unresolved
decisions before training. Companion software implements the workflow, and
Lean~4 machine-checks the exact discrete core. Depth replaces precision only
relative to a declared library, horizon, execution semantics, and routing
model.
\end{abstract}
\vspace{3mm}


\clearpage
\thispagestyle{plain}
\begingroup
\setlength{\parindent}{0pt}
\setlength{\parskip}{0pt}
\setlength{\fboxsep}{4.5pt}
\setlength{\fboxrule}{0.7pt}
\setlength{\abovedisplayskip}{2pt}
\setlength{\belowdisplayskip}{2pt}
\footnotesize

\begin{center}
{\Large\bfseries\color{DeepBlue} At a Glance: Master Law, Central Claims, and Use}\par
\vspace{1pt}
{\small The theorem spine, deployment workflow, and reading paths on one page.}
\end{center}
\vspace{3pt}

\glancebox{DeepBlue}{SoftBlue}{%
\textbf{Master law.}
Under the integrated theorem's assumptions,
\begin{equation}
\boxed{\left|E^{\mathrm{impl}}_{D,s}(F^\star)-E_{\Omega,\infty}(F^\star)\right|
\le
\underbrace{\frac{C_{\mathrm{syn}}}{D}}_{\text{pure synthesis}}
+
\underbrace{C_{\mathrm{meta}}2^{-s/m}}_{\text{metadata}}
+
\underbrace{\mathcal A_D}_{\text{execution arithmetic}}}
\label{eq:integrated-law-intro}
\end{equation}
where $E_{\Omega,\infty}(F^\star)=\dist(F^\star,\mathscr R_{\Omega,\mathrm{rel}})$ is the
structural floor for the declared dictionary family.  The three radii quantify
finite pure scheduling, finite description of a learned dictionary, and the
declared numerical execution semantics.  For $\vartheta$-H\"older fields, the
synthesis term becomes $C_\vartheta D^{-\vartheta}+C_1D^{-1}$; state-dependent
hard routing contributes a separate event-window radius closed by a small-gain
condition.
}

\vspace{4pt}
\noindent
\begin{minipage}[t]{0.495\linewidth}
\glancebox{SeaGreen}{SoftGreen}{%
\textbf{Five central claims.}\vspace{2pt}

\textbf{1. Exact structural floor.}
For the declared family, finite-depth error converges to the distance from the
target map to the closed relaxed reachable set.  Depth and optimization cannot
cross this floor.  (\Cref{thm:floor-rate-main,thm:learned-codebook-main})

\vspace{3pt}
\textbf{2. Constructive pure synthesis.}
Pure schedules approximate relaxed computation at $O(D^{-1})$ under bounded
variation and at $O(D^{-\vartheta}+D^{-1})$ under $\vartheta$-H\"older time
dependence.  (\Cref{thm:floor-rate-main,thm:holder-time-main})

\vspace{3pt}
\textbf{3. Finite arithmetic changes the phase.}
Full-state write-back contributes $D\rho_z$ and can freeze residual updates.
Increment error feedback replaces this growth by a bounded carry and an exact
common-lattice conservation law.
(\Cref{thm:error-feedback-main}; \Cref{prop:fixed-grid-freeze-main,prop:bit-exact-error-feedback-main})

\vspace{3pt}
\textbf{4. The first-order price is necessary.}
A fixed teacher has exact $\Theta(D^{-1})$ optimal error, with residual-ReLU and
nonuniform two-token attention realizations.  Thus, for coherent high-precision
comparators with $\Theta(L^{-1})$ error, accuracy matching requires
$D_{\mathrm{match}}=\Theta(L)$.
(\Cref{thm:fixed-teacher-main,thm:attention-fixed-target-main}; \Cref{prop:relu-mlp-fixed-target-main}; \Cref{cor:matching-depth-two-sided-main})

\vspace{2pt}
\textbf{5. Certification precedes training.}
Verified primal and dual bounds yield certified feasible, certified impossible,
or unresolved decisions; unresolved is an intentional outcome.
(\Cref{cor:three-way-decision-main})
}
\end{minipage}\hfill
\begin{minipage}[t]{0.485\linewidth}
\glancebox{WarmOrange}{SoftOrange}{%
\textbf{Practitioners' guide.}\vspace{4pt}
\begin{enumerate}[leftmargin=1.3em,label=\arabic*.,itemsep=5.5pt,topsep=0pt,parsep=0pt]
\item \textbf{Declare the problem.}
Specify the target, domain, norm, low-bit library, shared horizon, schedule,
metadata budget, arithmetic, and routing model.

\item \textbf{Bracket the floor.}
Use a feasible relaxed trajectory for an upper bound and a verified HJB,
support, affine, or SOS witness for a lower bound.

\item \textbf{Allocate resources.}
Search candidate depth--metadata pairs $(D,s)$ using the synthesis and codebook
terms, rather than maximizing depth or minimizing precision by default.

\item \textbf{Audit execution.}
Check write-back, accumulator and carry ranges, saturation or wrapping,
physical increment visibility, route isolation, transversality, and small gain.

\item \textbf{Decide before training.}
For a floor bracket $[L,U]$ and finite-resource radius $R_{D,s}$,
\[
U+R_{D,s}\le\varepsilon_H\Rightarrow\text{feasible},
\]
\[
L-R_{D,s}>\varepsilon_H\Rightarrow\text{impossible}.
\]
Otherwise report unresolved, identify the next bound or measurement needed,
and train only configurations that survive the screen.
\end{enumerate}
}
\end{minipage}

\vspace{4pt}
\glancebox{Gold}{Gold!9}{%
\textbf{Reading paths.}\par\vspace{2pt}
\begin{minipage}[t]{0.49\linewidth}
\textbf{Core argument.}
\Cref{sec:formal-model,sec:finite-arithmetic,sec:necessity-main}, followed by the
decision rule in \Cref{sec:certification-main}.

\vspace{3pt}
\textbf{Quantization practitioner.}
\Cref{sec:finite-arithmetic,sec:architecture-specializations-main}, then
\Cref{sec:certified-evidence-main,sec:design} and the QReplace supplement.

\vspace{3pt}
\textbf{Certification and software.}
\Cref{sec:certification-main}, then Appendices~\ref{app:qreplace-software} and~\ref{app:lean-verification} for QReplace and the
machine-checked discrete core.
\end{minipage}\hfill
\begin{minipage}[t]{0.48\linewidth}
\textbf{Theory reader.}
\Cref{sec:formal-model,sec:necessity-main,sec:learned-dictionaries-main-section},
then the complete theorem appendices.

\vspace{3pt}
\textbf{MoE and routing reader.}
\Cref{sec:hybrid-routing-main}, followed by the routed-architecture
specialization and event diagnostics in
\Cref{sec:architecture-specializations-main,sec:certified-evidence-main}.
\end{minipage}
}
\endgroup
\clearpage

\section{Introduction}
\label{sec:introduction}

Quantization is usually posed as a local approximation problem: replace a
weight tensor, activation, or matrix product by a nearby low-precision object
and control the induced distortion.  This viewpoint has enabled binary and
ternary networks \citep{CourbariauxBengioDavid2015,HubaraEtAl2018,RastegariEtAl2016},
strong post-training reconstruction methods
\citep{NagelEtAl2020AdaRound,FrantarEtAl2023GPTQ,XiaoEtAl2023SmoothQuant,LinEtAl2024AWQ,TsengEtAl2024QuIPSharp,ZhangEtAl2026Qronos},
and native very-low-bit language models
\citep{WangEtAl2025BitNet,MaEtAl2024b158,MaEtAl2025b158,LiuEtAl2025ParetoQ}.
It does not, by itself, answer an architectural question that arises whenever
computation can be refined or reused:
\begin{quote}
\emph{If a high-precision computation is replaced by more low-bit residual
steps, what is the infinite-depth limit, how quickly is it approached, and
which numerical or routing resources must scale for that limit to survive
actual execution?}
\end{quote}

Residual networks and neural differential equations suggest a common
finite-horizon evolution \citep{HeEtAl2016,ChenRubanovaBettencourtDuvenaud2018,HaberRuthotto2018};
unfolded optimization turns iteration count into network depth
\citep{GregorLeCun2010,ChenLiuWangYin2018,MongaLiEldar2021}; looped
Transformers repeatedly apply shared blocks \citep{FangEtAl2026LoopQ}; and
sparse mixture-of-experts (MoE) models route each token through a changing
sequence of expert fields \citep{ShazeerEtAl2017,FedusZophShazeer2022}.  In
these settings, increasing depth is naturally interpreted as \emph{unrolling}
a restricted computation library over a fixed horizon.  The interpretation is
meaningful only when the family across depths is coherent: arbitrarily
subdividing an unrelated block can move away from its original map even in
high precision.

We model a depth-$D$ low-bit student by a \emph{pure schedule}: at each of $D$
microsteps it selects one residual field from a finite declared library.
Allowing a probability vector over the same fields at each instant gives a
\emph{relaxed control}.  The relaxed system is not a parameter-space convex
hull.  Its fields are state dependent, need not commute, and one shared
schedule acts on the complete input-indexed state.  This paper asks what the
best system in that declared family can represent before asking whether a
particular optimizer can find it.

The resulting framework is a target-specific \emph{resource theory of
quantized neural computation}.  It separates the operation library, executed
depth, global codebook metadata, activation and accumulator semantics, route
event geometry, and proof-producing certificates.  The separation is the main
organizing idea of the paper: nominal weight bit width alone does not determine
whether depth can replace precision.

\subsection{Three notions of replacement}
\label{sec:replacement-notions}

The word \emph{replace} is often used for three different goals.  Keeping them
separate is necessary for both theorems and experiments.  Let $F^\star$ be a
desired computation, $F_L^{\mathrm H}$ a depth-$L$ high-precision comparator,
and $F_{D,\sigma}^{\mathrm L}$ a depth-$D$ low-bit student with schedule
$\sigma$.

\paragraph{Direct map emulation.}
The student approximates the comparator itself:
\begin{equation}
 E_b^{\mathrm{emul}}(D;F_L^{\mathrm H})
 :=\inf_\sigma \norm{F_{D,\sigma}^{\mathrm L}-F_L^{\mathrm H}}.
 \label{eq:direct-emulation-intro}
\end{equation}
This is the literal input--output meaning of replacement.

\paragraph{Accuracy matching relative to a common target.}
The student need not reproduce the comparator pointwise; it must be at least
as accurate relative to $F^\star$:
\begin{equation}
 D_{\mathrm{match}}(L,b;F^\star)
 :=\min\left\{D:
 E_b(D;F^\star)\le \norm{F_L^{\mathrm H}-F^\star}\right\}.
 \label{eq:matching-depth-intro}
\end{equation}
This is the natural notion when high-precision and low-bit networks are two
discretizations or implementations of a more accurate common residual
computation.  A central consequence of the theory is the conditional linear
exchange law:
\begin{equation}
\boxed{\begin{gathered}
\varepsilon_H(L)=\Theta(L^{-1}),\qquad
E_b(D;F^\star)-E_{\infty,b}(F^\star)=\Theta(D^{-1})\\
\Longrightarrow\quad D_{\mathrm{match}}(L,b;F^\star)=\Theta(L)
\end{gathered}}
\label{eq:linear-matching-law-intro}
\end{equation}
Here both model families must be coherent refinements relative to the same
target.  Structural feasibility and two upper bounds imply only the sufficient
scaling $D=O(L)$; the lower $D^{-1}$ law is what makes linear depth necessary.

\paragraph{Behavioral replacement.}
For a deployed decision rule $a$, one may ask only that
\begin{equation}
 a\left(F_{D,\sigma}^{\mathrm L}(x)\right)=a\left(F_L^{\mathrm H}(x)\right)
 \quad\text{on a declared domain.}
 \label{eq:behavioral-intro}
\end{equation}
Classification agreement, identical decoded tokens, and identical top-$k$
routes are examples.  Uniform state or logit bounds imply behavioral agreement
when they fit inside a decision margin, but behavioral agreement is weaker
than map equality.  Our structural theorems address the first two notions;
route and task margins connect them to the third.

\subsection{From structural reachability to an implemented decision}
\label{sec:resource-law-intro}

The preceding synopsis states the paper's master envelope and five principal
results.  Its central quantity is the structural floor
$E_{\Omega,\infty}(F^\star)$: the distance from the target to the closed
relaxed reachable class generated by the declared dictionary family.  Pure
depth prices a finite schedule relative to that ideal class; metadata prices a
finite description of a learned dictionary; and the arithmetic radius records
what the execution system loses through activation, scale, accumulator,
saturation, wrapping, or carry semantics.  These are different resources and
must not be collapsed into a nominal weight-bit count.

The envelope separates three questions that are frequently conflated:
\begin{enumerate}[leftmargin=1.65em,itemsep=.25em,topsep=.25em]
\item \textbf{Structural replacement:} is $E_{\Omega,\infty}(F^\star)$ zero for
      the declared dictionary family?
\item \textbf{Synthesis replacement:} how much pure depth is required to
      approach the corresponding relaxed computation?
\item \textbf{Implementation replacement:} does the declared finite arithmetic
      preserve the residual increments used by that synthesis?
\end{enumerate}
The same target may therefore receive different answers under different
libraries, activation grids, accumulator widths, metadata budgets, or route
semantics.  Figure~\ref{fig:resource-overview} summarizes this flow from a
declared comparison problem to a certified decision.

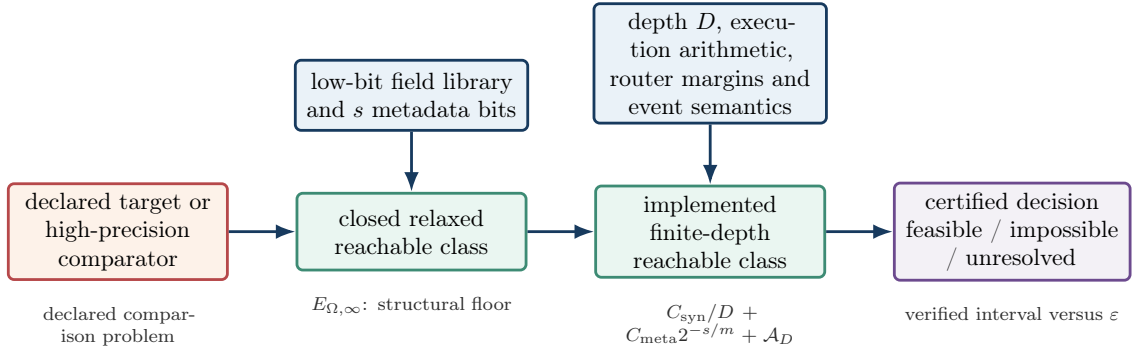
\begin{figure}[t]
\centering
\resizebox{.98\linewidth}{!}{%
\begin{tikzpicture}[
    >=Latex,
    font=\small,
    node distance=9mm and 10mm
]
\tikzset{
  target/.style={
    rounded corners=2.5pt,
    draw=WarmRed,
    very thick,
    fill=SoftOrange,
    text width=3.0cm,
    minimum height=1.15cm,
    align=center,
    inner sep=4pt
  },
  reach/.style={
    rounded corners=2.5pt,
    draw=SeaGreen,
    very thick,
    fill=SoftGreen,
    text width=3.15cm,
    minimum height=1.15cm,
    align=center,
    inner sep=4pt
  },
  resource/.style={
    rounded corners=2.5pt,
    draw=DeepBlue,
    very thick,
    fill=SoftBlue,
    text width=3.2cm,
    minimum height=1.05cm,
    align=center,
    inner sep=4pt
  },
  cert/.style={
    rounded corners=2.5pt,
    draw=Purple,
    very thick,
    fill=Purple!8,
    text width=3.25cm,
    minimum height=1.15cm,
    align=center,
    inner sep=4pt
  },
  arr/.style={->, very thick, draw=DeepBlue},
  note/.style={
    align=center,
    font=\scriptsize,
    text=black!82,
    text width=3.2cm
  }
}

\node[target] (teacher)
  {declared target or\\high-precision comparator};

\node[reach, right=10mm of teacher] (relaxed)
  {closed relaxed\\reachable class};

\node[reach, right=10mm of relaxed] (implemented)
  {implemented finite-depth\\reachable class};

\node[cert, right=10mm of implemented] (certificate)
  {certified decision\\feasible / impossible / unresolved};

\node[resource, above=9mm of relaxed] (dict)
  {low-bit field library\\and $s$ metadata bits};

\node[resource, above=9mm of implemented] (execution)
  {depth $D$, execution arithmetic,\\router margins and event semantics};

\draw[arr] (teacher) -- (relaxed);
\draw[arr] (dict) -- (relaxed);
\draw[arr] (relaxed) -- (implemented);
\draw[arr] (execution) -- (implemented);
\draw[arr] (implemented) -- (certificate);

\node[note, below=2.5mm of teacher] {declared comparison problem};
\node[note, below=2.5mm of relaxed] {$E_{\Omega,\infty}$: structural floor};
\node[note, below=2.5mm of implemented]
  {$C_{\rm syn}/D + C_{\rm meta}2^{-s/m} + \mathcal A_D$};
\node[note, below=2.5mm of certificate]
  {verified interval versus $\varepsilon$};

\end{tikzpicture}}
\caption{\textnormal{The resource-theoretic view.}  The low-bit library
determines the relaxed geometry and therefore the structural floor.  Depth
synthesizes that geometry with a pure schedule, metadata selects or describes
the library, and the execution semantics determine whether the microsteps
remain numerically visible.  Hard routing couples state and event-time errors.
Primal and dual certificates compare the resulting achievable interval with
the requested tolerance.}
\label{fig:resource-overview}
\end{figure}

\subsection{Why finite arithmetic changes the qualitative answer}
\label{sec:arithmetic-preview}

The phrase ``depth replaces precision'' is false without an execution model.
If every microstep writes the complete state to a fixed activation grid of
radius $\rho_z$, the generic envelope becomes
\begin{equation}
 E_D^{\mathrm{writeback}}
 \lesssim E_\infty+\frac{C}{D}+D\rho_z.
 \label{eq:writeback-preview}
\end{equation}
It has a finite optimum $D_{\mathrm{opt}}\asymp\sqrt{C/\rho_z}$; deeper
computation can be worse.  In the scalar recursion
\[
 x_{k+1}=Q_\Delta(x_k+D^{-1}),\qquad x_0=0,
\]
every update rounds away whenever $D>2/\Delta$, so the implemented state
remains zero although the ideal endpoint is one.  The same depth refinement
that enriches the ideal schedule can therefore make every physical residual
update numerically invisible.

Increment quantization with an error-feedback carry instead yields
\begin{equation}
 E_D^{\mathrm{EF}}
 \lesssim E_\infty+\frac{C}{D}+\rho e^{L_zT}+\eta_G\Phi_{L_z}(T),
 \label{eq:error-feedback-preview}
\end{equation}
where $\rho$ bounds the carry and $\eta_G$ bounds field-computation error.  The
quantizer residual telescopes rather than appearing as $D$ unrelated state
perturbations.  On a common integer lattice, an exact conservation identity
and sufficient two's-complement widths guarantee overflow-free residual
execution.  Precision is relocated into a persistent carry or compensated
accumulator; it is not eliminated.  Figure~\ref{fig:arithmetic-resource}
visualizes the two regimes.

\begin{figure}[t]
\centering
\includegraphics[width=.99\linewidth]{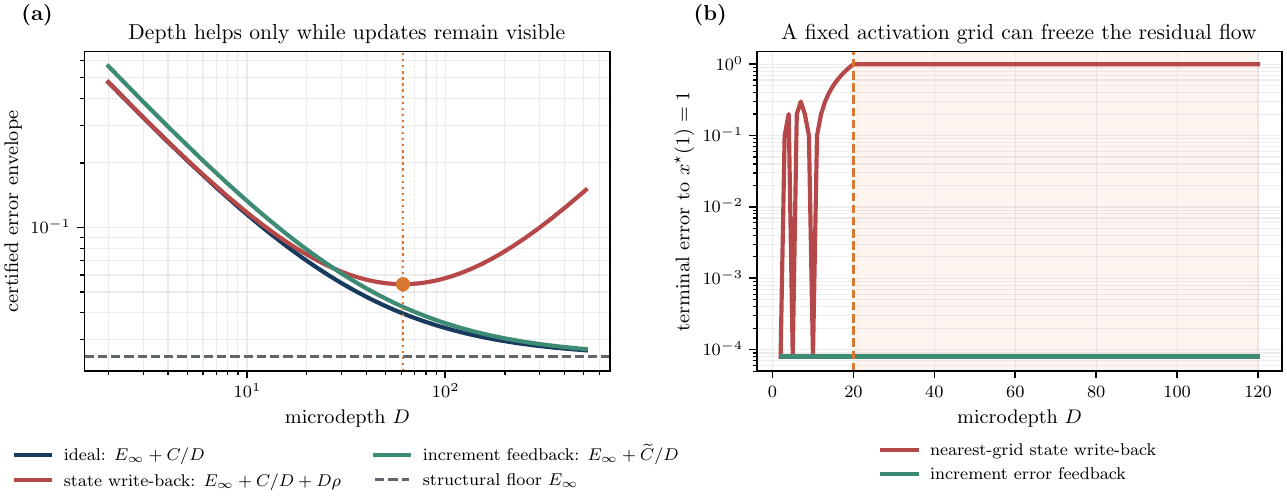}
\caption{\textnormal{Execution arithmetic creates a genuine phase diagram.}
\textnormal{(a)} Ideal synthesis approaches the structural floor, whereas
fixed-grid state write-back has a U-shaped envelope and a finite optimal depth.
Increment error feedback preserves the first-order law when its physical
residual unit scales with the microstep.  \textnormal{(b)} In an exact scalar
diagnostic, a fixed activation grid eventually erases every residual update;
the error-feedback execution does not. }
\label{fig:arithmetic-resource}
\end{figure}

\paragraph{Operational use.}
The five-step workflow on the synopsis page is implemented directly by the
companion recommender: declare the comparison problem, bracket the structural
floor, allocate finite resources, audit execution and routing, and return a
feasible, impossible, or unresolved decision.  Section~\ref{sec:design}
revisits the workflow after all constants and trust boundaries have been
developed.

\subsection{Companion software: recommendation and formal verification}
\label{sec:software-intro}

Two interoperable software supplements turn this workflow into a reusable
artifact.

\begin{itemize}
    \item \textbf{QReplace} accepts a versioned specification of the target,
low-bit dictionary, floor bracket, synthesis constants, metadata budget,
execution arithmetic, route events, and deployment costs.  It searches a
depth--metadata grid, reports the certified or estimated interval
\[
 \max\{0,L_\infty-R_{D,s}\}
 \le E_{D,s}^{\mathrm{impl}}(F^\star)
 \le U_\infty+R_{D,s},
\]
ranks the cost--error Pareto frontier, and returns certified feasible,
certified impossible, conditionally feasible, diagnostically promising, or
unresolved.  Every output includes an assumption ledger and the evidence
provenance of each nonzero radius.  The Python package, command-line and
browser interfaces, JSON schema, and four executable Colab notebooks are
specified in Appendix~\ref{app:qreplace-software}; \cref{tab:qreplace-schema} summarizes the input and trust contract.

    \item \textbf{QReplaceLean} is a pinned Lean~4 companion for seven exact
result groups at the center of the paper: fixed-grid freezing,
error-feedback conservation, common-lattice recovery and sufficient register
ranges, the QReplace feasible/impossible implications, metadata--depth
counting, fixed-teacher residual-recursion algebra, and two-token attention
mode algebra.  Its Colab verifies these groups in numbered stages and exports
a hash-locked result archive containing the exact sources, kernel-build logs,
claim-specific axiom audits, theorem crosswalk, and machine-readable claim
manifest.  The archived run accompanying this paper marks all twelve mapped
claim-level statements \texttt{kernel\_checked}. 
Appendix~\ref{app:lean-verification} gives the theorem crosswalk, faithfulness
audit, build protocol, runtime guidance, evidence identity, and trust boundary;
\cref{tab:lean-crosswalk} states the scope of the principal declarations.
The two artifacts communicate through certificate manifests, so a formal
result may support a QReplace premise without silently certifying unrelated
checkpoint-specific assumptions.
\end{itemize}

\subsection{Novelty boundary and evidence hierarchy}
\label{sec:novelty-boundary-intro}

Several ingredients have mature histories.  Quantized approximation theory
proves class-level depth--precision tradeoffs
\citep{BlumenfeldGilboaSoudry2019,OuSchenkelBolcskei2024}; sigma--delta and
error-feedback methods shape accumulated rounding error
\citep{DaubechiesDeVore2003,Gunturk2003,KrahmerSaabWard2012,OConnorWelling2017};
PTQ analyses study reconstruction, propagation, and accumulator constraints
\citep{ZhangZhouSaab2023PTQ,ColbertEtAl2024AXE,ZhangEtAl2025ProvablePTQ,ZhangEtAl2026Qronos};
learned quantizers construct implicit or adaptive codebooks
\citep{HuijbenEtAl2024QINCo,ZhongEtAl2026RQMoE}; hybrid-system theory treats
transverse events \citep{BurdenEtAl2016EventSelected,KongEtAl2024Saltation};
and formal verification certifies fixed quantized networks
\citep{TeuberEtAl2021Equivalence,HuangEtAl2024QNNVerify}.  The contribution
here is their target-specific synthesis around the closed relaxed reachable
set, with matching necessity mechanisms and a common certification interface.
Section~\ref{sec:related-work} states this boundary in detail.

The paper also separates levels of evidence.  An analytic formula or rational
certificate supports an exact claim over its declared domain.  Exhaustive
finite verification is exact only for the enumerated instance.  A controlled
mechanistic simulation tests one theorem mechanism.  A pretrained-model
diagnostic shows that the mechanism occurs in one checkpoint.  Uncertified
training curves or local searches are evidence about optimization, not proofs
of reachability.  Section~\ref{sec:evidence} maintains this language throughout.

Table~\ref{tab:appendix-map} records where each proof, experiment, and reproducibility component enters the main narrative.

\begin{table}[t]
\centering
\caption{Appendix map.  Appendices are introduced from the main text and
contain complete proofs, experiment specifications, or reproducibility
material.}
\label{tab:appendix-map}
\footnotesize
\renewcommand{\arraystretch}{0.96}
\begin{tabularx}{\linewidth}{Y{.18\linewidth}Y{.34\linewidth}X}
\toprule
Appendix & Contents & Main-text entry point\\
\midrule
Appendix~\ref{app:full-resource-proofs} & Core proofs for finite arithmetic, fixed targets, learned codebooks, BV synthesis, hybrid routing, certification, Transformers, and MoE & Sections~\ref{sec:formal-model}--\ref{sec:certification}\\
Appendix~\ref{app:merged-exact-proofs} & H\"older synthesis, prescribed routing, common-lattice arithmetic, neural converses, matching depth, and compositional HJB & Sections~\ref{sec:formal-model}--\ref{sec:certification}\\
Appendix~\ref{app:soft-details} & Soft-threshold invariance, signed-ray laws, computable floor brackets, and rational PWA methodology & Section~\ref{sec:soft-threshold-main}\\
Appendix~\ref{app:experiments} & Search domains, resource accounting, numerical references, and figure conventions & Section~\ref{sec:certified-evidence-main}\\
Appendix~\ref{app:transformer-experiments} & Attention, scaling, stability, and routed-system settings & Figures~\ref{fig:attention-structural}--\ref{fig:attention-scale}\\
Appendix~\ref{app:pretrained-experiment} & DistilBERT split, QAT, evaluation, and pure-compilation details & Section~\ref{sec:pretrained-causal}\\
Appendix~\ref{app:multiteacher} & Eight-teacher breadth study and dictionary expansion & Figure~\ref{fig:broad-matrix}\\
Appendix~\ref{app:reproducibility} & Claim-to-artifact map, scripts, data, and verification entry points & All numerical sections\\
Appendix~\ref{app:qreplace-software} & QReplace schema, resource algorithm, evidence levels, interfaces, and tests & Section~\ref{sec:software-intro} and the design workflow\\
Appendix~\ref{app:lean-verification} & Lean theorem crosswalk, staged verifier, faithfulness audit, evidence archive, and status policy & Sections~\ref{sec:arithmetic-preview}, \ref{sec:necessity-main}, and~\ref{sec:certification-main}\\
\bottomrule
\end{tabularx}
\end{table}

\paragraph{Scope and trust boundary.}
Depth replaces precision only relative to a declared dictionary, horizon,
execution semantics, and routing model.  The paper does not claim that
arbitrary insertion of unscaled independent layers preserves a trained map,
that parameter-bit savings imply latency or energy savings, or that a failed
optimizer proves nonrepresentability.  An ideal invariant tube does not
certify an implemented trajectory; finite arithmetic requires its own range
and overflow audit.  Hard routing over a continuum of inputs is treated in
$L^p$ or on route cells because the routed map need not belong to $C(\Xi)$.
These are theorem boundaries, not after-the-fact caveats.

\paragraph{Road map.}
Section~\ref{sec:related-work} positions the contribution.  Section~\ref{sec:formal-model}
defines the lifted full-map dynamics and proves the structural-floor and
finite-depth laws.  Section~\ref{sec:finite-arithmetic} studies digital
execution.  Sections~\ref{sec:necessity-main} and
\ref{sec:learned-dictionaries-main-section} establish fixed-target necessity
and metadata resource laws.  Section~\ref{sec:hybrid-routing-main} treats
route-changing systems, and Section~\ref{sec:certification-main} develops the
certificate stack.  Section~\ref{sec:architectures} instantiates the abstract
terms for unfolded networks, Transformers, attention, and MoE blocks.
Section~\ref{sec:evidence} presents exact, controlled, and pretrained evidence;
Sections~\ref{sec:design} and~\ref{sec:discussion} translate the theory into
design guidance and state the remaining limitations.  Appendices~\ref{app:qreplace-software}
and~\ref{app:lean-verification} specify the recommender and formal-verification
companions, respectively.  Table~\ref{tab:reading-guide} gives audience-specific paths through the paper.

\begin{table}[t]
\centering
\caption{Audience-specific reading paths. }
\label{tab:reading-guide}
\footnotesize
\renewcommand{\arraystretch}{0.96}
\setlength{\tabcolsep}{4pt}
\begin{tabularx}{\linewidth}{@{}Y{.23\linewidth}Y{.38\linewidth}X@{}}
\toprule
Audience & Recommended path & Defer on a first pass\\
\midrule
Quantization and systems practitioner & Synopsis $\rightarrow$ \Cref{sec:finite-arithmetic,sec:architecture-specializations-main,sec:design} $\rightarrow$ Appendix~\ref{app:qreplace-software} & Occupation measures, SOS convergence, and detailed hybrid proofs\\
Theory reader & \Cref{sec:formal-model,sec:necessity-main,sec:learned-dictionaries-main-section,sec:certification-main} $\rightarrow$ Appendices~\ref{app:full-resource-proofs} and~\ref{app:merged-exact-proofs} & Pretrained diagnostics and deployment-specific ledgers\\
MoE and routing researcher & \Cref{sec:hybrid-routing-main,sec:transformer-specialization-main} $\rightarrow$ routed-system evidence in \Cref{sec:evidence} & Soft-threshold examples and metadata allocation details\\
Verification and certification reader & \Cref{sec:certification-main,sec:exact-verification-main} $\rightarrow$ Appendices~\ref{app:qreplace-software} and~\ref{app:lean-verification} & Detailed Transformer component bounds\\
Reader seeking only the core argument & Synopsis $\rightarrow$ \Cref{sec:formal-model,sec:finite-arithmetic,sec:necessity-main,sec:three-way-decision-main} & All architecture and software extensions on the first pass\\
\bottomrule
\end{tabularx}
\end{table}

\section{Related Work and Novelty Boundary}
\label{sec:related-work}

The paper sits at the intersection of quantized approximation, residual dynamical systems, mixed-integer control, finite-arithmetic analysis, adaptive codebooks, hybrid routing, and formal certification.  This section is intentionally explicit about the boundary: several ingredients have long and mature histories; the contribution is the target-specific synthesis and certification framework obtained by joining them.

\subsection{Quantized networks and depth--precision tradeoffs}

Binary and low-bit training established that useful neural computation can survive aggressively discrete weights and activations \citep{CourbariauxBengioDavid2015,HubaraEtAl2018,RastegariEtAl2016}.  Universal-approximation and complexity results later quantified the expressive power of quantized ReLU networks \citep{DingEtAl2019}.  Mean-field analyses exposed a quantization--depth tradeoff in randomly initialized deep networks \citep{BlumenfeldGilboaSoudry2019}, while recent minimax approximation theory identifies under-, proper-, and over-quantization regimes and constructs deeper low-precision ReLU networks that preserve memory-optimal approximation rates \citep{OuSchenkelBolcskei2024}.  Fixed-point arithmetic can also change expressivity relative to ideal real arithmetic \citep{HwangParkPark2024FixedPoint}, and concurrent work shows that reduction order and inexact activation implementations can likewise alter representability under floating-point execution \citep{ParkEtAl2026FloatingPoint}.

Recent Transformer lower bounds exhibit tasks with a sharp one-bit precision--expressivity threshold \citep{ChakrabartiPitassiAlman2026}.  Complementary expressivity work analyzes standard softmax decoders with rounded activations and attention weights, allowing depth and width to grow while keeping precision low \citep{BroesamleEckstein2026LowPrecisionSoftmax}.  Those results are task- or computation-class statements; our exact attention converse is instead a fixed-target residual-discretization law with an exact finite-depth optimum.  The results therefore mark complementary precision boundaries rather than one subsuming the others.

These results are class-level or universal: they ask how well a broad function class can be approximated as network size and bit precision vary.  Our object is different.  For one declared operation library and one target full map, we identify the exact infinite-depth relaxed reachable set, its target-specific distance, and the finite implementation radius.  The structural floor can be positive even when the broader architecture class is universal; conversely, a small target-specific library can realize a particular target efficiently without being universal.

At the systems level, native low-bit language models and extreme quantization methods demonstrate that bit width, training procedure, and codebook design interact strongly \citep{WangEtAl2025BitNet,MaEtAl2024b158,MaEtAl2025b158,LiuEtAl2025ParetoQ,BulatOualiTzimiropoulos2024QBB,LeeEtAl2025LittleBit,WangEtAl2026CATQ}.  Our theory does not compare those methods directly.  It supplies a language for asking whether a learned low-bit operation family has the right relaxed geometry, what depth is required to synthesize it, and which execution terms may destroy that gain.

\subsection{Post-training quantization, error propagation, and accumulation}

Post-training quantization (PTQ) typically minimizes a local, layerwise, or blockwise reconstruction objective.  Adaptive rounding, second-order reconstruction, activation smoothing, and incoherent or lattice transformations are prominent examples \citep{NagelEtAl2020AdaRound,FrantarEtAl2023GPTQ,XiaoEtAl2023SmoothQuant,LinEtAl2024AWQ,TsengEtAl2024QuIPSharp}.  Provable PTQ methods analyze greedy path-following and related rounding procedures \citep{ZhangZhouSaab2023PTQ,ZhangZhouSaab2024Corrigendum}; more recent analysis gives nonasymptotic guarantees for OPTQ/GPTQ and Qronos \citep{ZhangEtAl2025ProvablePTQ}, while Qronos explicitly corrects quantization errors propagated from previously processed layers \citep{ZhangEtAl2026Qronos}.  Quantization-error propagation has also been studied as a layerwise phenomenon \citep{AraiIchikawa2025QEP}.  Recent activation-PTQ work identifies phase-wise changes in the Transformer residual stream as a source of depth-dependent sensitivity and assigns higher precision selectively to the affected layers \citep{ZhaoEtAl2026DynamicPTQ}.  Architecture-level evidence further shows that removing residual connections can preserve more quantization-friendly activation statistics, at the cost of a different optimization and accuracy tradeoff \citep{JiEtAl2026ResidualFreeQuantization}.  These perspectives reinforce that quantization difficulty depends on residual-stream dynamics and architecture, rather than on one static activation-error number attached uniformly to every layer.

Finite accumulation is a separate resource from operand precision.  Integer-only systems model nonlinearities, scaling, and accumulators explicitly \citep{KimEtAl2021IBERT,YaoEtAl2021HAWQV3,HuEtAl2024ILLM}; accumulator-aware quantization constrains overflow during training, improves those constraints, and extends them to PTQ and multi-stage accumulation \citep{ColbertEtAl2023A2Q,ColbertEtAl2024A2QPlus,ColbertEtAl2024AXE}.  Low-bit summation schemes such as Markov greedy sums study the accuracy of the accumulator itself \citep{NateshKungKong2025MGS}.  Concurrent work studies Lyapunov-guided stabilization under two's-complement wrapping arithmetic, reinforcing that wraparound is a dynamical failure mode rather than a small projection perturbation \citep{HamadoucheHussain2026}.  Recursive or looped architectures amplify the importance of cross-step state alignment and error accumulation \citep{FangEtAl2026LoopQ}.  Post-training model expansion offers another way to trade additional computation or structure for quantization robustness \citep{FrancoEtAl2025Expansion}.

Our arithmetic theorem is complementary.  It begins from a schedulewise transition perturbation bound and asks how the error scales with microdepth.  The resulting distinction is qualitative: quantizing and rewriting the entire state generically introduces $D\rho_z$, whereas increment error feedback can replace that growth by a bounded carry term.  This produces a depth--activation-precision phase diagram, an exact fixed-grid no-go theorem, and a common-lattice integer realization with explicit register widths---not merely a layerwise error bound.

\subsection{Sigma--delta methods and error feedback}

Sigma--delta quantization trades amplitude resolution for oversampling by feeding cumulative quantization error back into subsequent samples \citep{DaubechiesDeVore2003,Gunturk2003,KrahmerSaabWard2012}.  Sigma--delta neural networks use temporal or spatial oversampling to shape quantization noise \citep{OConnorWelling2017}, and recent work applies related ideas to very-low-bit language models \citep{XiaEtAl2025SDQ}.  The transformed-error proof in Section~\ref{sec:finite-arithmetic} is structurally related to this literature: a carry state makes the accumulated rounding error telescope.

Error compensation also appears in quantized training: ECO injects post-update quantization error into optimizer momentum so that full-precision master weights can be removed \citep{NikdanEtAl2026ECO}.  That training-time mechanism and our inference-time residual carry share a telescoping principle but operate on different states and answer different questions.

The present contribution is not the invention of feedback quantization.  It is a finite-horizon neural reachability theorem that places the carry term alongside the structural floor and pure-schedule synthesis term, and that derives the precision scaling required for residual updates to remain visible as $D$ increases.

\subsection{Residual networks, neural ODEs, and unrolling}

Residual networks motivate the interpretation of depth as discretized evolution \citep{HeEtAl2016}.  Stable architectures and neural ODEs make that connection explicit \citep{HaberRuthotto2018,ChenRubanovaBettencourtDuvenaud2018}, but an arbitrary residual stack need not approximate a single coherent differential equation \citep{SanderAblinPeyre2022}.  This distinction is central to our setup: a depth exchange rate is meaningful only for a fixed-horizon refinement family or another declared coupling across depth.  Repeating or subdividing a block that was trained as one large step can move away from its original map even in high precision.

Algorithm unrolling supplies a particularly transparent family of coherent refinements.  ISTA and FISTA generate proximal iterations for sparse inverse problems \citep{DaubechiesDefriseDeMol2004,BeckTeboulle2009}; LISTA learns such iterations as a network \citep{GregorLeCun2010}; and deep unfolding has become a broad methodology for interpretable signal-processing architectures \citep{ChenLiuWangYin2018,MongaLiEldar2021}.  One-bit unrolling experiments provide direct motivation for exchanging arithmetic precision and iteration count \citep{EamazYeganegiSoltanalian2025}.  Our soft-threshold specialization exploits contraction and piecewise-affine structure to obtain exact floors, fixed-teacher converses, and rational certificates.

\subsection{Relaxed controls, discrepancy, and switched systems}

Relaxed controls replace pure switching by probability-valued controls and recover pure trajectories through rapid switching or chattering \citep{Young1969,Warga1972}.  Mixed-integer optimal control makes this approximation quantitative through sum-up rounding, combinatorial integral approximation, and projection methods \citep{SagerReineltBock2009,SagerBockDiehl2012,SagerJungKirches2011,KirchesLendersManns2020,VasudevanEtAl2013Part1,VasudevanEtAl2013Part2}.  The online simplex-rounding lemma used here belongs to that lineage.

Our use of relaxed controls differs in two ways.  First, the state is lifted to a complete input-indexed representation, so one schedule must approximate an entire map rather than one trajectory.  Second, we identify the distance to the closed relaxed reachable set as the exact asymptotic neural approximation floor and propagate pure-switching discrepancy through architecture-specific stability and implementation arithmetic.  The bounded-variation extension permits prescribed jumps without charging their number, only total variation.

\subsection{Transformers, attention, and learned operation libraries}

Transformers compose residual self-attention and MLP branches \citep{VaswaniEtAl2017}.  Interacting-particle and dynamical-system viewpoints illuminate their continuous-depth behavior \citep{DuttaEtAl2021,GeshkovskiEtAl2025}, while rank-collapse and smoothness analyses clarify architectural stability \citep{DongCordonnierLoukas2021,WuEtAl2024,AbellaSilvestreTabuada2025,KimPapamakariosMnih2021,DasoulasEtAl2021,CastinAblinPeyre2024}.  LayerNorm and RMSNorm control the normalized state domain \citep{BaKirosHinton2016,ZhangSennrich2019}.

Transformer quantization includes integer-only execution, attention-aware reconstruction, activation smoothing, binary or lattice codebooks, and attention-specific approximations \citep{KimEtAl2021IBERT,FrantarEtAl2023GPTQ,XiaoEtAl2023SmoothQuant,LinEtAl2024AWQ,TsengEtAl2024QuIPSharp,HuangEtAl2024BiLLM,LiEtAl2024QLLMEval,XuEtAl2024PTQScaling,AdepuEtAl2024FrameQuant,HwangNguyenLee2026LSViT,QianEtAl2026QMHSA,XiaoZhangZhang2026BinaryAttention}.  Section~\ref{sec:architectures} does not propose a new quantizer for each component.  It derives explicit bounds that convert LayerNorm, score, value, projection, MLP, scale, and accumulator errors into the single field-computation term $\eta_G$ required by the integrated theorem.

The dictionary itself may be learned.  QINCo constructs implicit codebooks conditioned on previous reconstruction states \citep{HuijbenEtAl2024QINCo}; RQ-MoE uses input-dependent expert mechanisms for vector quantization \citep{ZhongEtAl2026RQMoE}.  Our metadata theorem covers compact global codebook families and smooth state-dependent codebook generators.  Discrete input-dependent selection is treated by the hybrid routing theory rather than hidden inside a smooth parameter-Lipschitz constant.

\subsection{Mixture-of-experts routing and hybrid events}

Sparse MoE architectures activate only a subset of experts \citep{ShazeerEtAl2017,FedusZophShazeer2022}.  Expert-choice, stable routing, and soft mixtures address load allocation and route fluctuations \citep{ZhouEtAl2022ExpertChoice,DaiEtAl2022StableMoE,PuigcerverEtAl2024}.  Quantization introduces additional expert-specific and router-specific sensitivity.  Recent work studies mixed precision across experts and gives generalization-based allocation rules \citep{ChowdhuryEtAl2026MoETheory}, global expert-level allocation \citep{DengEtAl2026GEMQ}, and calibration-free spectral allocation \citep{YangEtAl2026AlphaQ}; routing-consistent objectives explicitly preserve score values and top-$k$ structure \citep{ParkEtAl2026VSRAQ}.

A positive score margin on an entire connected time interval implies a fixed support; it cannot describe an actual route change.  Our hybrid theorem instead assumes isolated score ties with transversal crossings.  This is aligned with event-selected vector fields and saltation analysis in hybrid dynamical systems \citep{SacconEtAl2014,BurdenEtAl2016EventSelected,KongEtAl2024Saltation}.  The resulting theorem localizes mismatches to explicit event windows and closes the route--state feedback loop through a small-gain condition.  It also states an anti-chattering condition separately: sup-norm score closeness alone does not guarantee a unique implemented switch.

\subsection{Certification, occupation measures, and verification}

Formal equivalence verification asks whether two fixed networks agree exactly or within a tolerance, using methods such as MILP, interval propagation, and path enumeration \citep{TeuberEtAl2021Equivalence}.  Quantized-network verifiers exploit integer structure to certify properties of a given QNN \citep{HuangEtAl2024QNNVerify}.  SimCert provides probabilistic behavioral-similarity certificates for a fixed compressed network, including quantization and pruning \citep{LiSongLi2026SimCert}.  Our question is set-valued: what can \emph{any} schedule or metadata choice in the declared family achieve?  A verifier for one compiled network cannot by itself certify a positive structural floor over the whole family.

For controlled dynamical systems, HJB subsolutions and support functions give dual lower bounds.  Occupation measures reformulate finite-dimensional optimal control as linear programs over measures; polynomial systems admit moment and SOS relaxations \citep{LasserreEtAl2008,ClaeysDaafouzHenrion2016}.  The superposition principle connects continuity equations and trajectory measures \citep{AmbrosioGigliSavare2008}.

Recent featurized occupation-measure work develops finite-dimensional primal--dual representations for structured global search without a full state-space grid \citep{WeiTaoTanNie2026FOM}.  Our use is narrower and certificate oriented: verified HJB and support witnesses lower-bound the best achievable relaxed target error, and a coupling budget permits blockwise HJB certificates to compose.  We assemble these tools into a certification stack for the relaxed floor, add an exact finite-witness affine dual, and combine lower and upper bounds with finite implementation radii to obtain the three-way decision rule used throughout the paper.

\paragraph{Summary of the boundary.}
The paper does not claim novelty for relaxed controls, sigma--delta feedback, codebook learning, hybrid transversality, or occupation measures individually.  Its contribution is a common, target-specific resource law with matching necessity mechanisms, explicit architecture constants, and proof-producing feasibility and impossibility interfaces for quantized residual computation.

\section{Full-Map Residual Reachability}
\label{sec:formal-model}
\label{sec:abstract-theory}

This section defines the common object used throughout the paper.  The most important modeling choice is to lift the state over the entire input family.  That prevents a different low-bit schedule from being chosen for every input unless input-dependent control is explicitly declared as a resource.

\subsection{Lifted state, target maps, and pure schedules}

Let $\Xi$ be a compact input set and let $q$ be the hidden-state dimension.  We work in the Banach space
\begin{equation}
 \Z=C(\Xi;\R^q),
 \qquad
 \norm{z}_{\Z}=\sup_{\xi\in\Xi}\norm{z(\xi)}_2,
 \label{eq:lifted-space}
\end{equation}
unless a finite witness set or an $L^p(\mu)$ routed formulation is stated explicitly.  A point $z\in\Z$ is therefore a complete input-indexed representation.  The initial state $z_0$ may include an input embedding, fixed positional information, and any declared auxiliary state such as an error-feedback carry.

Fix a residual horizon $T>0$ and a finite low-bit residual-field dictionary
\begin{equation}
 \mathcal G_b=\{G_1,\ldots,G_J\},
 \qquad G_j:[0,T]\times\Z\to\Z.
 \label{eq:field-dictionary}
\end{equation}
The index $b$ denotes the complete precision model, not only a scalar bit width: it may specify a codebook, per-channel scales, a block family, and any other finitely charged metadata.  At depth $D$, set $h=T/D$ and $t_k=kh$.  A pure schedule $\sigma=(\sigma_0,\ldots,\sigma_{D-1})\in[J]^D$ produces the Euler recursion
\begin{equation}
 z_{k+1}^{\sigma}
 =z_k^{\sigma}+hG_{\sigma_k}(t_k,z_k^{\sigma}),
 \qquad z_0^{\sigma}=z_0.
 \label{eq:pure-euler-main}
\end{equation}
The terminal state is a full map on $\Xi$.  Let $\RR_{b,D}$ be the finite set of all such endpoints and let $\mathscr P_{b,D}$ be the corresponding linearly interpolated path set.

A measurable relaxed control $p:[0,T]\to\Delta_J$ produces
\begin{equation}
 \dot z_p(t)=\sum_{j=1}^Jp_j(t)G_j(t,z_p(t)),
 \qquad z_p(0)=z_0.
 \label{eq:relaxed-ode-main}
\end{equation}
We write $\RR_{b,\Rel}$ for the closure of its endpoint set in the $\Z$ norm and $\mathscr P_{b,\Rel}$ for the closure of the trajectory set in the uniform path norm on $C([0,T];\Z)$.  Taking closures is not cosmetic: relaxed controls are defined only up to almost-everywhere equality, and limiting controls or trajectories may be reached without a single representative attaining the infimum.

For a terminal target $F^\star\in\Z$ and a target path $z^\star\in C([0,T];\Z)$, define
\begin{align}
 E_b^{\mathrm{end}}(D;F^\star)
 &:=\dist(F^\star,\RR_{b,D}),
 &
 E_{\infty,b}^{\mathrm{end}}(F^\star)
 &:=\dist(F^\star,\RR_{b,\Rel}),
 \label{eq:end-errors-main}\\
 E_b^{\mathrm{path}}(D;z^\star)
 &:=\dist(z^\star,\mathscr P_{b,D}),
 &
 E_{\infty,b}^{\mathrm{path}}(z^\star)
 &:=\dist(z^\star,\mathscr P_{b,\Rel}).
 \label{eq:path-errors-main}
\end{align}
The endpoint floor is the correct invariant for direct map emulation and common-target endpoint matching.  The path floor is required when routing, safety, or intermediate-state behavior matters.  When the endpoint-versus-path distinction and precision model are clear, we abbreviate these quantities by $E_D(F^\star)$ and $E_\infty(F^\star)$; superscripts and the subscript $b$ are restored whenever they carry information.  Table~\ref{tab:notation} collects the notation used throughout the paper.

\begin{table}[t]
\centering
\caption{Core notation.  All resource terms are defined relative to a declared target, input geometry, operation library, and execution semantics.}
\label{tab:notation}
\small
\begin{tabularx}{\linewidth}{Y{.26\linewidth}X}
\toprule
Symbol & Meaning \\
\midrule
$T$, $L$, $D$, $h=T/D$ & residual horizon, high-precision comparator depth, low-bit student depth, and microstep \\
$\Xi$, $q$, $\Z$ & compact input family, hidden-state dimension, and lifted full-map state space \\
$\mathcal G_b=\{G_j\}_{j=1}^J$ & finite declared low-bit residual-field dictionary \\
$\sigma\in[J]^D$, $p(t)\in\Delta_J$ & pure schedule and relaxed simplex control \\
$\RR_{b,D}$, $\RR_{b,\Rel}$ & pure depth-$D$ endpoint set and closed relaxed endpoint set \\
$E_D(F^\star)$, $E_\infty(F^\star)$ & abbreviated finite-depth error and structural relaxed floor; endpoint/path and precision decorations are added when needed \\
$B,L_z,L_t,V_t$ & field magnitude, state-Lipschitz, time-Lipschitz, and total-variation bounds \\
$C_{\mathrm{syn}}$ & finite-depth pure-synthesis constant \\
$\Omega,s,m,\delta_s$ & learned dictionary family, metadata bits, metric dimension, and metadata covering radius \\
$\rho_z,\rho_D,\eta_G,\mathcal A_D$ & state-grid radius, increment-carry radius, field-computation error, and schedulewise arithmetic radius \\
$\tau_r,\nu_r,\Delta_r,\chi,w_{r,D}$ & route-event time, transversality slope, field jump, small-gain factor, and mismatch-window radius \\
$L,U,R_D$ & certified lower floor bound, feasible upper bound, and finite implementation radius \\
\bottomrule
\end{tabularx}
\end{table}

For $L\ge0$, write
\begin{equation}
 \Phi_L(T):=
 \begin{cases}
 (e^{LT}-1)/L,&L>0,\\
 T,&L=0,
 \end{cases}
 \label{eq:gronwall-factor-main}
\end{equation}
for the finite-horizon Gronwall accumulation factor used below.

\subsection{Regularity, existence, and the common tube}

The synthesis theorem is local to a verified execution tube.  This is both mathematically necessary and practically useful: global Lipschitz bounds for attention or normalization may be infinite or far too loose, while the relevant deployment set can be bounded.

\begin{assumption}[Uniform Euler-admissible synthesis tube]
\label[assumption]{ass:common-tube-main}
\label[assumption]{ass:regularity}
There are a closed tube $\mathcal K$, an open enlarged tube
$\mathcal K^\rho\subset\Z$, constants $B,L_z,h_0>0$, and temporal
regularity bounds independent of depth and schedule such that, for every
$D$ with $h=T/D\le h_0$, every measurable relaxed control, every pure
schedule, and every mixed schedule used in the synthesis theorem:
\begin{enumerate}[label=(\roman*),leftmargin=2em,itemsep=.22em]
\item $\norm{G_j(t,z)}\le B$ and
$\norm{G_j(t,z)-G_j(t,z')}\le L_z\norm{z-z'}$ for all atoms, times,
and states $z,z'\in\mathcal K^\rho$;
\item either each $t\mapsto G_j(t,\cdot)$ is uniformly $L_t$-Lipschitz
in the sup field norm, or it has a right-continuous bounded-variation
representative with total variation measure $\nu_j$ and
$V_t=\sum_j\nu_j((0,T])<\infty$;
\item every relaxed control induces a unique Carath\'eodory solution in
$\mathcal K$, and every relaxed trajectory, mixed-Euler grid state,
pure-Euler grid state, and interpolation segment used in the proof lies in
$\mathcal K^\rho$.
\end{enumerate}
\end{assumption}

The synthesis tube does not automatically cover a finite-arithmetic or routed
implementation.  Hardware-transfer, learned-dictionary, routed, and
Transformer results below state separate implemented or route-tube hypotheses
that contain the ideal and executed states, together with every interpolation
segment used in a Lipschitz or mean-value estimate.  Saturation and carry
bounds are required on the complete implemented range.

The existence clause follows from standard Banach-space Carath\'eodory theory under strong measurability, boundedness, and uniform state Lipschitzness.  For bounded-variation time dependence, the field range is separable and admits a strongly measurable representative.  In the smooth-time specialization, uniform $L_t$-Lipschitz continuity implies $V_t\le J L_tT$.  We state the existence condition explicitly because an endpoint set is meaningless if some measurable controls do not generate unique trajectories.

\subsection{Balanced switching}

The bridge from relaxed to pure computation is a discrepancy argument.  Given averaged relaxed weights $p_0,\ldots,p_{D-1}\in\Delta_J$, one chooses one atom per cell while keeping every prefix count close to its relaxed mass.

\begin{lemma}[Online simplex rounding]
\label[lemma]{lem:online-rounding-main}
For every sequence $p_0,\ldots,p_{D-1}\in\Delta_J$, there exists $\sigma_0,\ldots,\sigma_{D-1}\in[J]$ such that, for every prefix $n\le D$ and coordinate $j$,
\begin{equation}
 \left|\sum_{k=0}^{n-1}
 \bigl(\one_{\{\sigma_k=j\}}-p_{k,j}\bigr)\right|<J.
 \label{eq:prefix-discrepancy-main}
\end{equation}
One online rule maintains $q_0=0$, selects any maximizer
\begin{equation}
 \sigma_k\in\arg\max_j(q_{k,j}+p_{k,j}),
 \qquad q_{k+1}=q_k+p_k-e_{\sigma_k},
 \label{eq:online-rounding-rule-main}
\end{equation}
and satisfies \eqref{eq:prefix-discrepancy-main}.
\end{lemma}

\begin{proof}
The vector $q_k$ is the negative cumulative discrepancy and always has zero coordinate sum.  If $-1<q_{k,j}<J$, then $r=q_k+p_k$ has sum one and every coordinate greater than $-1$.  The selected coordinate satisfies $r_{\sigma_k}\ge1/J$, so subtracting one keeps it above $-1$; it remains below $J$ because $r_{\sigma_k}<J+1$.  An unselected coordinate cannot reach $J$: together with the maximizing coordinate and the lower bounds on the remaining coordinates, that would force $\sum_jr_j>1$.  Induction gives $-1<q_{k,j}<J$, and $q_n=\sum_{k<n}(p_k-e_{\sigma_k})$ proves the claim.
\end{proof}

The prefix bound, rather than independent per-step rounding, is what makes the forcing term first order.  Abel summation converts bounded cumulative discrepancy into a bound involving the temporal and state variation of the fields.

\subsection{Structural floor and finite-depth law}

\begin{theorem}[Pure-to-relaxed full-map synthesis]
\label[theorem]{thm:floor-rate-main}
\label{thm:hausdorff}
Under Assumption~\ref{ass:common-tube-main}, there exist $D_0=\lceil T/h_0\rceil$ and constants $C_{\mathrm{syn}}^{\mathrm{end}},C_{\mathrm{syn}}^{\mathrm{path}}<\infty$, independent of $D$ and the target, such that for every $D\ge D_0$,
\begin{align}
 \dH(\RR_{b,D},\RR_{b,\Rel})
 &\le \frac{C_{\mathrm{syn}}^{\mathrm{end}}}{D},
 \label{eq:hausdorff-end-main}\\
 \dH^{\mathrm{path}}(\mathscr P_{b,D},\mathscr P_{b,\Rel})
 &\le \frac{C_{\mathrm{syn}}^{\mathrm{path}}}{D}.
 \label{eq:hausdorff-path-main}
\end{align}
Consequently, for every endpoint target $F^\star$ and path target $z^\star$,
\begin{align}
 \left|E_b^{\mathrm{end}}(D;F^\star)
       -E_{\infty,b}^{\mathrm{end}}(F^\star)\right|
 &\le \frac{C_{\mathrm{syn}}^{\mathrm{end}}}{D},
 \label{eq:target-end-main}\\
 \left|E_b^{\mathrm{path}}(D;z^\star)
       -E_{\infty,b}^{\mathrm{path}}(z^\star)\right|
 &\le \frac{C_{\mathrm{syn}}^{\mathrm{path}}}{D}.
 \label{eq:target-path-main}
\end{align}
In the bounded-variation case one valid endpoint constant is
\begin{equation}
 C_{\mathrm{syn}}^{\mathrm{end}}
 =Te^{L_zT}\left[
 V_t+\frac12L_zBT+J^2(B+L_zBT)+JV_t
 \right],
 \label{eq:bv-constant-main}
\end{equation}
with $C_{\mathrm{syn}}^{\mathrm{path}}=C_{\mathrm{syn}}^{\mathrm{end}}+2BT$.
\end{theorem}

\paragraph{Proof, step by step.}
Let $p$ be a relaxed control and $z(t)$ its trajectory.  On each cell $I_k=[t_k,t_{k+1})$, average the control to $p_k=h^{-1}\int_{I_k}p(t)\,dt$ and define the mixed Euler path
\begin{equation}
 u_{k+1}=u_k+h\sum_jp_{k,j}G_j(t_k,u_k).
 \label{eq:mixed-euler-main}
\end{equation}
There are then two errors.

First, continuous flow versus mixed Euler.  The local defect is the integral of
\[
 \sum_jp_j(t)\{G_j(t,z(t))-G_j(t_k,z(t_k))\}.
\]
The state-motion contribution is at most $L_zBh^2/2$ per cell.  The time contribution is at most $L_t h^2/2$ in the smooth case or $h\sum_j\nu_j((t_k,t_{k+1}))$ in the BV case.  Summation and discrete Gronwall give a uniform grid-point error of order $h$.

Second, mixed Euler versus pure Euler.  Apply Lemma~\ref{lem:online-rounding-main} and put
\begin{equation}
 f_k=\sum_j(\one_{\{\sigma_k=j\}}-p_{k,j})G_j(t_k,u_k),
 \qquad F_n=\sum_{k<n}f_k.
 \label{eq:forcing-main}
\end{equation}
Abel summation writes $F_n$ as prefix discrepancies multiplied by initial fields and successive field differences.  The discrepancy is uniformly bounded, while
\[
 \sum_k\norm{G_j(t_{k+1},u_{k+1})-G_j(t_k,u_k)}
 \le L_zBT+\nu_j((0,T])
\]
in the BV case.  Thus $\max_n\norm{F_n}$ is bounded independently of $D$.  Let $v_k$ denote the pure Euler state, put $e_k=v_k-u_k$, and set
$\widetilde e_k=e_k-hF_k$.  Then
\begin{equation}
 \norm{\widetilde e_{k+1}}
 \le(1+hL_z)\norm{\widetilde e_k}
   +h^2L_z\max_n\norm{F_n}.
 \label{eq:transformed-error-main}
\end{equation}
Discrete Gronwall gives $\max_k\norm{v_k-u_k}=O(h)$.

Combining the two estimates gives one directed distance from relaxed to pure endpoints.  For the reverse direction, use the vertex-valued relaxed control equal to $e_{\sigma_k}$ on each cell; its mixed Euler path is the given pure path, so the same continuous-versus-Euler estimate applies.  Passing to closures preserves both directed distances.  On each grid cell, both the continuous trajectory and its corresponding linear interpolant move by at most $Bh$ from a grid endpoint, so linear interpolation adds at most $2Bh$ in the path norm.  The target-distance bounds follow because distance to a nonempty set is one-Lipschitz with respect to Hausdorff perturbations.  Appendix~\ref{app:full-resource-proofs} gives the full smooth and BV derivations with every constant.

\subsection{When time regularity is only H\"older}
\label{sec:holder-time-main}

Bounded variation gives a first-order law even when all inputs share finitely
many jump times.  A different regime arises when route or scale-change times
are dispersed over a continuum input population: the induced field may be
continuous in an $L^p$ map norm but only H\"older in time.  The synthesis
exponent then records that regularity rather than remaining artificially first
order.

\begin{assumption}[H\"older-time residual tube]
\label[assumption]{ass:holder-time-main}
Retain the boundedness, state-Lipschitz, common-tube, Euler-admissibility, and
well-posedness clauses of Assumption~\ref{ass:common-tube-main}.  Replace its
time-regularity clause by: for some $\vartheta\in(0,1]$ and $H_t<\infty$,
\begin{equation}
 \sup_{z\in\mathcal K^\rho}
 \norm{G_j(t,z)-G_j(s,z)}
 \le H_t|t-s|^\vartheta
 \quad\text{for every }j,s,t.
 \label{eq:holder-time-main}
\end{equation}
\end{assumption}

\begin{theorem}[Pure-to-relaxed synthesis for H\"older-time fields]
\label[theorem]{thm:holder-time-main}
Under Assumption~\ref{ass:holder-time-main}, there are finite constants
$C_\vartheta^{\mathrm{end}},C_1^{\mathrm{end}}$, independent of $D$ and the
target, such that for every $D\ge D_0$,
\begin{equation}
 \dH(\RR_{b,D},\RR_{b,\Rel})
 \le \frac{C_\vartheta^{\mathrm{end}}}{D^\vartheta}
     +\frac{C_1^{\mathrm{end}}}{D}.
 \label{eq:holder-endpoint-main}
\end{equation}
One valid choice is
\begin{align}
 C_\vartheta^{\mathrm{end}}
 &=T^\vartheta H_t\left[
    \frac{\Phi_{L_z}(T)}{\vartheta+1}
    +J^2T e^{L_zT}\right],
 \label{eq:holder-Ctheta-main}\\
 C_1^{\mathrm{end}}
 &=T\left[
   \frac12L_zB\Phi_{L_z}(T)
   +J^2(B+L_zBT)e^{L_zT}\right].
 \label{eq:holder-Cone-main}
\end{align}
The linearly interpolated path sets obey the same bound with an additional
$2BT/D$.  The corresponding endpoint- and path-target errors differ from
their relaxed floors by at most these radii.
\end{theorem}

\paragraph{Why the exponent changes.}
Averaging the relaxed control on one cell creates a time quadrature defect
$H_th^{1+\vartheta}/(1+\vartheta)$ and a state-motion defect
$L_zBh^2/2$.  Balanced switching still controls cumulative symbol discrepancy,
but the sampled field variation over the grid is now
$O(h^{\vartheta-1})$.  Abel summation multiplies that variation by the
microstep $h$, producing $O(h^\vartheta)$; the state contribution remains
$O(h)$.  Appendix~\ref{app:merged-exact-proofs} gives both Hausdorff
directions, the closure step, the path interpolation bound, and the constants
in \eqref{eq:holder-Ctheta-main}--\eqref{eq:holder-Cone-main}.

\subsection{Interpretation: feasibility, depth, and coherence}

Theorem~\ref{thm:floor-rate-main} separates two questions that are often conflated:
\begin{equation}
 \underbrace{E_{\infty,b}(F^\star)}_{\text{what the library can express}}
 \quad\text{and}\quad
 \underbrace{C_{\mathrm{syn}}/D}_{\text{what a finite pure schedule costs}}.
 \label{eq:floor-synthesis-split}
\end{equation}
A positive floor is a structural incompatibility; no optimizer or asymptotic depth refinement can remove it.  A zero floor says only that the target can be approximated by increasingly fast switching.  It does not guarantee that the arithmetic, router, or optimization procedure preserves that representation.

The theorem also clarifies when a depth exchange rate is meaningful.  Suppose a high-precision comparator belongs to a coherent family with target error $\varepsilon_H(L)\asymp L^{-1}$.  If the low-bit structural floor is zero, the theorem supplies the sufficient scaling $D=O(L)$.  A matching class-uniform lower bound, or the fixed-target converse of Section~\ref{sec:necessity-main}, is required before concluding $D=\Theta(L)$.  Without a shared residual horizon or another coupling across depth, subdividing an arbitrary trained block need not approximate its original map.  The DistilBERT control experiment in Section~\ref{sec:pretrained-causal} makes this distinction visible before quantization is introduced.

\begin{figure}[t]
\centering
\includegraphics[width=.96\linewidth]{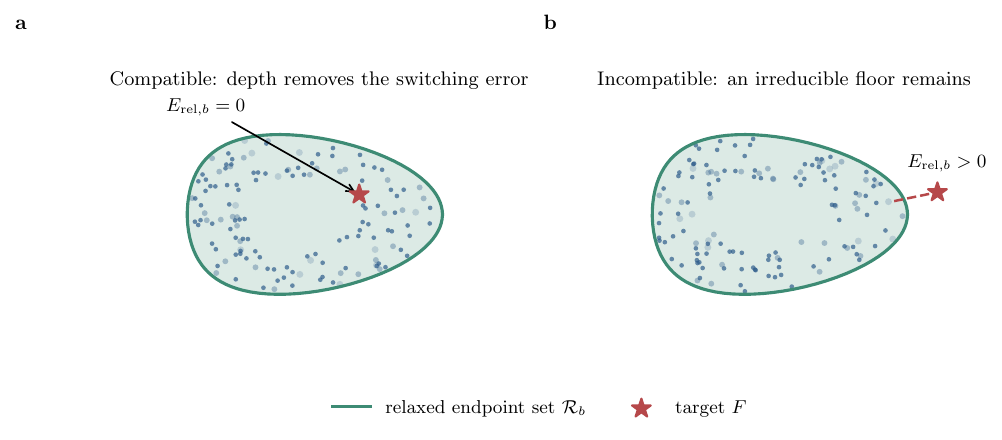}
\caption{\textnormal{Reachability is a geometric property of the declared operation library.}  Pure depth-$D$ schedules form a discrete cloud.  Balanced switching drives that cloud toward the relaxed reachable set.  A compatible target has zero structural floor and only a finite-depth remainder; an incompatible target remains separated even as $D\to\infty$.}
\label{fig:reachability-concept}
\end{figure}

Figure~\ref{fig:reachability-concept} gives the geometric picture.  The following sections add the resources omitted by ideal relaxed reachability: finite execution arithmetic, fixed-target necessity, learned metadata, discontinuous routes, and proof-producing certificates.

\section{Finite Execution Arithmetic}
\label{sec:finite-arithmetic}
\label{sec:finite-arithmetic-main}

The ideal theorem treats hidden states and residual accumulation in real arithmetic.  A useful quantized-computation theory must say when this idealization is benign and when it changes the answer.  We therefore separate the \emph{operation library} from the \emph{transition semantics}.  The same low-bit field may be executed with floating-point accumulation, fixed-grid activation write-back, saturating integer arithmetic, two's-complement wrapping, or an auxiliary error-feedback register.  These are different computational resources and need not have the same depth scaling.

\subsection{A schedulewise transfer theorem}

For a fixed depth $D$, write
\begin{equation}
 \Psi_{k,j}(z)=z+hG_j(t_k,z)
 \label{eq:ideal-transition-main}
\end{equation}
for the ideal transition, and let $\widehat\Psi_{k,j}$ be the actually executed transition.  It may contain quantized weights and activations, an approximate step size, integer multiply--accumulate operations, requantization, clipping, or an auxiliary carry.  The only requirement is that the ideal and implemented transitions be compared on one common tube.  This is a separate verification obligation: an invariant tube for the ideal flow does not by itself imply that a finite-arithmetic implementation remains inside it.

\begin{theorem}[Schedulewise hardware transfer]
\label[theorem]{thm:hardware-transfer-main}
Suppose every ideal and implemented prefix generated by an admissible schedule remains in the common tube.  Assume uniformly over the atom index that $\operatorname{Lip}(\Psi_{k,j})\le\kappa_k$ for every $j$, and
\begin{equation}
 \sup_{z\ \mathrm{in\ the\ tube}}
 \norm{\widehat\Psi_{k,j}(z)-\Psi_{k,j}(z)}
 \le \varepsilon_{k,j}.
 \label{eq:local-hardware-main}
\end{equation}
Then, for every schedule $\sigma\in[J]^D$,
\begin{equation}
 \norm{\widehat z_D^\sigma-z_D^\sigma}
 \le
 \mathcal A_D(\sigma)
 :=\sum_{k=0}^{D-1}\varepsilon_{k,\sigma_k}
      \prod_{\ell=k+1}^{D-1}\kappa_\ell.
 \label{eq:hardware-radius-main}
\end{equation}
If $\mathcal A_D=\sup_\sigma\mathcal A_D(\sigma)$, then
\begin{align}
 \dH(\widehat\RR_{b,D},\RR_{b,D})&\le\mathcal A_D,
 \label{eq:implemented-ideal-set-main}\\
 \dH(\widehat\RR_{b,D},\RR_{b,\Rel})
 &\le\frac{C_{\mathrm{syn}}^{\mathrm{end}}}{D}+\mathcal A_D,
 \label{eq:implemented-relaxed-set-main}
\end{align}
and the same radius bounds the difference between implemented and relaxed target errors.
\end{theorem}

\begin{proof}
Let $e_k=\widehat z_k-z_k$.  Adding and subtracting $\Psi_{k,\sigma_k}(\widehat z_k)$ gives
\[
 \norm{e_{k+1}}
 \le\varepsilon_{k,\sigma_k}+\kappa_k\norm{e_k}.
\]
Iterating from $e_0=0$ proves \eqref{eq:hardware-radius-main}.  Every ideal endpoint is paired with the implemented endpoint generated by the same schedule, and conversely, so both directed set distances are bounded by $\mathcal A_D$.  The relaxed-set claim follows from Theorem~\ref{thm:floor-rate-main} and the Hausdorff triangle inequality.
\end{proof}

A concrete local defect decomposition for a finite-dimensional saturating transition
\begin{equation}
 \widehat\Psi_{k,j}(z)
 =Q^{\mathrm{sat}}_{\alpha_k,M_k}
 \left(z+\widehat h_k\widehat G_{k,j}(z)+e^{\mathrm{acc}}_{k,j}(z)\right)
 \label{eq:saturating-transition-main}
\end{equation}
is
\begin{equation}
 \varepsilon_{k,j}
 \le B|\widehat h_k-h|
 +|\widehat h_k|\eta^G_{k,j}
 +\eta^{\mathrm{acc}}_{k,j}
 +\frac{\alpha_k\sqrt q}{2}
 +\rho^{\mathrm{sat}}_{k,j},
 \label{eq:local-decomposition-main}
\end{equation}
where $\eta^G_{k,j}$ is the field-computation error, $\eta^{\mathrm{acc}}_{k,j}$ is the accumulator remainder, $\alpha_k$ is the activation-grid spacing, and $\rho^{\mathrm{sat}}_{k,j}$ is the supremum, over the verified tube, of the norm of the displacement introduced by the final clipping or projection operation.  Equation~\eqref{eq:local-decomposition-main} is additive: weight/codebook error, scale metadata, accumulation, write-back, and saturation are not collapsed into an uninterpretable ``quantization error.''  Wrapping arithmetic is excluded from this projection-defect model unless a separate invariant proves that no wrap can occur.

\subsection{Why full-state write-back can make deeper networks worse}

Suppose every microstep writes the complete state to a grid with covering radius $\rho_z$, and suppose every downstream transition product in Theorem~\ref{thm:hardware-transfer-main} is bounded by $\overline{\mathcal S}$.  Then
\begin{equation}
 \mathcal A_D\le D\rho_z\overline{\mathcal S},
 \qquad
 E_D^{\mathrm{writeback}}
 \le E_\infty+\frac{C_{\mathrm{syn}}}{D}
     +D\rho_z\overline{\mathcal S}.
 \label{eq:writeback-envelope-main}
\end{equation}
For stable residual transitions one may take $\overline{\mathcal S}\le e^{L_zT}$.  The upper envelope is minimized at
\begin{equation}
 D_{\mathrm{opt}}
 \asymp\sqrt{\frac{C_{\mathrm{syn}}}
 {\rho_z\overline{\mathcal S}}},
 \label{eq:writeback-optimal-depth-main}
\end{equation}
and maintaining an overall $O(D^{-1})$ certificate requires $\rho_z=O(D^{-2})$.  This is a worst-case statement, not a claim that every architecture realizes the upper envelope.  It nevertheless reveals the correct generic scaling: one state-rounding defect is incurred at every microstep and can be amplified by all subsequent transitions.

The qualitative failure is exact in one dimension.

\begin{proposition}[Fixed-grid freeze]
\label[proposition]{prop:fixed-grid-freeze-main}
Let $Q_\Delta$ be nearest rounding to $\Delta\mathbb Z$ and consider
\begin{equation}
 \widehat x_{k+1}=Q_\Delta\!\left(\widehat x_k+\frac1D\right),
 \qquad \widehat x_0=0.
 \label{eq:fixed-grid-recurrence-main}
\end{equation}
The ideal terminal target is $x^\star(1)=1$.  If $D>2/\Delta$, then $\widehat x_k=0$ for every $k$ and the terminal error is one.
\end{proposition}

\begin{proof}
The strict inequality $D>2/\Delta$ makes the proposed increment smaller than half a grid cell and avoids any dependence on the tie rule.  Starting at a grid point, nearest rounding returns the same point.  Induction proves the claim.
\end{proof}

Thus, the same depth refinement that makes the ideal schedule more expressive can make every physical residual update numerically invisible.  Operand precision, activation precision, and residual scaling cannot be treated as interchangeable labels.

\subsection{Increment error feedback}

The obstruction is avoided when the quantization residual is stored and injected into the next \emph{increment}.  For a fixed schedule, let
\begin{align}
 q_k&=Q\bigl(h\widehat G_{k,\sigma_k}(\widehat z_k)+r_k\bigr),
 \label{eq:ef-q-main}\\
 \widehat z_{k+1}&=\widehat z_k+q_k,
 \label{eq:ef-state-main}\\
 r_{k+1}&=h\widehat G_{k,\sigma_k}(\widehat z_k)+r_k-q_k,
 \qquad r_0=0.
 \label{eq:ef-carry-main}
\end{align}
Assume the quantizer has residual radius at most $\rho$ on the entire declared carry range, the carry register does not overflow, and hence $\norm{r_k}\le\rho$ throughout execution.  The carry is an explicit resource: if it is stored in a finite register, its precision and overflow semantics must be counted.

\begin{theorem}[Error feedback with inexact fields]
\label[theorem]{thm:error-feedback-main}
Suppose the ideal and implemented paths remain in a common tube, every ideal field is $L_z$-Lipschitz, and
\begin{equation}
 \sup_{k,j,z}
 \norm{\widehat G_{k,j}(z)-G_j(t_k,z)}\le\eta_G.
 \label{eq:field-error-main}
\end{equation}
Let $z_k$ be the ideal pure Euler path with the same schedule.  Then
\begin{equation}
 \max_{0\le k\le D}\norm{\widehat z_k-z_k}
 \le \rho e^{L_zT}+\eta_G\Phi_{L_z}(T)
 \label{eq:error-feedback-bound-main}
\end{equation}
Consequently,
\begin{equation}
 \dH(\widehat\RR^{\mathrm{EF}}_{b,D},\RR_{b,\Rel})
 \le\frac{C_{\mathrm{syn}}^{\mathrm{end}}}{D}
 +\rho e^{L_zT}+\eta_G\Phi_{L_z}(T).
 \label{eq:error-feedback-reach-main}
\end{equation}
If the physical increment lattice has residual radius $\rho_D=O(D^{-1})$ and $\eta_{G,D}=O(D^{-1})$, the implemented system retains a first-order total law.
\end{theorem}

\begin{proof}
First suppose $\widehat G=G$.  Let $e_k=\widehat z_k-z_k$ and define the transformed error $s_k=e_k+r_k$.  Since
\[
 q_k=hG_{\sigma_k}(t_k,\widehat z_k)+r_k-r_{k+1},
\]
we obtain
\begin{equation}
 s_{k+1}
 =s_k+h\{G_{\sigma_k}(t_k,\widehat z_k)
          -G_{\sigma_k}(t_k,z_k)\}.
 \label{eq:ef-transformed-main}
\end{equation}
Because $e_k=s_k-r_k$ and $\norm{r_k}\le\rho$,
\begin{equation}
 \norm{s_{k+1}}
 \le(1+hL_z)\norm{s_k}+hL_z\rho.
 \label{eq:ef-recursion-main}
\end{equation}
With $s_0=0$, discrete Gronwall gives
\[
 \norm{s_k}\le\rho\{(1+hL_z)^k-1\}.
\]
Therefore $\norm{e_k}\le\norm{s_k}+\norm{r_k}\le\rho(1+hL_z)^k\le\rho e^{L_zT}$.

For inexact fields, insert and subtract $G_{\sigma_k}(t_k,\widehat z_k)$ in \eqref{eq:ef-transformed-main}.  The transformed recursion becomes
\begin{equation}
 \norm{s_{k+1}}
 \le(1+hL_z)\norm{s_k}+hL_z\rho+h\eta_G.
 \label{eq:ef-recursion-inexact-main}
\end{equation}
The additional per-step forcing has discrete Gronwall sum $\eta_G\Phi_{L_z}(T)$.  The reachable-set result follows by pairing equal schedules and applying Theorem~\ref{thm:floor-rate-main}.
\end{proof}

The difference between \eqref{eq:writeback-envelope-main} and \eqref{eq:error-feedback-reach-main} is the finite-arithmetic analogue of noise shaping: the carry makes the cumulative rounding residual telescope rather than appear as $D$ unrelated state perturbations.  The theorem does not require stochastic rounding and does not assume zero-mean errors.

\subsection{An overflow-free residual-path common-lattice realization}
\label{sec:common-lattice-main}

The preceding theorem isolates the carry radius but does not yet say how to
store the carry exactly.  The next result gives a complete integer
realization for the residual/requantization path.  It is intentionally a
narrow theorem: matrix products and nonlinear approximations must separately
establish that their emitted increments lie on the declared fine lattice.

\begin{proposition}[Bit-exact common-lattice error feedback]
\label[proposition]{prop:bit-exact-error-feedback-main}
Let $D\ge1$, $\delta>0$, $m\in\{1,2,\ldots\}$, and $\Delta=m\delta$.
Suppose
\begin{equation}
 v_k\in\delta\mathbb Z^q,\qquad
 \norm{v_k}_\infty\le V,\qquad k=0,\ldots,D-1,
 \label{eq:lattice-increments-main}
\end{equation}
and $\widehat z_0\in\Delta\mathbb Z^q$.  Let $Q_\Delta$ be coordinatewise
nearest rounding to $\Delta\mathbb Z$ with a fixed tie rule and execute
\begin{align}
 q_k&=Q_\Delta(v_k+r_k),&
 \widehat z_{k+1}&=\widehat z_k+q_k,\nonumber\\
 r_{k+1}&=v_k+r_k-q_k,&r_0&=0.
 \label{eq:lattice-ef-main}
\end{align}
Define
\begin{align}
 M_r&=\left\lceil\frac m2\right\rceil,&
 M_a&=\left\lceil\frac V\delta\right\rceil+M_r,\nonumber\\
 M_q&=\left\lceil\frac V\Delta+1\right\rceil,&
 M_{z,0}&=\norm{\widehat z_0/\Delta}_\infty.
 \label{eq:lattice-magnitudes-main}
\end{align}
A signed two's-complement implementation is overflow free whenever
\begin{align}
 b_r&\ge1+\left\lceil\log_2(M_r+1)\right\rceil,\nonumber\\
 b_a&\ge1+\left\lceil\log_2(M_a+1)\right\rceil,\nonumber\\
 b_z&\ge1+\left\lceil\log_2(M_{z,0}+DM_q+1)\right\rceil.
 \label{eq:lattice-widths-main}
\end{align}
Here $b_r$ stores $r_k/\delta$, $b_a$ stores $(v_k+r_k)/\delta$ before
rounding by $m$, and $b_z$ stores $\widehat z_k/\Delta$.  Under these bounds,
\begin{equation}
 \norm{r_k}_\infty\le\frac\Delta2,
 \qquad
 \widehat z_k+r_k
 =\widehat z_0+\sum_{i=0}^{k-1}v_i
 \quad(0\le k\le D).
 \label{eq:lattice-conservation-main}
\end{equation}
If the ideal accumulated endpoint belongs to $\Delta\mathbb Z^q$, then
$r_D=0$ and the implemented terminal state equals that endpoint exactly.
\end{proposition}

\paragraph{Interpretation.}
The state register grows logarithmically with the largest accumulated integer
magnitude, while the residual register depends only on the lattice ratio
$m=\Delta/\delta$.  The theorem locates where precision must persist: in a
small carry or an equivalent compensated accumulator.  It does not certify a
complete integer Transformer kernel; LayerNorm, softmax, matrix products,
heterogeneous scales, and overflow semantics enter through their own field
and range certificates in Section~\ref{sec:transformer-specialization-main}.
Appendix~\ref{app:merged-exact-proofs} proves the register induction and exact
terminal recovery.  The released verifier exhaustively checks all scalar
sequences in its registered grid and randomized multidimensional cases.

\subsection{How many bits must grow with depth?}

Suppose the state range is fixed and the activation lattice has $b_z$ fractional bits, so $\rho_z\asymp2^{-b_z}$.  Naive write-back preserves a first-order certificate only if
\begin{equation}
 2^{-b_z}=O(D^{-2}),
 \qquad\text{equivalently}\qquad
 b_z\ge 2\log_2D+O(1).
 \label{eq:writeback-bit-growth-main}
\end{equation}
For increment error feedback with physical residual unit proportional to $h$, fixed \emph{normalized} increment precision suffices.  If one instead fixes an absolute increment grid over a fixed range, then $\rho_D=O(D^{-1})$ requires only
\begin{equation}
 b_{\mathrm{inc}}\ge\log_2D+O(1).
 \label{eq:ef-bit-growth-main}
\end{equation}
The carry register must be accurate enough that its own write-back error does not reintroduce a $D\rho_r$ term.  One can therefore move precision from the full activation path into a smaller persistent error state, but cannot make it disappear from the resource accounting.  Table~\ref{tab:arithmetic-scaling} contrasts the resulting depth scalings.

\begin{table}[t]
\centering
\caption{Depth scaling under representative execution semantics.  The entries are sufficient worst-case conditions for preserving an $O(D^{-1})$ total law, not universal lower bounds for every architecture.}
\label{tab:arithmetic-scaling}
\small
\begin{tabularx}{\linewidth}{Y{.25\linewidth}Y{.28\linewidth}Y{.20\linewidth}X}
\toprule
Execution semantics & Generic arithmetic contribution & Sufficient grid scaling & Qualitative consequence \\
\midrule
Exact or high-precision state & $0$ or separately bounded $\eta_G\Phi_{L_z}(T)$ & $\eta_G=O(D^{-1})$ & Ideal synthesis law survives \\
Full-state write-back & $D\rho_z\overline{\mathcal S}$ & $\rho_z=O(D^{-2})$ & Finite optimal depth; possible freezing \\
Increment quantization, no carry & Persistent bias can remain $O(1)$ & problem dependent & Repeated small bias need not vanish \\
Increment error feedback & $\rho_De^{L_zT}+\eta_G\Phi_{L_z}(T)$ & $\rho_D,\eta_G=O(D^{-1})$ & First-order law retained \\
Finite carry write-back & previous row plus $O(D\rho_r)$ & $\rho_r=O(D^{-2})$, unless carry is otherwise compensated & Precision is relocated, not eliminated \\
\bottomrule
\end{tabularx}
\end{table}

\subsection{Saturation and wrapping are not the same perturbation}

Let $\mathcal B$ be the representable state box and suppose every ideal prefix lies at distance at least $m>0$ from its boundary.  If a prefix-uniform implementation bound is strictly smaller than $m$, saturation never activates.  This gives a verifiable \emph{tube margin} rather than an informal hope that overflow is rare.

If saturation does activate, its displacement can be included as $\rho^{\mathrm{sat}}$ in \eqref{eq:local-decomposition-main}, but the resulting bound may be large and state dependent.  Two's-complement wrapping is more severe: it can change sign and magnitude discontinuously and is not a nearby projection.  A wrapping implementation therefore requires an explicit no-overflow invariant for every pre-write value.  Treating wraparound as ordinary rounding is mathematically invalid.

\paragraph{Engineering consequence.}
Depth substitutes for parameter or operation precision only when the entire execution stack is co-designed so that residual increments remain visible.  Before increasing depth, one should determine whether the arithmetic radius decreases, remains bounded, or grows with the number of microsteps.  Full-state write-back and increment error feedback can therefore place the same low-bit field library in qualitatively different depth regimes.

\section{Necessity of the First-Order Depth Price}
\label{sec:necessity-main}

The structural theorem gives a target-independent upper rate, but an upper rate alone does not establish that a first-order depth price is intrinsic.  A lower bound whose target changes with $D$ proves a minimax statement; it does not rule out the possibility that every fixed target of practical interest is approximated faster.  We therefore construct one teacher, independent of $D$, for which the globally optimal low-bit error can be computed at every depth, and then lift the obstruction to residual ReLU and attention systems as well as approximately exposed channels of a general neural block.

\subsection{One fixed teacher forces a first-order depth price}
\label{sec:fixed-teacher-main}

Fix $a>0$ and consider the binary residual recursion
\begin{equation}
 x_{k+1}=\left(1-\frac1D\right)x_k+\frac{q_k}{D},
 \qquad q_k\in\{0,a\},\qquad x_0=0.
 \label{eq:binary-recursion-main}
\end{equation}
The comparator is not chosen adversarially after $D$ is known.  It is the single time-one endpoint
\begin{equation}
 x^\star=a(1-e^{-1})
 \label{eq:fixed-teacher-target-main}
\end{equation}
of the continuous teacher $\dot x=a-x$, $x(0)=0$.

\begin{theorem}[Exact best error for one fixed binary teacher]
\label[theorem]{thm:fixed-teacher-main}
For every integer $D\ge2$, the all-$a$ schedule is a globally closest depth-$D$ endpoint of \eqref{eq:binary-recursion-main} to the fixed target \eqref{eq:fixed-teacher-target-main}.  The exact optimum is
\begin{equation}
 E_D^{\mathrm{bin}}(x^\star)
 =a\left[e^{-1}-\left(1-\frac1D\right)^D\right].
 \label{eq:fixed-teacher-exact-main}
\end{equation}
Moreover,
\begin{equation}
 \frac{a}{4eD}
 \le E_D^{\mathrm{bin}}(x^\star)
 \le \frac{a}{2e(D-1)},
 \label{eq:fixed-teacher-bounds-main}
\end{equation}
and
\begin{equation}
 E_D^{\mathrm{bin}}(x^\star)
 =\frac{a}{2eD}+O(D^{-2}).
 \label{eq:fixed-teacher-asymptotic-main}
\end{equation}
\end{theorem}

\begin{proof}
Write $r_D=1-D^{-1}$.  Unrolling the recursion gives
\begin{equation}
 x_D=\frac1D\sum_{k=0}^{D-1}r_D^{D-1-k}q_k.
 \label{eq:binary-endpoint-expansion-main}
\end{equation}
The largest attainable endpoint is
\begin{equation}
 M_D=a(1-r_D^D),
 \label{eq:binary-largest-main}
\end{equation}
obtained by choosing $q_k=a$ at every step.  Because the coefficients in \eqref{eq:binary-endpoint-expansion-main} are positive and increase with $k$, the second-largest endpoint is obtained by replacing the earliest, least heavily weighted symbol by zero.  Hence the gap from the largest to the second largest endpoint is
\begin{equation}
 g_D=\frac{a}{D}r_D^{D-1}.
 \label{eq:binary-top-gap-main}
\end{equation}
It is therefore enough to prove that the fixed teacher belongs to the Voronoi cell of $M_D$, namely
\begin{equation}
 M_D-x^\star\le \frac{g_D}{2}.
 \label{eq:binary-voronoi-condition-main}
\end{equation}
Put $u=D^{-1}$.  For $0<u\le1/2$, the convergent series identity
\begin{align}
 1+\left(\frac1u-1\right)\log(1-u)+\log(1-u/2)
 &=\sum_{n=2}^{\infty}\frac{u^n}{n}
   \left(\frac1{n+1}-\frac1{2^n}\right)
 \ge0
 \label{eq:binary-series-main}
\end{align}
uses only $2^n\ge n+1$.  Exponentiating \eqref{eq:binary-series-main} yields
\begin{equation}
 e^{-1}\le r_D^{D-1}\left(1-\frac1{2D}\right).
 \label{eq:binary-exponential-ineq-main}
\end{equation}
Consequently,
\begin{align}
 M_D-x^\star
 &=a(e^{-1}-r_D^D) \\
 &\le \frac{a}{2D}r_D^{D-1}
 =\frac{g_D}{2},
 \label{eq:binary-voronoi-proof-main}
\end{align}
which proves global optimality and the exact formula.

For the quantitative envelope, define
\begin{equation}
 \theta_D=-1-D\log(1-D^{-1})
 =\sum_{n=2}^{\infty}\frac1{nD^{n-1}}.
 \label{eq:theta-main}
\end{equation}
Then
\begin{equation}
 \frac1{2D}\le\theta_D\le\frac1{2(D-1)}\le\frac12
 \label{eq:theta-bounds-main}
\end{equation}
and
\begin{equation}
 e^{-1}-r_D^D=e^{-1}(1-e^{-\theta_D}).
 \label{eq:theta-error-main}
\end{equation}
For $0\le\theta\le1/2$, $\theta/2\le1-e^{-\theta}\le\theta$.  Substitution gives \eqref{eq:fixed-teacher-bounds-main}.  Expanding \eqref{eq:theta-main} and then $1-e^{-\theta_D}$ gives \eqref{eq:fixed-teacher-asymptotic-main}.
\end{proof}

The exact formula is useful for two reasons.  It removes the dependence of the target on $D$, and it identifies the finite-depth obstruction as a genuine endpoint-lattice effect rather than an artifact of a loose proof.  Figure~\ref{fig:fixed-teacher-necessity-main} shows that the exact curve rapidly enters the asymptotic regime while remaining inside the nonasymptotic envelope \eqref{eq:fixed-teacher-bounds-main}.

\begin{figure}[t]
\centering
\includegraphics[width=.96\linewidth]{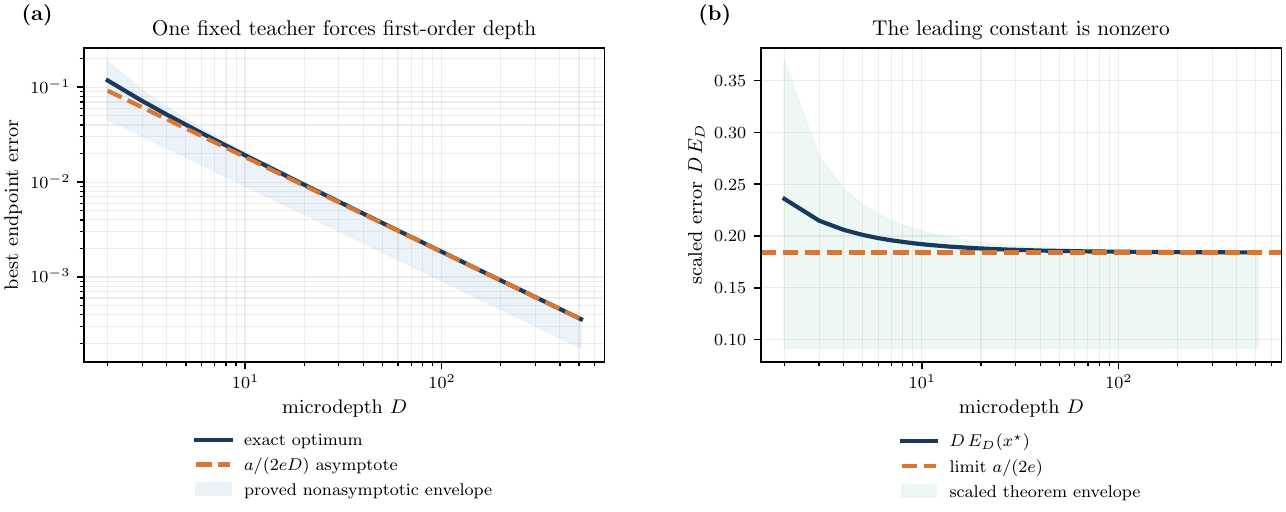}
\caption{\textnormal{A fixed target can require a first-order depth price.}
\textnormal{(a)} The exact optimum in Theorem~\ref{thm:fixed-teacher-main}
tracks its $a/(2eD)$ asymptote and remains inside the proved nonasymptotic
envelope.  \textnormal{(b)} Multiplication by $D$ exposes the nonzero limiting
constant $a/(2e)$; the lower law is not generated by changing the target with
$D$.  All curves are analytic consequences of the theorem, not fitted slopes.}
\label{fig:fixed-teacher-necessity-main}
\end{figure}

Figure~\ref{fig:fixed-teacher-necessity-main} makes the distinction between
minimax and fixed-target sharpness visible before we transfer the obstruction
to neural architectures.

\subsection{From a scalar obstruction to an architecture-level certificate}
\label{sec:exposed-mode-main}

A full network need not be one-dimensional.  It is enough that one observable channel behaves approximately like the binary system above.

\begin{corollary}[Exact exposed binary mode]
\label[corollary]{cor:exact-exposed-mode-main}
Normalize $T=1$.  Let $\ell\in\Z^*$ satisfy $\norm{\ell}_*\le1$ and $\ell(z_0)=0$.  Suppose that each atom has a label $q_j\in\{0,a\}$ and, on the common execution tube,
\begin{equation}
 \ell(G_j(t,z))=q_j-\ell(z).
 \label{eq:exact-exposed-mode-main}
\end{equation}
Assume at least one atom has $q_j=a$, and let $F^\star$ be the time-one map of that continuous $a$-atom.  Then $E_\infty(F^\star)=0$ and, for every $D\ge2$,
\begin{equation}
 E_D(F^\star)
 \ge a\left[e^{-1}-\left(1-\frac1D\right)^D\right]
 \ge\frac{a}{4eD}.
 \label{eq:exact-exposed-lower-main}
\end{equation}
Combining this with Theorem~\ref{thm:floor-rate-main} gives a two-sided $\Theta(D^{-1})$ law for one fixed architecture target.
\end{corollary}

\begin{proof}
Apply $\ell$ to the pure Euler recursion.  The projected state obeys \eqref{eq:binary-recursion-main}.  The continuous $a$-atom has projected endpoint $a(1-e^{-1})$.  Since $\norm{\ell}_*\le1$, the full-state norm dominates the scalar projection error.  The target itself is relaxed reachable, so its structural floor is zero.
\end{proof}

\subsection{Exact neural realizations of the fixed-target obstruction}
\label{sec:exact-neural-converses-main}

The scalar recursion is already a complete fixed-target converse, but it is
important to show that the obstruction is not created by an unnatural
one-dimensional notation.  We first embed it in a conventional residual ReLU
branch and then derive an exact self-attention example with a genuinely
nonuniform attention map and two dynamical modes.

\begin{proposition}[Residual ReLU MLP converse]
\label[proposition]{prop:relu-mlp-fixed-target-main}
Let $z=(x,u)\in\mathbb R\times\mathbb R^{q-1}$, $z_0=(0,u_0)$, and define two
residual MLP atoms
\begin{equation}
 G_c(z)=W_2\operatorname{ReLU}(W_1z)+b_c,
 \qquad c\in\{0,a\},
 \label{eq:relu-mlp-atoms-main}
\end{equation}
with $W_1=e_1^\top$, $W_2=-e_1$, and $b_c=ce_1$.  Every pure Euler path with
$h=1/D$ remains in $[0,a]\times\{u_0\}$ and its first coordinate obeys
\eqref{eq:binary-recursion-main}.  For the fixed continuous teacher generated
by the $c=a$ atom,
\begin{equation}
 z^\star(1)=\bigl(a(1-e^{-1}),u_0\bigr),
 \qquad
 E_D(z^\star(1))
 =a\left[e^{-1}-\left(1-\frac1D\right)^D\right].
 \label{eq:relu-mlp-exact-main}
\end{equation}
The same lower bound embeds unchanged in arbitrary additional residual
channels under any product norm that dominates the first coordinate; equality
continues to hold whenever the norm restricts to absolute value on the first
coordinate, as do the standard $\ell^p$ product norms.
\end{proposition}

\begin{theorem}[Exact fixed-target converse for nonuniform two-token attention]
\label[theorem]{thm:attention-fixed-target-main}
Let $X\in\mathbb R^{2\times2}$ carry the Frobenius norm and let $e_1,e_2$ be
the standard basis.  Define a one-head attention field by
\begin{align}
 Q(X)&=Xe_2,\qquad K(X)=Xe_2,\qquad V(X)=Xe_1,\nonumber\\
 \mathcal A(X)
 &:=-\softmax_{\rm row}\!\left(Q(X)K(X)^\top\right)V(X)e_1^\top,
 \label{eq:attention-head-main}
\end{align}
and atoms
\begin{equation}
 G_c(X)=\mathcal A(X)+c\one_2e_1^\top,
 \qquad c\in\{-a,+a\},\qquad a>0.
 \label{eq:attention-atoms-main}
\end{equation}
Initialize
\begin{equation}
 X_0=\begin{bmatrix}u_0&1\\-u_0&-1\end{bmatrix},
 \qquad u_0\ne0,
 \label{eq:attention-initial-main}
\end{equation}
put $\mu=\tanh(1)$, and define
\begin{equation}
 \delta_c(D)=e^{-c}-\left(1-\frac cD\right)^D,
 \qquad c\in(0,1].
 \label{eq:attention-delta-main}
\end{equation}
Let $X^\star$ be the time-one endpoint of $\dot X=G_{+a}(X)$.  The attention
matrix along every relevant path is the nonuniform bistochastic matrix
\begin{equation}
 P=\begin{bmatrix}p&1-p\\1-p&p\end{bmatrix},
 \qquad p=\frac{e}{e+e^{-1}},\qquad 2p-1=\mu.
 \label{eq:attention-P-main}
\end{equation}
For every $D\ge2$, the all-$+a$ schedule is globally optimal and
\begin{equation}
 E_D^{\rm attn}(X^\star)
 =\left[2a^2\delta_1(D)^2+2u_0^2\delta_\mu(D)^2\right]^{1/2}.
 \label{eq:attention-exact-main}
\end{equation}
The relaxed structural floor is zero and
\begin{equation}
 E_D^{\rm attn}(X^\star)
 =\frac{1}{\sqrt2D}
 \left(a^2e^{-2}+u_0^2\mu^4e^{-2\mu}\right)^{1/2}
 +O(D^{-2}).
 \label{eq:attention-asymptotic-main}
\end{equation}
\end{theorem}

The theorem is specific: it is a sharp witness for a two-token,
one-head residual attention system, not a lower bound for every complete
Transformer block.  Its value is that neither token mixing nor a nonuniform
softmax removes the first-order discretization price.  The invariant mean and
disagreement coordinates, the global schedule optimum, and the Frobenius
identity are proved in Appendix~\ref{app:merged-exact-proofs}.  Figure~\ref{fig:attention-fixed-target-main} compares the exact law, its asymptote, and exhaustive schedule enumeration.

Exact decoupling is stronger than one should expect from a trained Transformer or MLP block.  The next result makes every defect explicit and therefore turns the converse into a falsifiable checkpoint diagnostic.

\begin{theorem}[Perturbation-stable exposed-mode converse]
\label[theorem]{thm:approx-exposed-mode-main}
Normalize $T=1$, let $D\ge2$, and put $r_D=1-D^{-1}$.  Let $\ell\in\Z^*$ satisfy $\norm{\ell}_*\le1$ and $\ell(z_0)=0$.  Assign each atom a label $q_j\in\{0,a\}$ and suppose
\begin{equation}
 \abs{\ell(G_j(t,z))-(q_j-\ell(z))}
 \le\varepsilon_{\mathrm{mode}}
 \quad\text{for every }j,t,z
 \label{eq:approx-exposed-atom-main}
\end{equation}
on the complete pure-execution tube.  If
\begin{equation}
 \abs{\ell(F^\star)-a(1-e^{-1})}
 \le\varepsilon_\star,
 \label{eq:approx-exposed-target-main}
\end{equation}
then
\begin{equation}
 E_D(F^\star)
 \ge\max\left\{0,
 a(e^{-1}-r_D^D)
 -\varepsilon_\star
 -\varepsilon_{\mathrm{mode}}(1-r_D^D)
 \right\}.
 \label{eq:approx-exposed-lower-main}
\end{equation}
If $F^\star$ is generated by a continuous teacher whose projected path satisfies
\begin{equation}
 \abs{\dot y^\star-(a-y^\star)}\le\varepsilon_{\mathrm{teach}}
 \quad\text{a.e.},\qquad y^\star(0)=0,
 \label{eq:approx-exposed-teacher-main}
\end{equation}
then one may take $\varepsilon_\star=\varepsilon_{\mathrm{teach}}(1-e^{-1})$.
\end{theorem}

\begin{proof}
For one pure schedule define $y_k=\ell(z_k)$.  By \eqref{eq:approx-exposed-atom-main},
\begin{equation}
 y_{k+1}=r_Dy_k+\frac{q_{\sigma_k}}D+\frac{\xi_k}D,
 \qquad \abs{\xi_k}\le\varepsilon_{\mathrm{mode}}.
 \label{eq:approx-mode-recursion-main}
\end{equation}
Let $\bar y_k$ solve the exact binary recursion with the same labels.  Unrolling the difference gives
\begin{equation}
 \abs{y_D-\bar y_D}
 \le\frac{\varepsilon_{\mathrm{mode}}}{D}
       \sum_{m=0}^{D-1}r_D^m
 =\varepsilon_{\mathrm{mode}}(1-r_D^D).
 \label{eq:approx-mode-deviation-main}
\end{equation}
The reverse triangle inequality, \eqref{eq:approx-exposed-target-main}, and Theorem~\ref{thm:fixed-teacher-main} prove \eqref{eq:approx-exposed-lower-main}.  For the continuous teacher, variation of constants gives
\begin{equation}
 \abs{y^\star(1)-a(1-e^{-1})}
 \le\varepsilon_{\mathrm{teach}}
     \int_0^1e^{-(1-s)}\,ds
 =\varepsilon_{\mathrm{teach}}(1-e^{-1}).
 \label{eq:teacher-defect-main}
\end{equation}
\end{proof}

A fixed $O(1)$ modal defect eventually overwhelms a $D^{-1}$ discretization gap; it therefore cannot support an asymptotic converse.  Defects $o(D^{-1})$ preserve the constant $a/(2e)$, and $O(D^{-1})$ defects preserve the exponent with a modified constant.  This dependence is a feature rather than a weakness: the theorem says exactly what must be verified rather than converting an approximate empirical mode into an exact structural claim.

\begin{corollary}[Necessary and sufficient asymptotic matching-depth scale]
\label[corollary]{cor:matching-depth-two-sided-main}
Suppose that for some $D_\star\ge1$ and constants $c_->0$, $c_+>0$,
$c_2\ge0$,
\begin{equation}
 E_\infty+\frac{c_-}{D}-\frac{c_2}{D^2}
 \le E_D\le E_\infty+\frac{c_+}{D}
 \qquad(D\ge D_\star).
 \label{eq:two-sided-expansion-main}
\end{equation}
Set $D_0=\max\{D_\star,\lceil2c_2/c_-\rceil\}$ and let
$g=\varepsilon_H-E_\infty>0$.  Every matching depth on the asymptotic branch
satisfies $D\ge c_-/(2g)$, while
$D\ge\max\{D_\star,\lceil c_+/g\rceil\}$ is sufficient.  If $D_0>1$ and
\begin{equation}
 \Delta_0:=\min_{1\le D<D_0}(E_D-E_\infty)>0,
 \label{eq:finite-depth-exception-gap-main}
\end{equation}
then for all sufficiently small $g<\Delta_0$ the minimal matching depth is
$\Theta(g^{-1})$.  If $D_0=1$, the same conclusion holds without a finite
exception test.  A statement $D=\Theta(L)$ additionally requires an
independent comparator law $g(L)=\Theta(L^{-1})$; two upper bounds alone do
not imply it.
\end{corollary}

The finite-exception clause is logically necessary: an isolated shallow depth
can accidentally outperform the eventual asymptotic branch.  Appendix~\ref{app:merged-exact-proofs} gives the complete argument.

In particular, in the structurally feasible regime $E_{\infty,b}=0$, a coherent
high-precision comparator with first-order error and a two-sided first-order
low-bit law give the accuracy-matching exchange
\begin{equation}
\boxed{\begin{gathered}
\varepsilon_H(L)=\Theta(L^{-1}),\qquad
E_b(D;F^\star)=\Theta(D^{-1})\\
\Longrightarrow\quad D_{\mathrm{match}}(L,b;F^\star)=\Theta(L).
\end{gathered}}
\label{eq:linear-matching-law-main}
\end{equation}
This formalizes the widely used intuition that extra low-bit depth can replace
precision, while also stating its limit: feasibility alone yields only a
sufficient $D=O(L)$ law, and unrelated teacher families need not admit any
constant exchange rate.

\paragraph{Necessity consequence.}
In the zero-floor regime, first-order depth is not merely a sufficient artifact
of balanced switching: a single fixed teacher can require it exactly.  The
residual-MLP and attention constructions show that the obstruction survives in
recognizable neural architectures, while the perturbation-stable exposed-mode
theorem records how much mismatch can be tolerated before the lower certificate
becomes vacuous.

\section{Learned Dictionaries and Metadata Resources}
\label{sec:learned-dictionaries-main-section}

The finite-depth law prices execution once the operation library is fixed.
Practical quantizers also learn scales, centroids, basis tensors, clipping
levels, or the operation library itself.  These quantities are global side
information: they are stored once per model, layer, group, or codebook and then
reused across many examples and microsteps.  They should therefore be charged
separately from schedule depth.

\subsection{Compact learned dictionaries and global metadata}
\label{sec:learned-dictionaries-main}

Let $(\Omega,d_\Omega)$ be a nonempty compact parameter space.  Each $\omega\in\Omega$ indexes a label-aligned $J$-atom dictionary
\begin{equation}
 \mathcal G^\omega=\{G_1^\omega,\ldots,G_J^\omega\}.
 \label{eq:dictionary-family-main}
\end{equation}
When atom labels are exchangeable, $d_\Omega$ is understood after a canonical matching or on the quotient by permutations.  Otherwise a harmless relabeling could be charged as a large perturbation.

\begin{assumption}[Uniform learned-dictionary regularity]
\label[assumption]{ass:learned-dictionary-main}
The common-tube, boundedness, state regularity, temporal regularity, and existence assumptions of Assumption~\ref{ass:common-tube-main} hold uniformly over $\omega\in\Omega$, with $B,L_z$, the temporal variation or H\"older bounds, the tube radius, and the admissible step threshold independent of $D$, the pure schedule, and the encoded metadata value.  The constants $L_\Omega$ and $C_\Omega$ below are likewise uniform on the declared family.  There is $L_\Omega<\infty$ such that
\begin{equation}
 \sup_{t,z,j}
 \norm{G_j^\omega(t,z)-G_j^{\omega'}(t,z)}
 \le L_\Omega d_\Omega(\omega,\omega').
 \label{eq:dictionary-lipschitz-main}
\end{equation}
\end{assumption}

Let $\RR_{\omega,\Rel}$ and $\RR_{\omega,D}$ be the corresponding relaxed and pure endpoint sets, and let $\mathscr P_{\omega,\Rel}$ denote the relaxed path set in $C([0,T];\Z)$.  Define
\begin{equation}
 \RR_{\Omega,\Rel}
 :=\overline{\bigcup_{\omega\in\Omega}\RR_{\omega,\Rel}}.
 \label{eq:learned-relaxed-union-main}
\end{equation}
An $s$-bit metadata code is a nonempty finite set $\Omega_s\subset\Omega$ with $\abs{\Omega_s}\le2^s$ and covering radius
\begin{equation}
 \delta_s:=\sup_{\omega\in\Omega}
             \inf_{\widehat\omega\in\Omega_s}
             d_\Omega(\omega,\widehat\omega).
 \label{eq:metadata-covering-main}
\end{equation}
The implemented ideal-arithmetic class is
\begin{equation}
 \RR_{D,s}^{\mathrm{learn}}
 :=\bigcup_{\widehat\omega\in\Omega_s}\RR_{\widehat\omega,D}.
 \label{eq:learned-pure-union-main}
\end{equation}

\begin{lemma}[Lipschitz motion of relaxed endpoint and path sets]
\label[lemma]{lem:relaxed-codebook-lipschitz-main}
Under Assumption~\ref{ass:learned-dictionary-main},
\begin{align}
 \dH(\RR_{\omega,\Rel},\RR_{\omega',\Rel})
 &\le L_\Omega\Phi_{L_z}(T)d_\Omega(\omega,\omega'),
 \label{eq:relaxed-codebook-lipschitz-main}\\
 \dH(\mathscr P_{\omega,\Rel},\mathscr P_{\omega',\Rel})
 &\le L_\Omega\Phi_{L_z}(T)d_\Omega(\omega,\omega')
 \label{eq:relaxed-codebook-path-lipschitz-main}
\end{align}
\end{lemma}

\begin{proof}
Drive both systems with the same relaxed control.  Their state difference obeys
\begin{equation}
 \norm{e(t)}
 \le\int_0^t\left(L_z\norm{e(s)}
      +L_\Omega d_\Omega(\omega,\omega')\right)ds.
 \label{eq:codebook-gronwall-main}
\end{equation}
Gronwall's inequality gives the same bound uniformly for every
$t\in[0,T]$.  Taking $t=T$ gives
\eqref{eq:relaxed-codebook-lipschitz-main}; taking the supremum over $t$
gives \eqref{eq:relaxed-codebook-path-lipschitz-main}.  Reversing the roles of
$\omega$ and $\omega'$ and passing to closures proves both Hausdorff
estimates.
\end{proof}

\begin{theorem}[Structural floor plus depth plus metadata]
\label[theorem]{thm:learned-codebook-main}
Under Assumption~\ref{ass:learned-dictionary-main}, there are $D_0$ and a uniform synthesis constant $C_{\mathrm{syn}}^{\mathrm{unif}}$ such that, for every $D\ge D_0$,
\begin{equation}
 \dH(\RR_{D,s}^{\mathrm{learn}},\RR_{\Omega,\Rel})
 \le\frac{C_{\mathrm{syn}}^{\mathrm{unif}}}{D}
 +L_\Omega\Phi_{L_z}(T)\delta_s.
 \label{eq:learned-codebook-rate-main}
\end{equation}
Consequently, for every target $F^\star$,
\begin{equation}
 \left|\dist(F^\star,\RR_{D,s}^{\mathrm{learn}})
 -\dist(F^\star,\RR_{\Omega,\Rel})\right|
 \le\frac{C_{\mathrm{syn}}^{\mathrm{unif}}}{D}
 +L_\Omega\Phi_{L_z}(T)\delta_s.
 \label{eq:learned-target-rate-main}
\end{equation}
\end{theorem}

\begin{proof}
Take $y\in\RR_{\omega,\Rel}$.  Choose $\widehat\omega\in\Omega_s$ within $\delta_s$.  Lemma~\ref{lem:relaxed-codebook-lipschitz-main} places $y$ within $L_\Omega\Phi_{L_z}(T)\delta_s$ of $\RR_{\widehat\omega,\Rel}$, and Theorem~\ref{thm:floor-rate-main} places that set within $C_{\mathrm{syn}}^{\mathrm{unif}}/D$ of $\RR_{\widehat\omega,D}$.  This proves the difficult directed distance, including the closure by approximation.  Conversely, every pure endpoint in \eqref{eq:learned-pure-union-main} is within $C_{\mathrm{syn}}^{\mathrm{unif}}/D$ of the relaxed set with the same metadata, and that relaxed set is contained in \eqref{eq:learned-relaxed-union-main}.  The target-distance bound is the standard $1$-Lipschitz property of distance under Hausdorff perturbation.
\end{proof}

If the metric entropy satisfies
\begin{equation}
 \mathcal N(\Omega,d_\Omega,\delta)
 \le\left(\frac{C_\Omega}{\delta}\right)^m,
 \label{eq:metadata-entropy-main}
\end{equation}
then one may choose
\begin{equation}
 \delta_s\le C_\Omega2^{-s/m},
 \label{eq:metadata-radius-main}
\end{equation}
and Theorem~\ref{thm:learned-codebook-main} becomes
\begin{equation}
 E_{D,s}(F^\star)
 \le E_{\Omega,\infty}(F^\star)
 +\frac{C_{\mathrm{syn}}^{\mathrm{unif}}}{D}
 +C_{\mathrm{meta}}2^{-s/m},
 \quad
 C_{\mathrm{meta}}:=L_\Omega\Phi_{L_z}(T)C_\Omega.
 \label{eq:metadata-bit-law-main}
\end{equation}
Balancing the two vanishing terms gives $s=m\log_2D+O(1)$.  Thus a finite-dimensional learned codebook needs only logarithmically increasing global description precision to remain commensurate with first-order execution depth.

\subsection{A matching family-level counting obstruction}
\label{sec:metadata-converse-main}

The preceding result is constructive.  The following statement is purely information theoretic and does not rely on smoothness.

\begin{theorem}[Depth--dictionary--metadata packing converse]
\label[theorem]{thm:packing-main}
Consider a deterministic decoder with at most $2^s$ global metadata states and an input-independent length-$D$ schedule over $J$ atom labels.  The decoder receives no uncharged continuous side information, target-dependent randomness, hidden real-valued state, or post-hoc parameter selection after the target is revealed.  If the resulting class approximates every target in a metric family $\mathcal F$ to error at most $\varepsilon$, then
\begin{equation}
 s+D\log_2J
 \ge\log_2\mathcal M(\mathcal F,2\varepsilon),
 \label{eq:packing-main}
\end{equation}
where $\mathcal M(\mathcal F,2\varepsilon)$ is the largest cardinality of a subset whose pairwise distances are strictly greater than $2\varepsilon$.
\end{theorem}

\begin{proof}
There are at most $2^sJ^D$ declared configurations.  Two targets farther than $2\varepsilon$ apart cannot share one $\varepsilon$-accurate approximant by the triangle inequality.  Therefore every $2\varepsilon$ packing injects into the configuration set, so $\mathcal M(\mathcal F,2\varepsilon)\le2^sJ^D$.
\end{proof}

The qualifier ``declared'' is essential.  Variable-length programs, stochastic descriptions, input-conditioned controllers, externally supplied prompts, and continuous uncharged state can carry additional information and require a different accounting.  The theorem is not a universal Kolmogorov-complexity statement; it is a lower law for the same deterministic global-metadata and pure-schedule model used by the upper theorem.

\subsection{Joint allocation under a declared resource budget}
\label{sec:budget-allocation-main}

Let $\ell_\sigma>0$ be the charged units per executed microstep.  It may stand for schedule bits, operations, latency, or energy, provided the same quantity is used consistently.  Suppose
\begin{equation}
 \abs{E_{D,s}-E_{\Omega,\infty}}
 \le\frac{C_{\mathrm{syn}}}{D}+C_{\mathrm{meta}}2^{-s/m}.
 \label{eq:description-start-main}
\end{equation}
For total declared budget $B>1$, choose
\begin{align}
 s_B&=\left\lceil m\log_2\!\left(\max\{1,C_{\mathrm{meta}}B\}\right)\right\rceil,\\
 D_B&=\left\lfloor\frac{B-s_B}{\ell_\sigma}\right\rfloor.
 \label{eq:description-allocation-main}
\end{align}
For budgets large enough that $B\ge2(s_B+\ell_\sigma)$ and $D_B\ge\max\{D_0,2\}$,
\begin{equation}
 E_{D_B,s_B}
 \le E_{\Omega,\infty}
 +\frac{2C_{\mathrm{syn}}\ell_\sigma+1}{B}.
 \label{eq:description-upper-main}
\end{equation}
Indeed $C_{\mathrm{meta}}2^{-s_B/m}\le B^{-1}$ and $D_B\ge B/(2\ell_\sigma)$.  For $C_{\mathrm{meta}}>0$,
\begin{equation}
 s_B=m\log_2B+O(1),
 \qquad
 D_B=\frac{B}{\ell_\sigma}-\frac{m}{\ell_\sigma}\log_2B+O(1).
 \label{eq:description-asymptotic-main}
\end{equation}
Thus, under a positive per-step charge, global codebook metadata consumes only a logarithmic correction and nearly all asymptotic budget is allocated to executed depth.  The fixed-teacher converse simultaneously gives $E_D(F^\star)\ge a\ell_\sigma/(4eB)$ whenever $s+\ell_\sigma D\le B$ and the exposed mode is unaffected by metadata.  The budget exponent is therefore sharp in that regime.

\begin{figure}[t]
\centering
\includegraphics[width=.99\linewidth]{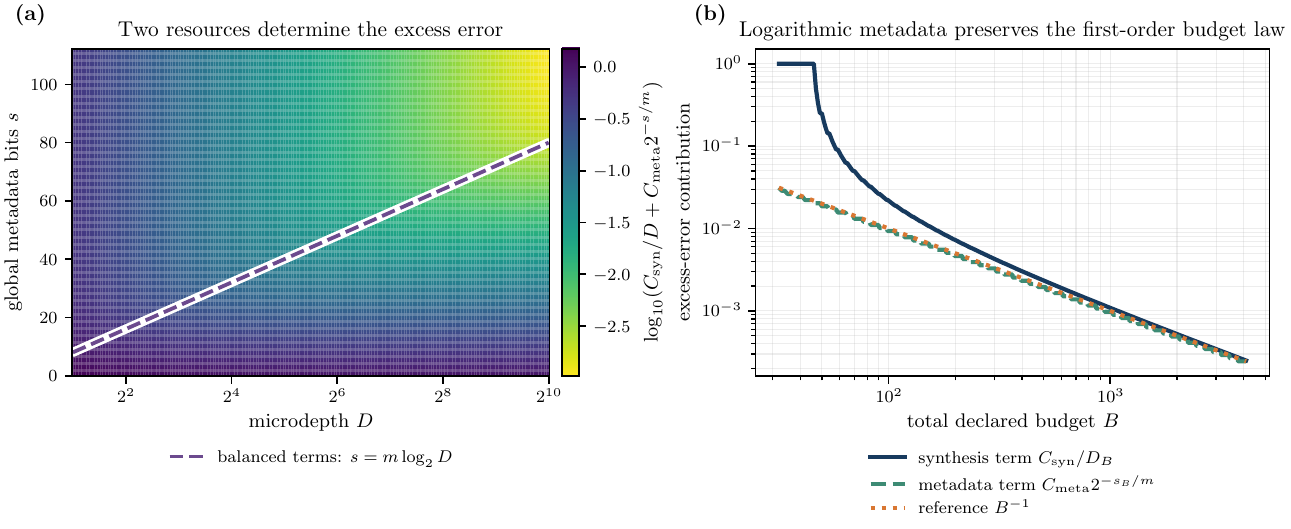}
\caption{\textnormal{Depth and global codebook metadata are distinct resources.}
\textnormal{(a)} The excess-error surface
$C_{\mathrm{syn}}/D+C_{\mathrm{meta}}2^{-s/m}$ for a representative
$m$-dimensional family.  The white-backed curve balances the two terms and is
not an empirical fit.  \textnormal{(b)} Under the declared unit-cost budget in
Section~\ref{sec:budget-allocation-main}, logarithmic metadata keeps its error
commensurate with the synthesis term and preserves the first-order budget law.}
\label{fig:metadata-resources-main}
\end{figure}

Figure~\ref{fig:metadata-resources-main} visualizes the two distinct decay laws and the logarithmic metadata allocation implied by the theorem.

\FloatBarrier

\paragraph{Metadata consequence.}
Executed depth and learned dictionary description are different resources.
Under the compact-family assumptions above, logarithmically growing metadata
is sufficient to keep codebook discretization commensurate with first-order
synthesis, while the packing theorem shows that a family-level approximation
guarantee must pay for both schedule labels and global metadata.  Nominal
weight precision alone does not encode either cost.

\section{Route-Changing Quantized Networks}
\label{sec:hybrid-routing-main}

Hard routing is not a smooth perturbation of a dense block.  A small score error can change the selected expert, the changed expert alters the state, and the state perturbation changes later scores.  The right object is therefore a hybrid feedback loop, not an additive ``router error'' attached after the fact.

A preliminary point prevents a common logical mistake.  If a continuous top-$k$ score path has a uniform strictly positive boundary margin on a connected time interval, its support cannot change on that interval.  A theorem based only on a positive margin is consequently a \emph{frozen-route stability theorem}; it does not describe actual expert swaps.  Route changes require score ties, and well-conditioned changes require those ties to be isolated and transversal.

The bounded-variation version of Theorem~\ref{thm:floor-rate-main} already handles a prescribed piecewise-smooth route: each route event contributes to the temporal variation constant but does not change the $D^{-1}$ exponent.  The remaining problem is to compare a target-route oracle with the state-driven implemented router.

\subsection{Prescribed route times over a continuum input population}
\label{sec:prescribed-route-lp-main}

Before treating a router whose event surface moves with the perturbed state,
consider the simpler but important case in which each input has a prescribed
route schedule.  Shared event times give bounded variation in the full-map
state.  Dispersed event times generally do not, but their concentration gives
a H\"older modulus in $L^p$.

\begin{corollary}[Distributional prescribed-route synthesis]
\label[corollary]{cor:distributional-route-main}
Let $(\Xi,\nu)$ be a probability space, $\Z=L^p(\nu;\mathbb R^q)$ with
$1\le p<\infty$, and let a prescribed route $m^\star(t,\xi)$ have at most
$R$ event-time functions $\tau_r:\Xi\to[0,T]$.  Suppose route changes can
occur only when $t$ crosses a $\tau_r(\xi)$ and, for some $\alpha\in(0,1]$,
\begin{equation}
 \nu\{\xi:\tau_r(\xi)\in(s,t]\}
 \le K_r|t-s|^\alpha
 \qquad(0\le s<t\le T).
 \label{eq:event-time-concentration-main}
\end{equation}
Assume the within-mode fields are pointwise $H_0$-Lipschitz in time, satisfy
the common $L^p$ state-Lipschitz and tube assumptions, and the pointwise field
jump across event $r$ is at most $\Delta_r$.  Then the induced routed atom is
$\vartheta$-H\"older in $L^p$ with
\begin{equation}
 \vartheta=\frac\alpha p,
 \qquad
 H_t\le H_0T^{1-\alpha/p}
       +\sum_{r=1}^R\Delta_rK_r^{1/p}.
 \label{eq:route-holder-main}
\end{equation}
Consequently, Theorem~\ref{thm:holder-time-main} gives a shared-schedule rate
$O(D^{-\alpha/p}+D^{-1})$ in the $L^p(\nu)$ map norm.
\end{corollary}

This result permits continuum-input prescribed route motion and states exactly
how the event-time distribution controls the exponent.  It does not establish
robustness of a state-dependent router to perturbed event surfaces; that
feedback problem is the subject of the next two subsections.

\subsection{Target route, actual route, and event isolation}
\label{sec:routing-setup-main}

Let $N_e$ be the number of experts, let $\TopK_k(s)$ denote a deterministic top-$k$ operator with a fixed tie rule, and let $m^\star(t)$ be a right-continuous target support.  Assume it has isolated event times
\begin{equation}
 0<\tau_1<\cdots<\tau_R<T.
 \label{eq:route-event-times-main}
\end{equation}
At event $r$, one expert $a_r$ enters and one expert $b_r$ leaves; the adjacent supports are denoted by $m_r^-$ and $m_r^+$.  The ideal and implemented score vectors are $s(t,z)$ and $\widehat s(t,z)$.  The relevant pair gap is
\begin{equation}
 \gamma_r(t,z):=s_{a_r}(t,z)-s_{b_r}(t,z),
 \qquad
 \widehat\gamma_r(t,z):=\widehat s_{a_r}(t,z)-\widehat s_{b_r}(t,z).
 \label{eq:route-pair-gaps-main}
\end{equation}

Event transversality alone does not exclude a third expert from entering the support.  We therefore separate \emph{pair crossing} from \emph{mode isolation}.

\begin{definition}[Two-mode isolation radius]
\label[definition]{def:two-mode-isolation-main}
An event neighborhood $U_r$ has isolation radius $\Gamma_r>0$ when every score vector $\widetilde s$ satisfying
\begin{equation}
 \norm{\widetilde s-s(t,z^\star(t))}_\infty<\Gamma_r
 \label{eq:isolation-radius-main}
\end{equation}
has top-$k$ support in $\{m_r^-,m_r^+\}$, and within these two modes the selected support is $m_r^+$ exactly when $\widetilde s_{a_r}>\widetilde s_{b_r}$.  Away from the event neighborhoods, $\Gamma_{\mathrm{out}}>0$ is an isolation radius if the same perturbation preserves the entire target support.
\end{definition}

A sufficient condition is directly checkable: every pairwise ranking inequality other than the $(a_r,b_r)$ swap has gap strictly larger than $2\Gamma_r$, while away from events the target $k$th-to-$(k+1)$st gap is strictly larger than $2\Gamma_{\mathrm{out}}$.  An $\ell_\infty$ score perturbation of radius $\Gamma$ changes any pair gap by less than $2\Gamma$, so these inequalities retain their signs.

For one shared atom schedule $\sigma$, define a target-route oracle and the actual state-driven execution from the same initial state:
\begin{align}
 y_{k+1}
 &=y_k+hG_{\sigma_k}^{m^\star(t_k)}(t_k,y_k),
 \label{eq:target-route-oracle-main}\\
 x_{k+1}
 &=x_k+hG_{\sigma_k}^{\widehat m(t_k,x_k)}(t_k,x_k),
 \qquad
 \widehat m(t,z)=\TopK_k(\widehat s(t,z)).
 \label{eq:actual-routed-main}
\end{align}
The oracle isolates approximation and arithmetic error under the intended route; the second recursion contains the additional hybrid feedback caused by route decisions.

\begin{assumption}[Simple isolated top-$k$ events]
\label[assumption]{ass:simple-route-events-main}
There are disjoint neighborhoods $U_r=(\tau_r-\ell_r,\tau_r+\ell_r)$ and constants $\nu_r,L_r,\eta_{r,D},\Delta_r>0$ such that:
\begin{enumerate}[label=(\roman*),leftmargin=2.1em,itemsep=.32em]
\item the target pair gap is continuous on $\overline U_r$, vanishes at $\tau_r$, and crosses with signed slope bound
\begin{align}
 \gamma_r(t,z^\star(t))&\le-\nu_r(\tau_r-t),
 &&t<\tau_r,\\
 \gamma_r(t,z^\star(t))&\ge \nu_r(t-\tau_r),
 &&t>\tau_r;
 \label{eq:signed-transversality-main}
\end{align}
\item the implemented pair gap obeys the target-path-relative estimate
\begin{equation}
 \abs{\widehat\gamma_r(t,z)-\gamma_r(t,z^\star(t))}
 \le\eta_{r,D}+L_r\norm{z-z^\star(t)};
 \label{eq:route-gap-error-main}
\end{equation}
\item the complete score vector obeys
\begin{equation}
 \norm{\widehat s(t,z)-s(t,z^\star(t))}_\infty
 \le\eta_{0,D}+L_0\norm{z-z^\star(t)};
 \label{eq:route-full-score-error-main}
\end{equation}
\item the target has event and away-event isolation radii as in Definition~\ref{def:two-mode-isolation-main};
\item every routed atom is $L_z$-Lipschitz in state on a common route tube, and adjacent route fields satisfy
\begin{equation}
 \sup_{j,t,z}
 \norm{G_j^{m_r^+}(t,z)-G_j^{m_r^-}(t,z)}
 \le\Delta_r.
 \label{eq:route-field-jump-main}
\end{equation}
\end{enumerate}
Events occur one at a time.  Simultaneous multi-surface intersections require a separate mode graph and are not covered by the theorem below.
\end{assumption}

\subsection{The route-window small-gain theorem}
\label{sec:route-window-theorem-main}

Assume the target-route oracle satisfies
\begin{equation}
 \max_{k\le D}\norm{y_k-z^\star(t_k)}\le\delta_D.
 \label{eq:oracle-route-error-main}
\end{equation}
The quantity $\delta_D$ may already include pure-to-relaxed synthesis, field quantization, increment arithmetic, and metadata error.  The hybrid theorem adds only the price of route disagreement.

\begin{theorem}[Transversal top-$k$ events have certified mismatch windows]
\label[theorem]{thm:route-window-main}
Under Assumption~\ref{ass:simple-route-events-main}, put $S=e^{L_zT}$ and define
\begin{align}
 \chi
 &:=2S\sum_{r=1}^R\frac{\Delta_rL_r}{\nu_r},
 \label{eq:route-small-gain-main}\\
 b_D
 &:=\delta_D
 +2S\sum_{r=1}^R\frac{\Delta_r\eta_{r,D}}{\nu_r}
 +2Sh\sum_{r=1}^R\Delta_r,
 \label{eq:route-forcing-main}\\
 \overline E_D
 &:=\frac{b_D}{1-\chi},
 \label{eq:route-envelope-main}\\
 w_{r,D}
 &:=\frac{\eta_{r,D}+L_r\overline E_D}{\nu_r}.
 \label{eq:route-window-radius-main}
\end{align}
Suppose $\chi<1$, every $w_{r,D}<\ell_r$, the resulting windows are disjoint, all prefixes remain in the common route tube, and
\begin{equation}
 \eta_{0,D}+L_0\overline E_D
 <\min\{\Gamma_{\mathrm{out}},\Gamma_1,\ldots,\Gamma_R\}.
 \label{eq:route-isolation-budget-main}
\end{equation}
The isolation inequality is an explicit a posteriori premise: compute $\overline E_D$ from \eqref{eq:route-envelope-main}, then verify \eqref{eq:route-isolation-budget-main}.
Then
\begin{equation}
 \max_{k\le D}\norm{x_k-z^\star(t_k)}\le\overline E_D,
 \label{eq:route-state-bound-main}
\end{equation}
and the actual and target supports agree at every grid time outside
\begin{equation}
 \bigcup_{r=1}^R
 [\tau_r-w_{r,D},\tau_r+w_{r,D}].
 \label{eq:route-window-union-main}
\end{equation}
The theorem does not assume that the implemented route switches only once inside a window.
\end{theorem}

\begin{proof}
We argue by induction on the grid index.  The initial equality
$x_0=y_0=z^\star(0)$ gives the claim at $n=0$.  Assume the state bound and
route-tube membership hold through index $n$.  First, the state bound at each
index $k\le n$, the target-path-relative score estimates, and signed
transversality imply agreement of the implemented and target pair-gap signs
whenever $|t_k-\tau_r|>w_{r,D}$.  The strict isolation budget rules out every
nonadjacent support.  Thus, through index $n$, mismatches for event $r$ occur
only in the boundary-clipped window and at no more than
$2w_{r,D}/h+2$ grid points.

Second, compare $x$ with the target-route oracle $y$.  A matched update has
only the factor $1+hL_z$; a mismatched update in window $r$ has additional
forcing at most $h\Delta_r$.  Applying discrete Gronwall to the forcing terms
through update $n$ gives
\begin{align}
 \max_{k\le n+1}\norm{x_k-y_k}
 &\le S\sum_{r=1}^R\Delta_r(2w_{r,D}+2h)\nonumber\\
 &=2S\sum_{r=1}^R\frac{\Delta_r\eta_{r,D}}{\nu_r}
   +\chi\overline E_D+2Sh\sum_{r=1}^R\Delta_r.
 \label{eq:route-gronwall-main}
\end{align}
Adding the oracle error gives
\[
 \norm{x_{n+1}-z^\star(t_{n+1})}
 \le b_D+\chi\overline E_D=\overline E_D.
\]
The theorem's route-tube premise places the new state and the interpolation
segment used in the one-step estimate inside the domain of the constants,
which closes the induction.  Reapplying the sign and isolation argument at all
grid times proves \eqref{eq:route-window-union-main}.
\end{proof}

The denominator $1-\chi$ has a direct interpretation.  The factor $L_r/\nu_r$ converts state error into event-window width; $\Delta_r$ converts window width back into state forcing; $S$ propagates that forcing through the remaining horizon.  The sum of these loop gains must be smaller than one.

\begin{corollary}[First-order route-changing law]
\label[corollary]{cor:first-order-routing-main}
If
\begin{equation}
 \delta_D\le\frac{C_\delta}{D},
 \qquad
 \eta_{r,D}\le\frac{C_{\eta,r}}{D},
 \qquad h=\frac TD,
 \label{eq:first-order-route-input-main}
\end{equation}
and the $D$-independent small-gain, tube, and isolation conditions hold, then
\begin{equation}
 \max_{k\le D}\norm{x_k-z^\star(t_k)}
 \le\frac{C_{\mathrm{route}}}{D},
 \qquad
 w_{r,D}\le
 \frac{C_{\eta,r}+L_rC_{\mathrm{route}}}{\nu_rD},
 \label{eq:first-order-route-output-main}
\end{equation}
where
\begin{equation}
 C_{\mathrm{route}}
 =\frac{C_\delta
 +2S\sum_r\Delta_rC_{\eta,r}/\nu_r
 +2ST\sum_r\Delta_r}{1-\chi}.
 \label{eq:first-order-route-constant-main}
\end{equation}
Thus isolated route changes preserve the first-order synthesis exponent when oracle and score errors are first order and the hybrid feedback gain is below one.  Because $w_{r,D}=O(D^{-1})$, the disjoint-window and $w_{r,D}<\ell_r$ premises hold automatically for all sufficiently large $D$ whenever the target event neighborhoods have fixed positive separation and width.
\end{corollary}

Figure~\ref{fig:hybrid-route-window-main} visualizes the two linked conclusions: localization of support errors and control of the downstream state.

\begin{figure}[t]
\centering
\includegraphics[width=.99\linewidth]{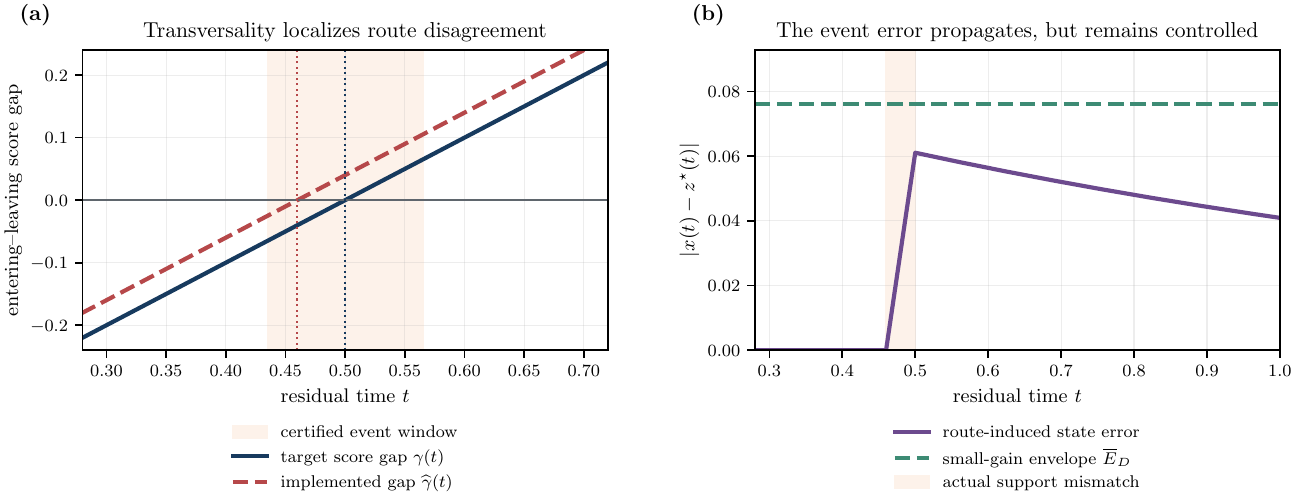}
\caption{A route-changing diagnostic consistent with Theorem~\ref{thm:route-window-main}.  \textbf{Left:} a direct score perturbation makes the implemented route switch one grid point early, but transversality confines the mismatch to the shaded certified event window.  \textbf{Right:} the resulting state error is visible and remains below the small-gain envelope.  The plotted curves are a deterministic theorem illustration, not a claim about a particular MoE checkpoint.}
\label{fig:hybrid-route-window-main}
\end{figure}

\subsection{Event-time uniqueness requires an anti-chattering premise}
\label{sec:anti-chattering-main}

A uniformly small score perturbation can oscillate arbitrarily often near a zero.  Transversality of the target therefore localizes all mismatches but does not by itself imply one implemented switch.  Let $\widetilde x_D$ be the linear interpolation of the Euler states.  If target and routed fields have speed at most $B$, then
\begin{equation}
 \sup_{t\in[0,T]}\norm{\widetilde x_D(t)-z^\star(t)}
 \le\widetilde E_D:=\overline E_D+2Bh.
 \label{eq:continuous-route-state-main}
\end{equation}
The corresponding continuous event window is
\begin{equation}
 \widetilde w_{r,D}
 :=\frac{\eta_{r,D}+L_r\widetilde E_D}{\nu_r}.
 \label{eq:continuous-route-window-main}
\end{equation}
Suppose the target and implemented pair gaps along their respective paths are continuous and absolutely continuous on $U_r$, and there exist $\varsigma_r\in\{-1,1\}$ and $\kappa_r>\zeta_r>0$ such that almost everywhere
\begin{equation}
 \varsigma_r g_r^{\star\prime}(t)\ge\kappa_r,
 \qquad
 \abs{\widehat g_r'(t)-g_r^{\star\prime}(t)}
 \le\kappa_r-\zeta_r.
 \label{eq:event-derivative-separation-main}
\end{equation}
Then $\varsigma_r\widehat g_r$ is strictly increasing, it has exactly one zero $\widehat\tau_r$, and
\begin{equation}
 \abs{\widehat\tau_r-\tau_r}\le\widetilde w_{r,D}.
 \label{eq:event-time-bound-main}
\end{equation}
The a.e. formulation permits derivative jumps at Euler grid points.  Hysteresis or a minimum dwell-time condition can replace derivative separation when that better matches the router implementation.

\subsection{Continuum inputs: why the ambient space must change}
\label{sec:routing-lp-main}

For a continuum of inputs, hard top-$k$ routing can make $\xi\mapsto F(\xi)$ discontinuous.  Such a map may leave $C(\Xi)$ even when every expert is continuous.  A theorem in the continuous-map sup norm is therefore invalid without route-cell restrictions or a positive uniform margin.

Let $(\Xi,\mu)$ be a finite measure space and assume the routed dynamics are input separable apart from the shared atom schedule.  If the hypotheses of Theorem~\ref{thm:route-window-main} hold almost everywhere with measurable input-dependent constants, $\operatorname*{ess\,sup}_\xi\chi(\xi)<1$, and $\overline E_D\in L^p(\mu)$, then pointwise application followed by integration gives
\begin{equation}
 \max_{k\le D}
 \norm{x_k-z^\star(t_k)}_{L^p(\mu)}
 \le\left\|\frac{b_D}{1-\chi}\right\|_{L^p(\mu)}.
 \label{eq:routing-lp-state-main}
\end{equation}
Moreover, if $\mathcal M_k$ is the set of inputs with mismatched support at $t_k$, and $R(\xi)$ is the finite measurable number of isolated target events for input $\xi$ with $R\in L^1(\mu)$,
\begin{equation}
 h\sum_{k=0}^{D-1}\mu(\mathcal M_k)
 \le2\int_\Xi\sum_r w_{r,D}(\xi)\,d\mu(\xi)
 +2h\int_\Xi R(\xi)\,d\mu(\xi).
 \label{eq:routing-lp-occupancy-main}
\end{equation}
A distributional boundary-margin condition
\begin{equation}
 \mu\{\Gamma^\star\le u\}\le C_\Gamma u^\alpha
 \label{eq:margin-density-main}
\end{equation}
then converts an $\ell_\infty$ score error $\varepsilon$ into mismatch mass at most $C_\Gamma(2\varepsilon)^\alpha$ and an $L^p$ route-field perturbation of order $\varepsilon^{\alpha/p}$.  No analogous improvement follows in $L^\infty$ without a positive essential margin.

\subsection{Auditable constants for weighted MoE blocks}
\label{sec:moe-constants-main}

Let each expert field satisfy
\begin{equation}
 \norm{E_e(t,z)}\le B_E,
 \qquad
 \norm{E_e(t,z)-E_e(t,z')}\le L_E\norm{z-z'},
 \label{eq:moe-expert-bounds-main}
\end{equation}
and let the zero-padded normalized gate vector on support $S$ satisfy
\begin{equation}
 \norm{\pi^S(t,z)-\pi^S(t,z')}_1
 \le L_\pi\norm{z-z'}.
 \label{eq:moe-gate-lipschitz-main}
\end{equation}
For
\begin{equation}
 G^S(t,z)=\sum_e\pi_e^S(t,z)E_e(t,z),
 \label{eq:weighted-moe-field-main}
\end{equation}
one may take
\begin{equation}
 L_z\le L_E+B_EL_\pi.
 \label{eq:moe-field-lipschitz-main}
\end{equation}
For one expert swap,
\begin{equation}
 \Delta_r
 \le B_E\sup_{U_r\times\mathcal K}
       \norm{\pi^{S_r^+}-\pi^{S_r^-}}_1
 \le2B_E.
 \label{eq:moe-field-jump-main}
\end{equation}
On a common support, expert error $\eta_E$ and gate error $\eta_\pi$ give
\begin{equation}
 \norm{\widehat G^S-G^S}\le\eta_E+B_E\eta_\pi.
 \label{eq:moe-field-error-main}
\end{equation}
Per-expert score error $\eta_{\mathrm{score}}$ and state Lipschitz constant $L_{\mathrm{score}}$ imply
\begin{equation}
 \eta_{r,D}\le2\eta_{\mathrm{score}},
 \qquad
 L_r\le2L_{\mathrm{score}}.
 \label{eq:moe-gap-constants-main}
\end{equation}
For uniform top-$k$ averaging, a one-expert swap improves \eqref{eq:moe-field-jump-main} to $\Delta_r\le2B_E/k$.  Every term in the small-gain test can therefore be measured or bounded from experts, normalized gates, router scores, and event slopes.

\paragraph{Hybrid-stability consequence.}
Quantized route-changing networks obey a first-order depth law only in a hybrid stability regime: target events must be isolated and transversal, score and oracle errors must shrink at the required rate, and the route--state feedback gain must be below one.  A frozen-margin argument is not a substitute for this theorem.

\FloatBarrier

\section{Certifying the Structural Floor Before Training}
\label{sec:certification-main}
\label{sec:certification}

The structural floor is the decisive asymptotic invariant, but a definition by an infinite-dimensional reachable set is not yet a practical decision procedure.  This section develops a primal--dual hierarchy.  A feasible relaxed trajectory gives an upper bound; a dual witness gives a lower bound.  Their gap quantifies what remains unresolved.  The hierarchy ranges from exact finite-dimensional affine programs to HJB inequalities and modal occupation-measure relaxations.

Throughout this section, $z\in\mathbb R^n$ is a finite-dimensional lifted state, for example the product state obtained by evaluating a shared network on a finite witness set.  Let $K$ be a compact tube containing every relaxed trajectory and let $\phi:K\to\mathbb R$ be a continuous terminal loss.  Define
\begin{equation}
 E_\infty^\phi
 :=\inf_{p(\cdot)}\phi(z_p(T)).
 \label{eq:general-floor-objective-main}
\end{equation}
For a norm-distance problem one may use an epigraph state or a support-function formulation; for a finite witness set the resulting optimization is finite dimensional.

\subsection{HJB subsolutions are rigorous lower certificates}
\label{sec:hjb-main}

\begin{theorem}[HJB subsolution lower certificate]
\label[theorem]{thm:hjb-lower-main}
Let $v$ be the restriction to $[0,T]\times K$ of a $C^1$ function on an open neighborhood.  Suppose
\begin{align}
 v(T,z)&\le\phi(z),
 &&z\in K,
 \label{eq:hjb-terminal-main}\\
 \partial_tv(t,z)+\nabla_zv(t,z)^\top G_j(t,z)&\ge0,
 &&(t,z)\in[0,T]\times K,
 \quad j\in[J].
 \label{eq:hjb-atom-main}
\end{align}
Then
\begin{equation}
 v(0,z_0)\le E_\infty^\phi.
 \label{eq:hjb-value-main}
\end{equation}
The conclusion remains valid for the closed relaxed endpoint set.
\end{theorem}

\begin{proof}
For a relaxed control $p$ and its absolutely continuous trajectory,
\begin{align}
 \frac{d}{dt}v(t,z_p(t))
 &=\partial_tv(t,z_p(t))
   +\sum_jp_j(t)\nabla v(t,z_p(t))^\top G_j(t,z_p(t))\\
 &\ge0
 \label{eq:hjb-chain-main}
\end{align}
for almost every $t$, because $p_j(t)\ge0$, $\sum_jp_j(t)=1$, and every atom inequality is nonnegative.  Thus
\begin{equation}
 v(0,z_0)\le v(T,z_p(T))\le\phi(z_p(T)).
 \label{eq:hjb-path-main}
\end{equation}
Taking the infimum proves the claim; continuity extends it to endpoint closures.
\end{proof}

A neural or polynomial candidate is not automatically a certificate.  Inequalities \eqref{eq:hjb-terminal-main}--\eqref{eq:hjb-atom-main} must be verified on the entire declared tube by exact algebra, sound interval bounds, a proof-producing SDP, or another global enclosure.  Sampled residuals are only diagnostics.

\subsection{Compositional HJB certificates with a coupling budget}
\label{sec:compositional-hjb-main}

A monolithic certificate can be as difficult as the original control problem.
Residual and Transformer architectures, however, are naturally block
structured.  Local lower certificates remain valid globally when the omitted
couplings are charged explicitly.

\begin{theorem}[Coupling-robust compositional HJB certificate]
\label[theorem]{thm:compositional-hjb-main}
Let $K=K_1\times\cdots\times K_M$, $z=(z_1,\ldots,z_M)$, and suppose
\begin{equation}
 G_j(t,z)_i=g_{i,j}(t,z_i)+c_{i,j}(t,z),
 \qquad i\in[M].
 \label{eq:compositional-dynamics-main}
\end{equation}
Let the terminal loss satisfy $\phi(z)\ge\sum_i\phi_i(z_i)$.  For every block,
let each $v_i$ be the restriction to $[0,T]\times K_i$ of a $C^1$ function defined on an open neighborhood and satisfy
\begin{align}
 v_i(T,z_i)&\le\phi_i(z_i),\nonumber\\
 \partial_tv_i(t,z_i)+\nabla v_i(t,z_i)^\top g_{i,j}(t,z_i)&\ge0
 \quad\text{for every }j.
 \label{eq:local-hjb-main}
\end{align}
If a nonnegative continuous function $\beta$ verifies
\begin{equation}
 \sum_{i=1}^M\nabla v_i(t,z_i)^\top c_{i,j}(t,z)
 \ge-\beta(t)
 \quad\text{for every }(t,z,j),
 \label{eq:coupling-budget-main}
\end{equation}
then
\begin{equation}
 v_{\rm comp}(t,z)
 =\sum_{i=1}^Mv_i(t,z_i)-\int_t^T\beta(s)\,ds
 \label{eq:compositional-value-main}
\end{equation}
is a valid global HJB lower certificate and
\begin{equation}
 E_\infty^\phi
 \ge\sum_{i=1}^Mv_i(0,z_{i,0})-\int_0^T\beta(s)\,ds.
 \label{eq:compositional-lower-main}
\end{equation}
In particular, if continuous nonnegative envelopes $a_i,b_i$ satisfy
$\norm{\nabla v_i}_*\le a_i(t)$ and
$\norm{c_{i,j}}\le b_i(t)$ uniformly, one may take
$\beta(t)=\sum_i a_i(t)b_i(t)$.
\end{theorem}

The theorem does not claim that arbitrary local certificates compose for
free.  It identifies the exact loss incurred by inter-block dependence.  It
is therefore useful both constructively and diagnostically: a large coupling
budget shows where a proposed modular certificate is too weak.  The sign and
terminal calculations are given in Appendix~\ref{app:merged-exact-proofs}.

\subsection{Full-map support witnesses}
\label{sec:support-witness-main}

Let $F^\star$ be a target in a Banach space $\Z$ and let $\lambda\in\Z^*$ satisfy $\norm{\lambda}_*\le1$.  For every reachable $y$,
\begin{equation}
 \norm{F^\star-y}\ge\inner{\lambda}{F^\star-y},
 \label{eq:dual-norm-main}
\end{equation}
so
\begin{equation}
 E_\infty(F^\star)
 \ge\inner{\lambda}{F^\star}
 -\sup_{y\in\RR_{\Rel}}\inner{\lambda}{y}.
 \label{eq:support-basic-main}
\end{equation}
If a $C^1$ function $w$ satisfies
\begin{align}
 w(T,z)&\ge\inner{\lambda}{z},
 \label{eq:support-terminal-main}\\
 \partial_tw+\nabla w^\top G_j&\le0
 \quad\text{for every atom},
 \label{eq:support-hjb-main}
\end{align}
then $w$ is nonincreasing along every relaxed trajectory and
\begin{equation}
 E_\infty(F^\star)
 \ge\inner{\lambda}{F^\star}-w(0,z_0).
 \label{eq:support-certified-main}
\end{equation}
For $\Z=C(\Xi;\mathbb R^q)$ with the supremum norm, every finite signed point-evaluation functional
\begin{equation}
 \lambda(F)=\sum_{i=1}^Na_i^\top F(\xi_i),
 \qquad
 \sum_i\norm{a_i}_2\le1,
 \label{eq:finite-point-witness-main}
\end{equation}
is admissible.  This is the bridge from a finite witness computation to a rigorous lower bound on the full-map norm: the witness is a legitimate dual functional, not a claim that the finite sample uniformly covers the domain.

\subsection{An exact affine dual and rational certificates}
\label{sec:affine-certificate-main}

A particularly useful scalable case occurs when the atoms share one linear drift:
\begin{equation}
 G_j(t,z)=A(t)z+b_j(t).
 \label{eq:common-drift-affine-main}
\end{equation}
Let $\Phi(t,s)$ be the fundamental matrix of $\dot z=A(t)z$, let $P:\mathbb R^n\to\mathbb R^q$ be linear, and define
\begin{equation}
 a=P\Phi(T,0)z_0,
 \qquad
 c_j(t)=P\Phi(T,t)b_j(t).
 \label{eq:affine-output-data-main}
\end{equation}
Variation of constants gives
\begin{equation}
 Pz_p(T)=a+\int_0^T\sum_jp_j(t)c_j(t)\,dt.
 \label{eq:affine-variation-main}
\end{equation}
The closed relaxed output set is compact and convex, and its support function is
\begin{equation}
 h_{\mathcal Y_{\Rel}}(\lambda)
 =\inner{\lambda}{a}
 +\int_0^T\max_j\inner{\lambda}{c_j(t)}\,dt.
 \label{eq:affine-support-main}
\end{equation}
Indeed, the integrand ranges over the convex hull of finitely many $c_j(t)$, and a maximizing atom may be selected measurably with a fixed tie rule.  Fenchel duality for distance to a closed convex set then yields the exact formula
\begin{equation}
 \dist(y^\star,\mathcal Y_{\Rel})
 =\max_{\norm{\lambda}_*\le1}
 \left\{
 \inner{\lambda}{y^\star-a}
 -\int_0^T\max_j\inner{\lambda}{c_j(t)}\,dt
 \right\}.
 \label{eq:affine-exact-distance-main}
\end{equation}
No state-space HJB grid is required.

For the constant-drift-free, finite-witness case, put $T=1$, $A=0$, $z_0=0$, and $B=[b_1\ \cdots\ b_J]$.  In the $\ell^\infty$ output norm,
\begin{align}
 E_{\mathrm{aff}}
 &=\min_{\alpha\ge0,\,\one^\top\alpha=1}
   \norm{B\alpha-y^\star}_\infty,
 \label{eq:affine-primal-main}\\
 &=\max_{\norm{\lambda}_1\le1}
   \left\{\lambda^\top y^\star-
   \max_j\lambda^\top b_j\right\}.
 \label{eq:affine-dual-main}
\end{align}
When $B$ and $y^\star$ have exact rational representations, any rational feasible pair $(\widehat\alpha,\widehat\lambda)$ gives an exact bracket
\begin{equation}
 L:=\widehat\lambda^\top y^\star-
       \max_j\widehat\lambda^\top b_j
 \le E_{\mathrm{aff}}
 \le
 U:=\norm{B\widehat\alpha-y^\star}_\infty.
 \label{eq:affine-rational-bracket-main}
\end{equation}
After clearing denominators of the complete problem data and witnesses, every inequality is checked with integer arithmetic.  Floating-point optimization may propose the witness, but it does not certify it.

The accompanying matrix example lifts sixteen width-eight witness states under sixteen signed 4-bit $8\times8$ residual atoms, giving dimension $128$.  A rational primal point and rational dual witness certify a floor bracket of width approximately $1.1\times10^{-11}$ around $0.12575462209$.  Section~\ref{sec:certified-evidence-main} returns to its nonlinear full-domain counterpart.

\subsection{Nonlinear systems through modal occupation measures}
\label{sec:occupation-main}

The affine dual is exact but specialized.  For nonlinear finite-dimensional fields, introduce one nonnegative occupation measure $\mu_j$ per atom and a terminal probability measure $\mu_T$.  The modal measures record where and for how long each operation is used.  Consider
\begin{align}
 V_{\mathrm{LP}}:=\inf_{\mu_1,\ldots,\mu_J,\mu_T}
 &\quad\int_K\phi(z)\,d\mu_T(z)
 \label{eq:occupation-objective-main}\\
 \text{subject to }&\quad
 \mu_j\in\mathcal M_+([0,T]\times K),
 \quad\mu_T\in\mathcal P(K),\nonumber\\
 &\quad\sum_j\pi_{t\#}\mu_j=dt,
 \label{eq:occupation-time-marginal-main}\\
 &\quad
 \int_Kv(T,z)\,d\mu_T(z)-v(0,z_0)\nonumber\\
 &\qquad=
 \sum_j\int_{[0,T]\times K}
 \left(\partial_tv+\nabla v^\top G_j\right)d\mu_j
 \quad\forall v\in C^1.
 \label{eq:occupation-liouville-main}
\end{align}

\begin{theorem}[Exact modal occupation-measure representation]
\label[theorem]{thm:occupation-main}
Assume $K$ is nonempty and compact, $\phi\in C(K)$, the fields extend to bounded continuous functions on a neighborhood of $[0,T]\times K$, are uniformly Lipschitz in state, and every relaxed trajectory remains in $K$.  Then
\begin{equation}
 V_{\mathrm{LP}}=E_\infty^\phi,
 \label{eq:occupation-equality-main}
\end{equation}
and the measure program attains its optimum.  The time-marginal constraint is redundant once \eqref{eq:occupation-liouville-main} is imposed for all state-independent tests and $\mu_T$ has unit mass.  The HJB program in Theorem~\ref{thm:hjb-lower-main} is the conic dual after removing this redundancy; weak duality always holds.  Equality of the infinite-dimensional primal and dual is claimed only under the additional closed-image qualification detailed in Appendix~\ref{app:resource-theory-proofs}; that qualification is not automatic for every instance.
\end{theorem}

\begin{proof}[Proof of the primal identity]
Every relaxed trajectory $z_p$ gives feasible measures
\begin{equation}
 d\mu_j(t,z)=p_j(t)\,dt\,\delta_{z_p(t)}(dz),
 \qquad
 \mu_T=\delta_{z_p(T)},
 \label{eq:path-occupation-main}
\end{equation}
so $V_{\mathrm{LP}}\le E_\infty^\phi$.  Conversely, set $\mu=\sum_j\mu_j$ and disintegrate
$\mu(dt,dz)=dt\,\mu_t(dz)$.  The weak Liouville equation and bounded
velocity admit a narrowly continuous representative $t\mapsto\mu_t$ with
initial law $\delta_{z_0}$ and terminal law $\mu_T$.  Since $\mu_j\ll\mu$, the Radon--Nikodym derivatives $p_j=d\mu_j/d\mu$ are simplex valued almost everywhere.  The Liouville identity is the continuity equation for
\begin{equation}
 b(t,z)=\sum_jp_j(t,z)G_j(t,z).
 \label{eq:occupation-velocity-main}
\end{equation}
The superposition principle for continuity equations gives a probability measure on absolutely continuous integral curves of $b$, all starting at $z_0$, whose terminal law is $\mu_T$ \citep{AmbrosioGigliSavare2008}.  Along almost every curve, the composition $t\mapsto p(t,\gamma(t))$ is measurable and simplex valued.  It is therefore an admissible open-loop relaxed control for that individual curve, even though the disintegrated field $p(t,z)$ is state dependent before the curve is selected.  Therefore
\begin{equation}
 \int\phi\,d\mu_T
 =\int\phi(\gamma(T))\,d\Pi(\gamma)
 \ge E_\infty^\phi.
 \label{eq:occupation-superposition-main}
\end{equation}
Taking the state-independent test $v(t,z)=\psi(t)$ in \eqref{eq:occupation-liouville-main} shows that $\sum_j\pi_{t\#}\mu_j=dt$, so the time-marginal constraint is indeed redundant once $\mu_T$ has unit mass.  Compact support gives weak-* compactness and attainment.  The remaining conic-duality details are stated explicitly in Appendix~\ref{app:resource-theory-proofs} rather than being assumed automatically.
\end{proof}

The superposition step matters.  Without it, a feasible positive solution of the Liouville equation could be mistaken for one physical trajectory.  The theorem shows exactly what it represents: a probability mixture of terminal points of admissible relaxed trajectories, which is sufficient because an average terminal loss cannot beat the minimum loss among the mixture components.

\subsection{Moment--SOS lower hierarchies}
\label{sec:sos-main}

Suppose $\mathcal K=[0,T]\times K$ and a declared terminal set
$K_T\subseteq K$ are compact basic semialgebraic sets with Archimedean
quadratic modules, and the fields and terminal loss are polynomial.  Define the
terminal-constrained relaxed value
\begin{equation}
 E_{\infty,K_T}^\phi
 :=\inf\{\phi(z_p(T)):p(\cdot)\text{ relaxed-admissible},\ z_p(T)\in K_T\}.
 \label{eq:terminal-constrained-floor-main}
\end{equation}
Take $K_T=K$ when no additional terminal constraint is intended.  Introduce
one truncated moment sequence $y_j$ per modal measure and $y_T$ for the
terminal measure.  At order $r$, impose positive semidefinite moment and
localizing matrices on $\mathcal K$ and $K_T$, together with
\begin{equation}
 L_{y_T}(v(T,\cdot))-v(0,z_0)
 =\sum_jL_{y_j}(\partial_tv+\nabla v^\top G_j)
 \label{eq:moment-liouville-main}
\end{equation}
for every monomial whose displayed moments have degree at most $2r$.  If $L_r$
denotes the resulting minimum, then for all sufficiently large $r$,
\begin{equation}
 L_r\le L_{r+1}\le E_{\infty,K_T}^\phi,
 \qquad
 \lim_{r\to\infty}L_r=E_{\infty,K_T}^\phi.
 \label{eq:sos-convergence-main}
\end{equation}
This is the modal switched-system hierarchy of
\citet{ClaeysDaafouzHenrion2016}, augmented by a terminal measure and
objective.  The finite conic dual is a degree-restricted SOS HJB certificate.
Asymptotic hierarchy convergence is distinct from finite-order SDP
primal--dual equality, which requires its own no-gap qualification.

A degree-two example is exact.  For $\dot x\in\{-1,1\}$, $x(0)=0$, $t\in[0,1]$, $x\in[-1,1]$, and $\phi(x)=(x-2)^2$, the value is one and
\begin{equation}
 v(t,x)=(1+t-x)^2
 \label{eq:exact-sos-witness-main}
\end{equation}
satisfies $v(1,x)=\phi(x)$,
\begin{equation}
 \partial_tv+\partial_xv(+1)=0,
 \qquad
 \partial_tv+\partial_xv(-1)=4t+4(1-x),
 \label{eq:exact-sos-residual-main}
\end{equation}
which is an exact quadratic-module certificate on the box.  The all-$+1$ trajectory supplies the matching primal cost.  The bundle includes the symbolic identity and verifier.

\subsection{From bounds to a decision}
\label{sec:three-way-decision-main}

Let a feasible relaxed computation give $E_\infty\le U$, let a verified dual witness give $L\le E_\infty$, and let the finite implemented architecture satisfy
\begin{equation}
 \abs{E_D^{\mathrm{impl}}-E_\infty}\le R_D,
 \label{eq:implementation-radius-decision-main}
\end{equation}
where $R_D$ contains the depth, metadata, arithmetic, and---when relevant---hybrid route radius.

\begin{corollary}[Feasible, impossible, or unresolved]
\label[corollary]{cor:three-way-decision-main}
For a required tolerance $\varepsilon_H$,
\begin{align}
 U+R_D\le\varepsilon_H
 &\Longrightarrow\text{certified feasible at depth $D$},
 \label{eq:decision-feasible-main}\\
 L-R_D>\varepsilon_H
 &\Longrightarrow\text{certified impossible at depth $D$},
 \label{eq:decision-finite-impossible-main}\\
 L>\varepsilon_H
 &\Longrightarrow\text{certified asymptotically impossible}.
 \label{eq:decision-asymptotic-impossible-main}
\end{align}
Every other case is unresolved.
\end{corollary}

\begin{proof}
Use $E_D^{\mathrm{impl}}\le E_\infty+R_D\le U+R_D$ for feasibility and $E_D^{\mathrm{impl}}\ge E_\infty-R_D\ge L-R_D$ for finite-depth impossibility.  The asymptotic statement follows from $E_\infty\ge L$.
\end{proof}

Figure~\ref{fig:certificate-decisions-main} makes the decision geometry explicit: the method returns a proof-backed verdict only when the certified interval lies wholly on one side of the requested tolerance.

\begin{figure}[t]
\centering
\includegraphics[width=.88\linewidth]{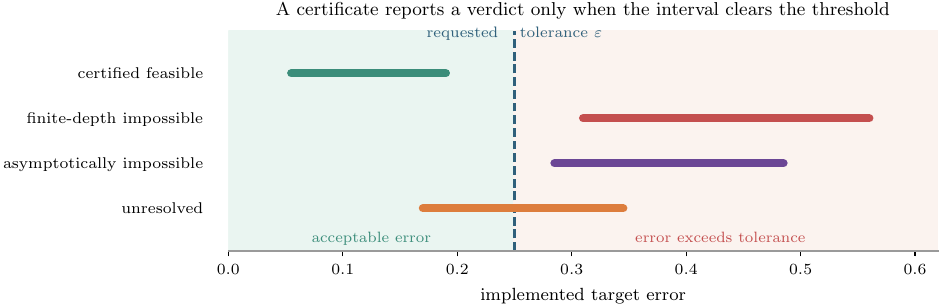}
\caption{The decision interface.  A primal upper bound and a verified dual lower bound define an interval for the implemented target error after adding the finite-resource radius.  The interval can certify feasibility, finite-depth impossibility, or asymptotic impossibility; overlap with the required tolerance is reported as unresolved rather than converted into a post-hoc verdict.}
\label{fig:certificate-decisions-main}
\end{figure}

\paragraph{Certification consequence.}
The resource law can be used before expensive training.  A sound implementation radius plus a primal--dual floor bracket yields a three-way decision.  The unresolved outcome is scientifically essential: it distinguishes lack of representability from lack of certificate strength.

\FloatBarrier

\section{Architecture Specializations}
\label{sec:architecture-specializations-main}
\label{sec:architectures}

The abstract theory is useful only if its tube, Lipschitz, dictionary, arithmetic, and routing constants can be instantiated without hiding architecture-specific behavior inside one symbol.  We develop two complementary specializations.  Contractive soft-threshold networks provide exact nonlinear floors and exposed-mode converses.  Pre-normalized Transformer and MoE blocks provide explicit normalization, attention, MLP, accumulator, and router error ledgers on bounded deployment tubes.

\subsection{Contractive soft-threshold and unfolded networks}
\label{sec:soft-threshold-main}

For $\lambda\ge0$, define coordinatewise soft thresholding
\begin{equation}
 \soft_\lambda(u)_i
 =\operatorname{sign}(u_i)(\abs{u_i}-\lambda)_+.
 \label{eq:soft-threshold-main}
\end{equation}
It is nonexpansive and obeys
\begin{equation}
 \norm{\soft_\lambda(u)-\soft_{\lambda'}(v)}_2
 \le\norm{u-v}_2+\sqrt d\,\abs{\lambda-\lambda'}.
 \label{eq:soft-lipschitz-main}
\end{equation}
On the input ball $\mathcal X=B_2^m(Y)$, take $\Z=C(\mathcal X;\mathbb R^d)$ and let a complete atom $a_j=(W_j,S_j,A_j,\lambda_j)$ induce
\begin{align}
 T_j(z)(y)&=\soft_{\lambda_j}(W_jy+S_jz(y))+A_jz(y),\\
 G_j(z)&=T_j(z)-z.
 \label{eq:soft-field-main}
\end{align}
Assume
\begin{equation}
 \norm{S_j}_2+\norm{A_j}_2\le\rho<1,
 \qquad
 \norm{W_j}_2\le W_\star.
 \label{eq:soft-contraction-main}
\end{equation}
Then all continuous reference and relaxed trajectories starting from zero, and all pure or mixed Euler trajectories with $0<h\le1$ whose atoms satisfy \eqref{eq:soft-contraction-main}, remain in the ball
\begin{equation}
 \norm z_\infty\le R:=\frac{W_\star Y}{1-\rho},
 \label{eq:soft-invariant-radius-main}
\end{equation}
and on that ball one may take
\begin{equation}
 L_z=1+\rho,
 \qquad B=2R.
 \label{eq:soft-field-constants-main}
\end{equation}
The invariant ball follows by the scalar inequality $\norm{T_j(z)}_\infty\le W_\star Y+\rho\norm z_\infty$; the same contraction gives well-posedness and input Lipschitzness.  Detailed proofs appear in Appendix~\ref{app:soft-threshold-proofs}.

\paragraph{Why parameter-space balancing is not enough.}
Balancing or averaging parameter tuples does not in general balance the
induced nonlinear residual fields.  A scalar threshold target and a low-bit
alphabet containing only neighboring thresholds already exhibit the issue:
the parameter discrepancy can remain uniformly small while the target map
retains a kink that is absent from every induced low-bit linear region.  The
relevant convexification is therefore the convex hull of \emph{induced vector
fields}, not the convex hull of parameter tuples.  Nonlinear parameter-to-map
transformations do not commute with convex averaging.

\paragraph{An exact same-dictionary structural split.}
Use the four atoms
\begin{equation}
 W_{\sigma,\ell}=
 \begin{bmatrix}\sigma&0\\0&0\end{bmatrix},
 \quad S_{\sigma,\ell}=0,
 \quad A_{\sigma,\ell}=sI_2,
 \quad \lambda_{\sigma,\ell}=\ell,
 \quad \sigma\in\{-1,+1\},\ \ell\in\{0,1\},
 \label{eq:structural-split-atoms-main}
\end{equation}
with $0\le s<1$.  Put
\begin{equation}
 \mu_s=1-s,
 \qquad
 \alpha_s=\frac{1-e^{-\mu_s}}{\mu_s}.
 \label{eq:alpha-s-main}
\end{equation}
The target using the atom threshold $1$ has zero floor.  The target with the same $W,S,A$ but threshold $1/2$ has the exact floor
\begin{equation}
 E_{\infty,b}=\frac{\alpha_s}{6}.
 \label{eq:exact-positive-floor-main}
\end{equation}
To see why, restrict to inputs $y=(t,0)$.  Every relaxed first-coordinate endpoint is exactly
\begin{equation}
 at+b\soft_1(t),
 \qquad \abs a+\abs b\le\alpha_s.
 \label{eq:split-relaxed-class-main}
\end{equation}
The target is $\alpha_s\soft_{1/2}(t)$.  Evaluating at $t=1/2$ and $t=1$ gives a lower bound $\alpha_s/6$, and the feasible coefficients $a=\alpha_s/3$, $b=2\alpha_s/3$ attain it uniformly.  Thus two targets with identical dimensions, norm bounds, masks, recurrence, and dictionary can lie on opposite sides of the feasibility boundary.  Figure~\ref{fig:structural-split-evidence} visualizes this distinction.

\paragraph{Signed-ray LISTA families.}
Let
\begin{equation}
 W_c=cB_W,
 \quad S_c=cB_S,
 \quad A_c=A_0,
 \quad \lambda_c=c\overline\lambda,
 \label{eq:signed-ray-parameters-main}
\end{equation}
for gains $c$ in a finite nonnegative alphabet $\mathcal C_b$.  Positive homogeneity gives
\begin{equation}
 G_c(z)=cT(z)+A_0z-z,
 \qquad
 T(z)(y)=\soft_{\overline\lambda}(B_Wy+B_Sz(y)).
 \label{eq:signed-ray-field-main}
\end{equation}
Hence relaxed gain paths in $\operatorname{conv}\mathcal C_b$ are compatible by construction.  If $\mathfrak d_b$ denotes the optimal scalar prefix discrepancy of the gain alphabet, a Lipschitz gain path admits a pure schedule with
\begin{equation}
 E_b(D;F_\infty^c)
 \le\frac{C_{\mathrm{disc}}+C_{\mathrm{quant}}\mathfrak d_b}{D}.
 \label{eq:signed-ray-upper-main}
\end{equation}
When $\mathcal C_b$ contains $0$ and its smallest positive edge is $\Delta_0$, an exposed channel gives a matching class-level lower law proportional to $\Delta_0/D$.  The fixed binary-gain teacher in Section~\ref{sec:fixed-teacher-main} strengthens this to one target independent of $D$.  The gain edge and discrepancy radius, rather than nominal bit count alone, determine the finite-depth constant.

\subsection{Pre-normalized Transformer blocks}
\label{sec:transformer-specialization-main}

\paragraph{Exact attention witness versus complete-block error ledger.}
Theorem~\ref{thm:attention-fixed-target-main} proves an exact global converse
for a deliberately minimal attention system: two tokens, one head, a
nonuniform softmax, and an invariant two-mode residual plane.  The results
below answer the complementary implementation question for a general
pre-normalized block: how errors in normalization, scores, values,
projections, MLPs, accumulators, and routing contribute to one verified field
radius.  The exact witness establishes necessity in a recognizable neural
operator; the ledger establishes transfer under architecture-specific bounds.

Consider a parallel pre-normalized residual field
\begin{equation}
 G(t,X)=\alpha(t)A(N_1(X))+\beta(t)M(N_2(X)),
 \label{eq:transformer-field-main}
\end{equation}
and its executed counterpart
\begin{equation}
 \widehat G_k(X)=
 \widehat\alpha_k\widehat A_k(\widehat N_{1,k}(X))
 +\widehat\beta_k\widehat M_k(\widehat N_{2,k}(X)).
 \label{eq:transformer-executed-field-main}
\end{equation}
For a token matrix $X\in\mathbb R^{n\times d}$, write
\begin{equation}
 \norm X_{2,\infty}:=\max_i\norm{X_{i:}}_2.
 \label{eq:token-norm-main}
\end{equation}
The following component bounds all use this norm and therefore compose without changing geometry midway through the argument.

\subsubsection{LayerNorm}

Let $P_d=I-d^{-1}\one\one^\top$ and
\begin{equation}
 N_{\gamma,\beta}(x)
 =\gamma\odot\frac{P_dx}
 {\sqrt{d^{-1}\norm{P_dx}_2^2+\epsilon}}+\beta,
 \qquad\epsilon>0.
 \label{eq:layernorm-map-main}
\end{equation}
Differentiating the centered normalization map shows
\begin{equation}
 \operatorname{Lip}_2(N_{\gamma,\beta})
 \le\frac{\norm\gamma_\infty}{\sqrt\epsilon},
 \qquad
 \norm{N_{\gamma,\beta}(x)-\beta}_2
 \le\sqrt d\,\norm\gamma_\infty.
 \label{eq:layernorm-bounds-main}
\end{equation}
If the executed map uses the same $\epsilon$ but perturbed input and affine parameters, then
\begin{align}
 \norm{N_{\widehat\gamma,\widehat\beta}(\widehat x)
       -N_{\gamma,\beta}(x)}_2
 \le{}&\frac{\norm{\widehat\gamma}_\infty}{\sqrt\epsilon}
        \norm{\widehat x-x}_2
 +\sqrt d\,\norm{\widehat\gamma-\gamma}_\infty
 +\norm{\widehat\beta-\beta}_2.
 \label{eq:layernorm-error-main}
\end{align}
A finite-arithmetic LayerNorm remainder is added once.  A coordinatewise activation quantizer of radius $\rho_x$ contributes at most $\sqrt d\rho_x$ to the input term.

\subsubsection{Softmax attention}

For one head,
\begin{equation}
 S=QK^\top/\sqrt{d_k},
 \qquad
 P=\operatorname{softmax}_{\mathrm{row}}(S),
 \qquad
 Z=PV.
 \label{eq:attention-definition-main}
\end{equation}
The softmax Jacobian maps $\ell^\infty$ to $\ell^1$ with norm at most one:
\begin{equation}
 \norm{\operatorname{softmax}(\widehat s)
       -\operatorname{softmax}(s)}_1
 \le\norm{\widehat s-s}_\infty.
 \label{eq:softmax-lipschitz-main}
\end{equation}
Indeed, along a line segment its derivative is $p_i(\delta_i-\mathbb E_p\delta)$; the $\ell^1$ norm is a mean absolute deviation bounded by the radius of the value interval.

Assume
\begin{align*}
 \norm Q_{2,\infty}&\le B_Q,&
 \norm K_{2,\infty}&\le B_K,&
 \norm V_{2,\infty}&\le B_V,\\
 \norm{\widehat Q-Q}_{2,\infty}&\le\eta_Q,&
 \norm{\widehat K-K}_{2,\infty}&\le\eta_K,&
 \norm{\widehat V-V}_{2,\infty}&\le\eta_V.
\end{align*}
Then
\begin{align}
 \norm{\widehat S-S}_{\max}
 &\le\eta_S
 :=\frac{\eta_Q(B_K+\eta_K)+B_Q\eta_K}{\sqrt{d_k}},
 \label{eq:attention-score-error-main}\\
 \norm{\widehat Z-Z}_{2,\infty}
 &\le\eta_V+B_V\eta_S.
 \label{eq:attention-value-error-main}
\end{align}
Because each row of $P$ is a probability vector, the ideal attended value satisfies $\norm Z_{2,\infty}\le B_V$.  For $A=ZW_O+b_O$, decompose $\widehat Z\widehat W_O-ZW_O=(\widehat Z-Z)\widehat W_O+Z(\widehat W_O-W_O)$ to obtain
\begin{align}
 \norm{\widehat A-A}_{2,\infty}
 \le{}&\norm{\widehat W_O}_{\mathrm{op}}
        (\eta_V+B_V\eta_S)
 +B_V\norm{\widehat W_O-W_O}_{\mathrm{op}}
 +\norm{\widehat b_O-b_O}_2.
 \label{eq:attention-output-error-main}
\end{align}
Per-head errors aggregate in Euclidean norm under concatenation.  These inequalities expose the distinct effects of $Q/K$ score distortion, value distortion, output projection, and arithmetic remainder.

\subsubsection{MLP and complete field error}

On $\norm z\le R$, an affine map obeys
\begin{equation}
 \sup_{\norm z\le R}\norm{\widehat Wz+\widehat b-(Wz+b)}
 \le\norm{\widehat W-W}_{\mathrm{op}}R+\norm{\widehat b-b}.
 \label{eq:affine-error-main}
\end{equation}
For $M(z)=W_2\varphi(W_1z+b_1)+b_2$, with $\varphi$ $L_\varphi$-Lipschitz and bounded by $B_\varphi$ on the tube,
\begin{align}
 \sup_{\norm z\le R}\norm{\widehat M(z)-M(z)}
 \le{}&\norm{\widehat W_2-W_2}_{\mathrm{op}}B_\varphi
 +\norm{\widehat W_2}_{\mathrm{op}}L_\varphi
  \left(\norm{\widehat W_1-W_1}_{\mathrm{op}}R
       +\norm{\widehat b_1-b_1}\right)\\
 &+\norm{\widehat b_2-b_2}.
 \label{eq:mlp-error-main}
\end{align}
Suppose $\eta_{A,k},\eta_{M,k}$ are uniform branch-output errors on the declared tube, $L_A,L_M$ are Lipschitz constants of the ideal branch maps on that tube, $\eta_{N_1,k},\eta_{N_2,k}$ are normalization errors, and $B_A,B_M$ are tube bounds.  Then
\begin{align}
 \eta_{G,k}:=
 \sup_X\norm{\widehat G_k(X)-G(t_k,X)}
 \le{}&\abs{\widehat\alpha_k-\alpha(t_k)}B_A
 +\abs{\widehat\alpha_k}(\eta_{A,k}+L_A\eta_{N_1,k})\\
 &+\abs{\widehat\beta_k-\beta(t_k)}B_M
 +\abs{\widehat\beta_k}(\eta_{M,k}+L_M\eta_{N_2,k}).
 \label{eq:transformer-field-error-main}
\end{align}
This is the $\eta_G$ term in the error-feedback theorem.  Scale errors, normalization errors, operator quantization, and finite MAC remainders remain visible rather than being combined into one undocumented norm.

\subsection{The executed Transformer resource law}
\label{sec:executed-transformer-law-main}

Let a nonempty compact family of Transformer dictionaries be indexed by $\omega\in\Omega$ and satisfy the uniform tube and learned-codebook assumptions.  Suppose every implemented microstep uses increment error feedback with residual radius $\rho_D$, the field error satisfies $\sup_k\eta_{G,k}\le\eta_{G,D}$, carry and MAC accumulators are exact on the declared range or contribute a separately bounded term, and saturation is inactive by a verified tube margin.  Then the endpoint and path reachable sets obey
\begin{align}
 \dH(\widehat\RR_{D,s}^{\mathrm{Tr,EF}},
       \RR_{\Omega,\Rel}^{\mathrm{Tr}})
 &\le r_D^{\mathrm{end}},
 \label{eq:executed-transformer-end-main}\\
 \dH^{\mathrm{path}}(\widehat{\mathscr P}_{D,s}^{\mathrm{Tr,EF}},
       \mathscr P_{\Omega,\Rel}^{\mathrm{Tr}})
 &\le r_D^{\mathrm{path}},
 \label{eq:executed-transformer-path-main}
\end{align}
where
\begin{align}
 r_D^{\mathrm{end}}
 &=\frac{C_{\mathrm{syn}}^{\mathrm{Tr}}}{D}
 +L_\Omega\Phi_{L_z}(T)\delta_s
 +\rho_De^{L_zT}
 +\eta_{G,D}\Phi_{L_z}(T),
 \label{eq:transformer-end-radius-main}\\
 r_D^{\mathrm{path}}
 &=\frac{C_{\mathrm{syn,path}}^{\mathrm{Tr}}}{D}
 +L_\Omega\Phi_{L_z}(T)\delta_s
 +\rho_De^{L_zT}
 +\eta_{G,D}\Phi_{L_z}(T).
 \label{eq:transformer-path-radius-main}
\end{align}
The proof pairs equal schedules through the error-feedback theorem, adds ideal pure-to-relaxed synthesis, and then pairs equal relaxed controls across nearby metadata values using the endpoint-and-path continuity statement in Lemma~\ref{lem:relaxed-codebook-lipschitz-main}.  The path statement is essential for routing; an endpoint bound cannot be substituted for the oracle-path premise of Theorem~\ref{thm:route-window-main}.

A first-order executed law follows when
\begin{equation}
 \delta_s=O(D^{-1}),
 \qquad
 \rho_D=O(D^{-1}),
 \qquad
 \eta_{G,D}=O(D^{-1}),
 \label{eq:transformer-resource-scaling-main}
\end{equation}
and all other arithmetic terms satisfy $\mathcal A_D^{\mathrm{other}}=O(D^{-1})$ or are identically zero.  For route-changing MoE Transformers, use $r_D^{\mathrm{path}}$ as the oracle error $\delta_D$ and insert the expert, gate, score, and event constants of Section~\ref{sec:moe-constants-main} into the hybrid theorem.

\subsection{The integrated resource theorem}
\label{sec:integrated-resource-main}

The preceding results combine without double counting.

\begin{theorem}[Floor--depth--metadata--arithmetic resource law]
\label[theorem]{thm:integrated-resource-main}
Let a compact learned dictionary family satisfy the uniform smooth or bounded-variation assumptions.  Let an $s$-bit metadata code have covering radius $\delta_s$, and suppose the implemented and ideal depth-$D$ systems have a schedulewise arithmetic radius $\mathcal A_D$ uniform over every encoded metadata value.  Then
\begin{equation}
 \dH(\RR_{D,s}^{\mathrm{impl}},\RR_{\Omega,\Rel})
 \le
 \frac{C_{\mathrm{syn}}^{\mathrm{unif}}}{D}
 +L_\Omega\Phi_{L_z}(T)\delta_s
 +\mathcal A_D.
 \label{eq:integrated-resource-main}
\end{equation}
If $\mathcal N(\Omega,\delta)\le(C_\Omega/\delta)^m$, execution uses scaled increment feedback with normalized radius $\overline\rho_b$, field error $\eta_{G,D}$, and remaining implementation radius $\mathcal A_D^{\mathrm{other}}$, then
\begin{equation}
 \dH(\RR_{D,s}^{\mathrm{impl}},\RR_{\Omega,\Rel})
 \le
 \frac{C_{\mathrm{syn}}^{\mathrm{unif}}
       +T\overline\rho_be^{L_zT}}{D}
 +L_\Omega\Phi_{L_z}(T)C_\Omega2^{-s/m}
 +\eta_{G,D}\Phi_{L_z}(T)
 +\mathcal A_D^{\mathrm{other}}.
 \label{eq:integrated-scaled-resource-main}
\end{equation}
The same radii bound the absolute difference between the optimal implemented target error and the learned relaxed structural floor.  Prefix-uniform radii give the corresponding path statement.
\end{theorem}

\begin{proof}
Insert the ideal pure set with encoded metadata and the relaxed set with the same metadata between the implemented pure set and $\RR_{\Omega,\Rel}$.  Apply, in order, the schedulewise hardware-transfer theorem, pure-to-relaxed synthesis, and Lemma~\ref{lem:relaxed-codebook-lipschitz-main}; then use the Hausdorff triangle inequality.  The scaled form substitutes $\rho_D=T\overline\rho_b/D$ and the metric-entropy cover.  Distance to a target is one-Lipschitz under Hausdorff perturbation.
\end{proof}

The integrated theorem is the formal version of Equation~\eqref{eq:integrated-law-intro}.  Under H\"older-time fields, $C_{\mathrm{syn}}/D$ is replaced by the two-term radius in Theorem~\ref{thm:holder-time-main}.  Its upper law is matched by three independent necessity mechanisms: the fixed-teacher $D^{-1}$ converse, the depth--dictionary--metadata packing law, and positive dual floor certificates.  Actual route-changing networks additionally require the hybrid small-gain conditions; their state and event-window radius replaces the symmetric route-free term rather than being silently added as an independent Hausdorff radius.

\paragraph{Architecture consequence.}
A quantized Transformer block is certifiable only when its low-bit dictionary geometry and its executed numerical semantics are stated in the same norm on the same tube.  Weight error alone is not a block-level resource certificate, and an endpoint-only theorem is not enough for routed execution.

\section{Certified Evidence: What Is Proved, What Is Diagnosed}
\label{sec:certified-evidence-main}
\label{sec:evidence}

The numerical program is organized by the logical strength of the evidence, not by model size.  The hierarchy in Table~\ref{tab:evidence-hierarchy} governs every claim in this section.  Analytic formulas and verified rational witnesses support theorem-level conclusions.  Exhaustive finite checks establish exactness only for the enumerated instance.  Controlled simulations isolate one predicted mechanism.  Pretrained-model studies diagnose whether that mechanism is present in a learned checkpoint.  Training curves and local searches remain empirical upper bounds, never proofs of representability or impossibility.  Prospective experiments are stated as evidence gates rather than reported as completed results.

Appendix~\ref{app:experiments} supplies architecture-level parameters and search details; Appendix~\ref{app:transformer-experiments} collects the complete experiment-setting tables; and Appendix~\ref{app:reproducibility} maps every numerical claim to its data, script, and verification status.  This separation is intentional: the paper's resource theory is mathematical, while the numerical studies show which parts of that theory are exact, mechanistically visible, or presently only diagnostic.

\begin{table}[t]
\centering
\caption{Evidence hierarchy and permitted inference.  A lower row may motivate or diagnose a theorem mechanism, but it does not inherit the stronger claims permitted to a row above it.}
\label{tab:evidence-hierarchy}
\small
\begin{tabularx}{\linewidth}{Y{.20\linewidth}Y{.29\linewidth}X}
\toprule
Evidence class & What is verified & Permitted conclusion \\
\midrule
Analytic or rationally certified & Closed-form identity, exact symbolic inequality, rational LP/MILP/PWA witness, or globally verified HJB/SOS certificate & The stated formula or bound holds on the declared domain and under the theorem assumptions \\
Exhaustive finite verification & Every configuration in a finite registered instance & Exact optimum or invariant for that finite instance; no automatic extrapolation beyond it \\
Controlled mechanistic experiment & A simulation designed to vary one resource while holding the others fixed & The predicted mechanism is present in the controlled system; not a general representability claim \\
Pretrained-model diagnostic & A fixed learned checkpoint, declared data, seeds, and target construction & The mechanism is observable in that checkpoint; not integer-only execution, hardware efficiency, or architecture-wide certification \\
Uncertified empirical result & Training curve, local search, held-out slope, or finite witness estimate without a global certificate & An observed upper bound or diagnostic trend; never proof of a positive floor or global optimum \\
Prospective evidence gate & A preregistered or fully specified experiment not yet run & A falsifiable next test, not a completed contribution \\
\bottomrule
\end{tabularx}
\end{table}

\subsection{Exact executable checks for the new converse and arithmetic laws}
\label{sec:exact-verification-main}

Two load-bearing results have independent executable verifiers.  For
Theorem~\ref{thm:attention-fixed-target-main}, the released program enumerates
every binary schedule through depth $12$, verifies that the all-$+a$ schedule
is globally optimal, checks the invariant-plane reduction, and matches
\eqref{eq:attention-exact-main} to machine precision.  For
Proposition~\ref{prop:bit-exact-error-feedback-main}, the registered verifier
checks residual bounds, every declared register width, conservation, and exact
terminal recovery over $2{,}991{,}904$ exhaustive scalar sequences and $500$
randomized multidimensional trials.  These checks do not replace the proofs;
they guard the algebra, boundary cases, and executable artifact against
independent implementation mistakes.

\begin{figure}[t]
\centering
\includegraphics[width=.96\linewidth]{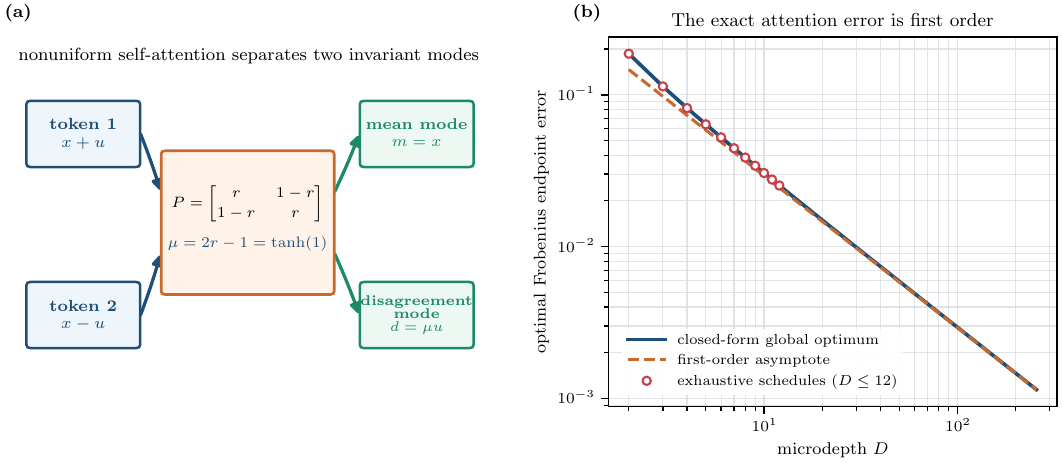}
\caption{\textnormal{An exact fixed-target attention converse.}
\textnormal{(a)} The two-token state decomposes into a mean mode, which carries
the binary scheduling obstruction, and a disagreement mode, which experiences
a nonuniform attention contraction with eigenvalue $\mu=\tanh(1)$.
\textnormal{(b)} The closed-form global optimum in
\eqref{eq:attention-exact-main} agrees with exhaustive schedule enumeration
through depth $12$ and approaches its predicted first-order asymptote.}
\label{fig:attention-fixed-target-main}
\end{figure}

\subsection{The Structural Floor Is the Decisive Invariant Across Architectures}

We first isolate the central structural dichotomy.  Figure~\ref{fig:soft-structural} uses one soft-threshold residual dictionary and two nonlinear targets.  One target lies in the relaxed endpoint set and its globally optimized finite-depth error approaches zero.  The other is separated from that endpoint set and converges to the exact floor $\alpha_s/6$.  Every marker is an exact LP/MILP value over all piecewise-affine breakpoints.

\begin{figure}[t]
\centering
\includegraphics[width=.96\linewidth]{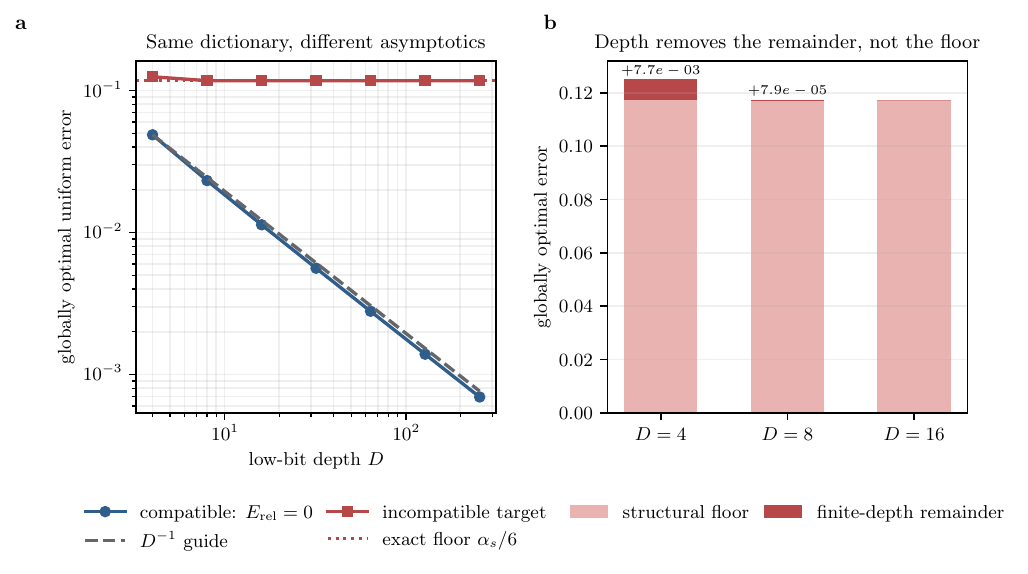}
\caption{\textnormal{Exact structural split in a soft-threshold residual network.}  Under one finite dictionary, a compatible nonlinear target converges toward zero, while an incompatible target converges to the exact infinite-depth floor $\alpha_s/6$.  The dashed floor is analytic and the finite-depth points are globally solved.}
\label{fig:soft-structural}
\label{fig:structural-split-evidence}
\end{figure}

The same phenomenon survives state-dependent token interaction.  These attention studies use ideal real-valued execution; the arithmetic terms of Section~\ref{sec:finite-arithmetic-main} are intentionally held at zero so that operation-library geometry is isolated.  In Figure~\ref{fig:attention-structural}, the compatible attention target decreases from $1.75\times10^{-2}$ at $D=4$ to $5.93\times10^{-4}$ at $D=128$.  The incompatible target has an output component orthogonal to every dictionary atom and therefore retains the exact missing-subspace floor $1.077349$.  The representative attention matrix confirms that the experiment is not a disguised tokenwise linear map.

\begin{figure}[t]
\centering
\includegraphics[width=.98\linewidth]{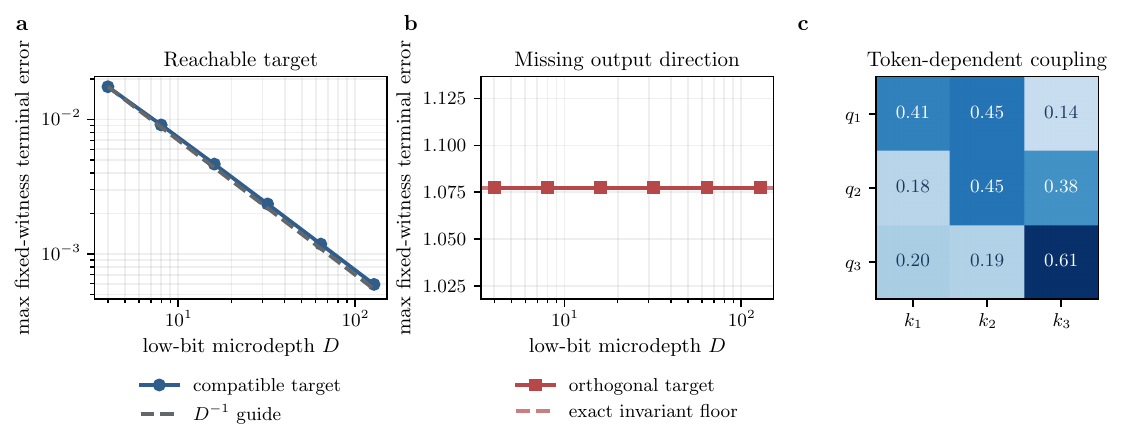}
\caption{\textnormal{A state-dependent attention structural split.}  The compatible target follows a first-order decay, whereas a target with a component in a missing output subspace remains at an exact invariant floor.  The right panel displays a nonuniform, input-dependent attention matrix from the experiment.}
\label{fig:attention-structural}
\end{figure}

Together, Figures~\ref{fig:soft-structural} and~\ref{fig:attention-structural} establish the first message of the paper: neither nominal bit width nor added depth determines feasibility by itself.  The decisive object is the target's position relative to the infinite-depth low-bit reachable class.

\FloatBarrier
\subsection{The First-Order Exponent Is Sharp; the Constant Is Codebook Dependent}

Figure~\ref{fig:soft-rate} verifies the minimax $D^{-1}$ law in the signed-ray soft-threshold family.  Globally solved finite programs coincide with the analytic edge-gap formula at every tested depth.  At $D=128$, the exact values of $DE_D$ are $0.2371$, $0.1185$, $0.0790$, and $0.0339$ for binary, ternary, uniform 2-bit, and uniform 3-bit gain alphabets.

\begin{figure}[H]
\centering
\includegraphics[width=.96\linewidth]{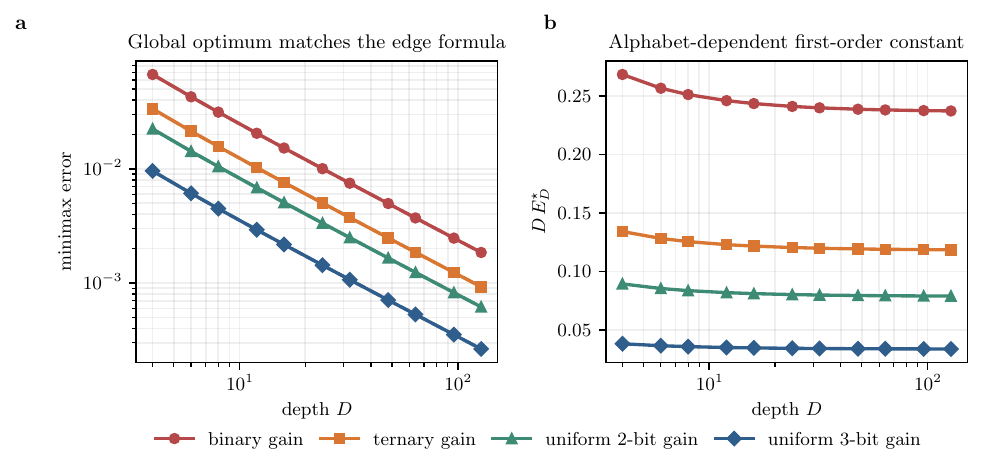}
\caption{\textnormal{Sharp minimax $D^{-1}$ law in a nonlinear signed-ray family.}  Globally solved finite-program values coincide with the exact edge-gap formula.  Multiplication by depth exposes distinct alphabet-dependent constants.}
\label{fig:soft-rate}
\end{figure}

Figure~\ref{fig:attention-rate} repeats the mechanism around a fixed nonlinear attention field.  Binary, ternary, uniform 2-bit, and uniform 3-bit schedules share the exponent but not the constant.  In this instance the ternary dictionary slightly outperforms the chosen uniform 2-bit dictionary, emphasizing that codebook geometry is more informative than nominal bits alone.

\begin{figure}[t]
\centering
\includegraphics[width=.96\linewidth]{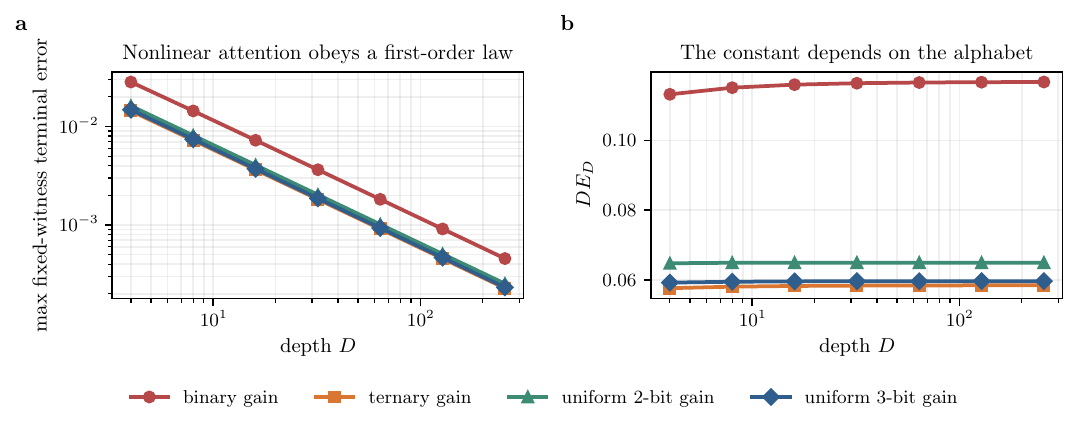}
\caption{\textnormal{First-order synthesis around a nonlinear attention field.}  Compatible gain alphabets all produce $D^{-1}$ terminal error, while $DE_D$ approaches a dictionary-dependent constant.  The experiment separates the universal exponent from the nonuniversal depth price.}
\label{fig:attention-rate}
\end{figure}

\subsection{Reachability Floors Create Phase Boundaries and Precision Sweet Spots}

The theory yields more than a convergence rate: a positive dual floor can identify precisions that can never meet a target tolerance, while a zero or small floor leaves a finite-depth synthesis problem.  Figure~\ref{fig:soft-phase} reports globally solved matching depths for a fixed nonlinear soft-threshold target.  Binary and uniform 2-bit codebooks have floors above every tested high-resolution tolerance.  Ternary, 3-bit, and 4-bit codebooks are feasible in different accuracy ranges, and the matched-storage optimum shifts from ternary at $L=8$, to 3-bit at $L=16$, to 4-bit at $L=32$ and $64$.

\begin{figure}[t]
\centering
\includegraphics[width=.96\linewidth]{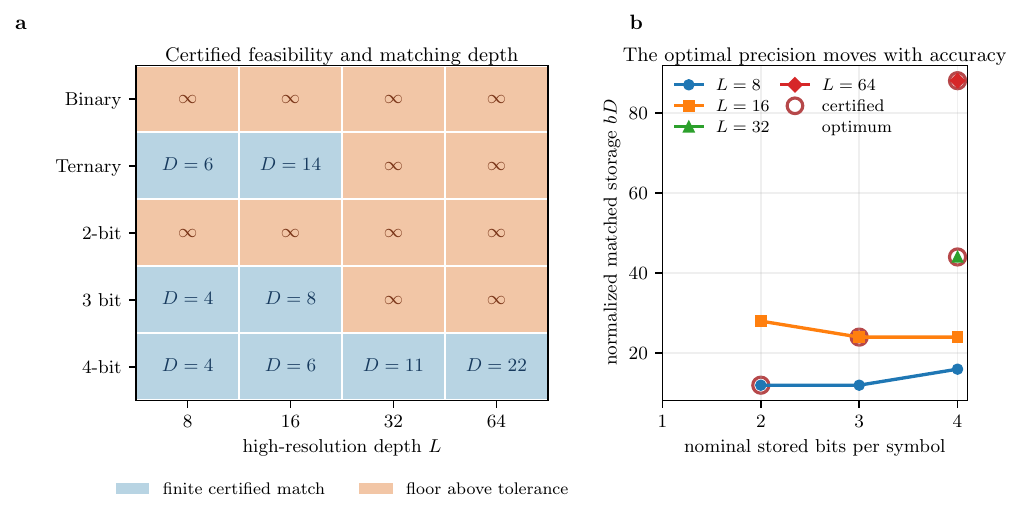}
\caption{\textnormal{Certified soft-threshold feasibility boundary and precision sweet spot.}  Orange cells have certified infinite-depth floors above the requested tolerance and are asymptotically infeasible.  The plot does not infer finite-depth impossibility from the floor alone; isolated shallow exceptions are checked by the accompanying exact search.  Blue cells report globally solved matching depths.  After codebook and schedule costs are counted, the optimum need not occur at the lowest precision and moves in this experiment as accuracy tightens.}
\label{fig:soft-phase}
\end{figure}

The attention-output codebook experiment computes infinite-depth floors exactly by linear programming and finite-depth synthesis by exact count search.  Figure~\ref{fig:attention-codebook-geometry} displays the operation geometry directly: the floors decrease from $0.7654$ for binary to $0.3867$ for ternary and $0.1507$ for uniform 2-bit, and vanish for the tested 3- and 4-bit codebooks.

\begin{figure}[t]
\centering
\includegraphics[width=.98\linewidth]{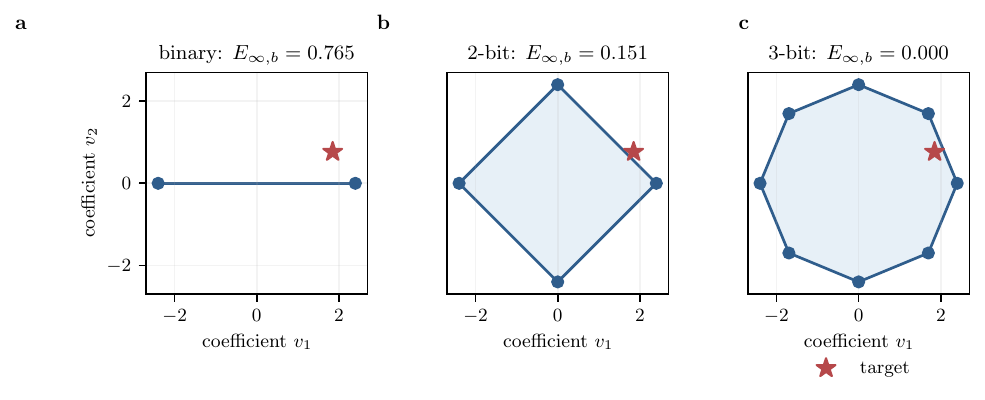}
\caption{\textnormal{Attention-output codebook geometry determines the infinite-depth floor.}  The target is compared with the convex hull generated by each finite codebook.  Binary and uniform 2-bit codebooks leave a visible geometric gap; the tested 3-bit polygon contains the target.  This feasibility calculation is completed before any finite-depth search.}
\label{fig:attention-codebook-geometry}
\end{figure}

Inside the feasible region, exact finite-depth count synthesis produces the matching depths in Figure~\ref{fig:attention-phase}.  Once codebook and schedule costs are included, the matched-storage optimum moves from ternary to 2-bit to 3-bit as the requested tolerance tightens.

\begin{figure}[t]
\centering
\includegraphics[width=.96\linewidth]{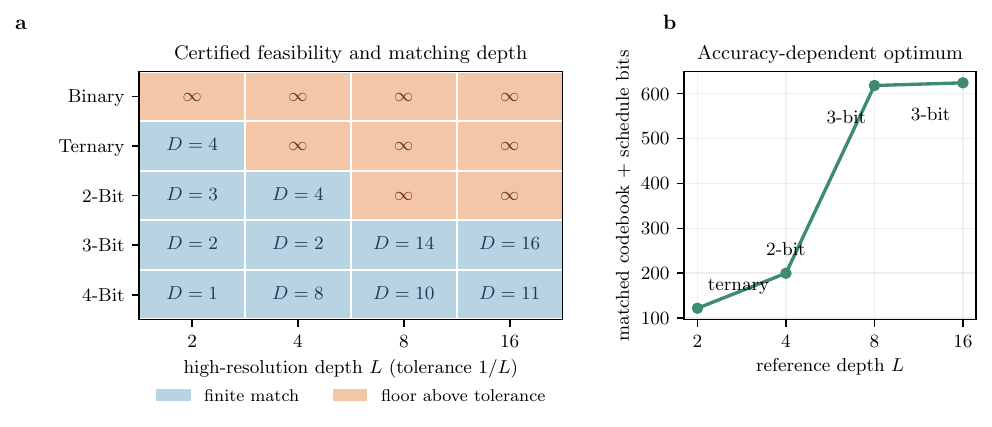}
\caption{\textnormal{Attention-codebook phase boundary and matched-resource optimum.}  The left panel separates positive-floor asymptotic infeasibility from finite matching depths.  The right panel shows that the resource-optimal precision can be interior and can change with the target tolerance.}
\label{fig:attention-phase}
\end{figure}

For the finite soft-threshold cases, a floor-plus-$C/D$ model is calibrated before the held-out global search and predicts matching depth with median multiplicative factor $1.32$.  The exact matrix case below supplies a second prospective check: the calibration model predicts $D=9$ and the full certificate establishes $D_{\min}=8$.

\subsection{Architecture Stability Changes the Depth Price}

The stability budget in the master theorem is not merely a worst-case proof artifact.  Figure~\ref{fig:attention-stability} varies RMSNorm regularization and residual scale over a $3\times3$ grid.  The first-order exponent persists, while the large-depth constant changes by roughly a factor of twenty and the first tested matching depth moves from $32$ to $128$.  Stronger amplification therefore raises the required depth even when the relaxed floor is unchanged.

\begin{figure}[t]
\centering
\includegraphics[width=.96\linewidth]{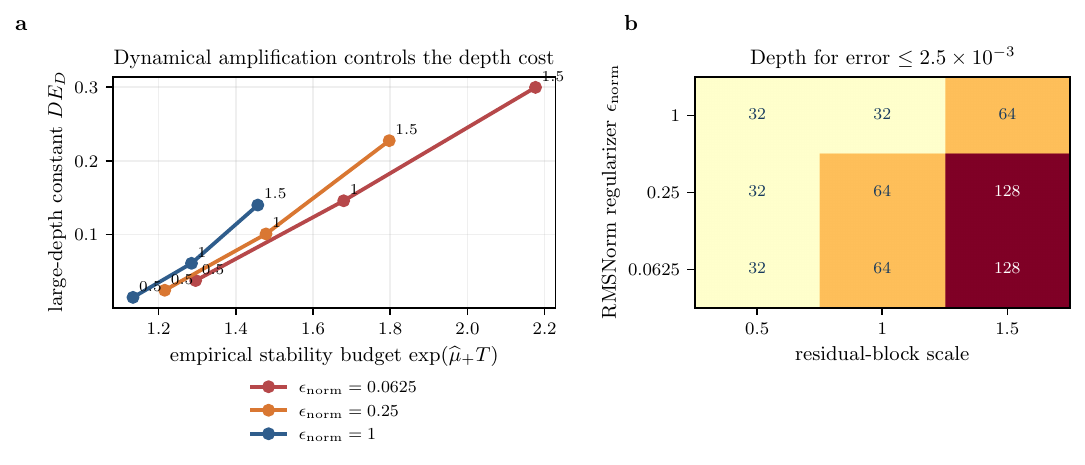}
\caption{\textnormal{Finite-horizon amplification controls the depth multiplier.}  Left: the observed large-depth constant $DE_D$ grows with the empirical stability budget and with more aggressive normalization/residual scaling.  Right: the first tested depth attaining error $2.5\times10^{-3}$.  Stability changes the constant, not the first-order mechanism.}
\label{fig:attention-stability}
\end{figure}

\subsection{Hard Routing Obeys a Margin Law}

Figure~\ref{fig:routing-margin} isolates the hybrid part of the theory.  Below the score-margin threshold, every active quantized router atom misroutes at least one token and a nonzero routing floor remains.  Once the margin dominates router and state error, route mismatch disappears and the smooth first-order depth law returns.  This is the frozen-route margin regime that complements Theorem~\ref{thm:route-window-main}.  The experiment varies a positive static margin and therefore does not stand in for an actual transversal route-change test.

\begin{figure}[t]
\centering
\includegraphics[width=.96\linewidth]{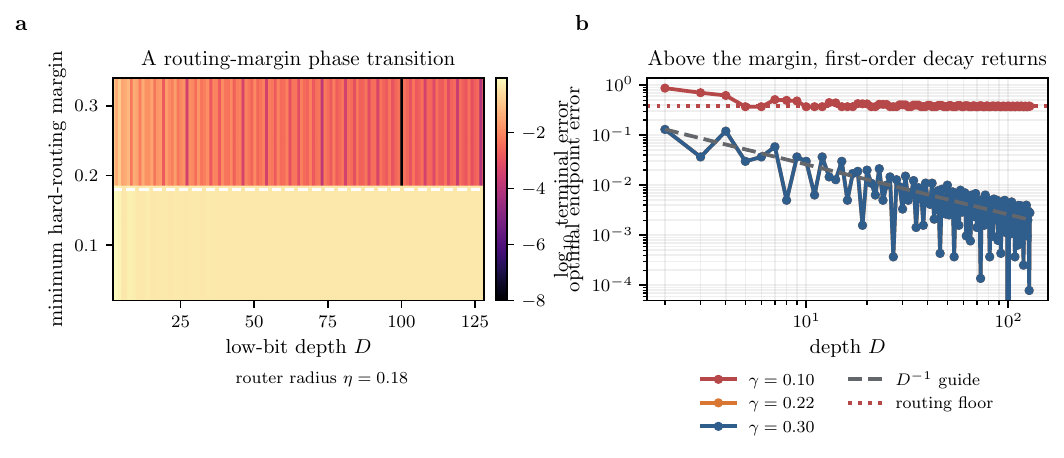}
\caption{\textnormal{Hard routing has a margin-controlled phase transition.}  The dashed boundary is the quantized-router perturbation radius.  Below it, route changes create a persistent error; above it, routes agree and the $D^{-1}$ synthesis regime reappears.}
\label{fig:routing-margin}
\end{figure}

\subsection{Exact Nonlinear Matrix-Valued Accuracy Matching}
\label{sec:matrix-cert}

We now give the strongest finite-depth accuracy-matching certificate.  The desired map is the depth-32 reference; the depth-4 high-precision network and the low-bit students are compared by their uniform errors to this common target.  Consider a recurrent $2\times2$ soft-threshold path with modes
\begin{align}
W^{(0)}&=\begin{bmatrix}0.65&0\\0&0.65\end{bmatrix},
&S^{(0)}&=\begin{bmatrix}0.10&0\\0&-0.10\end{bmatrix},
&\lambda^{(0)}&=0.35,\\
W^{(1)}&=\begin{bmatrix}0&-0.65\\0.65&0\end{bmatrix},
&S^{(1)}&=\begin{bmatrix}0&0.10\\-0.10&0\end{bmatrix},
&\lambda^{(1)}&=0.
\end{align}
The path uses modes $0,1,0$ on $[0,.30)$, $[.30,.65)$, and $[.65,1]$.  Let $F_{32}^{\mathrm H}$ denote the depth-$32$ common reference, $F_4^{\mathrm H}$ the high-precision depth-$L=4$ comparator, and $F_{D,\sigma}^{\mathrm L}$ the depth-$D$ low-bit student under schedule $\sigma$.  The low-bit residual-field dictionary consists of the nearest uniform 4-bit quantizations of the two $(W,S,\lambda)$ modes.  Every depth uses the shared microstep $1/D$, determined exactly by the unit horizon and the integer depth.

All $2^D$ schedules are enumerated for $D=1,\ldots,12$ on $1281$ witness points fixed before schedule evaluation.  The best depth-$8$ schedule is $(0,0,0,0,0,1,0,1)$.  A rational piecewise-affine propagation over inscribed and circumscribed 16-gons proves
\begin{align}
  \inf_{\sigma\in\{0,1\}^7}\sup_{y\in B_2^2(1)}
  \norm{F_{7,\sigma}^{\mathrm L}(y)-F_{32}^{\mathrm H}(y)}_2
  &\ge 0.02553769,\\
  \sup_{y\in B_2^2(1)}\norm{F_4^{\mathrm H}(y)-F_{32}^{\mathrm H}(y)}_2
  &\le 0.02534776,\label{eq:D7-fail}\\[.2em]
  \sup_{y\in B_2^2(1)}\norm{F_{8,\sigma^\star}^{\mathrm L}(y)-F_{32}^{\mathrm H}(y)}_2
  &\le 0.02416130,\\
  \sup_{y\in B_2^2(1)}\norm{F_4^{\mathrm H}(y)-F_{32}^{\mathrm H}(y)}_2
  &\ge 0.02443160.\label{eq:D8-win}
\end{align}
Every smaller-depth witness optimum is larger than the depth-$7$ bound.  Table~\ref{tab:matrix-cert} collects the decisive inequalities, the prospective depth prediction, and the stored-bit count under the fixed-metadata convention.  Hence
\begin{equation}
  \boxed{D_{\mathrm{match}}=8=2L}
  \label{eq:exact-2L}
\end{equation}
as the minimum accuracy-matching depth for this dictionary and common-reference tolerance.

\begin{table}[t]
\centering
\caption{Certified matrix-valued accuracy-matching margins, prospective prediction, and stored-bit accounting under the fixed-metadata convention.}
\label{tab:matrix-cert}
\small
\begin{tabularx}{\linewidth}{@{}Y{0.29\linewidth}Y{0.27\linewidth}>{\raggedright\arraybackslash}X@{}}
\toprule
Quantity & Certified value & Consequence\\
\midrule
Best depth-$7$ witness error & $0.02553769$ & Exceeds the full-disk upper bound of the high-precision comparator.\\
FP depth-$4$ error interval & $[0.02443160,0.02534776]$ & Defines the comparator's accuracy relative to the common reference.\\
Depth-$8$ low-bit interval & $[0.02344455,0.02416130]$ & Lies below the comparator's full-disk lower bound.\\
Calibration prediction & $\widehat D=9$ & Multiplicative factor $9/8=1.125$.\\
Exact accuracy-matching depth & $D_{\mathrm{match}}=8=2L$ & Establishes a nontrivial sequential depth exchange.\\
Variable low-bit storage & $80$ bits & 72-bit codebook plus 8-bit schedule.\\
Untied FP16 storage & $576$ bits & $7.20\times$ compression under the stated accounting.\\
Tied-codebook FP16 storage & $292$ bits & $3.65\times$ compression under the stated accounting.\\
\bottomrule
\end{tabularx}
\end{table}

Figure~\ref{fig:matrix-cert} displays the certified separation on the full input disk; the image is explanatory, while the proof uses exact rational cell bounds.

\begin{figure}[t]
\centering
\includegraphics[width=\linewidth]{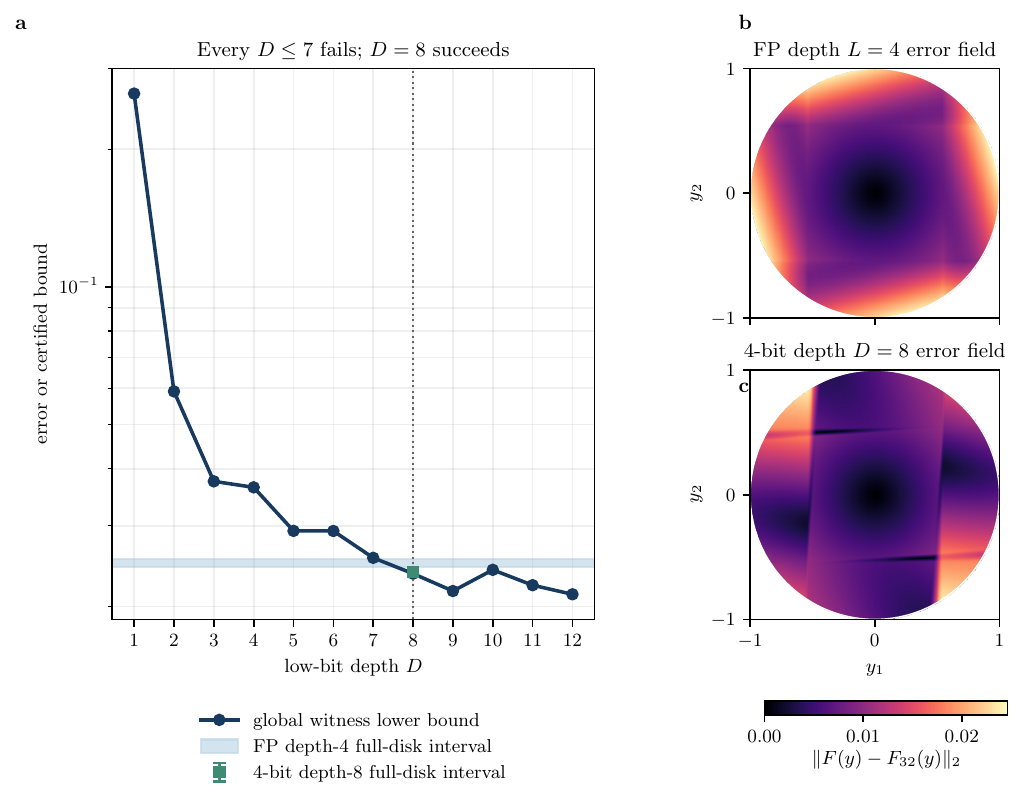}
\caption{\textnormal{Certified minimum accuracy-matching depth in a nonlinear matrix system.}  The large left panel excludes every $D\le7$ and certifies that the selected depth-8 student is more accurate relative to the common depth-32 reference than the depth-4 high-precision comparator.  The two right panels show the comparator and low-bit error fields relative to the depth-$32$ reference on a common color scale.  The proof uses exact rational activation-cell bounds rather than rasterized fields.}
\label{fig:matrix-cert}
\end{figure}

The 80-bit count includes the complete two-atom quantized coefficient codebook and the 8-bit schedule.  The quantizer levels, unit horizon, rule $h_D=1/D$, and decoder convention are globally fixed architectural metadata; any learned scales, zero-points, masks, or separately encoded residual scale would have to be added under a different implementation.  The result is an accuracy-matching and stored-bit certificate under this explicit convention, not a direct-emulation, bit-exact, or hardware-latency claim.  The second predeclared matrix target exhibited witness-level feasibility but its full-disk intervals overlapped at smaller depths; we therefore make no exact minimum-depth claim for that case and report it transparently in Appendix~\ref{app:multiteacher}.

\subsection{A Pretrained-Model Causal Test of Depth Coherence}
\label{sec:pretrained-causal}

The exact examples isolate the theory, but they do not show whether the coherence distinction matters inside a trained model.  We therefore study the feed-forward residual branches of layers 1 and 3 of the public SST-2-finetuned DistilBERT checkpoint \citep{SanhEtAl2019DistilBERT,WangEtAl2019GLUE,SocherEtAl2013SST,HuggingFaceDistilBERTSST2}, evaluated with the Transformers software stack \citep{WolfEtAl2020Transformers}.  For each branch, the \emph{direct} target is its original one-step residual map.  The \emph{coherent} target is a depth-32 high-precision Euler refinement of the same pretrained residual field over horizon $T=1$.  This creates a causal control: if added depth fails even without quantization on the direct target, quantized optimization cannot be the sole explanation.

Figure~\ref{fig:pretrained-coherence} gives the decisive comparison.  On the original one-step target, unquantized subdivision moves sharply away from the target: at layer 1 the relative hidden-map RMSE rises from $2.65\times10^{-4}$ at $D=1$ to $1.383$ at $D=16$, and at layer 3 from $5.34\times10^{-4}$ to $1.248$.  For the coherent targets, the direction reverses: error falls by $96.0\%$ at layer 1 and $91.0\%$ at layer 3.  The fitted log--log slopes are approximately $-1.16$ and $-0.86$, respectively.  Thus the failure of direct subdivision is already present in high precision, whereas the coherent construction recovers the predicted first-order refinement mechanism.

\begin{figure}[t]
\centering
\includegraphics[width=.99\linewidth]{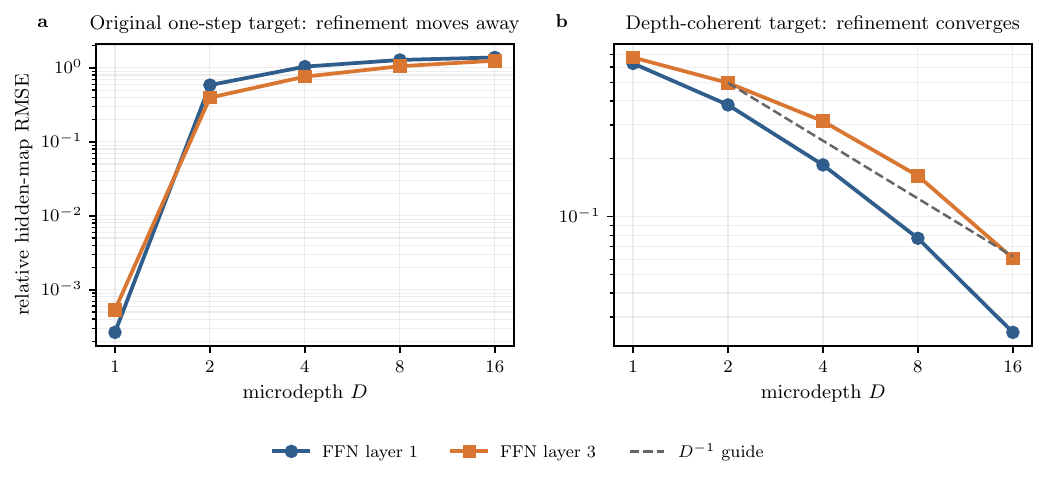}
\caption{\textnormal{Pretrained-model causal control for depth coherence.}  The same DistilBERT feed-forward residual fields are evaluated under two targets.  \textnormal{(a)} Subdividing the original one-step block moves away from its map even without quantization.  \textnormal{(b)} Refining a common depth-32 residual-flow target reduces error rapidly at both tested layers.  The contrast isolates target coherence before any low-bit optimization is introduced.}
\label{fig:pretrained-coherence}
\end{figure}

We next learn a two-atom 4-bit dictionary and time-dependent relaxed controls for each layer and depth $D\in\{4,8,16\}$, then harden the controls to a pure schedule and locally refine that schedule.  Training is repeated with three independent QAT seeds; evaluation uses all 872 SST-2 validation examples.  Figure~\ref{fig:pretrained-qat} and Table~\ref{tab:pretrained-qat} report the locally hardened pure schedules, which are the primary deployable students in this study.

At layer 3, increasing depth from 4 to 16 reduces hidden-map error by $10.4\%$ and logit RMSE by $29.5\%$.  The paired seed contrasts are
\[
  \Delta E=-0.01825,\qquad 95\%\ \mathrm{CI}=[-0.03419,-0.00230],
\]
for hidden-map error and
\[
  \Delta_{\mathrm{logit}}=-0.04993,\qquad 95\%\ \mathrm{CI}=[-0.06319,-0.03668]
\]
for logit RMSE.  Layer 1 is a near-saturation case: all three seeds improve in hidden-map error, but the conservative three-seed $t$ interval crosses zero, while logit error decreases by $13.8\%$.  Prediction agreement with the coherent target remains above $99\%$ throughout and reaches $99.66\%$ at layer 1 and $99.24\%$ at layer 3 for $D=16$.  Downstream SST-2 accuracy is not monotone in map fidelity: at layer 3 it peaks at $90.48\%$ for $D=8$ while endpoint and logit errors continue to improve through $D=16$.  This separation is expected: common-target map matching and task accuracy are related but distinct objectives.

\begin{figure}[t]
\centering
\includegraphics[width=.99\linewidth]{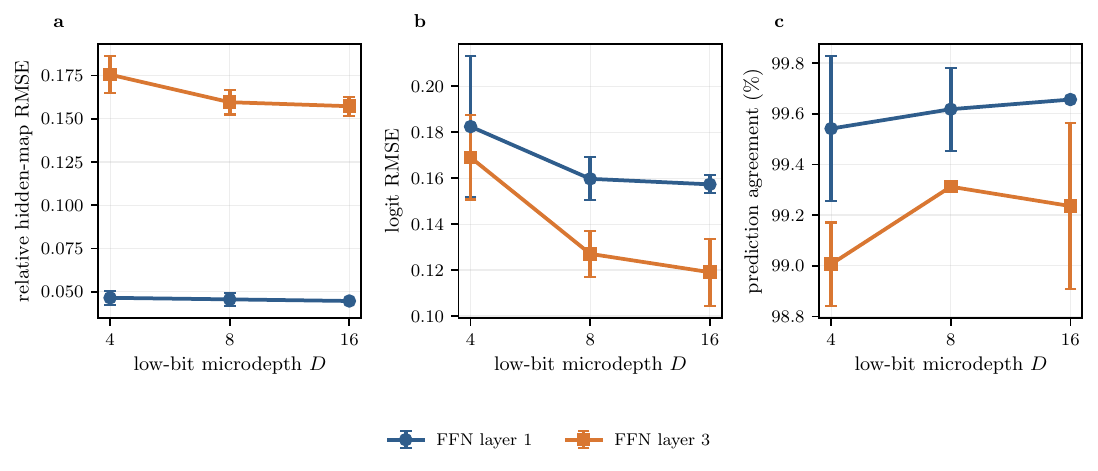}
\caption{\textnormal{Four-bit learned dictionaries preserve the coherent depth benefit.}  Means and Student-$t$ 95\% intervals are computed across three independent QAT seeds for locally hardened pure schedules.  \textnormal{(a)} Hidden-map error decreases modestly at the near-saturated first layer and more clearly at layer 3.  \textnormal{(b)} Logit RMSE improves at both layers, with a seed-robust depth effect at layer 3.  \textnormal{(c)} Prediction agreement with the coherent target model remains above $99\%$.}
\label{fig:pretrained-qat}
\end{figure}

\begin{table}[t]
\centering
\caption{Pretrained DistilBERT robustness study. Values are means across three QAT seeds; intervals are 95\% Student-$t$ intervals across seeds. The student is a locally hardened pure schedule from a learned two-atom 4-bit dictionary.}
\label{tab:pretrained-qat}
\small
\begin{tabular}{ccccc}
\toprule
Layer & $D$ & hidden-map RMSE & logit RMSE & target agreement \\
\midrule
1 & 4  & $0.04648\ [0.04235,0.05060]$ & $0.18243\ [0.15190,0.21297]$ & $99.54\%$\\
1 & 8  & $0.04558\ [0.04185,0.04931]$ & $0.15970\ [0.15035,0.16905]$ & $99.62\%$\\
1 & 16 & $0.04464\ [0.04397,0.04531]$ & $0.15730\ [0.15334,0.16125]$ & $99.66\%$\\
3 & 4  & $0.17548\ [0.16494,0.18602]$ & $0.16900\ [0.15070,0.18729]$ & $99.01\%$\\
3 & 8  & $0.15955\ [0.15245,0.16666]$ & $0.12705\ [0.11700,0.13711]$ & $99.31\%$\\
3 & 16 & $0.15723\ [0.15180,0.16267]$ & $0.11906\ [0.10444,0.13369]$ & $99.24\%$\\
\bottomrule
\end{tabular}
\end{table}

The relaxed and hardened programs are nearly indistinguishable at layer 3 and at the deeper layer-1 settings; the released seedwise compilation-gap table reports differences on the order of a percent or less.  Under the experiment's parameter-storage accounting, the two-atom 4-bit dictionary plus schedule uses approximately $3.97\times$ fewer stored parameter bits than the corresponding high-precision component.  This is a pretrained-model diagnostic, not a certificate of activation memory, integer-only execution, latency, energy, full-model quantization, or asymptotic reachability.  Within that scope, it provides a causal validation of the theory's central distinction: depth can refine a coherent computation, while no optimizer can repair a microstep construction that already moves away from the target in high precision.  Full settings are summarized in Table~\ref{tab:pretrained-settings}; the relaxed-to-pure compilation gaps appear in Figure~\ref{fig:pretrained-compilation}; seedwise trajectories, bootstrap intervals, and artifact paths are given in Appendix~\ref{app:pretrained-experiment}.

\FloatBarrier

\subsection{Medium-Scale Attention Retains the Mechanism}

Figure~\ref{fig:medium-attention} evaluates a state-dependent system with $n=8$ tokens, width $16$, and two heads over 96 held-out random sequences.  Appendix~\ref{app:transformer-experiments} records the fixed dictionary and held-out protocol.  Compatible targets retain first-order decay, while a component in a missing output direction remains at a positive floor.

\begin{figure}[t]
\centering
\includegraphics[width=.98\linewidth]{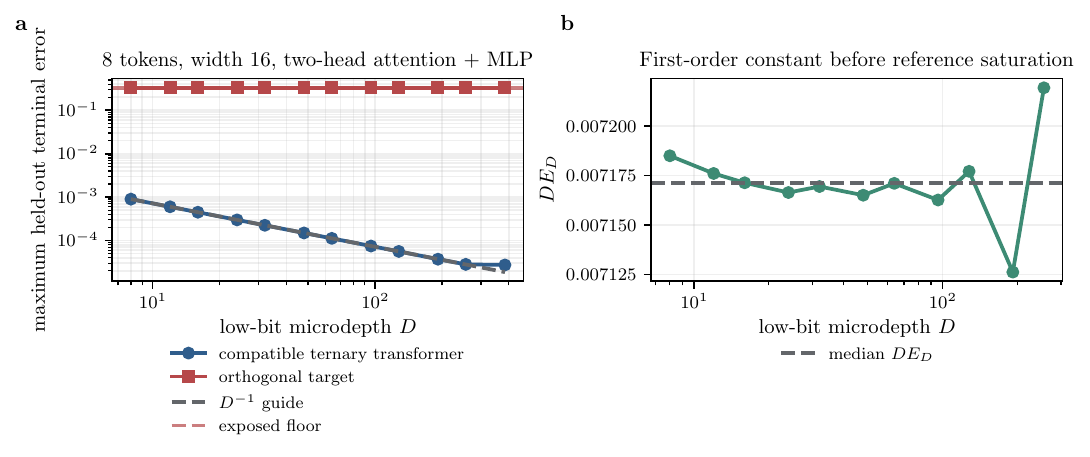}
\caption{\textnormal{Medium-scale attention validation.}  In an eight-token, width-16, two-head attention--MLP system, a compatible target preserves the $D^{-1}$ mechanism, while a component in a missing output direction remains at its exposed floor.  Curves show held-out errors over an input ensemble fixed before the depth sweep.  The right panel omits the final reference-integration-limited point when estimating the first-order plateau.}
\label{fig:medium-attention}
\end{figure}

The separate scale study in Figure~\ref{fig:attention-scale} varies $(n,d,H)$ over $(4,8,1)$, $(8,16,2)$, and $(16,32,4)$.  It shows that the first-order mechanism persists, while both the observed synthesis constant and the missing-direction floor change with architecture size.  Tables~\ref{tab:transformer-settings} and~\ref{tab:scale-sweep-summary} give the exact settings and numerical values.

\begin{figure}[t]
\centering
\includegraphics[width=.96\linewidth]{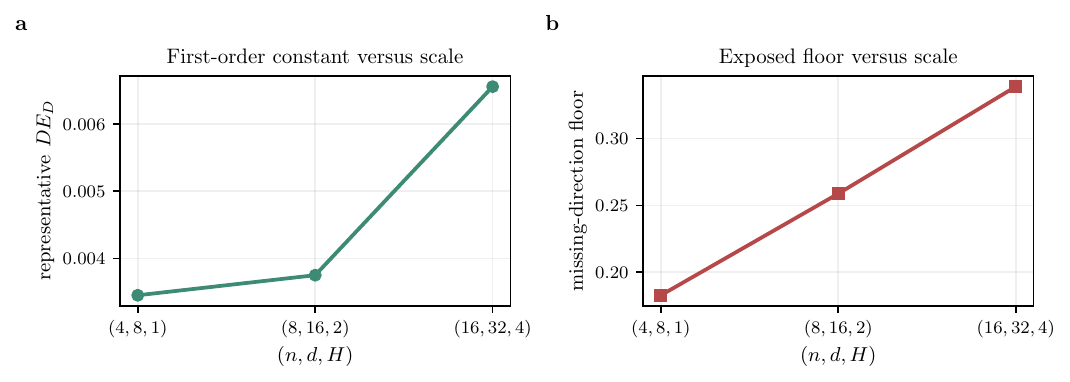}
\caption{\textnormal{Attention scale sweep.}  Left: the representative large-depth constant $D E_D$ for compatible targets.  Right: the exposed missing-direction floor.  The horizontal labels report $(n,d,H)$: token count, model width, and number of heads.}
\label{fig:attention-scale}
\end{figure}

\subsection{Representability and Optimization Are Different Frontiers}

For the recurrent soft-threshold dictionary, exhaustive search reveals a severe local-search gap: at $D=20$, the best of 64 local restarts is about $120\times$ above the global optimum and the median about $419\times$ above it.  Figure~\ref{fig:soft-opt-gap} shows this separation directly.

\begin{figure}[t]
\centering
\includegraphics[width=.98\linewidth]{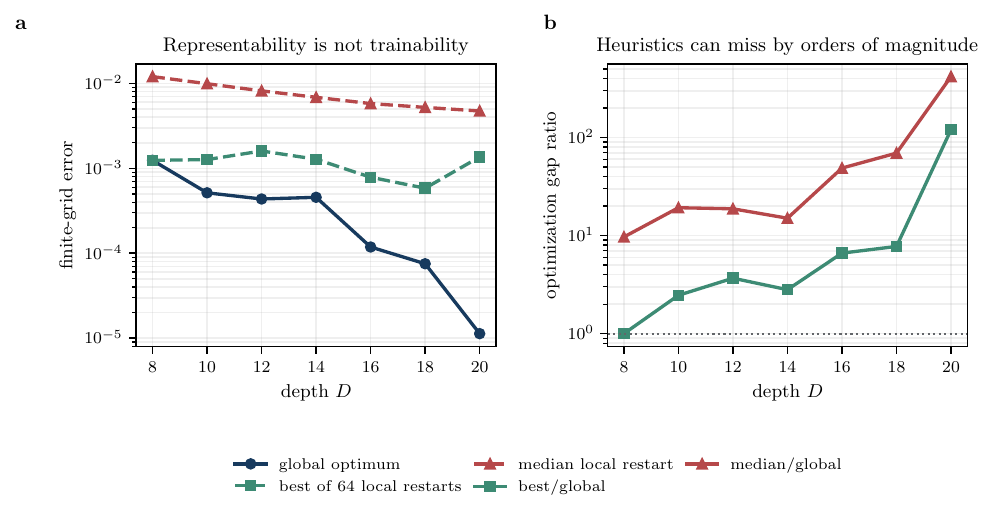}
\caption{\textnormal{Optimization gap for recurrent soft-threshold atoms.}  Globally exhaustive schedule search reveals representations that local coordinate descent misses by orders of magnitude.  This experiment demonstrates that failed local schedule search is not a certificate of positive-floor asymptotic infeasibility; more generally, no failed optimizer establishes nonrepresentability by itself.}
\label{fig:soft-opt-gap}
\end{figure}

Noncommuting attention--MLP atoms exhibit the same distinction at a smaller scale.  Figure~\ref{fig:attention-opt-gap} compares every schedule through $D=16$ with balanced switching and 16-restart coordinate descent.  Balanced schedules stay near the global frontier, whereas the local-search upper tail reaches roughly $1.95\times$ the optimum.

\begin{figure}[t]
\centering
\includegraphics[width=.98\linewidth]{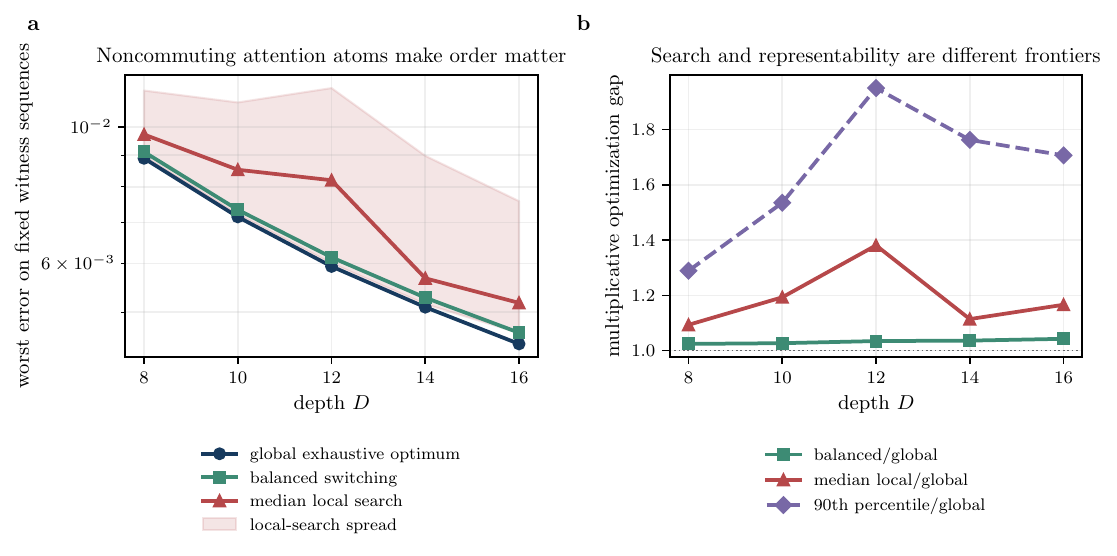}
\caption{\textnormal{Order and optimization matter for noncommuting attention atoms.}  Exhaustive global optima, balanced schedules, and local-search distributions are reported separately.  The theorem supplies a constructive representability law; it does not guarantee that an arbitrary training heuristic reaches the frontier.}
\label{fig:attention-opt-gap}
\end{figure}

A broader eight-teacher matrix stress test and a globally exhaustive four-atom dictionary expansion are reported in Appendix~\ref{app:multiteacher} and Figure~\ref{fig:broad-matrix}.  The claim-to-artifact paths for every numerical result are listed in Table~\ref{tab:repro} of Appendix~\ref{app:reproducibility}.  These studies provide breadth and reproducibility evidence without replacing the exact full-disk certificate above.

\section{From Theorem to a Quantized-Architecture Decision}
\label{sec:design-law-main}
\label{sec:design}

The theory changes the order in which a low-bit architecture should be designed.  It is usually inefficient to choose a bit width, train a large grid of depths, and only then ask why some configurations failed.  The resource view begins with the target and separates five questions.

\begin{enumerate}[leftmargin=1.8em,itemsep=.38em]
\item \textbf{Geometry:} does the learned low-bit operation family have a sufficiently small structural floor for the target?
\item \textbf{Synthesis:} if it does, how much pure depth is needed to implement a compatible relaxed computation?
\item \textbf{Execution:} do activation, scale, accumulator, carry, saturation, and field-computation errors shrink quickly enough for those microsteps to remain visible?
\item \textbf{Routing:} for discontinuous paths, do score errors fit inside isolated transversal event windows and a stable route--state feedback loop?
\item \textbf{Search:} can the optimizer or compiler discover a schedule and metadata value near the representational frontier?
\end{enumerate}

These questions should be answered in this order because later resources cannot repair an earlier structural failure.  A more powerful optimizer does not remove a positive floor.  More depth does not repair an activation grid that freezes the residual increments.  A small endpoint error does not certify a routed path.  Conversely, a failed local search is not evidence that the target is unreachable.

\subsection{A replacement workflow}
\label{sec:replacement-workflow-main}

For a declared target $F^\star$, tolerance $\varepsilon_H$, dictionary family $\Omega$, and execution specification, the following workflow implements the paper's logic.

\begin{enumerate}[leftmargin=1.8em,itemsep=.62em]
\item \textbf{Declare the replacement problem and every charged resource.}  Specify direct emulation, common-reference accuracy matching, or behavioral replacement; then record the target, input geometry, norm, low-bit atom family, shared schedule class, residual horizon, metadata budget, activation and increment grids, scale format, accumulators, overflow rule, carry state, and router.  Verify a common tube for both ideal and implemented prefixes, including state, normalization, field, Lipschitz, and representable-range bounds.

\item \textbf{Bracket the structural floor before allocating depth.}  Construct a feasible relaxed control for an upper bound $U$ and a verified support, HJB, affine, or SOS witness for a lower bound $L$.  If $L>\varepsilon_H$, reject the declared precision family: a more powerful optimizer and additional depth cannot repair that incompatibility.  If the bracket remains wide, report a certificate gap rather than a negative result.

\item \textbf{Allocate pure depth and learned metadata.}  Use
\begin{equation}
 R_{D,s}^{\mathrm{ideal}}
 =\frac{C_{\mathrm{syn}}}{D}+C_{\mathrm{meta}}2^{-s/m}
 \label{eq:ideal-resource-radius-design}
\end{equation}
to identify a candidate region, not a universal depth multiplier.  Count the dictionary, scales, zero-points, masks, schedules, and any other serialized global information.  Compare the achieved schedule with the certified frontier so that representation and optimization errors remain separate.

\item \textbf{Audit execution arithmetic and routes pathwise.}  Under full-state write-back, check whether $D\rho_z$ creates a finite optimal depth or a zero-update regime.  Under increment feedback, verify the carry range, its precision, and the $O(D^{-1})$ physical residual scale; prove no overflow for wrapping arithmetic.  For routed computation, certify event isolation, transversal slopes, score errors, field jumps, path error, and the small-gain condition.  Simultaneous, grazing, or chattering events fall outside the simple-event guarantee and must be reported rather than silently discarded.

\item \textbf{Decide first, then optimize and charge the platform.}  Combine $[L,U]$ with the finite-resource radius and return feasible, impossible, or unresolved before expensive training.  Train or compile only surviving configurations.  Finally count sequential latency, energy, activation and carry traffic, memory, serialized bytes, and kernel utilization.  Parameter bits alone are not a deployment objective, and a certified mathematical architecture is not automatically an efficient kernel.
\end{enumerate}

Table~\ref{tab:claim-interface-main} is a compact guardrail against overinterpreting any one certificate or experiment.

\begin{table}[t]
\centering
\caption{What each mathematical object can and cannot establish.  The separation prevents a local diagnostic from being promoted to a stronger global claim.}
\label{tab:claim-interface-main}
\small
\begin{tabularx}{\linewidth}{Y{.22\linewidth}Y{.33\linewidth}X}
\toprule
Object & What it certifies & What it does not certify \\
\midrule
Positive dual floor $L$ & Asymptotic incompatibility when $L$ exceeds the tolerance & Failure of one optimizer or one finite schedule \\
Zero/small primal upper bound $U$ & Existence of a compatible relaxed computation & Efficient pure compilation or successful training \\
$C_{\mathrm{syn}}/D$ & Pure finite-horizon implementation price under the declared tube & A universal depth multiplier across unrelated target families \\
Arithmetic radius $\mathcal A_D$ & Distance between ideal and declared executed transitions & Behavior of an undocumented hardware kernel \\
Route-window theorem & Path and support control near isolated transversal events & Arbitrary simultaneous events, chattering, or uniform continuity across route cells \\
Finite-witness rational certificate & Exact statement on the lifted witness object or a valid support functional & Uniform coverage of the input domain without a separate remainder proof \\
Task-margin argument & Behavioral agreement on the certified domain & Uniform map emulation \\
\bottomrule
\end{tabularx}
\end{table}

\subsection{Matching depth is target specific}
\label{sec:matching-depth-main}

Let $\varepsilon_H(L)=\norm{F_L^{\mathrm H}-F^\star}$ be a high-precision comparator's error relative to a common target.  Ignoring route events for the moment, a sufficient implemented matching condition is
\begin{equation}
 E_{\Omega,\infty}(F^\star)
 +\frac{C_{\mathrm{syn}}}{D}
 +C_{\mathrm{meta}}2^{-s/m}
 +\mathcal A_D
 \le\varepsilon_H(L).
 \label{eq:matching-condition-main}
\end{equation}
When the nonfloor terms are first order and the floor is strictly below the tolerance, this gives the predictive approximation
\begin{equation}
 D_{\mathrm{match}}
 \approx
 \frac{C_{\mathrm{eff}}}
 {\varepsilon_H(L)-E_{\Omega,\infty}(F^\star)-R_{\mathrm{nondepth}}},
 \label{eq:predictive-matching-depth-main}
\end{equation}
where every term in $C_{\mathrm{eff}}$ and $R_{\mathrm{nondepth}}$ is tied to the declared implementation.  The denominator explains the phase boundary: as the requested tolerance approaches the structural floor, the required depth diverges.

If a coherent high-precision family has $\varepsilon_H(L)\asymp L^{-1}$, a zero floor and an upper first-order low-bit law give $D=O(L)$.  A fixed-target or class-uniform lower law is required for $D=\Theta(L)$.  Theorem~\ref{thm:fixed-teacher-main} supplies such a fixed-target lower mechanism.  By contrast, if an unrestricted depth-$L$ teacher family retains $\Omega(L)$ independent functional choices, the packing law can force $D/L$ to grow logarithmically with $L$.  There is therefore no architecture-independent constant exchange rate.

\subsection{Resource-optimal precision need not be the lowest precision}
\label{sec:resource-optimal-main}

Let $R_{\mathrm{code}}(b,s)$ denote stored dictionary and metadata cost, $R_{\mathrm{step}}(b)$ the charged cost per executed depth, and $R_{\mathrm{state}}(b,D)$ the activation, carry, and accumulator cost.  A deployment-facing design is
\begin{equation}
 \min_{b,s,D}
 \ R_{\mathrm{code}}(b,s)
 +D R_{\mathrm{step}}(b)
 +R_{\mathrm{state}}(b,D)
 \quad\text{subject to a certified error tolerance.}
 \label{eq:resource-optimum-main}
\end{equation}
Lower precision may reduce per-step storage but enlarge the floor, the synthesis constant, the route-error constant, or the depth needed to reach the tolerance.  The optimum can consequently be an interior precision and can move as the requested accuracy changes.  Figures~\ref{fig:soft-phase} and~\ref{fig:attention-phase} show this effect in exact or globally solved examples.

The symbolic charge in \eqref{eq:resource-optimum-main} should be replaced by measured platform costs only after the mathematical feasibility screen.  Sequential depth, codebook reuse, sparsity, kernel fusion, and activation traffic can reverse a conclusion based on parameter bits alone.

\section{Discussion and Limitations}
\label{sec:discussion-main}
\label{sec:discussion}

\subsection{What is genuinely new}
\label{sec:novelty-discussion-main}

The empirical understanding that neural-network depth and parameter precision can trade off is not new.  Quantized approximation theory, oversampled and feedback quantization, mixed-integer control, residual-flow interpretations, and learned codebooks all provide important precedents.  
The contribution here is the connection of these ideas into one target-specific theory with the following jointly supported claims:

\begin{enumerate}[leftmargin=1.8em,itemsep=.35em]
\item the distance to the closed relaxed full-map reachable set is the exact asymptotic floor for a fixed operation library;
\item pure finite depth approaches that set at a target-independent first-order rate under bounded variation and at an explicit H\"older-dependent rate under weaker temporal regularity;
\item one fixed teacher, an exact residual ReLU MLP, a nonuniform two-token attention system, and a defect-robust exposed channel show that the first-order depth price can be necessary;
\item learned dictionary metadata and pure schedule depth are different description resources, with upper, packing, and allocation laws;
\item finite arithmetic can reverse the benefit of depth, while increment error feedback restores the law and admits an exact common-lattice implementation with auditable register widths;
\item prescribed route motion and actual state-driven route changes require different nonsmooth analyses; the latter needs transversality, path control, and hybrid small gain rather than a frozen positive-margin argument; and
\item the structural floor can be lower- and upper-bounded by proof-producing primal--dual certificates, including coupling-robust blockwise HJB witnesses, before training.
\end{enumerate}

\subsection{Representation and optimization remain different frontiers}

The theorems optimize over all declared schedules and metadata values.  They do not show that gradient descent, alternating minimization, local coordinate search, or a compiler finds the optimum.  Figure~\ref{fig:soft-opt-gap} demonstrates that local search can remain orders of magnitude above the exact frontier.  A useful deployment report should therefore contain both a certified or globally optimized representational bracket and the achieved optimization error.

\subsection{Exact finite-depth equivalence is a different problem}
\label{sec:exactness-discussion-main}

Exact bit-level emulation is stricter than vanishing error or accuracy matching.  With a fixed finite dictionary, fixed finite metadata, and input-independent schedules, the union of all finite-depth pure maps is countable, whereas a continuously parameterized teacher family can contain uncountably many maps.  Generic exact emulation therefore requires an additional resource: teacher-specific metadata, auxiliary state or width, input-dependent finite-state control, a finite digital deployment domain, or a weaker behavioral notion.  Formal equivalence verification under integer rounding and overflow semantics is a complementary compilation problem \citep{TeuberEtAl2021Equivalence}; the present resource law supplies approximation radii and structural impossibility certificates rather than a universal exact compiler.

\subsection{Trust boundaries}

Several boundaries should remain visible in every use of the theory.

\paragraph{Tube verification.}
A theorem constant is valid only on the tube for which boundedness, regularity, saturation margin, and field error are established.  Sample maxima are empirical estimates unless accompanied by a sound enclosure.

\paragraph{Finite arithmetic.}
A ``4-bit network'' is not an execution specification.  Weight format, activation format, increment scaling, accumulator width, rescaling, rounding, saturation or wrapping, and carry representation all affect $\mathcal A_D$.  A floating-point carry means the system is not integer only.

\paragraph{Routing.}
The simple-event theorem excludes simultaneous swaps, grazing contacts, and arbitrary chattering.  Its $L^p$ continuum extension permits discontinuous route cells; a uniform $C(\Xi)$ claim does not.

\paragraph{Certification.}
A sampled HJB residual is not a lower certificate.  A numerical SDP value is not a proof until the polynomial identity or interval enclosure is verified.  A finite witness certificate is only as broad as its valid dual interpretation or its separately proved covering remainder.

\paragraph{Evidence.}
The controlled DistilBERT study tests two feed-forward branches and three QAT seeds.  It supports the depth-coherence mechanism but does not establish a universal Transformer scaling law, integer-only execution, or hardware advantage.  The exact and synthetic certificates support the mathematical claims.

\subsection{Validation}
\label{sec:prospective-main}

The artifact bundle contains a strict checkpoint-output contract, exact certificate verifiers, bit-semantic arithmetic modules, and a frozen registry designed for this program.  It also contains QReplace, which turns the declared floor and finite-resource bounds into an auditable recommendation, and QReplaceLean, whose archived pinned build kernel-checks all twelve mapped statements in the exact discrete core; their scopes and status policies are given in Appendices~\ref{app:qreplace-software} and~\ref{app:lean-verification}.  The supplied compact pretrained archive does not contain the QAT checkpoint tensors needed to claim a formal pretrained-block floor, so no such result is manufactured here.

\paragraph{Proof and reproducibility map.}
The complete core resource-theory derivations are collected in Appendix~\ref{app:full-resource-proofs}; the H\"older, prescribed-route, common-lattice, exact neural-converse, matching-depth, and compositional-certificate proofs are collected in Appendix~\ref{app:merged-exact-proofs}.  The specialized soft-threshold chain is modular: well-posedness and invariant tubes appear in Appendix~\ref{soft:app:wellposed}, Euler and switching estimates in Appendix~\ref{soft:app:euler}, the signed-ray sharpness law in Appendix~\ref{soft:app:signed-ray}, certified floor computation in Appendix~\ref{soft:app:computability}, and the exact piecewise-affine matrix certificate in Appendix~\ref{soft:app:pwa}.  Table~\ref{soft:tab:certified-floor} itemizes every error contribution in the terminating floor algorithm.  Architecture, search, checkpoint, broad-stress, and numerical-artifact details are cross-indexed in Appendices~\ref{app:experiments}--\ref{app:reproducibility}; the two reusable software companions are specified in Appendices~\ref{app:qreplace-software} and~\ref{app:lean-verification}.  Thus every figure and table, and every supplementary block used to support a claim in the main paper, has an explicit entry point from the main narrative.

\paragraph{Disclosure of AI-assisted tools.} A general-purpose large language model was used as an assistive tool for drafting and editing, code generation and debugging, and proof exploration. The scopes of machine-checked and executable verification are stated explicitly in the manuscript and supplements.

\section{Conclusion}
\label{sec:conclusion-main}

Depth can replace missing numerical precision, but only in a specific sense and only under identifiable conditions.  The operation library determines a closed relaxed reachable class and therefore an exact structural floor.  Pure finite depth implements compatible bounded-variation relaxed computation at a first-order price, while H\"older temporal regularity yields the corresponding fractional rate.  Fixed teachers in scalar, residual-MLP, and nonuniform-attention systems prove that the first-order price can be necessary; when a coherent high-precision comparator also has first-order error, accuracy matching obeys $D_{\mathrm{match}}(L)=\Theta(L)$.  Learned codebooks add a global metadata resource.  Finite execution can preserve, saturate, or reverse the depth benefit depending on whether small residual increments remain visible; exact common-lattice error feedback shows one concrete way to preserve their accumulated effect.  Route-changing networks introduce a hybrid feedback loop controlled by event transversality and small gain.  Primal--dual certificates turn these objects into feasible, impossible, or unresolved decisions.

The resulting principle is:
\begin{center}
\fbox{\parbox{0.93\linewidth}{\centering
\textbf{Dictionary geometry determines what is reachable.  Depth pays for pure synthesis.  Metadata describes the learned library.  Arithmetic determines whether the microsteps exist physically.  Event geometry determines whether routed paths remain stable.  Certificates determine whether training is warranted.}}}
\end{center}

\appendix

\section{Complete Resource-Theory Proofs and Extensions}
\label{app:full-resource-proofs}
\label{app:resource-theory-proofs}
\label{sec:resource-theory-candidate}

This section strengthens the ideal-arithmetic reachability theorem in five
orthogonal directions: finite arithmetic, fixed-target converses, learned
codebooks, route-changing hybrid dynamics, and certified computation of the
structural floor.  The results are stated for the same lifted state space and
shared-control convention as in \cref{sec:abstract-theory}.  Each theorem
distinguishes the resource that is being varied from the resources held fixed;
no uncounted schedule, scale, accumulator, router, or codebook parameter is
permitted.  When a result requires a finite-dimensional state, this is stated
explicitly.

\subsection{Finite arithmetic: what survives actual execution}
\label{sec:finite-arithmetic-candidate}

For a depth $D$ and $h=T/D$, write
\[
  \Psi_{k,j}(z)=z+hG_j(t_k,z),\qquad t_k=kh,
\]
for the ideal microstep.  Let $\widehat\Psi_{k,j}$ be the corresponding
implemented microstep, including any weight, activation, scale, accumulator,
requantization, or clipping operations.  The implemented and ideal maps are
compared on one common execution tube.

\begin{theorem}[Schedulewise hardware perturbation and reachable-set transfer]
\label[theorem]{thm:hardware-transfer-candidate}
Suppose that every ideal and implemented prefix generated by an admissible
schedule remains in the common tube.  For every $k$ and $j$, assume the ideal
map $\Psi_{k,j}$ is $\kappa_k$-Lipschitz on that tube and
\begin{equation}
  \sup_{z\text{ in the tube}}
  \norm{\widehat\Psi_{k,j}(z)-\Psi_{k,j}(z)}
  \le \varepsilon_{k,j}.
  \label{eq:local-hardware-defect-candidate}
\end{equation}
For a schedule $\sigma\in[J]^D$, define
\begin{equation}
  \mathcal A_D(\sigma)
  :=\sum_{k=0}^{D-1}\varepsilon_{k,\sigma_k}
       \prod_{\ell=k+1}^{D-1}\kappa_\ell,
  \qquad
  \mathcal A_D:=\sup_{\sigma\in[J]^D}\mathcal A_D(\sigma),
  \label{eq:hardware-budget-candidate}
\end{equation}
where an empty product equals one.  Let $\widehat\RR_{b,D}$ be the set of
implemented endpoints, indexed by the same schedules as $\RR_{b,D}$.  Then
\begin{align}
  \norm{\widehat z_D^\sigma-z_D^\sigma}
  &\le \mathcal A_D(\sigma),
  \label{eq:schedule-hardware-bound-candidate}\\
  \dH(\widehat\RR_{b,D},\RR_{b,D})
  &\le \mathcal A_D.
  \label{eq:hardware-hausdorff-candidate}
\end{align}
Consequently, under \Cref{ass:regularity}, for every $D\ge D_0$,
\begin{align}
  \dH(\widehat\RR_{b,D},\RR_{b,\Rel})
  &\le \frac{C_{\mathrm{fh},b}}{D}+\mathcal A_D,
  \label{eq:hardware-relaxed-candidate}\\
  \left|\dist(F^\star,\widehat\RR_{b,D})
       -E_{\infty,b}^{\mathrm{end}}(F^\star)\right|
  &\le \frac{C_{\mathrm{fh},b}}{D}+\mathcal A_D.
  \label{eq:hardware-floor-rate-candidate}
\end{align}
\end{theorem}

\begin{proof}
For a fixed schedule, let $e_k=\widehat z_k-z_k$.  By adding and subtracting
$\Psi_{k,\sigma_k}(\widehat z_k)$,
\[
  \norm{e_{k+1}}
  \le \varepsilon_{k,\sigma_k}+\kappa_k\norm{e_k}.
\]
Iteration gives \eqref{eq:schedule-hardware-bound-candidate}.  Every ideal
endpoint is paired with the implemented endpoint having the same schedule,
and conversely, so the two directed set distances are bounded by the same
quantity.  This proves \eqref{eq:hardware-hausdorff-candidate}.  The remaining
claims follow from the triangle inequality for Hausdorff distance and the fact
that distance to a nonempty set is one-Lipschitz under Hausdorff perturbation.
\end{proof}

\begin{remark}[A concrete local-defect decomposition]
For a saturating fixed-point microstep of the form
\[
  \widehat\Psi_{k,j}(z)
  =Q^{\mathrm{sat}}_{\alpha_k,M_k}
   \left(z+\widehat h_k\widehat G_{k,j}(z)
               +e^{\mathrm{acc}}_{k,j}(z)\right),
\]
a valid bound in a finite-dimensional Euclidean state is
\begin{equation}
  \varepsilon_{k,j}
  \le B\abs{\widehat h_k-h}
     +\abs{\widehat h_k}\eta^G_{k,j}
     +\eta^{\mathrm{acc}}_{k,j}
     +\frac{\alpha_k\sqrt q}{2}
     +\rho^{\mathrm{sat}}_{k,j},
  \label{eq:hardware-decomposition-candidate}
\end{equation}
provided
$\norm{\widehat G_{k,j}(z)-G_j(t_k,z)}\le\eta^G_{k,j}$ and
$\norm{e^{\mathrm{acc}}_{k,j}(z)}\le\eta^{\mathrm{acc}}_{k,j}$.
Here $\rho^{\mathrm{sat}}_{k,j}$ is the distance moved by clipping.  The same
formula applies pointwise, followed by a supremum, in a finite-witness lifted
state.  Thus weight/codebook error, scale metadata, accumulation, write-back
rounding, and saturation enter as distinct resources rather than one
undifferentiated ``quantization error.''
\end{remark}

\begin{corollary}[Naive microstep write-back has a depth-dependent optimum]
\label[corollary]{cor:naive-rounding-candidate}
Suppose every write-back contributes a local defect at most $\rho_z$ and every
transition product in \eqref{eq:hardware-budget-candidate} is at most
$\overline{\mathcal S}_{\mathrm E}$.  Then
\begin{equation}
  \mathcal A_D\le D\rho_z\overline{\mathcal S}_{\mathrm E},
  \qquad
  \left|E_D^{\mathrm{hw}}-E_\infty\right|
  \le \frac{C_{\mathrm{fh}}}{D}
      +D\rho_z\overline{\mathcal S}_{\mathrm E}.
  \label{eq:naive-depth-envelope-candidate}
\end{equation}
In particular, this schedule-uniform worst-case envelope can certify a total
$O(D^{-1})$ law under write-back after every microstep only when
$\rho_z=O(D^{-2})$.  The statement concerns the generic certificate; special
architectures may exhibit cancellations that the envelope does not use.  For fixed $\rho_z$, the upper
envelope is minimized at
\[
  D_{\mathrm{opt}}\asymp
  \sqrt{\frac{C_{\mathrm{fh}}}
              {\rho_z\overline{\mathcal S}_{\mathrm E}}},
\]
with minimum order
$\sqrt{C_{\mathrm{fh}}\rho_z\overline{\mathcal S}_{\mathrm E}}$.
\end{corollary}

\begin{proposition}[A fixed activation grid can annihilate all small residual steps]
\label[proposition]{prop:fixed-grid-no-go-candidate}
Let $Q_\Delta$ denote nearest rounding to $\Delta\mathbb Z$, and consider
\begin{equation}
  \widehat x_{k+1}=Q_\Delta\left(\widehat x_k+\frac1D\right),
  \qquad \widehat x_0=0.
  \label{eq:fixed-grid-no-go-recursion-candidate}
\end{equation}
The continuous target is $x^\star(1)=1$.  If $D>2/\Delta$, then
$\widehat x_k=0$ for every $k$ and
\begin{equation}
  \abs{\widehat x_D-x^\star(1)}=1.
  \label{eq:fixed-grid-no-go-error-candidate}
\end{equation}
Thus increasing depth need not improve an implementation with a fixed
activation lattice.
\end{proposition}

\begin{proof}
The condition $D>2/\Delta$ gives $1/D<\Delta/2$.  Starting from a grid point,
nearest rounding maps every proposed increment back to the same point.  The
claim follows by induction.
\end{proof}

The preceding obstruction is avoided when the quantization residual is kept
and fed into the next increment.  Let $Q:\Z\to\Z$ satisfy
\begin{equation}
  \norm{v-Q(v)}\le\rho
  \label{eq:bounded-increment-quantizer-candidate}
\end{equation}
throughout the nonsaturating execution range.  For a fixed schedule, define
\begin{align}
  q_k&=Q\bigl(hG_{\sigma_k}(t_k,\widehat z_k)+r_k\bigr),
  \label{eq:ef-quantize-candidate}\\
  \widehat z_{k+1}&=\widehat z_k+q_k,
  \label{eq:ef-state-candidate}\\
  r_{k+1}&=hG_{\sigma_k}(t_k,\widehat z_k)+r_k-q_k,
  \qquad r_0=0.
  \label{eq:ef-residual-candidate}
\end{align}

\begin{theorem}[Error feedback removes the factor $D$]
\label[theorem]{thm:error-feedback-candidate}
Assume the ideal and error-feedback paths stay in a common tube on which every
field is $L_z$-Lipschitz.  Let $z_k$ be the ideal pure Euler path with the same
schedule.  Then
\begin{equation}
  \max_{0\le k\le D}\norm{\widehat z_k-z_k}
  \le \rho(1+hL_z)^D
  \le \rho e^{L_zT}.
  \label{eq:ef-bound-candidate}
\end{equation}
Consequently,
\begin{equation}
  \dH(\widehat\RR^{\mathrm{EF}}_{b,D},\RR_{b,\Rel})
  \le \frac{C_{\mathrm{fh},b}}{D}+\rho e^{L_zT}.
  \label{eq:ef-floor-rate-candidate}
\end{equation}
Therefore $\rho=O(D^{-1})$ is sufficient to retain the first-order total law,
in contrast with the $O(D^{-2})$ requirement for naive microstep write-back.
\end{theorem}

\begin{proof}
The residual definition and \eqref{eq:bounded-increment-quantizer-candidate}
give $\norm{r_k}\le\rho$.  Let
$e_k=\widehat z_k-z_k$ and $s_k=e_k+r_k$.  Since
\[
  q_k=hG_{\sigma_k}(t_k,\widehat z_k)+r_k-r_{k+1},
\]
we have
\[
  s_{k+1}
  =s_k+h\left[G_{\sigma_k}(t_k,\widehat z_k)
                   -G_{\sigma_k}(t_k,z_k)\right].
\]
Because $e_k=s_k-r_k$,
\[
  \norm{s_{k+1}}
  \le(1+hL_z)\norm{s_k}+hL_z\rho.
\]
With $s_0=0$, iteration yields
$\norm{s_k}\le\rho((1+hL_z)^k-1)$.  Hence
\[
  \norm{e_k}\le\norm{s_k}+\norm{r_k}
  \le\rho(1+hL_z)^k,
\]
which proves \eqref{eq:ef-bound-candidate}.  The reachable-set conclusion
follows from \cref{thm:hausdorff}.
\end{proof}

\begin{theorem}[Error feedback with inexact executed fields]
\label[theorem]{thm:error-feedback-field-v4}
For each grid point and atom, let $\widehat G_{k,j}$ be the field actually
computed before increment quantization.  Suppose the ideal and implemented
paths remain in a common tube, every ideal atom is $L_z$-Lipschitz there, and
\begin{equation}
  \sup_{k,j,z\text{ in the tube}}
  \norm{\widehat G_{k,j}(z)-G_j(t_k,z)}
  \le \eta_G.
  \label{eq:executed-field-error-v4}
\end{equation}
Run the error-feedback recursion
\begin{align}
 q_k&=Q\!\left(h\widehat G_{k,\sigma_k}(\widehat z_k)+r_k\right),\notag\\
 \widehat z_{k+1}&=\widehat z_k+q_k,\notag\\
 r_{k+1}&=h\widehat G_{k,\sigma_k}(\widehat z_k)+r_k-q_k,
 \qquad r_0=0,
 \label{eq:field-aware-ef-v4}
\end{align}
where $\norm{v-Q(v)}\le\rho$ on the entire nonsaturating execution range.
For the ideal pure Euler path with the same schedule,
\begin{equation}
 \max_{k\le D}\norm{\widehat z_k-z_k}
 \le \rho e^{L_zT}+\eta_G\Phi_{L_z}(T).
 \label{eq:field-aware-ef-bound-v4}
\end{equation}
Consequently,
\begin{equation}
 \dH(\widehat\RR^{\mathrm{EF}}_{b,D},\RR_{b,\Rel})
 \le \frac{C_{\mathrm{fh},b}}{D}
      +\rho e^{L_zT}+\eta_G\Phi_{L_z}(T).
 \label{eq:field-aware-ef-reach-v4}
\end{equation}
The $\eta_G$ term is an error in the residual \emph{field}; the microstep
multiplies it by $h$, so it accumulates over the fixed horizon rather than as
$D\eta_G$.
\end{theorem}

\begin{proof}
Let $e_k=\widehat z_k-z_k$ and $s_k=e_k+r_k$.  The residual identity gives
\[
 s_{k+1}=s_k+h\left[
 \widehat G_{k,\sigma_k}(\widehat z_k)
 -G_{\sigma_k}(t_k,z_k)
 \right].
\]
Since $e_k=s_k-r_k$ and $\norm{r_k}\le\rho$,
\[
 \norm{s_{k+1}}
 \le(1+hL_z)\norm{s_k}+hL_z\rho+h\eta_G.
\]
Iteration from $s_0=0$ yields, when $L_z>0$,
\[
 \norm{s_k}
 \le \rho\bigl((1+hL_z)^k-1\bigr)
   +\frac{\eta_G}{L_z}\bigl((1+hL_z)^k-1\bigr).
\]
Adding $\norm{r_k}\le\rho$ and using
$(1+hL_z)^D\le e^{L_zT}$ proves
\eqref{eq:field-aware-ef-bound-v4}.  When $L_z=0$, the same recursion gives
$\norm{e_k}\le\rho+t_k\eta_G$, which is the stated formula with
$\Phi_{L_z}(T)=T$.  Pair schedules and apply the baseline Hausdorff theorem.
\end{proof}

\begin{corollary}[Depth-scaled increments preserve the first-order law]
\label[corollary]{cor:scaled-error-feedback-v2}
Let $\overline Q_b$ be a normalized increment quantizer satisfying
\begin{equation}
  \norm{u-\overline Q_b(u)}\le \overline\rho_b
  \label{eq:normalized-quantizer-v2}
\end{equation}
throughout a normalized nonsaturating range.  For $h=T/D$, run
\begin{align}
  \overline q_k
  &=\overline Q_b\!\left(
      G_{\sigma_k}(t_k,\widehat z_k)+\overline r_k
    \right),
  \label{eq:scaled-ef-q-v2}\\
  \widehat z_{k+1}
  &=\widehat z_k+h\overline q_k,
  \label{eq:scaled-ef-state-v2}\\
  \overline r_{k+1}
  &=G_{\sigma_k}(t_k,\widehat z_k)+\overline r_k-\overline q_k,
  \qquad \overline r_0=0.
  \label{eq:scaled-ef-residual-v2}
\end{align}
If the ideal and implemented paths remain in a common tube, then
\begin{equation}
  \max_{k\le D}\norm{\widehat z_k-z_k}
  \le \frac{T\overline\rho_b e^{L_zT}}{D},
  \label{eq:scaled-ef-schedule-v2}
\end{equation}
and
\begin{equation}
  \dH(\widehat\RR^{\mathrm{scaled\text{-}EF}}_{b,D},\RR_{b,\Rel})
  \le
  \frac{C_{\mathrm{fh},b}+T\overline\rho_b e^{L_zT}}{D}.
  \label{eq:scaled-ef-reach-v2}
\end{equation}
Thus a fixed normalized increment alphabet can preserve the $D^{-1}$ law;
its physical scale shrinks with the residual step $h$.
\end{corollary}

\begin{proof}
Set $Q_{h,b}(v)=h\overline Q_b(v/h)$ and
$r_k=h\overline r_k$.  Then
$\norm{v-Q_{h,b}(v)}\le h\overline\rho_b$ and the recursion is exactly
\eqref{eq:ef-quantize-candidate}--\eqref{eq:ef-residual-candidate} with
$\rho=h\overline\rho_b$.  Apply
\cref{thm:error-feedback-candidate}.
\end{proof}

\begin{corollary}[Activation-bit and accumulator-bit scaling]
\label[corollary]{cor:activation-bit-scaling-v2}
Suppose a $q$-dimensional full state is quantized after every microstep by a
coordinatewise $b_z$-bit uniform quantizer on $[-R,R]$.  Its Euclidean
covering radius obeys
\begin{equation}
  \rho_z\le \frac{\sqrt q\,R}{2^{b_z}-1}.
  \label{eq:state-grid-radius-v2}
\end{equation}
The naive write-back envelope in
\cref{cor:naive-rounding-candidate} can certify an $O(D^{-1})$ law only if
\begin{equation}
  b_z\ge 2\log_2D+\frac12\log_2q+O(1).
  \label{eq:naive-bit-growth-v2}
\end{equation}
By contrast, if the normalized increment in
\cref{cor:scaled-error-feedback-v2} lies in $[-M,M]^q$, then a fixed
$b_\Delta$-bit normalized quantizer has
\begin{equation}
  \overline\rho_{b_\Delta}
  \le \frac{\sqrt q\,M}{2^{b_\Delta}-1}
  \label{eq:increment-radius-v2}
\end{equation}
and preserves a first-order total law with fixed $b_\Delta$.  If the emitted
increment codes are integers of coordinate magnitude at most $M_b$, exact
worst-case accumulation over $D$ steps requires a signed accumulator of width
at least
\begin{equation}
  1+\left\lceil\log_2(DM_b+1)\right\rceil,
  \label{eq:accumulator-guard-bits-v2}
\end{equation}
namely only $\log_2D+O(1)$ additional guard bits beyond the increment code.
\end{corollary}

\begin{proof}
The radius formulas are half the coordinate grid spacing, converted to the
Euclidean norm.  Substitute \eqref{eq:state-grid-radius-v2} into
\eqref{eq:naive-depth-envelope-candidate}; requiring the write-back term to
be $O(D^{-1})$ gives \eqref{eq:naive-bit-growth-v2}.  The increment statement
uses \cref{cor:scaled-error-feedback-v2}.  Finally, the absolute value of a
sum of $D$ integer codes of magnitude at most $M_b$ is at most $DM_b$.
\end{proof}

\begin{corollary}[Bounded-domain Transformer residual streams]
\label[corollary]{cor:transformer-arithmetic-v2}
Consider a finite-witness or pointwise bounded-domain Transformer residual
field for which the normalization, attention, and MLP estimates supply a
common field bound $B_{\mathrm{tr}}$ and state-Lipschitz constant
$L_{\mathrm{tr}}$ on a $q$-dimensional residual stream.  Suppose the executed
field has uniform realization error $\eta_{G,D}$, the increment quantizer has
normalized radius $\overline\rho_b$, residual accumulation is exact, no
saturation occurs, and the shared residual scale is $h=T/D$.  Suppose all
remaining implementation effects, after the field computation and
error-feedback increment quantization, have a uniform schedulewise endpoint
radius at most $\mathcal A_D^{\mathrm{other}}$.  Then
\begin{equation}
 \left|E_D^{\mathrm{tr,impl}}-E_{\infty}^{\mathrm{tr}}\right|
 \le
 \frac{C_{\mathrm{fh,tr}}+T\overline\rho_b e^{L_{\mathrm{tr}}T}}{D}
 +\eta_{G,D}\Phi_{L_{\mathrm{tr}}}(T)
 +\mathcal A_D^{\mathrm{other}},
 \label{eq:transformer-arithmetic-law-v2}
\end{equation}
where
\[
 \Phi_L(T):=\begin{cases}(e^{LT}-1)/L,&L>0,\\T,&L=0,\end{cases}
\]
and $\mathcal A_D^{\mathrm{other}}$ contains only scale, accumulator, or
other local defects not represented by the field and increment bounds.
Consequently, depth substitution requires execution that preserves residual
increments; quantizing the complete residual state after every microblock
obeys the stricter bit-growth law in
\cref{cor:activation-bit-scaling-v2}.
\end{corollary}

\begin{proof}
Use physical increment radius
$\rho_D=h\overline\rho_b=T\overline\rho_b/D$ in
\cref{thm:error-feedback-field-v4}.  Its schedulewise path bound, evaluated at
the endpoint, contributes
$T\overline\rho_b e^{L_{\mathrm{tr}}T}/D+
\eta_{G,D}\Phi_{L_{\mathrm{tr}}}(T)$.  Add the ideal pure-to-relaxed radius
$C_{\mathrm{fh,tr}}/D$ and the declared remaining schedulewise radius
$\mathcal A_D^{\mathrm{other}}$ by the Hausdorff triangle inequality.  The
target-distance conclusion follows because distance to a nonempty set is
one-Lipschitz under Hausdorff perturbations.
\end{proof}

\begin{remark}[Finite residual accumulators]
If the stored residual in \eqref{eq:ef-residual-candidate} is itself rounded at
every step with error $a_k$, $\norm{a_k}\le\rho_a$, the transformed recursion
acquires the additive term $a_k$.  The generic additional penalty is therefore
$D\rho_a\overline{\mathcal S}_{\mathrm E}$.  An exact common-lattice integer
accumulator, a sufficiently wide residual accumulator, a second feedback
loop, or terminal-only requantization is needed to avoid reintroducing the
factor $D$.
\end{remark}

\begin{proposition}[Saturation-safe tube]
\label[proposition]{prop:saturation-tube-candidate}
Let $\mathcal K_{\mathrm{id}}$ contain every ideal prefix state for every
admissible schedule, let $\mathcal B$ be the representable activation set,
and suppose
\begin{equation}
 \inf_{z\in\mathcal K_{\mathrm{id}}}
 \operatorname{dist}(z,\mathcal B^c)\ge m>0.
 \label{eq:saturation-margin-v4}
\end{equation}  Form the unsaturated implemented maps by deleting the clipping
operation.  If the prefix version of
\eqref{eq:schedule-hardware-bound-candidate} is strictly less than $m$ for
every schedule and every prefix, then clipping never activates; the saturated
and unsaturated executions coincide.
\end{proposition}

\begin{proof}
Assume the two executions agree through step $k$.  The unsaturated next state
is within $m$ of the corresponding ideal next state, which lies in
$\mathcal K_{\mathrm{id}}$.  It therefore belongs to $\mathcal B$, so
clipping leaves it unchanged.  Induction proves the claim.
\end{proof}

\begin{remark}[Wrapping arithmetic is a different regime]
A wraparound operation is not a projection onto a representable set.  One
overflow can change both sign and magnitude by an order-range amount.  A small
rounding-radius theorem is therefore impossible for wrapping arithmetic
without a separately verified no-overflow invariant.
\end{remark}

\subsection{A fixed teacher with an exact first-order converse}
\label{sec:fixed-target-converse-candidate}

The current minimax construction may be strengthened to one target that does
not depend on $D$.

\begin{theorem}[Exact best error for one fixed binary target]
\label[theorem]{thm:fixed-target-candidate}
Fix $a>0$, $D\ge2$, and $r_D=1-D^{-1}$.  Consider
\begin{equation}
  x_{k+1}=r_Dx_k+\frac{q_k}{D},
  \qquad q_k\in\{0,a\},
  \qquad x_0=0.
  \label{eq:fixed-target-recursion-candidate}
\end{equation}
Let the fixed target be the endpoint of
\begin{equation}
  \dot x=a-x,\qquad x(0)=0,
  \label{eq:fixed-target-ode-candidate}
\end{equation}
namely $x^\star=a(1-e^{-1})$.  Then the all-$a$ schedule is a globally closest
pure depth-$D$ endpoint, and the exact optimal error is
\begin{equation}
  E_D^{\mathrm{bin}}(x^\star)
  =a\left[e^{-1}-\left(1-\frac1D\right)^D\right].
  \label{eq:fixed-target-exact-candidate}
\end{equation}
For every $D\ge2$,
\begin{equation}
  \frac{a}{4eD}
  \le E_D^{\mathrm{bin}}(x^\star)
  \le \frac{a}{2e(D-1)},
  \label{eq:fixed-target-bounds-candidate}
\end{equation}
and
\begin{equation}
  E_D^{\mathrm{bin}}(x^\star)
  =\frac{a}{2eD}+O(D^{-2}).
  \label{eq:fixed-target-asymptotic-candidate}
\end{equation}
\end{theorem}

\begin{proof}
Every endpoint has the form
\[
  x_D=\frac1D\sum_{k=0}^{D-1}r_D^{D-1-k}q_k.
\]
The largest endpoint is
$M_D=a(1-r_D^D)$, obtained by the all-$a$ schedule.  The second largest is
obtained by replacing the earliest, and hence least heavily weighted, symbol
by zero.  The gap is
\[
  g_D=\frac{a}{D}r_D^{D-1}.
\]
It remains to show that $x^\star$ lies in the Voronoi cell of $M_D$, that is,
$M_D-x^\star\le g_D/2$.  Put $x=1/D$.  The convergent power series identity
\begin{align*}
 &1+\left(\frac1x-1\right)\log(1-x)+\log(1-x/2)\\
 &\qquad=\sum_{n=2}^{\infty}
   \frac{x^n}{n}\left(\frac1{n+1}-\frac1{2^n}\right)\ge0
\end{align*}
uses $2^n\ge n+1$.  Exponentiating gives
\[
  e^{-1}\le r_D^{D-1}\left(1-\frac1{2D}\right).
\]
Therefore
\[
  M_D-x^\star
  =a(e^{-1}-r_D^D)
  \le \frac{a}{2D}r_D^{D-1}=\frac{g_D}{2}.
\]
Every other pure endpoint is at most $M_D-g_D$, proving the exact formula.

For the quantitative bounds, set
\[
  \theta_D=-1-D\log(1-D^{-1})
  =\sum_{n=2}^{\infty}\frac{1}{nD^{n-1}}.
\]
Then
\[
  \frac1{2D}\le\theta_D\le\frac1{2(D-1)}\le\frac12
\]
and
\[
  e^{-1}-r_D^D=e^{-1}(1-e^{-\theta_D}).
\]
The inequalities $\theta/2\le1-e^{-\theta}\le\theta$ for
$0\le\theta\le1/2$ prove \eqref{eq:fixed-target-bounds-candidate}.  Expanding
$\theta_D$ and $1-e^{-\theta_D}$ gives
\eqref{eq:fixed-target-asymptotic-candidate}.
\end{proof}

\begin{corollary}[Exposed binary modes give fixed-target lower bounds]
\label[corollary]{cor:exposed-fixed-target-candidate}
Normalize the horizon to $T=1$.  Let $\ell\in\Z^*$ satisfy
$\norm{\ell}_*\le1$ and $\ell(z_0)=0$.  Suppose that for every atom $j$ there
is $q_j\in\{0,a\}$ such that
\begin{equation}
  \ell(G_j(t,z))=q_j-\ell(z)
  \label{eq:exposed-binary-mode-candidate}
\end{equation}
throughout the tube, and suppose at least one atom has $q_j=a$.  Let $F^\star$
be the terminal map of the continuous flow generated by any such $a$-atom.
Then $E_{\infty,b}^{\mathrm{end}}(F^\star)=0$ and, for every $D\ge2$,
\begin{equation}
  E_b^{\mathrm{end}}(D;F^\star)
  \ge a\left[e^{-1}-\left(1-\frac1D\right)^D\right]
  \ge\frac{a}{4eD}.
  \label{eq:exposed-fixed-target-lower-candidate}
\end{equation}
Together with \cref{thm:hausdorff}, for every
$D\ge\max\{2,D_0\}$,
\begin{equation}
  \frac{a}{4eD}
  \le E_b^{\mathrm{end}}(D;F^\star)
  \le\frac{C_{\mathrm{fh},b}}{D}.
  \label{eq:fixed-target-two-sided-candidate}
\end{equation}
Thus the first-order exponent is necessary and sufficient for one fixed
architecture target, not only in a $D$-dependent minimax sense.
\end{corollary}

\begin{proof}
Applying $\ell$ to every pure Euler recursion gives exactly
\eqref{eq:fixed-target-recursion-candidate}.  The continuous $a$-atom target
has projection $a(1-e^{-1})$.  Since $\norm{\ell}_*\le1$, full-state distance
dominates projected scalar distance, and
\cref{thm:fixed-target-candidate} gives the lower bound.  The target is a
relaxed-reachable endpoint, so its structural floor is zero; the upper bound
is \cref{thm:hausdorff}.
\end{proof}

\begin{theorem}[Perturbation-stable exposed-mode converse]
\label[theorem]{thm:approx-exposed-mode-v7}
Normalize the horizon to $T=1$, let $D\ge2$, put
$r_D=1-D^{-1}$, and let $\ell\in\Z^*$ satisfy
$\norm{\ell}_*\le1$ and $\ell(z_0)=0$.  Assign each atom a symbol
$q_j\in\{0,a\}$, where $a>0$, and assume that on the entire pure execution
tube
\begin{equation}
 \abs{\ell(G_j(t,z))-(q_j-\ell(z))}
 \le\varepsilon_{\mathrm{mode}}
 \qquad\forall j,t,z.
 \label{eq:approx-exposed-atom-v7}
\end{equation}
Let the fixed target $F^\star$ obey
\begin{equation}
 \abs{\ell(F^\star)-a(1-e^{-1})}
 \le\varepsilon_\star.
 \label{eq:approx-exposed-target-v7}
\end{equation}
Then every pure depth-$D$ implementation satisfies
\begin{align}
 E_b^{\mathrm{end}}(D;F^\star)
 \ge \max\Bigl\{0,{}&
 a\left[e^{-1}-r_D^D\right]
 -\varepsilon_\star
 -\varepsilon_{\mathrm{mode}}(1-r_D^D)
 \Bigr\}.
 \label{eq:approx-exposed-lower-v7}
\end{align}
If $F^\star$ is generated by a continuous teacher whose projected path
$y^\star(t)=\ell(z^\star(t))$ satisfies
\begin{equation}
 \abs{\dot y^\star-(a-y^\star)}\le\varepsilon_{\mathrm{teach}}
 \quad\text{a.e.},\qquad y^\star(0)=0,
 \label{eq:approx-exposed-teacher-v7}
\end{equation}
then one may take
$\varepsilon_\star=\varepsilon_{\mathrm{teach}}(1-e^{-1})$.
Consequently, a measured approximate mode gives a positive finite-depth
necessity certificate whenever the right-hand side of
\eqref{eq:approx-exposed-lower-v7} is positive.  If both defects are
$o(D^{-1})$, the lower law remains
$a/(2eD)+o(D^{-1})$.
\end{theorem}

\begin{proof}
For one pure schedule put $y_k=\ell(z_k)$.  By
\eqref{eq:approx-exposed-atom-v7},
\[
 y_{k+1}=r_Dy_k+\frac{q_{\sigma_k}}D+\frac{\xi_k}D,
 \qquad \abs{\xi_k}\le\varepsilon_{\mathrm{mode}}.
\]
Let $\bar y_k$ solve the exact binary recursion with the same symbols.  Then
\[
 \abs{y_D-\bar y_D}
 \le\frac{\varepsilon_{\mathrm{mode}}}{D}
       \sum_{m=0}^{D-1}r_D^m
 =\varepsilon_{\mathrm{mode}}(1-r_D^D).
\]
The reverse triangle inequality, the target defect, and
\cref{thm:fixed-target-candidate} now give
\eqref{eq:approx-exposed-lower-v7}; the full-state norm dominates the
projection because $\norm{\ell}_*\le1$.  For the teacher statement, variation
of constants gives
\[
 \abs{y^\star(1)-a(1-e^{-1})}
 \le\varepsilon_{\mathrm{teach}}
      \int_0^1e^{-(1-s)}\,ds
 =\varepsilon_{\mathrm{teach}}(1-e^{-1}).
\]
The asymptotic conclusion follows from
\eqref{eq:fixed-target-asymptotic-candidate}.
\end{proof}

\begin{remark}[What must be measured on a trained block]
The theorem needs only one norm-one residual-stream functional, the best
binary labels $q_j\in\{0,a\}$ for the projected atom fields, a uniform or
certified tube residual $\varepsilon_{\mathrm{mode}}$, and the teacher
residual $\varepsilon_{\mathrm{teach}}$.  Unlike the exact decoupling
corollary, it is directly falsifiable on checkpoint traces.  It also shows why
an $O(1)$ projection defect cannot establish an asymptotic $D^{-1}$ converse:
the exact discretization gap itself vanishes as $D^{-1}$.
\end{remark}

\begin{corollary}[A fixed binary-gain LISTA teacher has a two-sided depth law]
\label[corollary]{cor:binary-lista-fixed-teacher-v2}
In the signed-ray soft-threshold specialization, take the binary gain alphabet
$\mathcal C_b=\{0,\Delta\}$ with $\Delta>0$.  Suppose an exposed scalar
channel satisfies the row conditions in the signed-ray theorem and
\begin{equation}
  a_0=(B_Wy_0)_i-\overline\lambda>0.
  \label{eq:lista-exposed-amplitude-v2}
\end{equation}
Let $F_\Delta^\star$ be the fixed continuous teacher generated by the constant
gain $c(t)\equiv\Delta$.  Then $E_{\infty,b}(F_\Delta^\star)=0$ and, for every
$D\ge2$,
\begin{equation}
  E_b(D;F_\Delta^\star)
  \ge
  a_0\Delta\left[
    e^{-1}-\left(1-\frac1D\right)^D
  \right]
  \ge \frac{a_0\Delta}{4eD}.
  \label{eq:lista-fixed-lower-v2}
\end{equation}
For this constant-gain teacher the gain-discrepancy term in the signed-ray
upper theorem vanishes.  Define $C_{\mathrm{LISTA},b}:=C_{\mathrm{disc}}$
with $K_c=0$ in that theorem (or use the larger master-theorem constant).  Then,
for every $D\ge\max\{2,D_0\}$,
\begin{equation}
  \frac{a_0\Delta}{4eD}
  \le E_b(D;F_\Delta^\star)
  \le \frac{C_{\mathrm{LISTA},b}}{D}.
  \label{eq:lista-fixed-two-sided-v2}
\end{equation}
Thus the manuscript's unfolded-network class contains a recognizable fixed
teacher, independent of $D$, for which first-order depth is both necessary
and sufficient.
\end{corollary}

\begin{proof}
On the exposed channel, every pure schedule obeys
\[
 x_{k+1}=\left(1-\frac1D\right)x_k
          +\frac{a_0q_k}{D},
 \qquad q_k\in\{0,\Delta\}.
\]
Apply \cref{thm:fixed-target-candidate} with scalar amplitude
$a=a_0\Delta$.  Full-map error dominates error in the exposed coordinate,
and the constant-$\Delta$ teacher is relaxed reachable.
\end{proof}

\begin{corollary}[A fixed residual-MLP channel inherits the converse]
\label[corollary]{cor:residual-mlp-fixed-v4}
Consider a residual MLP, adapter, or residual-stream block on a verified
activation region.  Let $\ell\in\Z^*$ have $\norm{\ell}_*\le1$, assume
$\ell(z_0)=0$, and expose two quantized atoms with
\begin{equation}
 \ell(G_0(t,z))=-\ell(z),
 \qquad
 \ell(G_1(t,z))=a-\ell(z),
 \qquad a>0.
 \label{eq:residual-mlp-exposed-v4}
\end{equation}
Let $F^\star$ be the fixed continuous teacher generated by $G_1$ over the
unit horizon.  Then, for every $D\ge2$,
\begin{equation}
 E_D(F^\star)
 \ge a\left[e^{-1}-\left(1-\frac1D\right)^D\right]
 \ge\frac{a}{4eD}.
 \label{eq:residual-mlp-lower-v4}
\end{equation}
If the baseline synthesis assumptions hold, then for
$D\ge\max\{2,D_0\}$,
\begin{equation}
 \frac{a}{4eD}\le E_D(F^\star)\le\frac{C_{\mathrm{fh}}}{D}.
 \label{eq:residual-mlp-two-sided-v4}
\end{equation}
Exact full-state equality with the scalar formula holds under the additional
\emph{attainment condition}
\begin{equation}
 \norm{z_D^{(1)}-F^\star}
 =\abs{\ell(z_D^{(1)}-F^\star)},
 \label{eq:residual-mlp-decoupling-v4}
\end{equation}
where $z_D^{(1)}$ is the all-$G_1$ Euler endpoint.  This equality is, for
example, implied by an orthogonal direct-sum decomposition in which the
continuous teacher and the all-$G_1$ Euler path have identical transverse
endpoints.  Without it, the projection proves the lower bound but not exact
full-state optimality.
\end{corollary}

\begin{proof}
Projecting every pure recursion by $\ell$ gives the binary scalar recursion in
\cref{thm:fixed-target-candidate}, so the norm-one projection yields
\eqref{eq:residual-mlp-lower-v4}.  The teacher is relaxed reachable, hence the
baseline theorem gives the upper bound.  Under
\eqref{eq:residual-mlp-decoupling-v4}, the all-$G_1$ endpoint attains in the
full norm the scalar lower bound; no other schedule can do better because its
projected error is at least that value.
\end{proof}

\begin{corollary}[Orthogonally decoupled Transformer residual channel]
\label[corollary]{cor:transformer-exposed-channel-v4}
Let a Transformer residual stream split as
$\Z=\operatorname{span}\{v\}\oplus\mathcal Y$, with $\norm v=1$ and a
norm-one functional $\ell$ satisfying $\ell(v)=1$ and
$\mathcal Y\subset\ker\ell$.  Assume $z_0\in\mathcal Y$, so
$\ell(z_0)=0$.  Suppose two admissible microblock atoms have
\begin{equation}
 G_q(t,xv+y)=(q-x)v+H(t,y),
 \qquad q\in\{0,a\},
 \label{eq:transformer-decoupled-channel-v4}
\end{equation}
with the same $H$ for both atoms.  The continuous all-$a$ teacher obeys the
fixed-target lower law \eqref{eq:residual-mlp-lower-v4}.  If $H\equiv0$ (or if
its continuous and all-$a$ Euler endpoints coincide exactly), the full-state
best error equals the scalar formula in
\eqref{eq:fixed-target-exact-candidate}.  Thus a standard residual-stream
direction supplies an architecture-level fixed teacher; exact equality
requires, and explicitly records, orthogonal decoupling.
\end{corollary}

\begin{proof}
The functional $\ell$ eliminates $H$ and exposes
$\ell(G_q)=q-\ell(z)$.  Apply
\cref{cor:residual-mlp-fixed-v4}.  When the orthogonal endpoint error vanishes,
condition \eqref{eq:residual-mlp-decoupling-v4} holds.
\end{proof}

\begin{remark}[General contraction and horizon]
For $\dot x=a-\lambda x$ on $[0,T]$ and the corresponding Euler recursion,
assume $D>\lambda T$ so that the Euler contraction coefficient is positive.
The all-high endpoint error is
\[
  \frac{a}{\lambda}
  \left[e^{-\lambda T}
        -\left(1-\frac{\lambda T}{D}\right)^D\right].
\]
It is the exact optimum whenever the target lies in the top endpoint's Voronoi
cell; the displayed condition is automatic in the normalized binary theorem
above and can be checked explicitly for other $\lambda T$ and alphabets.
\end{remark}

\subsection{Learned codebooks, adaptive dictionaries, and metadata bits}
\label{sec:learned-codebooks-candidate}

Let $(\Omega,d_\Omega)$ be a nonempty compact parameter space.  A global codebook
parameter $\omega\in\Omega$ determines
\[
  \mathcal G_b^\omega
  =\{G_1^\omega,\ldots,G_J^\omega\}.
\]
The parameter is fixed for one complete network and is therefore metadata,
not a per-input control.

\begin{assumption}[Uniform learned-dictionary regularity]
\label[assumption]{ass:learned-dictionary-candidate}
All dictionaries indexed by $\omega\in\Omega$ satisfy
\Cref{ass:regularity} with the same tube and constants.  In addition, for some
$L_\Omega<\infty$,
\begin{equation}
  \sup_{j,t,z\text{ in the tube}}
  \norm{G_j^\omega(t,z)-G_j^{\omega'}(t,z)}
  \le L_\Omega d_\Omega(\omega,\omega').
  \label{eq:dictionary-parameter-lipschitz-candidate}
\end{equation}
\end{assumption}

\begin{remark}[Atom alignment and exchangeable codebooks]
The metric $d_\Omega$ compares dictionaries with their atom labels aligned.
When labels are exchangeable, either fix a canonical alignment or replace
$d_\Omega$ by the permutation-invariant quotient pseudometric
\[
 \overline d_\Omega(\omega,\omega')
 :=\min_{\pi\in\mathfrak S_J}
   \max_j\sup_{t,\,z\text{ in the tube}}
   \norm{G_j^\omega(t,z)-G_{\pi(j)}^{\omega'}(t,z)}.
\]
Without one of these conventions, the parameter-Lipschitz assumption can
charge a harmless relabeling as a large metadata perturbation.
\end{remark}

Let $\RR_{\omega,\Rel}$ and $\RR_{\omega,D}$ be the relaxed and pure endpoint
sets for parameter $\omega$, and define
\begin{equation}
  \RR_{\Omega,\Rel}
  :=\overline{\bigcup_{\omega\in\Omega}\RR_{\omega,\Rel}}.
  \label{eq:learned-relaxed-union-candidate}
\end{equation}
Let $\varnothing\ne\Omega_s\subset\Omega$ be a finite metadata code with
$\abs{\Omega_s}\le2^s$ and covering radius
\begin{equation}
  \delta_s:=\sup_{\omega\in\Omega}
             \inf_{\widehat\omega\in\Omega_s}
             d_\Omega(\omega,\widehat\omega).
  \label{eq:metadata-covering-radius-candidate}
\end{equation}
Set
\begin{equation}
  \RR_{D,s}^{\mathrm{learn}}
  :=\bigcup_{\widehat\omega\in\Omega_s}\RR_{\widehat\omega,D}.
  \label{eq:learned-pure-union-candidate}
\end{equation}

\begin{lemma}[Lipschitz motion of relaxed endpoint sets]
\label[lemma]{lem:learned-endpoint-lipschitz-candidate}
Under \cref{ass:learned-dictionary-candidate},
\begin{equation}
  \dH(\RR_{\omega,\Rel},\RR_{\omega',\Rel})
  \le L_\Omega\Phi_{L_z}(T)d_\Omega(\omega,\omega').
  \label{eq:learned-relaxed-lipschitz-candidate}
\end{equation}
\end{lemma}

\begin{proof}
Drive the two systems by the same relaxed control.  Their state difference
satisfies
\[
  \norm{e(t)}
  \le\int_0^t\left(L_z\norm{e(s)}
        +L_\Omega d_\Omega(\omega,\omega')\right)ds.
\]
Gronwall's inequality gives the stated endpoint bound.  The same construction
works in both directions and extends to the closures.
\end{proof}

\begin{theorem}[Learned-codebook floor plus depth plus metadata]
\label[theorem]{thm:learned-codebook-candidate}
Under \cref{ass:learned-dictionary-candidate}, for every $D\ge D_0$,
\begin{equation}
  \dH(\RR_{D,s}^{\mathrm{learn}},\RR_{\Omega,\Rel})
  \le \frac{C_{\mathrm{fh}}^{\mathrm{unif}}}{D}
       +L_\Omega\Phi_{L_z}(T)\delta_s,
  \label{eq:learned-codebook-rate-candidate}
\end{equation}
where $C_{\mathrm{fh}}^{\mathrm{unif}}$ is the common finite-horizon constant.
Consequently, for every target $F^\star$,
\begin{equation}
  \left|\dist(F^\star,\RR_{D,s}^{\mathrm{learn}})
        -\dist(F^\star,\RR_{\Omega,\Rel})\right|
  \le \frac{C_{\mathrm{fh}}^{\mathrm{unif}}}{D}
       +L_\Omega\Phi_{L_z}(T)\delta_s.
  \label{eq:learned-codebook-target-candidate}
\end{equation}
\end{theorem}

\begin{proof}
Take an endpoint in $\RR_{\omega,\Rel}$ and choose
$\widehat\omega\in\Omega_s$ within $\delta_s$.  By
\cref{lem:learned-endpoint-lipschitz-candidate}, the endpoint is within
$L_\Omega\Phi_{L_z}(T)\delta_s$ of $\RR_{\widehat\omega,\Rel}$, and by
\cref{thm:hausdorff} the latter is within
$C_{\mathrm{fh}}^{\mathrm{unif}}/D$ of
$\RR_{\widehat\omega,D}$.  Passing to the closure proves one directed bound.
Conversely, every endpoint in $\RR_{D,s}^{\mathrm{learn}}$ is within
$C_{\mathrm{fh}}^{\mathrm{unif}}/D$ of the relaxed set having the same
metadata, which is contained in the union defining
$\RR_{\Omega,\Rel}$.  The target-distance statement follows from Hausdorff
stability.
\end{proof}

\begin{corollary}[Metadata-bit law]
\label[corollary]{cor:metadata-bit-law-candidate}
Suppose the metric entropy of $\Omega$ obeys
\begin{equation}
  \mathcal N(\Omega,d_\Omega,\delta)
  \le\left(\frac{C_\Omega}{\delta}\right)^m
  \label{eq:metadata-entropy-assumption-candidate}
\end{equation}
for the relevant range of $\delta$.  Then one may choose an $s$-bit metadata
code with
\begin{equation}
  \delta_s\le C_\Omega 2^{-s/m},
  \label{eq:metadata-radius-rate-candidate}
\end{equation}
and hence
\begin{equation}
  \dH(\RR_{D,s}^{\mathrm{learn}},\RR_{\Omega,\Rel})
  \le \frac{C_{\mathrm{fh}}^{\mathrm{unif}}}{D}
       +L_\Omega\Phi_{L_z}(T)C_\Omega2^{-s/m}.
  \label{eq:metadata-depth-law-candidate}
\end{equation}
Thus, depth and global codebook metadata have different approximation exponents
and should be optimized jointly.
\end{corollary}

\begin{theorem}[Metadata--depth packing converse]
\label[theorem]{thm:metadata-packing-candidate}
Fix an input-independent schedule class with $J$ atom choices per depth and at
most $2^s$ stored global metadata configurations.  If the resulting depth-$D$
class approximates every target in a family $\mathcal F$ to error at most
$\varepsilon$, then
\begin{equation}
  s+D\log_2J
  \ge\log_2\mathcal M(\mathcal F,2\varepsilon),
  \label{eq:metadata-packing-candidate}
\end{equation}
where $\mathcal M(\mathcal F,2\varepsilon)$ is the maximum cardinality of a
subset with pairwise distance strictly greater than $2\varepsilon$.
\end{theorem}

\begin{proof}
There are at most $2^sJ^D$ implemented maps.  Two targets separated by more
than $2\varepsilon$ cannot share one $\varepsilon$-accurate approximant.
Therefore
$\mathcal M(\mathcal F,2\varepsilon)\le2^sJ^D$, which is equivalent to
\eqref{eq:metadata-packing-candidate}.
\end{proof}

\begin{proposition}[Smooth adaptive codebooks]
\label[proposition]{prop:smooth-adaptive-codebook-candidate}
Suppose
\[
  G_j(t,z)=\widetilde G_j(t,z,\omega(t,z)),
\]
where $\widetilde G_j$ is $L_z^0$-Lipschitz in $z$, $L_\omega$-Lipschitz in
$\omega$, and $L_t^0$-Lipschitz in $t$, while the adaptive codebook map
$\omega(t,z)$ is $L_g$-Lipschitz in $z$ and $K_g$-Lipschitz in $t$.  Then the
composite field is Lipschitz with constants
\begin{equation}
  L_z\le L_z^0+L_\omega L_g,
  \qquad
  L_t\le L_t^0+L_\omega K_g.
  \label{eq:adaptive-composite-lipschitz-candidate}
\end{equation}
Provided the common tube and boundedness assumptions hold, all conclusions of
\cref{thm:hausdorff} apply.  Discrete or top-$k$ adaptation is not covered by
this proposition and belongs to the hybrid regime below.
\end{proposition}

\begin{proof}
Add and subtract
$\widetilde G_j(t,z',\omega(t,z))$ for the state bound and
$\widetilde G_j(s,z,\omega(t,z))$ for the time bound, then apply the stated
Lipschitz estimates.
\end{proof}

\subsubsection{Joint metadata--execution allocation under a declared charge}

Let $\ell_\sigma>0$ be the charged resource units per executed microstep.  It
may represent fixed-length schedule bits, operations, latency, or energy, but
it must be declared and used consistently.  Write
$E_{D,s}=\dist(F^\star,\RR_{D,s}^{\mathrm{learn}})$ and
$E_{\Omega,\infty}=\dist(F^\star,\RR_{\Omega,\Rel})$.

\begin{theorem}[Explicit first-order metadata--depth allocation]
\label[theorem]{thm:description-upper-v4}
Suppose
\begin{equation}
 \abs{E_{D,s}-E_{\Omega,\infty}}
 \le \frac{C_{\mathrm{syn}}}{D}+C_{\mathrm{meta}}2^{-s/m}
 \label{eq:description-start-v4}
\end{equation}
for integers $D\ge D_0$ and $s\ge0$.  Given total declared budget
$B_{\mathrm{tot}}>1$, set
\begin{align}
 s_B&:=\left\lceil m\log_2\!\left(
        \max\{1,C_{\mathrm{meta}}B_{\mathrm{tot}}\}\right)\right\rceil,\\
 D_B&:=\left\lfloor
       \frac{B_{\mathrm{tot}}-s_B}{\ell_\sigma}\right\rfloor.
 \label{eq:description-allocation-v4}
\end{align}
If $B_{\mathrm{tot}}\ge2(s_B+\ell_\sigma)$ and $D_B\ge D_0$, then
$s_B+\ell_\sigma D_B\le B_{\mathrm{tot}}$ and
\begin{equation}
 E_{D_B,s_B}
 \le E_{\Omega,\infty}
 +\frac{2C_{\mathrm{syn}}\ell_\sigma+1}{B_{\mathrm{tot}}}.
 \label{eq:description-upper-v4}
\end{equation}
When $C_{\mathrm{meta}}>0$,
\begin{equation}
 s_B=m\log_2B_{\mathrm{tot}}+O(1),
 \qquad
 D_B=\frac{B_{\mathrm{tot}}}{\ell_\sigma}
      -\frac{m}{\ell_\sigma}\log_2B_{\mathrm{tot}}+O(1).
 \label{eq:description-asymptotic-v4}
\end{equation}
Thus globally shared codebook metadata consumes only a logarithmic correction;
almost all asymptotic declared resource should be spent on executed depth.
\end{theorem}

\begin{proof}
The definition of $s_B$ gives
$C_{\mathrm{meta}}2^{-s_B/m}\le B_{\mathrm{tot}}^{-1}$.  The budget
condition implies
$D_B\ge B_{\mathrm{tot}}/(2\ell_\sigma)$, so
$C_{\mathrm{syn}}/D_B\le2C_{\mathrm{syn}}\ell_\sigma/B_{\mathrm{tot}}$.
Substitution proves \eqref{eq:description-upper-v4}; the asymptotics follow
directly from \eqref{eq:description-allocation-v4}.
\end{proof}

\begin{corollary}[Matching fixed-teacher law under per-step charging]
\label[corollary]{cor:description-lower-v4}
If the exposed binary mode is unchanged by metadata and
$s+\ell_\sigma D\le B_{\mathrm{tot}}$, the fixed teacher in
\cref{cor:exposed-fixed-target-candidate} satisfies
\begin{equation}
 E_D(F^\star)\ge\frac{a}{4eD}
 \ge\frac{a\ell_\sigma}{4eB_{\mathrm{tot}}}.
 \label{eq:description-lower-v4}
\end{equation}
The $B_{\mathrm{tot}}^{-1}$ exponent is therefore sharp under the same positive
per-step execution charge.  This is not a universal compressed-description
lower bound: a repeated schedule may have sublinear code length.
\end{corollary}

\begin{proof}
Use the fixed-target lower bound and
$D\le B_{\mathrm{tot}}/\ell_\sigma$.
\end{proof}

\subsection{Bounded-variation dictionaries and route-changing hybrid networks}
\label{sec:hybrid-routing-v2}

A positive margin on a connected interval freezes a continuous top-$k$ route.
Actual route changes therefore require two ingredients, in the spirit of
transversal-event sensitivity for hybrid flows \citep{SacconEtAl2014}: a synthesis theorem
that permits temporal jumps, and a target-specific perturbation theorem that
localizes switching errors near transverse events.

\begin{assumption}[C\`adl\`ag bounded-variation dictionary]
\label[assumption]{ass:bv-fields-v2}
Retain the boundedness, state-Lipschitz, and common-tube clauses of
\Cref{ass:regularity}.  Assume $\Z$ is Banach and the declared tube is contained in an open set on
which all fields are defined.  For each atom, regard
\[
 g_j(t):=G_j(t,\cdot)|_{\mathcal K^{\rho_{\mathcal K}}}
\]
as a map into the Banach space
$\mathcal B(\mathcal K^{\rho_{\mathcal K}};\Z)$ of bounded fields with the
uniform norm.  Assume that $g_j$ has a right-continuous representative of bounded
variation.  Let $\nu_j$ be its variation measure, so that for $0\le s<t\le T$,
\begin{equation}
 \sup_{z\in\mathcal K^{\rho_{\mathcal K}}}
 \norm{G_j(t,z)-G_j(s,z)}
 \le \nu_j((s,t]),
 \label{eq:bv-variation-measure-v2}
\end{equation}
and set
\begin{equation}
  V_t:=\sum_{j=1}^J\nu_j((0,T])<\infty.
  \label{eq:bv-total-variation-v2}
\end{equation}
The representative is used at Euler grid points; changing its value at
finitely many times does not change a Carath\'eodory trajectory.
\end{assumption}

\begin{lemma}[Well-posed relaxed flow under BV time dependence]
\label[lemma]{lem:bv-carath-v2}
Under \cref{ass:bv-fields-v2}, every measurable relaxed control
$p:[0,T]\to\Delta_J$ induces a unique absolutely continuous solution of
\begin{equation}
  \dot z(t)=\sum_jp_j(t)G_j(t,z(t)),\qquad z(0)=z_0,
  \label{eq:bv-relaxed-flow-v2}
\end{equation}
inside the declared tube.  It satisfies
$\norm{z(t)-z(s)}\le B|t-s|$.
\end{lemma}

\begin{proof}
A Banach-valued BV map has totally bounded, hence separable, range and admits
a strongly measurable representative.  Consequently the relaxed right-hand
side is strongly measurable in time, uniformly bounded by $B$, and
$L_z$-Lipschitz in state.  Picard iteration for the Bochner integral equation
gives local existence and uniqueness.  The declared forward tube and the speed
bound prevent finite-time escape and extend the solution to $[0,T]$.  Integrating the
speed bound gives the last claim.
\end{proof}

\begin{lemma}[Online simplex rounding]
\label[lemma]{lem:online-simplex-rounding-v4}
For every sequence $p_0,\ldots,p_{D-1}\in\Delta_J$, there is a sequence
$\sigma_0,\ldots,\sigma_{D-1}\in[J]$ such that, for every prefix $n\le D$
and every coordinate $j$,
\begin{equation}
 \left|\sum_{k=0}^{n-1}
  \bigl(\one_{\{\sigma_k=j\}}-p_{k,j}\bigr)\right|<J.
 \label{eq:online-rounding-prefix-v4}
\end{equation}
One online construction maintains $q_0=0$ and chooses
\begin{equation}
 \sigma_k\in\arg\max_j(q_{k,j}+p_{k,j}),
 \qquad q_{k+1}=q_k+p_k-e_{\sigma_k},
 \label{eq:online-rounding-alg-v4}
\end{equation}
where $e_j$ is the $j$th coordinate vector.
\end{lemma}

\begin{proof}
The case $J=1$ is immediate.  For $J\ge2$, $q_k$ is the negative cumulative
discrepancy and $\sum_jq_{k,j}=0$.  Inductively suppose
$-1<q_{k,j}<J$.  Put $r=q_k+p_k$, so $\sum_jr_j=1$ and $r_j>-1$.
The selected coordinate has $r_{\sigma_k}\ge1/J$, hence
$q_{k+1,\sigma_k}=r_{\sigma_k}-1>-1$, and
$r_{\sigma_k}<J+1$ gives the upper bound.  For an unselected coordinate,
$q_{k+1,j}=r_j>-1$.  If some unselected $r_j\ge J$, then also
$r_{\sigma_k}\ge J$; together with $r_\ell>-1$ on the remaining $J-2$
coordinates this would give
$\sum_\ell r_\ell>2J-(J-2)>1$, a contradiction.  Thus every unselected
$r_j<J$.  Finally,
$q_n=\sum_{k<n}(p_k-e_{\sigma_k})$, which is exactly
\eqref{eq:online-rounding-prefix-v4}.
\end{proof}

\begin{theorem}[Pure-to-relaxed approximation for BV fields]
\label[theorem]{thm:bv-reachability-v2}
Under \cref{ass:bv-fields-v2}, let $D_0=\lceil T/h_0\rceil$.  For every
$D\ge D_0$,
\begin{align}
  \dH(\RR_{b,D},\RR_{b,\Rel})
  &\le\frac{C_{\mathrm{BV},b}}{D},
  \label{eq:bv-endpoint-rate-v2}\\
  C_{\mathrm{BV},b}
  &:=Te^{L_zT}\left[
       V_t+\frac12L_zBT
       +J^2(B+L_zBT)+JV_t
     \right],
  \label{eq:bv-constant-v2}\\
  \dH^{\mathrm{path}}(\mathscr P_{b,D},\mathscr P_{b,\Rel})
  &\le\frac{C_{\mathrm{BV},b}+2BT}{D}.
  \label{eq:bv-path-rate-v2}
\end{align}
The same bounds hold for distances to arbitrary target endpoints and paths.
In particular, finitely many prescribed route changes alter the constant
through total variation but not the $D^{-1}$ exponent.
\end{theorem}

\begin{proof}
Let $h=T/D$, $I_k=[t_k,t_{k+1})$, and average a relaxed control on each
interval:
$p_{k,j}=h^{-1}\int_{I_k}p_j(t)\,dt$.  Let $z(t)$ be the relaxed solution and
$u_k$ the associated mixed Euler path.  The one-step defect is
\[
 \int_{I_k}\sum_jp_j(t)
 \left[G_j(t,z(t))-G_j(t_k,z(t_k))\right]dt.
\]
By \cref{lem:bv-carath-v2}, the state-motion part is at most
$L_zBh^2/2$.  By \eqref{eq:bv-variation-measure-v2}, the time part is at most
$h\sum_j\nu_j((t_k,t_{k+1}))$.  Therefore the sum of all local defects is at
most
\begin{equation}
  hV_t+\frac12L_zBTh.
  \label{eq:bv-total-local-defect-v2}
\end{equation}
Discrete Gronwall gives
\begin{equation}
  \max_{k\le D}\norm{u_k-z(t_k)}
  \le he^{L_zT}\left(V_t+\frac12L_zBT\right).
  \label{eq:bv-mixed-error-v2}
\end{equation}

Apply \cref{lem:online-simplex-rounding-v4} to $p_k$.  With
$G_{j,k}=G_j(t_k,u_k)$, Abel summation and the prefix discrepancy bound give
\begin{align}
 \sup_{n\le D}\left\|\sum_{k<n}\sum_j
   (\one_{\{\sigma_k=j\}}-p_{k,j})G_{j,k}\right\|
 &\le \sum_jJ\left[B+
       \sum_{k=0}^{D-2}\norm{G_{j,k+1}-G_{j,k}}\right]\\
 &\le J^2(B+L_zBT)+JV_t
 =:W_{\mathrm{BV}}.
 \label{eq:bv-prefix-forcing-v2}
\end{align}
Indeed,
\[
 \sum_k\norm{G_j(t_{k+1},u_{k+1})-G_j(t_k,u_k)}
 \le L_zBT+\nu_j((0,T]).
\]
For completeness, put
\[
 f_k:=\sum_j(\one_{\{\sigma_k=j\}}-p_{k,j})G_{j,k},
 \qquad F_n:=\sum_{k<n}f_k,
\]
so \eqref{eq:bv-prefix-forcing-v2} gives
$\max_n\norm{F_n}\le W_{\mathrm{BV}}$.  If
$e_k=z_k-u_k$ and $\widetilde e_k=e_k-hF_k$, then the pure and mixed Euler
recursions imply
\[
 \norm{\widetilde e_{k+1}}
 \le(1+hL_z)\norm{\widetilde e_k}+h^2L_zW_{\mathrm{BV}}.
\]
Iteration from $\widetilde e_0=0$, followed by
$e_k=\widetilde e_k+hF_k$, yields
\begin{equation}
  \max_{k\le D}\norm{z_k-u_k}
  \le hW_{\mathrm{BV}}(1+hL_z)^D
  \le hW_{\mathrm{BV}}e^{L_zT}.
  \label{eq:bv-pure-mixed-v2}
\end{equation}
Combining \eqref{eq:bv-mixed-error-v2} and
\eqref{eq:bv-pure-mixed-v2} proves one directed endpoint estimate.

For the reverse direction, take the vertex-valued relaxed control that equals
$e_{\sigma_k}$ on $I_k$.  Its mixed Euler path is the given pure path, so
\eqref{eq:bv-mixed-error-v2} places that pure endpoint within the same bound
of a relaxed endpoint.  Passing to the endpoint closures preserves both
directed distances: approximate a closed-set point by raw relaxed endpoints
and use continuity of distance to the finite pure set.  The identical argument
works for path closures.  Both the relaxed path and the linear Euler
interpolant move by at most $Bh$ inside one cell, adding at most $2Bh$ and
proving \eqref{eq:bv-path-rate-v2}.
\end{proof}

We now formalize actual top-$k$ changes.  Let $N_e$ be the number of experts,
let $\TopK_k(s)$ denote a deterministic top-$k$ support with a fixed tie rule,
and let $m^\star(t)$ be a right-continuous target support with isolated event
times $0<\tau_1<\cdots<\tau_R<T$.  The target support is required to be
self-consistent: away from event ties,
$m^\star(t)=\TopK_k(s(t,z^\star(t)))$, and the prescribed tie rule supplies
its value at each event.  At event $r$, one expert $a_r$ enters and
one expert $b_r$ leaves, so the adjacent supports $m_r^-$ and $m_r^+$ differ
by one swap.  Let $s(t,z)\in\R^{N_e}$ be the ideal score vector and
$\widehat s(t,z)$ its implementation.

\begin{definition}[Two-mode isolation radius]
\label[definition]{def:two-mode-isolation-v2}
An event neighborhood $U_r$ has isolation radius $\Gamma_r>0$ if every score
vector $\widetilde s$ satisfying
$\norm{\widetilde s-s(t,z^\star(t))}_\infty<\Gamma_r$ has top-$k$ support in
$\{m_r^-,m_r^+\}$, and within those two modes the selected support is
$m_r^+$ exactly when
$\widetilde s_{a_r}>\widetilde s_{b_r}$.  Outside the event neighborhoods,
$\Gamma_{\mathrm{out}}>0$ is an isolation radius if the same perturbation
preserves the complete target support.
\end{definition}

\begin{lemma}[A checkable sufficient top-$k$ isolation condition]
\label[lemma]{lem:topk-isolation-v2}
The event isolation property holds with radius $\Gamma_r$ if, along the target
path in $U_r$, every ranking inequality needed to distinguish experts other
than the pair $(a_r,b_r)$ has gap strictly larger than $2\Gamma_r$.  The
away-event property holds if the target $k$th-to-$(k+1)$st score gap is
strictly larger than $2\Gamma_{\mathrm{out}}$.
\end{lemma}

\begin{proof}
An $\ell_\infty$ score perturbation of radius $\Gamma$ changes every pairwise
gap by less than $2\Gamma$.  Hence all ranking inequalities with larger
margin retain their sign.  At an event, only the order of $a_r$ and $b_r$ can
change, which gives exactly the two adjacent supports.
\end{proof}

For one shared atom schedule $\sigma$, take the common initial state
$x_0=y_0=z^\star(0)$ and define the target-route oracle and the actual routed
recursion by
\begin{align}
  y_{k+1}
  &=y_k+hG_{\sigma_k}^{m^\star(t_k)}(t_k,y_k),
  \label{eq:oracle-route-recursion-v2}\\
  x_{k+1}
  &=x_k+hG_{\sigma_k}^{\widehat m(t_k,x_k)}(t_k,x_k),
  \qquad
  \widehat m(t,z)=\TopK_k(\widehat s(t,z)).
  \label{eq:actual-route-recursion-v2}
\end{align}

\begin{assumption}[Simple isolated top-$k$ swaps]
\label[assumption]{ass:simple-topk-events-v2}
There are disjoint neighborhoods
$U_r=(\tau_r-\ell_r,\tau_r+\ell_r)$ and constants
$\nu_r,L_r,\eta_{r,D},\Delta_r>0$ such that:
\begin{enumerate}[label=(\roman*),leftmargin=2em]
\item the target pair gap
$t\mapsto\gamma_r(t,z^\star(t))$, where
$\gamma_r(t,z)=s_{a_r}(t,z)-s_{b_r}(t,z)$, is continuous on
$\overline U_r$, satisfies
$\gamma_r(\tau_r,z^\star(\tau_r))=0$, and obeys
\begin{align}
 \gamma_r(t,z^\star(t))&\le-\nu_r(\tau_r-t),
 &&t\in(\tau_r-\ell_r,\tau_r),\\
 \gamma_r(t,z^\star(t))&\ge \nu_r(t-\tau_r),
 &&t\in(\tau_r,\tau_r+\ell_r);
 \label{eq:signed-transversality-v2}
\end{align}
\item the implemented pair gap obeys
\begin{equation}
 \abs{\widehat\gamma_r(t,z)-\gamma_r(t,z^\star(t))}
 \le \eta_{r,D}+L_r\norm{z-z^\star(t)};
 \label{eq:route-gap-error-v2}
\end{equation}
\item the full score perturbation obeys
\begin{equation}
 \norm{\widehat s(t,z)-s(t,z^\star(t))}_\infty
 \le \eta_{0,D}+L_0\norm{z-z^\star(t)};
 \label{eq:full-score-error-v2}
\end{equation}
\item the target scores have event isolation radii $\Gamma_r$ and away-event
radius $\Gamma_{\mathrm{out}}$ as in
\cref{def:two-mode-isolation-v2};
\item every routed atom is $L_z$-Lipschitz on the common route tube and
\begin{equation}
  \sup_{j,t,z\text{ in the tube}}
  \norm{G_j^{m_r^+}(t,z)-G_j^{m_r^-}(t,z)}
  \le\Delta_r.
  \label{eq:route-field-jump-v2}
\end{equation}
\end{enumerate}
Events are one-at-a-time; simultaneous swaps require a separate
multi-surface statement.
\end{assumption}

\begin{theorem}[Transversal top-$k$ changes have shrinking mismatch windows]
\label[theorem]{thm:route-window-v2}
Suppose \cref{ass:simple-topk-events-v2} holds and the oracle satisfies
\begin{equation}
  \max_{k\le D}\norm{y_k-z^\star(t_k)}\le\delta_D.
  \label{eq:oracle-route-error-v2}
\end{equation}
Let $S=e^{L_zT}$ and define
\begin{align}
  \chi
  &:=2S\sum_{r=1}^R\frac{\Delta_rL_r}{\nu_r},
  \label{eq:route-small-gain-v2}\\
  b_D
  &:=\delta_D
      +2S\sum_{r=1}^R\frac{\Delta_r\eta_{r,D}}{\nu_r}
      +2Sh\sum_{r=1}^R\Delta_r,
  \label{eq:route-forcing-v2}\\
  \overline E_D
  &:=\frac{b_D}{1-\chi},
  \label{eq:route-envelope-v2}\\
  w_{r,D}
  &:=\frac{\eta_{r,D}+L_r\overline E_D}{\nu_r}.
  \label{eq:route-window-v2}
\end{align}
Assume $\chi<1$, every $w_{r,D}<\ell_r$, the resulting time windows are
disjoint, every oracle and actual prefix remains in the common route tube, and
\begin{equation}
  \eta_{0,D}+L_0\overline E_D
  <\min\{\Gamma_{\mathrm{out}},\Gamma_1,\ldots,\Gamma_R\}.
  \label{eq:route-isolation-budget-v2}
\end{equation}
Then
\begin{equation}
  \max_{k\le D}\norm{x_k-z^\star(t_k)}\le\overline E_D,
  \label{eq:route-state-bound-v2}
\end{equation}
and the actual and target supports agree at every grid time outside
\begin{equation}
  \bigcup_{r=1}^R
  [\tau_r-w_{r,D},\tau_r+w_{r,D}].
  \label{eq:route-window-union-v2}
\end{equation}
The theorem localizes route error without assuming that the implemented
router switches only once inside a window.
\end{theorem}

\begin{proof}
The common initialization $x_0=y_0=z^\star(0)$ starts an induction.  Assume
that the state discrepancy and route-tube membership hold through grid index
$n$.  At each $k\le n$, the state bound inserted into
\eqref{eq:route-gap-error-v2} and \eqref{eq:full-score-error-v2}, together
with signed transversality and the strict isolation radii, proves route
agreement outside the certified event windows and excludes all nonadjacent
supports.  The boundary-clipped number of mismatched grid points for event
$r$ through index $n$ is at most $2w_{r,D}/h+2$.

For the update to $n+1$, matched modes contribute only the state-Lipschitz
factor, whereas a mismatch in window $r$ adds at most $h\Delta_r$.  Discrete
Gronwall applied to the forcing through step $n$ gives
\begin{align}
 \max_{k\le n+1}\norm{x_k-y_k}
 &\le S\sum_{r=1}^R\Delta_r(2w_{r,D}+2h)\\
 &=2S\sum_r\frac{\Delta_r\eta_{r,D}}{\nu_r}
   +\chi\overline E_D+2Sh\sum_r\Delta_r.
 \label{eq:route-gronwall-v2}
\end{align}
Adding the oracle error yields the state bound at $n+1$.  The declared common
route tube contains the new state and all interpolation segments used in the
estimate, so the induction closes.  The sign and isolation comparison then
proves \eqref{eq:route-window-union-v2}.
\end{proof}

\begin{corollary}[First-order hard-routing law]
\label[corollary]{cor:first-order-routing-v2}
If
\begin{equation}
  \delta_D\le\frac{C_\delta}{D},
  \qquad
  \eta_{r,D}\le\frac{C_{\eta,r}}{D},
  \qquad h=\frac TD,
  \label{eq:first-order-router-input-v2}
\end{equation}
and the $D$-independent small-gain and isolation hypotheses of
\cref{thm:route-window-v2} hold, then
\begin{align}
 \max_{k\le D}\norm{x_k-z^\star(t_k)}
 &\le \frac{C_{\mathrm{route}}}{D},\\
 w_{r,D}
 &\le \frac{C_{\eta,r}+L_rC_{\mathrm{route}}}{\nu_rD},
 \label{eq:first-order-router-output-v2}
\end{align}
where
\begin{equation}
 C_{\mathrm{route}}
 =\frac{C_\delta
 +2S\sum_r\Delta_rC_{\eta,r}/\nu_r
 +2ST\sum_r\Delta_r}{1-\chi}.
 \label{eq:route-constant-v2}
\end{equation}
Thus actual route changes preserve the synthesis exponent when oracle and
score errors are first order and the hybrid feedback gain is below one.
\end{corollary}

\begin{corollary}[Unique event time under a.e. derivative separation]
\label[corollary]{cor:event-time-v4}
Under \cref{thm:route-window-v2}, suppose all routed fields and the target path
have speed at most $B$, and let $\widetilde x_D$ be the linear interpolant of
the actual Euler states.  Then
\begin{equation}
 \sup_{t\in[0,T]}\norm{\widetilde x_D(t)-z^\star(t)}
 \le \widetilde E_D:=\overline E_D+2Bh.
 \label{eq:event-continuous-state-v4}
\end{equation}
Define
$\widehat g_r(t)=\widehat\gamma_r(t,\widetilde x_D(t))$,
$g_r^\star(t)=\gamma_r(t,z^\star(t))$, and
\begin{equation}
 \widetilde w_{r,D}
 :=\frac{\eta_{r,D}+L_r\widetilde E_D}{\nu_r}.
 \label{eq:event-continuous-window-v4}
\end{equation}
Assume $\widetilde w_{r,D}<\ell_r$, both gap functions are continuous and
absolutely continuous on $U_r$, and there are an orientation
$\varsigma_r\in\{-1,1\}$ and constants $\kappa_r>\zeta_r>0$ such that, for
almost every $t\in U_r$,
\begin{equation}
 \varsigma_r g_r^{\star\prime}(t)\ge\kappa_r,
 \qquad
 \abs{\widehat g_r'(t)-g_r^{\star\prime}(t)}
 \le\kappa_r-\zeta_r.
 \label{eq:event-derivative-separation-v4}
\end{equation}
Then $\widehat g_r$ is strictly monotone, has exactly one zero
$\widehat\tau_r$, and
\begin{equation}
 \abs{\widehat\tau_r-\tau_r}\le\widetilde w_{r,D}.
 \label{eq:event-time-bound-v4}
\end{equation}
The hypotheses allow the Euler interpolant to have derivative jumps at grid
points.  Without an anti-chattering condition such as
\eqref{eq:event-derivative-separation-v4}, a uniformly small score error need
not imply a unique implemented event.
\end{corollary}

\begin{proof}
Both the target path and each linear Euler interpolant move by at most $Bh$
inside one cell, proving \eqref{eq:event-continuous-state-v4}.  The pair-gap
perturbation and signed transversality imply sign agreement whenever
$\abs{t-\tau_r}>\widetilde w_{r,D}$.  From
\eqref{eq:event-derivative-separation-v4},
$\varsigma_r\widehat g_r'(t)\ge\zeta_r$ almost everywhere.  Absolute
continuity therefore makes $\varsigma_r\widehat g_r$ strictly increasing.
Its one-sided signs outside the closed localized window are opposite, so
continuity gives a zero in that window; strict monotonicity makes it unique.
\end{proof}

\begin{theorem}[Continuum-input hard routing in $L^p$]
\label[theorem]{thm:continuum-routing-lp-v2}
Let $(\Xi,\mu)$ be a finite measure space, fix $1\le p\le\infty$, and
suppose the routed network is
input separable: for almost every $\xi$, its state and router depend on
$z(t,\xi)$ and the same shared atom schedule $\sigma$, but not on other
inputs.  Assume that for almost every $\xi$ the hypotheses of
\cref{thm:route-window-v2} hold with measurable input-dependent quantities
$R(\xi)$, $\delta_D(\xi)$, $\eta_{r,D}(\xi)$,
$\nu_r(\xi)$, $L_r(\xi)$, and $\Delta_r(\xi)$.  Define pointwise
$\chi(\xi)$, $b_D(\xi)$, and $\overline E_D(\xi)$ by
\eqref{eq:route-small-gain-v2}--\eqref{eq:route-envelope-v2}, and suppose
\begin{equation}
  \mathop{\mathrm{ess\,sup}}_{\xi\in\Xi}\chi(\xi)<1,
  \qquad
  \overline E_D\in L^p(\mu),
  \qquad R\in L^1(\mu),
  \qquad W_D:=\sum_{r=1}^{R(\cdot)}w_{r,D}\in L^1(\mu).
  \label{eq:lp-routing-small-gain-v2}
\end{equation}
Then, even when $\xi\mapsto m(t,\xi)$ and
$\xi\mapsto z(t,\xi)$ are discontinuous,
\begin{equation}
 \max_{k\le D}
 \norm{x_k-z^\star(t_k)}_{L^p(\mu)}
 \le
 \left\|\frac{b_D}{1-\chi}\right\|_{L^p(\mu)}.
 \label{eq:lp-routing-state-v2}
\end{equation}
Moreover, with $\mathcal M_k$ the set of inputs whose support is mismatched at
$t_k$,
\begin{equation}
 h\sum_{k=0}^{D-1}\mu(\mathcal M_k)
 \le
 2\int_\Xi\sum_{r=1}^{R(\xi)}w_{r,D}(\xi)\,d\mu(\xi)
 +2h\int_\Xi R(\xi)\,d\mu(\xi).
 \label{eq:lp-routing-occupancy-v2}
\end{equation}
This provides a mathematically valid continuum-input replacement for the
uniform $C(\Xi)$ formulation, which hard routing need not preserve.
\end{theorem}

\begin{proof}
Apply \cref{thm:route-window-v2} pointwise for almost every input and then
take the $L^p$ norm.  For the second claim, the pointwise number of mismatched
grid cells is at most
$\sum_r(2w_{r,D}(\xi)/h+2)$; multiply by $h$ and integrate.  Measurability is
part of the hypotheses.
\end{proof}

\begin{corollary}[Distributional top-$k$ margin law]
\label[corollary]{cor:margin-density-v2}
At a fixed time, let
$\Gamma^\star(\xi)$ be the target gap between its $k$th and $(k+1)$st scores.
Suppose
\begin{equation}
 \mu\{\xi:\Gamma^\star(\xi)\le u\}
 \le C_\Gamma u^\alpha,
 \qquad 0\le u\le u_0,
 \label{eq:margin-density-v2}
\end{equation}
and
$\norm{\widehat s(\xi)-s^\star(\xi)}_\infty\le\varepsilon\le u_0/2$
almost everywhere.  Then
\begin{equation}
 \mu\{\xi:\TopK_k(\widehat s(\xi))
              \ne\TopK_k(s^\star(\xi))\}
 \le C_\Gamma(2\varepsilon)^\alpha.
 \label{eq:margin-route-measure-v2}
\end{equation}
For $1\le p<\infty$, suppose the two routed fields agree whenever their
supports agree and differ by at most $\Delta$ pointwise whenever their
supports differ.  Their $L^p$ routing perturbation is then at most
\begin{equation}
  \Delta\,C_\Gamma^{1/p}(2\varepsilon)^{\alpha/p}.
  \label{eq:margin-field-lp-v2}
\end{equation}
For $p=\infty$, the corresponding worst-case bound is simply $\Delta$; a
distributional margin assumption does not improve an essential-supremum norm.
\end{corollary}

\begin{proof}
A top-$k$ support can change under an $\ell_\infty$ score perturbation of size
$\varepsilon$ only if the target boundary gap is at most $2\varepsilon$.
Apply \eqref{eq:margin-density-v2}; the field bound follows by integrating the
indicator of the mismatch set.
\end{proof}

\begin{proposition}[Checkable weighted top-$k$ MoE constants]
\label[proposition]{prop:topk-moe-constants-v4}
Let expert fields satisfy
\begin{equation}
 \sup_{e,t,z}\norm{E_e(t,z)}\le B_E,
 \qquad
 \norm{E_e(t,z)-E_e(t,z')}\le L_E\norm{z-z'}.
 \label{eq:moe-expert-bounds-v4}
\end{equation}
For each fixed active support $S$, let the zero-padded normalized gate vector
$\pi^S(t,z)$ be supported on $S$, nonnegative, sum to one, and obey
\begin{equation}
 \norm{\pi^S(t,z)-\pi^S(t,z')}_1
 \le L_\pi\norm{z-z'}.
 \label{eq:moe-gate-lipschitz-v4}
\end{equation}
For
$G^S(t,z)=\sum_e\pi_e^S(t,z)E_e(t,z)$, one may take
\begin{equation}
 L_z\le L_E+B_EL_\pi.
 \label{eq:moe-field-lipschitz-v4}
\end{equation}
For a one-expert swap between $S_r^-$ and $S_r^+$,
\begin{equation}
 \Delta_r\le B_E\Delta_{\pi,r},
 \qquad
 \Delta_{\pi,r}:=
 \sup_{U_r\times\mathcal K}
 \norm{\pi^{S_r^+}-\pi^{S_r^-}}_1\le2.
 \label{eq:moe-field-jump-v4}
\end{equation}
If, on a common support, $\widehat\pi^S$ is also nonnegative and
normalized, $\sup_e\norm{\widehat E_e-E_e}\le\eta_E$, and
$\norm{\widehat\pi^S-\pi^S}_1\le\eta_\pi$, then
\begin{equation}
 \norm{\widehat G^S-G^S}\le\eta_E+B_E\eta_\pi.
 \label{eq:moe-field-error-v4}
\end{equation}
Finally, per-expert score error $\eta_{\mathrm{score}}$ and score
state-Lipschitz constant $L_{\mathrm{score}}$ imply
\begin{equation}
 \eta_{r,D}\le2\eta_{\mathrm{score}},
 \qquad
 L_r\le2L_{\mathrm{score}}.
 \label{eq:moe-router-gap-v4}
\end{equation}
Thus every constant in the route-window theorem can be audited from expert,
gate, and router quantities.
\end{proposition}

\begin{proof}
For a fixed support, insert and subtract
$\sum_e\pi_e^S(t,z)E_e(t,z')$ and use that the weights sum to one, obtaining
$L_E+B_EL_\pi$.  At a route change, use the same expert values and bound the
zero-padded weight difference in $\ell^1$.  For implementation error, insert
$\sum_e\widehat\pi_e^SE_e$ and use normalization.  A pair gap is the
difference of two scores, giving the final factor two.
\end{proof}

\begin{corollary}[Top-$1$ and uniformly averaged top-$k$ route constants]
\label[corollary]{cor:uniform-topk-constants-v4}
Under \cref{prop:topk-moe-constants-v4}, top-$1$ routing permits
\begin{equation}
 L_z\le L_E,\qquad \Delta_r\le2B_E.
 \label{eq:top1-constants-v4}
\end{equation}
For the uniform top-$k$ average
$G^S=k^{-1}\sum_{e\in S}E_e$ and a one-expert swap,
\begin{equation}
 L_z\le L_E,\qquad \Delta_r\le\frac{2B_E}{k}.
 \label{eq:uniform-topk-constants-v4}
\end{equation}
In the latter case the hybrid small-gain coefficient obeys
\begin{equation}
 \chi
 \le \frac{8e^{L_ET}B_EL_{\mathrm{score}}}{k}
       \sum_{r=1}^R\frac1{\nu_r}.
 \label{eq:uniform-topk-small-gain-v4}
\end{equation}
This is a directly checkable sufficient condition for stable route changes.
\end{corollary}

\begin{proof}
The fixed-support weights are constant.  Two one-hot vectors have $\ell^1$
distance two; two uniform top-$k$ vectors differing by one swap have distance
$2/k$.  Substitute \eqref{eq:moe-router-gap-v4} and
\eqref{eq:uniform-topk-constants-v4} into
\eqref{eq:route-small-gain-v2}.
\end{proof}

\subsection{Scalable primal--dual certificates for the structural floor}
\label{sec:floor-certificates-v2}

The next results use a finite-dimensional joint state $z\in\R^n$, for example
the product state induced by a finite witness set.  Let $K$ be a compact tube
containing every relaxed trajectory, let $\phi:K\to\R$ be a continuous
terminal loss, and define
\begin{equation}
  E_\infty^\phi
  :=\inf_{p(\cdot)}\phi(z_p(T)).
  \label{eq:general-floor-objective-v2}
\end{equation}

\begin{theorem}[HJB subsolution lower certificate]
\label[theorem]{thm:hjb-lower-v2}
Let $v$ be the restriction to $[0,T]\times K$ of a $C^1$ function on an
open neighborhood and satisfy
\begin{align}
  v(T,z)&\le\phi(z), &&z\in K,
  \label{eq:hjb-terminal-v2}\\
  \partial_tv(t,z)+\nabla_zv(t,z)^\top G_j(t,z)&\ge0,
  &&(t,z)\in[0,T]\times K,\quad j\in[J].
  \label{eq:hjb-atom-v2}
\end{align}
Then
\begin{equation}
  v(0,z_0)\le E_\infty^\phi.
  \label{eq:hjb-value-v2}
\end{equation}
The conclusion remains valid for the closed relaxed endpoint set.
\end{theorem}

\begin{proof}
Along every relaxed trajectory,
\[
 \frac d{dt}v(t,z_p(t))
 =\partial_tv+\sum_jp_j\nabla v^\top G_j\ge0
\]
for almost every $t$.  Hence
$v(0,z_0)\le v(T,z_p(T))\le\phi(z_p(T))$.  Take the infimum; continuity
extends the inequality to endpoint closures.
\end{proof}

\begin{proposition}[Support-function witness for full-map lower bounds]
\label[proposition]{prop:support-witness-v2}
Let $\lambda\in\Z^*$ with $\norm{\lambda}_*\le1$.  Then
\begin{equation}
 E_{\infty,b}^{\mathrm{end}}(F^\star)
 \ge \inner{\lambda}{F^\star}
      -\sup_{y\in\RR_{b,\Rel}}\inner{\lambda}{y}.
 \label{eq:support-basic-v2}
\end{equation}
If $w$ is the restriction of a $C^1$ function on an open neighborhood of
$[0,T]\times K$ and satisfies
\begin{align}
 w(T,z)&\ge\inner{\lambda}{z},
 \label{eq:support-terminal-v2}\\
 \partial_tw(t,z)+\nabla_zw(t,z)^\top G_j(t,z)&\le0
 \quad\text{for every }j,
 \label{eq:support-hjb-v2}
\end{align}
then
\begin{equation}
 E_{\infty,b}^{\mathrm{end}}(F^\star)
 \ge\inner{\lambda}{F^\star}-w(0,z_0).
 \label{eq:support-certified-v2}
\end{equation}
For $\Z=C(\Xi;\R^q)$, every finite signed point-evaluation functional
\[
 \lambda(F)=\sum_{i=1}^Na_i^\top F(\xi_i),
 \qquad \sum_i\norm{a_i}_2\le1,
\]
is admissible.
\end{proposition}

\begin{proof}
For every relaxed endpoint $y$,
$\norm{F^\star-y}\ge\inner{\lambda}{F^\star-y}$, which gives the first
claim.  The differential inequality makes $w$ nonincreasing along every
relaxed trajectory, so
$\inner{\lambda}{z_p(T)}\le w(T,z_p(T))\le w(0,z_0)$.  Substitute this
support bound into \eqref{eq:support-basic-v2}.
\end{proof}

\begin{theorem}[Exact common-drift affine floor dual]
\label[theorem]{thm:affine-floor-dual-v5}
Let the finite-dimensional relaxed atoms have one common linear drift,
\begin{equation}
 G_j(t,z)=A(t)z+b_j(t),\qquad j\in[J],
 \label{eq:common-drift-affine-v5}
\end{equation}
where $A$ is integrable, the $b_j$ are integrable, and $\Phi(t,s)$ is the
fundamental matrix of $\dot z=A(t)z$.  Let $P:\R^n\to\R^q$ be linear, set
\begin{equation}
 a=P\Phi(T,0)z_0,\qquad
 c_j(t)=P\Phi(T,t)b_j(t),
 \label{eq:affine-output-data-v5}
\end{equation}
and let $\mathcal Y_{\Rel}$ be the closed set of relaxed outputs $Pz_p(T)$.
Then $\mathcal Y_{\Rel}$ is compact and convex and, for every
$\lambda\in\R^q$,
\begin{equation}
 h_{\mathcal Y_{\Rel}}(\lambda)
 =\inner{\lambda}{a}
  +\int_0^T\max_{j\in[J]}\inner{\lambda}{c_j(t)}\,dt.
 \label{eq:affine-support-v5}
\end{equation}
Consequently, for every target $y^\star$ and every norm with dual norm
$\norm{\cdot}_*$,
\begin{equation}
 \dist(y^\star,\mathcal Y_{\Rel})
 =\max_{\norm{\lambda}_*\le1}
 \left\{
   \inner{\lambda}{y^\star-a}
   -\int_0^T\max_j\inner{\lambda}{c_j(t)}\,dt
 \right\}.
 \label{eq:affine-exact-distance-v5}
\end{equation}
Thus the structural floor is the value of a finite-dimensional concave dual
problem, without a state-space HJB discretization.
\end{theorem}

\begin{proof}
Variation of constants gives
\[
 Pz_p(T)=a+\int_0^T\sum_jp_j(t)c_j(t)\,dt.
\]
The integrand ranges over the compact convex hull of
$\{c_1(t),\ldots,c_J(t)\}$.  The Aumann integral of this integrably bounded
compact convex multifunction is compact and convex.  Maximizing a linear
functional can be done pointwise: measurable selection is immediate because
there are finitely many atoms and ties may be broken by the smallest index.
This proves \eqref{eq:affine-support-v5}.  For a nonempty closed convex set
$C$ in finite dimensions, Fenchel duality for the norm distance gives
\[
 \dist(y,C)=\max_{\norm{\lambda}_*\le1}
 \{\inner{\lambda}{y}-h_C(\lambda)\}.
\]
Substituting the support formula proves
\eqref{eq:affine-exact-distance-v5}.
\end{proof}

\begin{corollary}[Rational finite-witness LP certificate]
\label[corollary]{cor:rational-affine-lp-v5}
In \cref{thm:affine-floor-dual-v5}, suppose $T=1$, $A=0$, $z_0=0$,
$P=I$, $b_j\in\mathbb Q^q$ are constant, and the target
$y^\star\in\mathbb Q^q$ is represented exactly.  Put
$B=[b_1\ \cdots\ b_J]$.  In the $\ell^\infty$ output norm, the exact floor is
both the primal and dual LP value
\begin{align}
 E_{\mathrm{aff}}
 &=\min_{\alpha\ge0,\,\one^\top\alpha=1}
      \norm{B\alpha-y^\star}_\infty,
 \label{eq:affine-primal-lp-v5}\\
 &=\max_{\norm{\lambda}_1\le1}
      \left\{\lambda^\top y^\star-
                 \max_j\lambda^\top b_j\right\}.
 \label{eq:affine-dual-lp-v5}
\end{align}
For any rational feasible $\widehat\alpha$ and $\widehat\lambda$, define
\begin{equation}
 U=\norm{B\widehat\alpha-y^\star}_\infty,
 \qquad
 L=\widehat\lambda^\top y^\star-
       \max_j\widehat\lambda^\top b_j.
 \label{eq:rational-affine-bracket-v5}
\end{equation}
Then $L\le E_{\mathrm{aff}}\le U$ exactly.  All inequalities can be checked with integer arithmetic after clearing the
denominators of $B$, $y^\star$, $\widehat\alpha$, and
$\widehat\lambda$.  In a finite-witness
full-map lift, $\widehat\lambda$ is a sparse signed combination of coordinate
point evaluations and is therefore an interpretable obstruction witness.
\end{corollary}

\begin{proof}
The relaxed endpoint set is $\conv\{b_1,\ldots,b_J\}$.  The first display is
the epigraph LP for distance to that polytope.  The second follows from
\eqref{eq:affine-exact-distance-v5}, because the dual of $\ell^\infty$ is
$\ell^1$ and the support function of the convex hull is
$\max_j\lambda^\top b_j$.  Weak duality applied to any feasible pair gives
the exact rational bracket.
\end{proof}

\begin{theorem}[Certified nonlinear enclosure around an affine dictionary]
\label[theorem]{thm:affine-remainder-enclosure-v5}
Let $\widetilde G_j(t,z)=A(t)z+b_j(t)$ be the affine atoms of
\cref{thm:affine-floor-dual-v5}, and let $G_j$ be actual atoms on a common tube.
Assume every actual and surrogate relaxed trajectory driven by the same
control remains in that tube, and let $\eta:[0,T]\to[0,\infty)$ be measurable
with $\eta\in L^1([0,T])$ such that
\begin{equation}
 \sup_{z\text{ in the tube}}
 \norm{G_j(t,z)-\widetilde G_j(t,z)}
 \le\eta(t)
 \quad\text{for a.e. }t,\quad j\in[J].
 \label{eq:affine-remainder-v5}
\end{equation}
Then their closed relaxed output sets satisfy
\begin{equation}
 \dH(\mathcal Y_{\Rel},\widetilde{\mathcal Y}_{\Rel})
 \le R_\eta
 :=\int_0^T\norm{P\Phi(T,t)}_{\mathrm{op}}\eta(t)\,dt.
 \label{eq:affine-remainder-radius-v5}
\end{equation}
Hence, for every target,
\begin{equation}
 \abs{E_\infty-\widetilde E_{\mathrm{aff}}}\le R_\eta.
 \label{eq:affine-floor-enclosure-v5}
\end{equation}
In particular, an exact affine dual lower certificate $L$ yields the rigorous
nonlinear lower bound $E_\infty\ge\max\{0,L-R_\eta\}$, while an affine primal
upper witness $U$ yields $E_\infty\le U+R_\eta$.
\end{theorem}

\begin{proof}
Pair actual and surrogate trajectories by the same relaxed control.  Their
difference satisfies
\[
 e(T)=\int_0^T\Phi(T,t)r(t)\,dt,
 \qquad \norm{r(t)}\le\eta(t),
\]
after subtracting the common linear drift.  Applying $P$ proves the uniform
pairing radius in one direction, and the same control pairing proves the
reverse direction.  Take closures and use the one-Lipschitz dependence of
distance to a target on Hausdorff perturbations.
\end{proof}

\begin{remark}[Why this is the scalable certificate tier]
The affine dual is exact for linear residual channels and for finite-witness
libraries whose atom outputs are state independent.  For a nonlinear
Transformer or MLP block, interval, Lipschitz, or verified linearization bounds
can supply $\eta(t)$ in \eqref{eq:affine-remainder-v5}.  The result then gives
a formally checkable lower floor without solving a high-dimensional HJB PDE.
Occupation measures and SOS remain the more general nonlinear tier.
\end{remark}

\paragraph{A verified 128-dimensional block certificate.}
The accompanying exact artifact evaluates sixteen signed 4-bit $8\times8$
linear residual atoms on sixteen width-8 witness inputs, giving a lifted
full-map dimension of $128$.  A high-precision target is generated by a
rational convex atom mixture plus a rational rank-one perturbation.  A floating
LP proposes the witnesses, but the reported simplex point and support
functional are rounded to denominator $10^{12}$ and checked independently
using exact rational arithmetic.  The resulting structural-floor bracket is
\begin{center}
\begin{tabular}{@{}ll@{}}
\toprule
Exact dual lower bound & $0.1257546220866386$\\
Exact primal upper bound & $0.1257546220976685$\\
Exact bracket width & $1.1030\times10^{-11}$\\
Relative bracket width & $8.7710\times10^{-11}$\\
\bottomrule
\end{tabular}
\end{center}
The certificate hash begins \texttt{ec23f09ccf8e1984}.  This is a scalable,
architecture-shaped proof of concept; it is not yet a certificate for a
pretrained checkpoint.

\begin{lemma}[Superposition realization of feasible modal measures]
\label[lemma]{lem:modal-superposition-v7}
Assume the hypotheses of \cref{thm:occupation-lp-v2}.  If
$(\mu_1,\ldots,\mu_J,\mu_T)$ satisfies
\eqref{eq:occupation-time-marginal-v2}--
\eqref{eq:occupation-liouville-v2}, then there are a probability measure
$\Pi$ on $AC([0,T];K)$ and Borel functions
$p_j:[0,T]\times K\to[0,1]$ such that
\begin{equation}
 \sum_jp_j=1\quad dt\,\mu_t\text{-a.e.},
 \qquad
 \dot\gamma(t)=\sum_jp_j(t,\gamma(t))G_j(t,\gamma(t))
 \quad\text{a.e. for $\Pi$-a.e. }\gamma,
 \label{eq:modal-superposition-curve-v7}
\end{equation}
all curves start at $z_0$, and
$(e_T)_\#\Pi=\mu_T$.  In particular, $\mu_T$ is a probability mixture of
terminal points of admissible relaxed trajectories.
\end{lemma}

\begin{proof}
Put $\mu=\sum_j\mu_j$.  The time marginal gives a disintegration
$\mu(dt,dz)=dt\,\mu_t(dz)$ with $\mu_t\in\mathcal P(K)$ almost everywhere.
Because $\mu_j\ll\mu$, choose Borel Radon--Nikodym densities
$p_j=d\mu_j/d\mu$ and extend them arbitrarily on the $\mu$-null set; then
$p=(p_1,\ldots,p_J)$ is simplex valued $\mu$-almost everywhere.  The
Liouville identity is the distributional continuity equation for the bounded
Borel velocity
\[
 b(t,z)=\sum_jp_j(t,z)G_j(t,z),
\]
with initial law $\delta_{z_0}$ and terminal law $\mu_T$.  The finite-dimensional
superposition principle for continuity equations
\citep[Theorem~8.2.1]{AmbrosioGigliSavare2008} supplies a probability measure
$\Pi$ on absolutely continuous integral curves of $b$ with the prescribed
marginals.  Since the time marginals are supported on the closed set $K$,
almost every continuous curve lies in $K$ for all times.  Along each such
curve, $t\mapsto p(t,\gamma(t))$ is measurable, so it is an admissible relaxed
open-loop control for that curve.  The initial Dirac law and terminal marginal
give the remaining claims.  This is the same modal-measure realization used
for switched systems in \citet[Theorems~1--2]{ClaeysDaafouzHenrion2016}.
\end{proof}

\begin{theorem}[Exact occupation-measure representation of the relaxed floor]
\label[theorem]{thm:occupation-lp-v2}
Assume $\phi\in C(K)$, the finite-dimensional fields extend to bounded
continuous maps on a neighborhood of $[0,T]\times K$, are uniformly
Lipschitz in $z$, and every relaxed trajectory remains in nonempty compact
$K$.  Let $\mathcal M_+$ denote finite nonnegative Borel measures.  Consider
\begin{align}
 V_{\mathrm{LP}}:=\inf_{\mu_1,\ldots,\mu_J,\mu_T}
 &\quad \int_K\phi(z)\,d\mu_T(z)
 \label{eq:occupation-objective-v2}\\
 \text{subject to }&\quad
 \mu_j\in\mathcal M_+([0,T]\times K),\quad
 \mu_T\in\mathcal P(K),\nonumber\\
 &\quad \sum_{j=1}^J\pi_{t\#}\mu_j=dt,
 \label{eq:occupation-time-marginal-v2}\\
 &\quad
 \int_Kv(T,z)\,d\mu_T(z)-v(0,z_0)\nonumber\\
 &\qquad=\sum_{j=1}^J\int_{[0,T]\times K}
   \left(\partial_tv+\nabla v^\top G_j\right)d\mu_j
 \quad\forall v\in C^1([0,T]\times\mathbb R^n).
 \label{eq:occupation-liouville-v2}
\end{align}
Then
\begin{equation}
  V_{\mathrm{LP}}=E_\infty^\phi,
  \label{eq:occupation-equality-v2}
\end{equation}
and the measure program attains its optimum.  The time-marginal constraint
is redundant once \eqref{eq:occupation-liouville-v2} is imposed for all
state-independent tests $v(t,z)=\psi(t)$ and $\mu_T$ has unit mass.  After
removing that redundant constraint, the conic dual lower program is
\begin{equation}
 \sup_{v\in C^1}v(0,z_0)
 \quad\text{subject to}\quad
 v(T,z)\le\phi(z),\quad
 \partial_tv+\nabla v^\top G_j\ge0\ \forall j.
 \label{eq:occupation-dual-v2}
\end{equation}
Weak duality always holds.  If the image cone of the nonnegative measure cone
under the Liouville boundary operator, augmented by the cost coordinate, is
weak-* closed, infinite-dimensional conic separation gives equality of the
measure and HJB values.  The exact primal floor identity
\eqref{eq:occupation-equality-v2} does not require this additional dual-closure
hypothesis.
\end{theorem}

\begin{proof}
Every relaxed path $z_p$ gives feasible modal measures
\[
 d\mu_j(t,z)=p_j(t)\,dt\,\delta_{z_p(t)}(dz),
 \qquad \mu_T=\delta_{z_p(T)},
\]
so $V_{\mathrm{LP}}\le E_\infty^\phi$.  Conversely,
\cref{lem:modal-superposition-v7} writes the terminal law of every feasible
measure tuple as a probability mixture of admissible relaxed endpoints.
Therefore
\[
 \int\phi\,d\mu_T
 =\int\phi(\gamma(T))\,d\Pi(\gamma)
 \ge\inf_{p(\cdot)}\phi(z_p(T))=E_\infty^\phi.
\]
This proves equality.

For redundancy, let $\eta=\sum_j\pi_{t\#}\mu_j$ and test the Liouville
identity with $v(t,z)=\psi(t)$.  Then
$\int\psi'\,d\eta=\psi(T)-\psi(0)$ for every $\psi\in C^1[0,T]$.
Given any $f\in C[0,T]$, choosing
$\psi(t)=\int_0^tf(s)\,ds$ gives
$\int f\,d\eta=\int_0^Tf(t)\,dt$, hence $\eta=dt$.
All feasible modal measures have total mass $T$ and $\mu_T$ has mass one;
compact support therefore makes the feasible set weak-* compact.  The
Liouville integrands are continuous, so the constraints and objective are
weak-* closed and continuous, proving attainment.  Lagranging the equivalent
formulation without the redundant marginal gives
\eqref{eq:occupation-dual-v2}; weak duality is
\cref{thm:hjb-lower-v2}.  The final no-gap statement is the standard closed-cone
separation criterion.
\end{proof}

\begin{theorem}[Moment--SOS hierarchy with terminal support]
\label[theorem]{thm:sos-hierarchy-v2}
Let $\mathcal K=[0,T]\times K$ and $K_T\subseteq K$ be compact basic
semialgebraic sets.  Write
\begin{align*}
 \mathcal K&=\{(t,z):g_i(t,z)\ge0,\ i=1,\ldots,m_K\},\\
 K_T&=\{z:h_i(z)\ge0,\ i=1,\ldots,m_T\}.
\end{align*}
Assume that the corresponding quadratic modules are Archimedean.  The set
$K_T$ is the declared support of the terminal measure; take $K_T=K$ when no
extra terminal constraint is imposed.
Assume every atom $G_j$ and the terminal loss $\phi$ are polynomial, all
feasible relaxed paths remain in $K$, and at least one admissible relaxed path
terminates in $K_T$.  Define
\begin{equation}
 E_{\infty,K_T}^{\phi}
 :=\inf\{\phi(z_p(T)):p(\cdot)\text{ relaxed-admissible},\ z_p(T)\in K_T\}.
 \label{eq:terminal-constrained-floor-v2}
\end{equation}
A compact polynomial epigraph lift may be used for a semialgebraic norm or
maximum loss.  When no extra terminal condition is imposed, take $K_T=K$, in
which case $E_{\infty,K_T}^{\phi}=E_\infty^\phi$.

For each modal measure introduce a truncated moment sequence $y_j$, and for
$\mu_T$ introduce $y_T$.  At relaxation order $r$, impose
\begin{align}
 M_r(y_j)&\succeq0,&
 M_{r-d_i}(g_i y_j)&\succeq0,\nonumber\\
 M_r(y_T)&\succeq0,&
 M_{r-e_i}(h_i y_T)&\succeq0,
 \label{eq:modal-moment-psd-v7}
\end{align}
with $d_i=\lceil\deg g_i/2\rceil$ and
$e_i=\lceil\deg h_i/2\rceil$, together with the truncated Liouville equations
\begin{equation}
 L_{y_T}(v(T,\cdot))-v(0,z_0)
 =\sum_jL_{y_j}(\partial_tv+\nabla v^\top G_j)
 \label{eq:modal-moment-liouville-v7}
\end{equation}
for every monomial $v$ satisfying
$\deg v(T,\cdot)\le2r$ and
$\deg(\partial_tv+\nabla v^\top G_j)\le2r$ for every atom $j$.
Let $L_r$ be the minimum of $L_{y_T}(\phi)$ over these constraints, including
unit terminal mass and all localizing constraints defining $K_T$.  Its conic dual is the corresponding degree-restricted
SOS HJB certificate.  Then, for all sufficiently large $r$,
\begin{equation}
 L_r\le L_{r+1}\le E_{\infty,K_T}^\phi,
 \qquad
 \lim_{r\to\infty}L_r=E_{\infty,K_T}^\phi.
 \label{eq:sos-convergence-v2}
\end{equation}
When a finite primal or dual SDP satisfies a standard no-gap condition, its
moment and SOS values agree.  A feasible relaxed rollout gives an independent
upper bound $U\ge E_{\infty,K_T}^\phi$ from a rollout terminating in
$K_T$, so $[L_r,U]$ is a rigorous bracket for the constrained floor.
\end{theorem}

\begin{proof}
Every admissible relaxed trajectory terminating in $K_T$ supplies feasible
moments at every order, so $L_r\le E_{\infty,K_T}^\phi$; increasing $r$
adds consistency and positivity constraints, hence $L_r\le L_{r+1}$.  For a
sequence of nearly optimal truncated solutions, the Archimedean localizing
constraints give uniform bounds on every fixed moment.  A diagonal
subsequence therefore converges coefficientwise to full moment sequences.
Putinar's representing-measure theorem yields nonnegative modal measures on
$\mathcal K$ and a terminal probability measure supported on $K_T$.  Passing
to the limit in every polynomial Liouville identity gives the weak Liouville
equation for polynomial tests.  On the compact supports, simultaneous
polynomial approximation of a $C^1$ test and its first derivatives extends the
identity to the $C^1$ class used in \cref{thm:occupation-lp-v2}.  The limiting
measures are therefore feasible for the exact modal LP with terminal support
$K_T$.  Applying the superposition realization to that LP identifies its value
with $E_{\infty,K_T}^\phi$, proving convergence.  This is the modal
switched-system hierarchy of \citet[Theorem~5]{ClaeysDaafouzHenrion2016},
augmented by the terminal support and objective.  The SOS program is the
finite conic dual; equality of finite primal and dual SDP values is a separate
qualification and is not used in the asymptotic argument.
\end{proof}

\begin{theorem}[Exact degree-two HJB/SOS floor certificate]
\label[theorem]{thm:exact-polynomial-sos-v7}
Consider the scalar switched system
\begin{equation}
 \dot x(t)=u(t),\qquad u(t)\in\{-1,1\},\qquad
 x(0)=0,\qquad (t,x)\in[0,1]\times[-1,1],
 \label{eq:exact-sos-system-v7}
\end{equation}
with terminal loss $\phi(x)=(x-2)^2$.  Its relaxed structural floor is exactly
\begin{equation}
 E_\infty^\phi=1.
 \label{eq:exact-sos-value-v7}
\end{equation}
Moreover the polynomial
\begin{equation}
 v(t,x)=(1+t-x)^2
 \label{eq:exact-sos-value-function-v7}
\end{equation}
is a degree-two SOS/HJB certificate attaining this value.  Indeed,
\begin{align}
 v(1,x)&=\phi(x),\nonumber\\
 \partial_tv+\partial_xv\,(+1)&=0,\nonumber\\
 \partial_tv+\partial_xv\,(-1)
 &=4(1+t-x)=4t+4(1-x),
 \label{eq:exact-sos-decomposition-v7}
\end{align}
which lies in the quadratic module generated by
$t$, $1-t$, $1-x$, and $1+x$ with constant SOS multipliers.  Thus a finite
SOS relaxation already returns the exact global floor.
\end{theorem}

\begin{proof}
Equation \eqref{eq:exact-sos-decomposition-v7} and the terminal identity make
$v$ feasible in the HJB dual and give the lower bound
$v(0,0)=1$.  The all-$+1$ trajectory reaches $x(1)=1$ and has terminal cost
one, giving a matching primal upper bound.  The displayed decomposition is an
exact rational SOS certificate, not a sampled inequality check.
\end{proof}

\paragraph{Machine-checkable certificate.}
The accompanying artifact
\nolinkurl{polynomial_sos_certificate_v7.json} stores the value polynomial,
terminal identity, modal residuals, quadratic-module decomposition, and
matching primal trajectory over the rational polynomial ring.  Its exact
symbolic verifier reports SHA-256
\nolinkurl{9196fc00171a964954e09898750162dc9f46f28e844d213189f9464b80c030ec}.

\begin{theorem}[Exact HJB certificate for the nonlinear soft-threshold floor]
\label[theorem]{thm:soft-threshold-hjb-exact-v2}
Consider the manuscript's same-dictionary structural-split example with
$\mu_s=1-s>0$ and
$\alpha_s=(1-e^{-\mu_s})/\mu_s$.  Lift only the exposed first coordinate at
inputs $y=1/2$ and $y=1$, writing
$z=(x_{1/2},x_1)$.  Every low-bit atom has joint dynamics
\begin{equation}
 \dot x_{1/2}=-\mu_sx_{1/2}+\frac12u_j,
 \qquad
 \dot x_1=-\mu_sx_1+u_j
 \label{eq:two-input-soft-dynamics-v2}
\end{equation}
for some $u_j\in\{-1,0,1\}$.  In the two-point sup norm, define the norm-one point-evaluation witness
\begin{equation}
  \lambda(z)=\frac13x_1-\frac23x_{1/2}
  \label{eq:soft-witness-v2}
\end{equation}
and
\begin{equation}
  w(t,z)=e^{\mu_s(t-1)}\lambda(z).
  \label{eq:soft-hjb-v2}
\end{equation}
Then
\begin{equation}
 w(1,z)=\lambda(z),
 \qquad
 \partial_tw+\nabla w^\top G_j=0
 \quad\text{for every atom }j.
 \label{eq:soft-hjb-identity-v2}
\end{equation}
For the target
$F^\star(y)=\alpha_s\soft_{1/2}(y)$,
\begin{equation}
  E_{\infty,b}(F^\star)=\frac{\alpha_s}{6}.
  \label{eq:soft-exact-floor-v2}
\end{equation}
The lower bound is certified by the single explicit HJB witness above, and the
matching upper bound is attained by the relaxed coefficient choice
$a=\alpha_s/3$, $b=2\alpha_s/3$.
\end{theorem}

\begin{proof}
The coefficients of $u_j$ cancel in the witness direction:
\[
 \nabla\lambda^\top G_j
 =\frac13(-\mu_sx_1+u_j)
  -\frac23(-\mu_sx_{1/2}+u_j/2)
 =-\mu_s\lambda(z).
\]
Therefore \eqref{eq:soft-hjb-identity-v2} holds.  Since $z_0=0$,
$w(0,z_0)=0$, and \cref{prop:support-witness-v2} gives
\[
 E_{\infty,b}(F^\star)\ge\lambda(F^\star).
\]
Now $F^\star(1/2)=0$ and $F^\star(1)=\alpha_s/2$, so
$\lambda(F^\star)=\alpha_s/6$.  The stated relaxed coefficients produce the reachable map
$a y+b\soft_1(y)$.  After dividing by $\alpha_s$ and using
$a/\alpha_s=1/3$, $b/\alpha_s=2/3$, the error on $[0,2]$ is
\[
 \begin{cases}
  y/3,&0\le y\le1/2,\\
  \abs{2y/3-1/2},&1/2\le y\le1,\\
  1/6,&1\le y\le2.
 \end{cases}
\]
Its supremum is $1/6$, and odd symmetry gives the same bound on $[-2,0]$.
Thus the primal upper bound is $\alpha_s/6$, and the primal and dual bounds
coincide.
\end{proof}

\begin{corollary}[Certified feasible, impossible, or unresolved]
\label[corollary]{cor:primal-dual-decision-v2}
Suppose a feasible relaxed trajectory gives $E_\infty\le U$, a verified HJB,
SOS, or support certificate gives $L\le E_\infty$, and the finite
implementation obeys
\begin{equation}
  |E_D^{\mathrm{impl}}-E_\infty|\le R_D.
  \label{eq:implementation-radius-v2}
\end{equation}
For a required tolerance $\varepsilon_H$,
\begin{align*}
 U+R_D\le\varepsilon_H
 &\Longrightarrow \text{certified feasible at depth $D$},\\
 L-R_D>\varepsilon_H
 &\Longrightarrow \text{certified impossible at depth $D$},\\
 L>\varepsilon_H
 &\Longrightarrow \text{certified asymptotically impossible}.
\end{align*}
Otherwise the certificate is unresolved rather than negative.
\end{corollary}

\begin{proof}
Use
$E_D^{\mathrm{impl}}\le E_\infty+R_D\le U+R_D$ for feasibility and
$E_D^{\mathrm{impl}}\ge E_\infty-R_D\ge L-R_D$ for finite-depth
impossibility.  The asymptotic statement follows from $E_\infty\ge L$.
\end{proof}

\subsection{Architecture specialization: Transformer and routed-MoE execution}
\label{sec:transformer-execution-v4}

Consider a parallel pre-normalized Transformer residual field
\begin{equation}
 G(t,X)=\alpha(t)A(N_1(X))+\beta(t)M(N_2(X))
 \label{eq:transformer-field-v4}
\end{equation}
and its executed field
\begin{equation}
 \widehat G_k(X)=
 \widehat\alpha_k\widehat A_k(\widehat N_{1,k}(X))
 +\widehat\beta_k\widehat M_k(\widehat N_{2,k}(X)).
 \label{eq:transformer-executed-field-v4}
\end{equation}
All norms below are the same norm used to verify the common execution tube.
The following propositions make the normalization and attention terms
checkable from matrix, activation, and quantizer error budgets rather than
leaving them as opaque constants.

\paragraph{Operational norms.}
For a token matrix $X\in\mathbb R^{n\times d}$ write
$\norm{X}_{2,\infty}:=\max_i\norm{X_{i:}}_2$.  For a score matrix $S$, write
$\norm{S}_{\max}:=\max_{i,j}\abs{S_{ij}}$ and, for a row matrix $P$,
$\norm{P}_{\infty,1}:=\max_i\sum_j\abs{P_{ij}}$.

\begin{proposition}[LayerNorm tube and quantization constants]
\label[proposition]{prop:layernorm-constants-v7}
Let $P_d=I-d^{-1}\one\one^\top$ and, for $\epsilon>0$, define
\begin{equation}
 N_{\gamma,\beta}(x)
 =\gamma\odot\frac{P_dx}
 {\sqrt{d^{-1}\norm{P_dx}_2^2+\epsilon}}+\beta.
 \label{eq:layernorm-map-v7}
\end{equation}
Then
\begin{equation}
 \Lip_2(N_{\gamma,\beta})
 \le\frac{\norm{\gamma}_\infty}{\sqrt\epsilon},
 \qquad
 \norm{N_{\gamma,\beta}(x)-\beta}_2
 \le\sqrt d\,\norm{\gamma}_\infty.
 \label{eq:layernorm-lipschitz-v7}
\end{equation}
If the same $\epsilon$ is used by an executed map with input
$\widehat x$ and parameters $(\widehat\gamma,\widehat\beta)$, then exact
LayerNorm arithmetic obeys
\begin{align}
 \norm{N_{\widehat\gamma,\widehat\beta}(\widehat x)
       -N_{\gamma,\beta}(x)}_2
 \le{}&\frac{\norm{\widehat\gamma}_\infty}{\sqrt\epsilon}
        \norm{\widehat x-x}_2
 +\sqrt d\,\norm{\widehat\gamma-\gamma}_\infty
 +\norm{\widehat\beta-\beta}_2.
 \label{eq:layernorm-error-v7}
\end{align}
Any finite-arithmetic LayerNorm remainder is added once to the right-hand
side.  A coordinatewise activation quantizer of radius $\rho_x$ contributes
at most $\sqrt d\,\rho_x$ to $\norm{\widehat x-x}_2$.
\end{proposition}

\begin{proof}
For
$f(x)=P_dx/s(x)$ with
$s(x)^2=d^{-1}\norm{P_dx}_2^2+\epsilon$, differentiation gives
\[
 Df(x)=\frac{P_d}{s(x)}
 -\frac{(P_dx)(P_dx)^\top}{d\,s(x)^3}.
\]
On the mean-zero subspace its eigenvalues are $1/s(x)$ orthogonally to
$P_dx$ and $\epsilon/s(x)^3$ in the $P_dx$ direction; it vanishes on
$\operatorname{span}\{\one\}$.  Hence $\norm{Df(x)}_{\mathrm{op}}
\le1/\sqrt\epsilon$.  Also $\norm{f(x)}_2\le\sqrt d$.  Splitting the executed
error into input, scale, and bias perturbations proves
\eqref{eq:layernorm-error-v7}.
\end{proof}

\begin{lemma}[Rowwise softmax is Lipschitz from $\ell^\infty$ to $\ell^1$]
\label[lemma]{lem:softmax-lipschitz-v7}
For all $s,\widehat s\in\mathbb R^n$,
\begin{equation}
 \norm{\operatorname{softmax}(\widehat s)
       -\operatorname{softmax}(s)}_1
 \le\norm{\widehat s-s}_\infty.
 \label{eq:softmax-lipschitz-v7}
\end{equation}
Consequently, for rowwise softmax,
\begin{equation}
 \norm{\operatorname{softmax}_{\rm row}(\widehat S)
       -\operatorname{softmax}_{\rm row}(S)}_{\infty,1}
 \le\norm{\widehat S-S}_{\max}.
 \label{eq:row-softmax-lipschitz-v7}
\end{equation}
\end{lemma}

\begin{proof}
Along the segment $s+\tau\delta$, the softmax derivative is
$p_i(\delta_i-\mathbb E_p\delta)$.  Its $\ell^1$ norm is the mean absolute
deviation of a random variable taking values in an interval of width at most
$2\norm\delta_\infty$.  Cauchy--Schwarz and the elementary variance bound
$\operatorname{Var}(X)\le(\sup X-\inf X)^2/4$ make this at most
$\norm\delta_\infty$.  Integrate over
$\tau\in[0,1]$ and apply the result row by row.
\end{proof}

\begin{proposition}[Auditable single-head attention error]
\label[proposition]{prop:attention-error-v7}
Let
\begin{equation}
 S=QK^\top/\sqrt{d_k},\qquad P=\operatorname{softmax}_{\rm row}(S),
 \qquad Z=PV,
 \label{eq:attention-definition-v7}
\end{equation}
and define the executed quantities analogously.  Suppose
\begin{align*}
 \norm{Q}_{2,\infty}&\le B_Q,&
 \norm{K}_{2,\infty}&\le B_K,&
 \norm{V}_{2,\infty}&\le B_V,\\
 \norm{\widehat Q-Q}_{2,\infty}&\le\eta_Q,&
 \norm{\widehat K-K}_{2,\infty}&\le\eta_K,&
 \norm{\widehat V-V}_{2,\infty}&\le\eta_V.
\end{align*}
Then
\begin{align}
 \norm{\widehat S-S}_{\max}
 &\le\eta_S
 :=\frac{\eta_Q(B_K+\eta_K)+B_Q\eta_K}{\sqrt{d_k}},
 \label{eq:attention-score-error-v7}\\
 \norm{\widehat Z-Z}_{2,\infty}
 &\le\eta_V+B_V\eta_S.
 \label{eq:attention-value-error-v7}
\end{align}
For an output projection $A=ZW_O+b_O$,
\begin{align}
 \norm{\widehat A-A}_{2,\infty}
 \le{}&\norm{\widehat W_O}_{\mathrm{op}}
        (\eta_V+B_V\eta_S)
 +B_V\norm{\widehat W_O-W_O}_{\mathrm{op}}
 +\norm{\widehat b_O-b_O}_2.
 \label{eq:attention-output-error-v7}
\end{align}
All bounds remain valid with additional certified arithmetic remainders added
to the relevant $\eta$ terms.
\end{proposition}

\begin{proof}
For every score entry, add and subtract $q_i^\top\widehat k_j$ to obtain
\eqref{eq:attention-score-error-v7}.  By
\cref{lem:softmax-lipschitz-v7},
$\norm{\widehat P-P}_{\infty,1}\le\eta_S$.  Since every row of
$\widehat P$ is a probability vector,
\[
 \norm{\widehat P(\widehat V-V)}_{2,\infty}\le\eta_V,
 \qquad
 \norm{(\widehat P-P)V}_{2,\infty}\le B_V\eta_S,
\]
which proves \eqref{eq:attention-value-error-v7}.  Split the output projection
error and use $\norm Z_{2,\infty}\le B_V$.
\end{proof}

\begin{corollary}[Multihead attention aggregation]
\label[corollary]{cor:multihead-attention-v7}
If head $h$ has preprojection error $\eta_{Z,h}$ in
$\norm{\cdot}_{2,\infty}$, concatenation satisfies
\begin{equation}
 \norm{\operatorname{concat}_h\widehat Z_h
       -\operatorname{concat}_h Z_h}_{2,\infty}
 \le\left(\sum_h\eta_{Z,h}^2\right)^{1/2}.
 \label{eq:multihead-concat-v7}
\end{equation}
Applying \eqref{eq:attention-output-error-v7} to the shared output projection
turns per-head quantizer, score, and value errors into the branch constant
$\eta_{A,k}$ in \cref{prop:transformer-field-error-v4}.
\end{corollary}

\begin{proof}
For each token, the squared Euclidean norm of a concatenation is the sum of
the squared head norms.  Maximize over tokens and apply the projection bound.
\end{proof}

\begin{proposition}[Componentwise Transformer field-error budget]
\label[proposition]{prop:transformer-field-error-v4}
Suppose on the normalized tube
\begin{align*}
 \norm{A(N_1(X))}&\le B_A,&
 \norm{M(N_2(X))}&\le B_M,\\
 \norm{\widehat N_{i,k}(X)-N_i(X)}&\le\eta_{N_i,k},&i&=1,2,\\
 \sup_Z\norm{\widehat A_k(Z)-A(Z)}&\le\eta_{A,k},&
 \operatorname{Lip}(A)&\le L_A,\\
 \sup_Z\norm{\widehat M_k(Z)-M(Z)}&\le\eta_{M,k},&
 \operatorname{Lip}(M)&\le L_M.
\end{align*}
Then
\begin{align}
 \sup_X\norm{\widehat G_k(X)-G(t_k,X)}
 \le\eta_{G,k}:={}&
 \abs{\widehat\alpha_k-\alpha(t_k)}B_A
 +\abs{\widehat\alpha_k}(\eta_{A,k}+L_A\eta_{N_1,k})\notag\\
 &+\abs{\widehat\beta_k-\beta(t_k)}B_M
 +\abs{\widehat\beta_k}(\eta_{M,k}+L_M\eta_{N_2,k}).
 \label{eq:transformer-field-error-v4}
\end{align}
If $G(t_k,\cdot)$ is $L_G$-Lipschitz, then
\begin{equation}
 \norm{\widehat G_k(X)-G(t_k,Y)}
 \le L_G\norm{X-Y}+\eta_{G,k}.
 \label{eq:transformer-cross-error-v4}
\end{equation}
\end{proposition}

\begin{proof}
For the attention term, add and subtract
$\widehat\alpha_kA(N_1(X))$ and
$\widehat\alpha_kA(\widehat N_{1,k}(X))$; bound the resulting scale,
operator, and normalization errors.  Repeat for the MLP term.  Add and
subtract $G(t_k,X)$ for the cross-state estimate.
\end{proof}

\begin{proposition}[Affine and two-layer MLP quantization errors]
\label[proposition]{prop:mlp-error-v4}
On $\norm z\le R$, an affine map $L(z)=Wz+b$ and its quantized version obey
\begin{equation}
 \sup_{\norm z\le R}\norm{\widehat L(z)-L(z)}
 \le\norm{\widehat W-W}_{\mathrm{op}}R+\norm{\widehat b-b}.
 \label{eq:affine-error-v4}
\end{equation}
For
$M(z)=W_2\varphi(W_1z+b_1)+b_2$, assume $\varphi$ is
$L_\varphi$-Lipschitz and
$\norm{\varphi(W_1z+b_1)}\le B_\varphi$ on the tube.  Then
\begin{align}
 \sup_{\norm z\le R}\norm{\widehat M(z)-M(z)}
 \le{}&\norm{\widehat W_2-W_2}_{\mathrm{op}}B_\varphi
 +\norm{\widehat W_2}_{\mathrm{op}}L_\varphi
   \left(\norm{\widehat W_1-W_1}_{\mathrm{op}}R
        +\norm{\widehat b_1-b_1}\right)\notag\\
 &+\norm{\widehat b_2-b_2}.
 \label{eq:mlp-error-v4}
\end{align}
\end{proposition}

\begin{proof}
Insert the exact intermediate affine map.  The first-layer perturbation passes
through the Lipschitz nonlinearity and the executed second-layer operator;
the remaining terms are direct second-layer and bias errors.
\end{proof}

\begin{theorem}[Executed Transformer endpoint and path resource laws]
\label[theorem]{thm:executed-transformer-v4}
Consider a nonempty learned family of Transformer dictionaries indexed by
$\omega\in\Omega$, with an $s$-bit metadata cover of radius $\delta_s$.
Assume the family has one common tube, uniform boundedness and state-Lipschitz
constants, and the parameter-Lipschitz bound
\eqref{eq:dictionary-parameter-lipschitz-candidate}.  Assume in addition that
its time dependence satisfies either the uniform smooth hypotheses or the
uniform BV hypotheses, with endpoint constant
$C_{\mathrm{syn}}^{\mathrm{Tr}}$ and path constant
$C_{\mathrm{syn,path}}^{\mathrm{Tr}}$.  Let
$\widehat{\mathscr P}_{D,s}^{\mathrm{Tr,EF}}$ be the linearly interpolated
implemented path set and set
\[
 \mathscr P_{\Omega,\Rel}^{\mathrm{Tr}}
 :=\overline{\bigcup_{\omega\in\Omega}
                 \mathscr P_{\omega,\Rel}^{\mathrm{Tr}}}.
\]
Suppose every implemented microstep uses increment error feedback with
physical residual radius $\rho_D$ and uniform field error
$\eta_{G,D}:=\sup_k\eta_{G,k}$.  Assume the carry and integer MAC accumulators
are exact on the declared tube, saturation is inactive by
\cref{prop:saturation-tube-candidate}, and any serial attention--MLP splitting
cost is included in the two synthesis constants.  Then
\begin{align}
 \dH\!\left(
 \widehat\RR_{D,s}^{\mathrm{Tr,EF}},
 \RR_{\Omega,\Rel}^{\mathrm{Tr}}
 \right)
 &\le r_D^{\mathrm{end}},
 \label{eq:executed-transformer-law-v4}\\
 \dH^{\mathrm{path}}\!\left(
 \widehat{\mathscr P}_{D,s}^{\mathrm{Tr,EF}},
 \mathscr P_{\Omega,\Rel}^{\mathrm{Tr}}
 \right)
 &\le r_D^{\mathrm{path}},
 \label{eq:executed-transformer-path-law-v4}
\end{align}
where
\begin{align}
 r_D^{\mathrm{end}}
 &:=\frac{C_{\mathrm{syn}}^{\mathrm{Tr}}}{D}
 +L_\Omega\Phi_{L_z}(T)\delta_s
 +\rho_De^{L_zT}
 +\eta_{G,D}\Phi_{L_z}(T),\\
 r_D^{\mathrm{path}}
 &:=\frac{C_{\mathrm{syn,path}}^{\mathrm{Tr}}}{D}
 +L_\Omega\Phi_{L_z}(T)\delta_s
 +\rho_De^{L_zT}
 +\eta_{G,D}\Phi_{L_z}(T).
 \label{eq:executed-transformer-radii-v4}
\end{align}
The same radii bound the differences between implemented and ideal learned-
dictionary endpoint or path target errors.  In particular, the full executed
block retains a first-order depth law when
\begin{equation}
 \delta_s=O(D^{-1}),\qquad
 \rho_D=O(D^{-1}),\qquad
 \eta_{G,D}=O(D^{-1}).
 \label{eq:transformer-resource-scaling-v4}
\end{equation}
\end{theorem}

\begin{proof}
For one metadata value and one schedule,
\cref{thm:error-feedback-field-v4} controls every grid prefix.  Linear
interpolation between paired grid states does not enlarge their maximum
difference.  The smooth baseline theorem or
\cref{thm:bv-reachability-v2} supplies the ideal endpoint and path synthesis
radii.  Driving two learned dictionaries by the same relaxed control and
applying Gronwall gives the codebook term uniformly over the whole path, not
only at $T$.  Pair schedules in both directions, take closures, and apply the
endpoint or path Hausdorff triangle inequality.  The target-distance claims
follow from one-Lipschitz continuity of distance under Hausdorff perturbation.
\end{proof}

\begin{corollary}[Executed route-changing MoE Transformer law]
\label[corollary]{cor:executed-moe-transformer-v4}
Let $z^\star$ be a target route-changing Transformer path with simple isolated
events, and suppose it belongs to the closed relaxed path class generated by
the target-route dictionary
\[
 \mathcal G^{m^\star,\omega}(t)
 =\{G_1^{m^\star(t),\omega}(t,\cdot),\ldots,
     G_J^{m^\star(t),\omega}(t,\cdot)\}
\]
for some $\omega\in\Omega$.  Assume this learned target-route family satisfies
the uniform BV and executed-Transformer hypotheses above.  Use
\cref{prop:topk-moe-constants-v4} to instantiate $L_z$, $\Delta_r$,
$\eta_{r,D}$, and $L_r$.  Then there exist a metadata code and one pure
schedule whose executed \emph{target-route oracle} obeys
\begin{equation}
 \delta_D:=\max_{k\le D}\norm{y_k-z^\star(t_k)}
 \le r_D^{\mathrm{path}}.
 \label{eq:moe-oracle-radius-v4}
\end{equation}
If this radius obeys the route-window small-gain, isolation, disjoint-window,
and tube conditions, the actual state-driven implementation satisfies
\begin{equation}
 \max_{k\le D}\norm{x_k-z^\star(t_k)}
 \le\overline E_D=\frac{b_D}{1-\chi},
 \label{eq:executed-moe-state-v4}
\end{equation}
and its route mismatches occur only in the certified event windows.  Under
\eqref{eq:transformer-resource-scaling-v4} and
$\eta_{\mathrm{score}}=O(D^{-1})$, both the state error and event-window widths
are $O(D^{-1})$.
\end{corollary}

\begin{proof}
Because $z^\star$ lies in the closed relaxed target-route path class, the
pathwise part of \cref{thm:executed-transformer-v4} supplies a metadata value,
pure schedule, and executed oracle whose entire interpolated path is within
$r_D^{\mathrm{path}}$ of $z^\star$; in particular,
\eqref{eq:moe-oracle-radius-v4} holds at every grid point.  This is the
pathwise premise needed by the hybrid theorem and does not follow from an
endpoint estimate alone.  The weighted-MoE proposition supplies the remaining
constants.  Apply \cref{thm:route-window-v2,cor:first-order-routing-v2}.
\end{proof}

\subsection{The integrated resource theorem}
\label{sec:integrated-resource-v2}

\begin{theorem}[Floor--depth--metadata--arithmetic resource law]
\label[theorem]{thm:integrated-resource-v2}
Let a compact learned dictionary family satisfy the uniform learned-codebook
assumptions, with either smooth time dependence or the uniform BV assumptions
of \cref{thm:bv-reachability-v2}.  Let $C_{\mathrm{syn}}^{\mathrm{unif}}$
denote the corresponding uniform finite-depth constant, let an $s$-bit
metadata code have covering radius $\delta_s$, and suppose the implemented and ideal depth-$D$ systems have a schedulewise
arithmetic radius $\mathcal A_D$ uniform over every encoded metadata value.
Then
\begin{equation}
 \dH(\RR_{D,s}^{\mathrm{impl}},\RR_{\Omega,\Rel})
 \le
 \frac{C_{\mathrm{syn}}^{\mathrm{unif}}}{D}
 +L_\Omega\Phi_{L_z}(T)\delta_s
 +\mathcal A_D.
 \label{eq:integrated-resource-law-v2}
\end{equation}
If the metadata entropy obeys
$\mathcal N(\Omega,\delta)\le(C_\Omega/\delta)^m$, execution uses scaled
error feedback with normalized increment radius $\overline\rho_b$, the
executed field error is uniformly at most $\eta_{G,D}$, and all remaining
implementation effects have schedulewise radius at most
$\mathcal A_D^{\mathrm{other}}$, then
\begin{equation}
 \dH(\RR_{D,s}^{\mathrm{impl}},\RR_{\Omega,\Rel})
 \le
 \frac{C_{\mathrm{syn}}^{\mathrm{unif}}
       +T\overline\rho_b e^{L_zT}}{D}
 +L_\Omega\Phi_{L_z}(T)C_\Omega2^{-s/m}
 +\eta_{G,D}\Phi_{L_z}(T)
 +\mathcal A_D^{\mathrm{other}}.
 \label{eq:integrated-scaled-law-v2}
\end{equation}
The same quantities bound the signed difference between the implemented
optimal target error and the learned relaxed structural floor.  Replacing the
endpoint synthesis and arithmetic radii by their prefix-uniform counterparts
gives the identical statement for linearly interpolated path sets.
\end{theorem}

\begin{proof}
Use the triangle inequality among the implemented pure set, the ideal pure set
with quantized metadata, the corresponding relaxed set, and the union over
all learned metadata.  Apply the hardware transfer theorem, the smooth or BV
pure-to-relaxed theorem, and the learned-endpoint Lipschitz lemma.  The scaled form uses \cref{thm:error-feedback-field-v4}, the physical
radius $\rho_D=T\overline\rho_b/D$, and the metadata entropy bound.
Distance to a target is one-Lipschitz under Hausdorff perturbations.
\end{proof}

The upper law is complemented by three independent necessity mechanisms:
\begin{enumerate}[leftmargin=2em]
\item the fixed LISTA teacher in
\cref{cor:binary-lista-fixed-teacher-v2} forces an architecture-level
$\Omega(D^{-1})$ synthesis price;
\item the metadata--depth packing converse requires
$s+D\log_2J\ge\log_2\mathcal M(2\varepsilon,\mathcal F)$ for a teacher
family;
\item HJB, support, or SOS certificates lower-bound a positive structural
floor that no depth can remove.
\end{enumerate}
For actual route-changing networks, the target-specific hybrid theorem and
its $L^p$ extension supply the additional event-window conditions.  These
conditions should remain separate from the symmetric Hausdorff law until a
genuine two-sided hybrid reachable-set theorem is proved.

\paragraph{Theoretical completion boundary.}
The smooth/BV reachable-set laws, finite-arithmetic transfer, exact and
perturbation-stable fixed-target converses, global-codebook resource laws,
simple-event hybrid theorem, explicit Transformer component bounds, and
primal--dual floor certificates above form a closed theorem package.  The
paper does not claim a symmetric Hausdorff theorem for arbitrary state-driven
hybrid reachable sets, a uniform $C(\Xi)$ theorem across discontinuous input
routes, or an integer-hardware guarantee without a verified tube, accumulator,
scale, and saturation audit.  Those exclusions are mathematical boundaries of
the stated resource theory, not informal caveats.

\section{Proofs for the H\"older, Integer, Neural-Converse, and Compositional Extensions}
\label{app:merged-exact-proofs}

This appendix proves the results added to the main theorem spine after the
core resource-theory development in Appendix~\ref{app:full-resource-proofs}.
The order follows the main text: H\"older-time synthesis, prescribed route
motion, exact common-lattice execution, neural fixed-target converses,
matching-depth consequences, and compositional HJB certificates.  We retain
all closure, boundary-case, and register-range details because they are the
parts most easily obscured by a proof sketch.

\subsection{H\"older-time pure-to-relaxed synthesis}
\label{app:holder-proof}

We prove Theorem~\ref{thm:holder-time-main}.  Put $h=T/D$ and
$t_k=kh$.  Let $z:[0,T]\to\mathcal K$ be the trajectory of a measurable
relaxed control $\alpha:[0,T]\to\Delta_J$.  On each cell
$I_k=[t_k,t_{k+1}]$, define the averaged control
\begin{equation}
 p_{k,j}:=\frac1h\int_{I_k}\alpha_j(s)\,ds,
 \qquad p_k\in\Delta_J,
 \label{eq:app-holder-average}
\end{equation}
and the mixed Euler path
\begin{equation}
 u_{k+1}=u_k+h\sum_{j=1}^Jp_{k,j}G_j(t_k,u_k),
 \qquad u_0=z(0).
 \label{eq:app-holder-mixed}
\end{equation}
All paths below lie in the enlarged tube by
Assumption~\ref{ass:holder-time-main} and $D\ge D_0$.

\paragraph{Step 1: continuous relaxed trajectory versus mixed Euler.}
The relaxed trajectory is $B$-Lipschitz, because for $s<t$,
\begin{equation}
 \norm{z(t)-z(s)}
 \le\int_s^t\sum_j\alpha_j(r)\norm{G_j(r,z(r))}\,dr
 \le B(t-s).
 \label{eq:app-holder-path-speed}
\end{equation}
Consequently, the local defect obtained by starting the mixed step at the
exact state $z(t_k)$ is
\begin{align}
 d_k
 &:={z(t_{k+1})-z(t_k)}
    -h\sum_jp_{k,j}G_j(t_k,z(t_k))\nonumber\\
 &=\int_{I_k}\sum_j\alpha_j(s)
   \{G_j(s,z(s))-G_j(t_k,z(t_k))\}\,ds,
 \label{eq:app-holder-local-defect}
\end{align}
and hence
\begin{align}
 \norm{d_k}
 &\le\int_0^h\{H_tr^\vartheta+L_zBr\}\,dr\nonumber\\
 &=\frac{H_t}{\vartheta+1}h^{1+\vartheta}
   +\frac12L_zBh^2.
 \label{eq:app-holder-local-bound}
\end{align}
For $e_k^{\rm mix}:=z(t_k)-u_k$,
\begin{equation}
 \norm{e_{k+1}^{\rm mix}}
 \le(1+hL_z)\norm{e_k^{\rm mix}}+\norm{d_k}.
 \label{eq:app-holder-mixed-recursion}
\end{equation}
Using
\begin{equation}
 h\sum_{r=0}^{D-1}(1+hL_z)^r
 \le \Phi_{L_z}(T),
 \qquad
 \Phi_{L_z}(T):=
 \begin{cases}
  (e^{L_zT}-1)/L_z,&L_z>0,\\
  T,&L_z=0,
 \end{cases}
 \label{eq:app-holder-phi}
\end{equation}
we obtain, uniformly in $k\le D$,
\begin{equation}
 \norm{z(t_k)-u_k}
 \le
 \left\{\frac{H_t}{\vartheta+1}h^\vartheta
       +\frac12L_zBh\right\}\Phi_{L_z}(T).
 \label{eq:app-holder-cont-mixed}
\end{equation}
This estimate also covers $L_z=0$ by the displayed definition of
$\Phi_{L_z}$.

\paragraph{Step 2: balanced pure schedule versus mixed Euler.}
Apply Lemma~\ref{lem:online-rounding-main} to $p_0,\ldots,p_{D-1}$ and let
$\sigma_0,\ldots,\sigma_{D-1}$ be the resulting pure schedule.  Set
\begin{equation}
 d_{k,j}:=\one_{\{\sigma_k=j\}}-p_{k,j},
 \qquad
 S_{n,j}:=\sum_{k=0}^{n-1}d_{k,j}.
 \label{eq:app-holder-discrepancy}
\end{equation}
Then $|S_{n,j}|<J$ for all $n,j$.  Along the mixed path, define
\begin{equation}
 f_k:=\sum_{j=1}^Jd_{k,j}G_j(t_k,u_k),
 \qquad
 F_n:=\sum_{k=0}^{n-1}f_k,
 \quad F_0=0.
 \label{eq:app-holder-forcing}
\end{equation}
For each fixed $j$, discrete Abel summation gives
\begin{equation}
 \sum_{k=0}^{n-1}d_{k,j}G_j(t_k,u_k)
 =S_{n,j}G_j(t_{n-1},u_{n-1})
  +\sum_{k=0}^{n-2}S_{k+1,j}
   \{G_j(t_k,u_k)-G_j(t_{k+1},u_{k+1})\}.
 \label{eq:app-holder-abel}
\end{equation}
The mixed Euler increments satisfy $\norm{u_{k+1}-u_k}\le Bh$, and therefore
\begin{equation}
 \norm{G_j(t_{k+1},u_{k+1})-G_j(t_k,u_k)}
 \le H_th^\vartheta+L_zBh.
 \label{eq:app-holder-grid-var}
\end{equation}
Combining \eqref{eq:app-holder-discrepancy}--\eqref{eq:app-holder-grid-var}
and summing over $j$ yields the prefix bound
\begin{equation}
 \max_{0\le n\le D}\norm{F_n}
 \le M_D
 :=J^2\{B+H_tT h^{\vartheta-1}+L_zBT\}.
 \label{eq:app-holder-prefix}
\end{equation}
When $\vartheta<1$, $M_D$ grows with $D$; the transformed-error argument below
multiplies it by $h$, which is precisely why the final rate is
$h^\vartheta$ rather than $h$.

Let $v_0=u_0$ and
\begin{equation}
 v_{k+1}=v_k+hG_{\sigma_k}(t_k,v_k)
 \label{eq:app-holder-pure}
\end{equation}
be the pure Euler path.  With $e_k:=v_k-u_k$,
\begin{equation}
 e_{k+1}=e_k+h\{G_{\sigma_k}(t_k,v_k)
                    -G_{\sigma_k}(t_k,u_k)\}+hf_k.
 \label{eq:app-holder-e-rec}
\end{equation}
Introduce the transformed error $\widetilde e_k=e_k-hF_k$.  Since
$F_{k+1}=F_k+f_k$,
\begin{equation}
 \widetilde e_{k+1}
 =\widetilde e_k+h\{G_{\sigma_k}(t_k,v_k)
                         -G_{\sigma_k}(t_k,u_k)\}.
 \label{eq:app-holder-transformed-exact}
\end{equation}
Because $e_k=\widetilde e_k+hF_k$,
\begin{equation}
 \norm{\widetilde e_{k+1}}
 \le(1+hL_z)\norm{\widetilde e_k}+h^2L_zM_D.
 \label{eq:app-holder-transformed-bound}
\end{equation}
Starting from $\widetilde e_0=0$ and using
\eqref{eq:app-holder-phi},
\begin{equation}
 \norm{\widetilde e_k}\le hL_zM_D\Phi_{L_z}(T),
 \qquad
 \norm{e_k}\le hM_D\{1+L_z\Phi_{L_z}(T)\}
              =hM_De^{L_zT}.
 \label{eq:app-holder-mixed-pure}
\end{equation}
Substituting \eqref{eq:app-holder-prefix} gives
\begin{equation}
 \max_k\norm{v_k-u_k}
 \le J^2e^{L_zT}
 \{Bh+H_tT h^\vartheta+L_zBT h\}.
 \label{eq:app-holder-switch-final}
\end{equation}

\paragraph{Step 3: the relaxed-to-pure directed distance.}
Combining \eqref{eq:app-holder-cont-mixed} and
\eqref{eq:app-holder-switch-final}, then substituting $h=T/D$, gives
\begin{align}
 \norm{v_D-z(T)}
 &\le \frac{T^\vartheta H_t}{D^\vartheta}
       \left\{\frac{\Phi_{L_z}(T)}{\vartheta+1}
                    +J^2Te^{L_zT}\right\}\nonumber\\
 &\quad+\frac{T}{D}
       \left\{\frac12L_zB\Phi_{L_z}(T)
                    +J^2(B+L_zBT)e^{L_zT}\right\}.
 \label{eq:app-holder-directed-rel-pure}
\end{align}
The right-hand side is exactly
$C_\vartheta^{\rm end}D^{-\vartheta}+C_1^{\rm end}D^{-1}$.
It is uniform over the relaxed control, so it bounds the directed distance
from the raw relaxed endpoint set to the pure endpoint set.

\paragraph{Step 4: the pure-to-relaxed directed distance.}
Conversely, fix any pure Euler schedule.  Define the measurable relaxed
control that equals the vertex $e_{\sigma_k}$ on $I_k$.  Its mixed Euler path
is exactly the prescribed pure path.  The continuous trajectory generated by
that vertex-valued control differs from the pure Euler path only by the
continuous-versus-mixed estimate \eqref{eq:app-holder-cont-mixed}, which is no
larger than the constants in \eqref{eq:app-holder-directed-rel-pure}.  This
proves the opposite directed distance.

The estimates hold for the raw endpoint sets.  If $A,B$ are two nonempty
bounded sets whose two directed distances are at most $r$, then the same is
true for $\overline A,\overline B$, because every point in a closure is a
limit of points in the raw set and the norm is continuous.  Thus the same
$r$ bounds the Hausdorff distance between the closed reachable sets.

\paragraph{Step 5: path norm and target distances.}
Linearly interpolate each Euler path.  On a cell, both a continuous relaxed
trajectory and an Euler interpolant move by at most $Bh$ from their left grid
states.  Therefore the grid-point estimate incurs at most an additional
$2Bh=2BT/D$ in the uniform path norm.  The same argument applies in both
Hausdorff directions.  Finally, for nonempty sets $A,B$,
\begin{equation}
 |\dist(x,A)-\dist(x,B)|\le d_{\rm H}(A,B),
 \label{eq:app-holder-distance-lipschitz}
\end{equation}
so the endpoint and path target-error statements follow.  This completes the
proof of Theorem~\ref{thm:holder-time-main}.

\subsection{Prescribed route times over a continuum input set}
\label{app:distributional-route-proof}

\begin{proof}[Proof of Corollary~\ref{cor:distributional-route-main}]
Fix $0\le s<t\le T$ and a state $z$ in the declared $L^p$ tube.  For almost
every input $\xi$, connect the routed field at time $s$ to that at time $t$ by
following the within-mode evolution and inserting one jump whenever an event
time $\tau_r(\xi)$ lies in $(s,t]$.  The pointwise triangle inequality gives
\begin{equation}
 \norm{G(t,z)(\xi)-G(s,z)(\xi)}_2
 \le H_0|t-s|
      +\sum_{r=1}^R\Delta_r
       \one_{\{\tau_r(\xi)\in(s,t]\}}.
 \label{eq:app-route-pointwise}
\end{equation}
The assumed jump bounds are uniform over the tube.  Minkowski's inequality and
\eqref{eq:event-time-concentration-main} imply
\begin{align}
 \norm{G(t,z)-G(s,z)}_{L^p(\nu)}
 &\le H_0|t-s|
   +\sum_{r=1}^R\Delta_r
      \nu\{\tau_r\in(s,t]\}^{1/p}\nonumber\\
 &\le H_0|t-s|
   +\sum_{r=1}^R\Delta_rK_r^{1/p}|t-s|^{\alpha/p}.
 \label{eq:app-route-lp}
\end{align}
Put $\vartheta=\alpha/p\in(0,1]$.  Since
$|t-s|\le T^{1-\vartheta}|t-s|^\vartheta$, the last display is bounded by
\begin{equation}
 \left\{H_0T^{1-\vartheta}
        +\sum_{r=1}^R\Delta_rK_r^{1/p}\right\}|t-s|^\vartheta.
 \label{eq:app-route-holder}
\end{equation}
The state-Lipschitz and tube properties are inherited from the within-mode
assumptions, so Theorem~\ref{thm:holder-time-main} applies in
$\Z=L^p(\nu;\mathbb R^q)$ and gives the claimed
$O(D^{-\alpha/p}+D^{-1})$ rate.  Notice that the proof never perturbs the
event surfaces: the route is prescribed.  State-driven route stability is
therefore correctly left to the hybrid small-gain theorem.
\end{proof}

\subsection{Exact common-lattice error feedback}
\label{app:common-lattice-proof}

\begin{proof}[Proof of Proposition~\ref{prop:bit-exact-error-feedback-main}]
All statements are coordinatewise, so it is enough to prove the scalar case
and then take the maximum over coordinates.  Because $v_k\in\delta\mathbb Z$,
$r_0=0\in\delta\mathbb Z$, and $q_k\in\Delta\mathbb Z\subset\delta\mathbb Z$,
the recursion implies inductively that
\begin{equation}
 r_k\in\delta\mathbb Z,
 \qquad \widehat z_k\in\Delta\mathbb Z
 \quad\text{for every }k.
 \label{eq:app-lattice-invariance}
\end{equation}
Nearest rounding gives
\begin{equation}
 |r_{k+1}|=|v_k+r_k-Q_\Delta(v_k+r_k)|\le\Delta/2.
 \label{eq:app-lattice-residual-bound}
\end{equation}
The fixed tie rule makes the update deterministic; the bound is valid for
either tie direction.

Write all stored values as integers in their native units.  From
\eqref{eq:app-lattice-residual-bound},
\begin{equation}
 |r_k/\delta|\le m/2\le M_r.
 \label{eq:app-lattice-r-int}
\end{equation}
Also
\begin{equation}
 |(v_k+r_k)/\delta|
 \le V/\delta+M_r\le M_a.
 \label{eq:app-lattice-a-int}
\end{equation}
Finally, nearest rounding can move its input by at most $\Delta/2$, and hence
\begin{equation}
 |q_k|\le |v_k+r_k|+\Delta/2\le V+\Delta,
 \qquad
 |q_k/\Delta|\le V/\Delta+1\le M_q.
 \label{eq:app-lattice-q-int}
\end{equation}
Induction on the state update gives
\begin{equation}
 |\widehat z_k/\Delta|
 \le M_{z,0}+kM_q
 \le M_{z,0}+DM_q.
 \label{eq:app-lattice-z-int}
\end{equation}
A signed $b$-bit two's-complement integer represents
$[-2^{b-1},2^{b-1}-1]$.  If
$b\ge1+\lceil\log_2(M+1)\rceil$, then
$2^{b-1}-1\ge M$, while the negative endpoint has at least the same
magnitude.  Applying this observation with $M_r$, $M_a$, and
$M_{z,0}+DM_q$ proves that every declared register is overflow free.

Rearranging the residual update gives
\begin{equation}
 \widehat z_{k+1}+r_{k+1}
 =\widehat z_k+q_k+v_k+r_k-q_k
 =\widehat z_k+r_k+v_k.
 \label{eq:app-lattice-one-step-conservation}
\end{equation}
Telescoping from $r_0=0$ yields the exact conservation law
\begin{equation}
 \widehat z_k+r_k
 =\widehat z_0+\sum_{i=0}^{k-1}v_i.
 \label{eq:app-lattice-telescope}
\end{equation}
Suppose the right-hand side at $k=D$ belongs to $\Delta\mathbb Z^q$.
Because $\widehat z_D\in\Delta\mathbb Z^q$, the conservation law implies
$r_D\in\Delta\mathbb Z^q$.  Together with
$\norm{r_D}_\infty\le\Delta/2$, this forces $r_D=0$: the only multiple of
$\Delta$ in the closed interval $[-\Delta/2,\Delta/2]$ is zero.  Hence
$\widehat z_D$ equals the ideal accumulated endpoint exactly.
\end{proof}

\paragraph{Boundary cases.}
The proof includes $D=1$ and $m=1$.  When $m=1$, the residual is identically
zero away from an impossible half-grid tie, because the fine and coarse
lattices coincide.  The bit-width formulas are sufficient rather than
minimal; their purpose is a portable overflow certificate.  Wrapping
arithmetic is excluded: the theorem certifies that no register reaches the
wrapping boundary.

\subsection{Exact fixed-target neural converses}
\label{app:neural-converse-proofs}

\begin{proof}[Proof of Proposition~\ref{prop:relu-mlp-fixed-target-main}]
For $z=(x,u)$,
\begin{equation}
 W_1z=x,
 \qquad
 W_2\operatorname{ReLU}(W_1z)
 =-\operatorname{ReLU}(x)e_1.
 \label{eq:app-mlp-eval}
\end{equation}
Assume $x_k\in[0,a]$.  Then $\operatorname{ReLU}(x_k)=x_k$ and the Euler step
with label $c_k\in\{0,a\}$ is
\begin{equation}
 x_{k+1}=x_k+\frac1D(c_k-x_k)
 =\left(1-\frac1D\right)x_k+\frac{c_k}{D},
 \qquad u_{k+1}=u_k.
 \label{eq:app-mlp-rec}
\end{equation}
The first expression is a convex combination of $x_k$ and $c_k$, so
$x_{k+1}\in[0,a]$.  Since $x_0=0$, induction proves invariance and reduces the
complete MLP to the scalar recursion of Theorem~\ref{thm:fixed-teacher-main}.
The continuous $c=a$ atom solves $\dot x=a-x$, giving
$x(1)=a(1-e^{-1})$, while all remaining coordinates equal $u_0$.  The exact
optimal error therefore equals the scalar formula.  Any product norm that
dominates the first coordinate preserves the same lower bound, and the
all-$a$ schedule attains it with no error in the remaining coordinates.
\end{proof}

\begin{proof}[Proof of Theorem~\ref{thm:attention-fixed-target-main}]
For $x,u\in\mathbb R$, define the affine plane
\begin{equation}
 X(x,u):=
 \begin{bmatrix}x+u&1\\x-u&-1\end{bmatrix}.
 \label{eq:app-attn-plane}
\end{equation}
Both the attention term and the bias in
\eqref{eq:attention-atoms-main} have zero second column, so every continuous
and Euler path starting from $X_0=X(0,u_0)$ remains on this plane.

On the plane,
\begin{equation}
 Q(X)=K(X)=\begin{bmatrix}1\\-1\end{bmatrix},
 \qquad
 QK^\top=\begin{bmatrix}1&-1\\-1&1\end{bmatrix}.
 \label{eq:app-attn-scores}
\end{equation}
Row-wise softmax therefore gives the constant, nonuniform matrix
\begin{equation}
 P=\begin{bmatrix}p&1-p\\1-p&p\end{bmatrix},
 \qquad p=\frac{e}{e+e^{-1}},
 \qquad 2p-1=\frac{e-e^{-1}}{e+e^{-1}}=\tanh(1)=\mu.
 \label{eq:app-attn-P}
\end{equation}
The value vector is $V=[x+u,x-u]^\top$.  Direct multiplication gives
\begin{equation}
 PV=\begin{bmatrix}x+\mu u\\x-\mu u\end{bmatrix},
 \qquad
 \mathcal A(X(x,u))
 =\begin{bmatrix}-x-\mu u&0\\-x+\mu u&0\end{bmatrix}.
 \label{eq:app-attn-field}
\end{equation}
Thus the atom with label $c\in\{-a,+a\}$ induces the decoupled coordinates
\begin{equation}
 \dot x=c-x,
 \qquad
 \dot u=-\mu u.
 \label{eq:app-attn-cont-coordinates}
\end{equation}
For the continuous $+a$ teacher,
\begin{equation}
 x^\star=a(1-e^{-1}),
 \qquad
 u^\star=u_0e^{-\mu}.
 \label{eq:app-attn-teacher}
\end{equation}
A pure depth-$D$ Euler schedule $c_0,\ldots,c_{D-1}$ obeys
\begin{align}
 x_{k+1}&=r x_k+\frac{c_k}{D},
 &r&:=1-\frac1D,\nonumber\\
 u_{k+1}&=r_\mu u_k,
 &r_\mu&:=1-\frac\mu D.
 \label{eq:app-attn-euler}
\end{align}
The disagreement endpoint is schedule independent:
$u_D=u_0r_\mu^D$.  The largest possible mean endpoint is produced by the
all-$+a$ schedule and equals
\begin{equation}
 x_D^{\max}=a(1-r^D).
 \label{eq:app-attn-xmax}
\end{equation}
Every other schedule has at least one $-a$.  Flipping the earliest, and hence
smallest-weight, symbol from $+a$ to $-a$ decreases the endpoint by
$2ar^{D-1}/D$; consequently every nonmaximal endpoint satisfies
\begin{equation}
 x_D\le x_D^{\max}-\frac{2a}{D}r^{D-1}.
 \label{eq:app-attn-second}
\end{equation}
It remains to prove that $x^\star$ lies in the Voronoi cell of
$x_D^{\max}$.  The elementary inequality
\begin{equation}
 \log(1-y)\ge-\frac{y}{1-y},
 \qquad 0<y<1,
 \label{eq:app-log-ineq}
\end{equation}
follows because $f(y)=\log(1-y)+y/(1-y)$ satisfies $f(0)=0$ and
$f'(y)=y/(1-y)^2\ge0$.  With $y=1/D$, it gives $r^{D-1}\ge e^{-1}$.  Therefore
\begin{equation}
 x_D^{\max}-x^\star
 =a(e^{-1}-r^D)
 \le\frac{a}{D}r^{D-1}.
 \label{eq:app-attn-voronoi}
\end{equation}
By \eqref{eq:app-attn-second}, the target is no farther from the maximal
endpoint than from any other endpoint.  Hence the all-$+a$ schedule is
globally optimal for the mean coordinate.

For two points on the invariant plane,
\begin{equation}
 \norm{X(x,u)-X(\bar x,\bar u)}_F^2
 =|x-\bar x+u-\bar u|^2+|x-\bar x-u+\bar u|^2
 =2|x-\bar x|^2+2|u-\bar u|^2.
 \label{eq:app-attn-frob}
\end{equation}
The optimal mean error is
$a\{e^{-1}-r^D\}=a\delta_1(D)$, and the disagreement error is
$u_0\{e^{-\mu}-r_\mu^D\}=u_0\delta_\mu(D)$.  Substitution into
\eqref{eq:app-attn-frob} proves \eqref{eq:attention-exact-main}.
The continuous teacher itself is an admissible relaxed trajectory, so the
relaxed structural floor is zero.

Finally, for each fixed $c\in(0,1]$,
\begin{align}
 D\log\left(1-\frac cD\right)
 &=-c-\frac{c^2}{2D}+O(D^{-2}),\nonumber\\
 \left(1-\frac cD\right)^D
 &=e^{-c}\left\{1-\frac{c^2}{2D}+O(D^{-2})\right\},
 \label{eq:app-attn-series}
\end{align}
so
\begin{equation}
 \delta_c(D)=\frac{c^2e^{-c}}{2D}+O(D^{-2}).
 \label{eq:app-attn-delta-series}
\end{equation}
Applying this with $c=1$ and $c=\mu$ in the exact Frobenius formula yields
\eqref{eq:attention-asymptotic-main}.
\end{proof}

\paragraph{Why $D\ge2$ is stated.}
For $D=1$, the disagreement Euler factor remains well defined, but the mean
comparison uses $r^{D-1}$ with $r=0$ and the two-point schedule set should be
checked separately.  The theorem is intended as a depth-refinement result and
therefore begins at $D=2$.  The executable verifier nevertheless includes the
relevant boundary arithmetic in its unit tests.

\subsection{Two-sided matching-depth scaling}
\label{app:matching-depth-proof}

\begin{proof}[Proof of Corollary~\ref{cor:matching-depth-two-sided-main}]
For every $D\ge D_0$, the definition of $D_0$ gives
$c_2/D^2\le c_-/(2D)$.  Hence the lower side of
\eqref{eq:two-sided-expansion-main} implies
\begin{equation}
 E_D-E_\infty\ge\frac{c_-}{2D}.
 \label{eq:app-match-lower}
\end{equation}
If a depth on this branch matches tolerance $E_\infty+g$, then
$c_-/(2D)\le g$, or $D\ge c_-/(2g)$.  Conversely, whenever
$D\ge D_\star$ and $D\ge c_+/g$, the upper side of
\eqref{eq:two-sided-expansion-main} gives $E_D\le E_\infty+g$; replacing the
real threshold by its ceiling yields the stated sufficient integer depth.

If $D_0>1$ and $0<g<\Delta_0$, no depth $D<D_0$ can match, by the definition
of $\Delta_0$.  The minimal matching depth must then lie on the asymptotic
branch, where the lower and upper bounds are constant multiples of $g^{-1}$.
If $D_0=1$, there is no finite exceptional set.  Finally, substituting
$g(L)=\Theta(L^{-1})$ gives $D=\Theta(L)$; without a two-sided comparator gap,
only a one-sided sufficient scaling follows.
\end{proof}

\subsection{Coupling-robust compositional HJB certificates}
\label{app:compositional-hjb-proof}

\begin{proof}[Proof of Theorem~\ref{thm:compositional-hjb-main}]
At terminal time, the integral in
\eqref{eq:compositional-value-main} vanishes.  Therefore
\begin{equation}
 v_{\rm comp}(T,z)
 =\sum_i v_i(T,z_i)
 \le\sum_i\phi_i(z_i)
 \le\phi(z).
 \label{eq:app-comp-terminal}
\end{equation}
For an atom $j$, differentiate the candidate on the product tube.  Since
$-\int_t^T\beta(s)\,ds$ has time derivative $+\beta(t)$,
\begin{align}
 &\partial_tv_{\rm comp}(t,z)
 +\sum_i\nabla_{z_i}v_{\rm comp}(t,z)^\top G_j(t,z)_i\nonumber\\
 &=\sum_i\left\{\partial_tv_i(t,z_i)
     +\nabla v_i(t,z_i)^\top g_{i,j}(t,z_i)\right\}
   +\sum_i\nabla v_i(t,z_i)^\top c_{i,j}(t,z)
   +\beta(t)\nonumber\\
 &\ge0+\{-\beta(t)\}+\beta(t)=0.
 \label{eq:app-comp-hjb}
\end{align}
Thus $v_{\rm comp}$ satisfies both hypotheses of
Theorem~\ref{thm:hjb-lower-main}, which gives
\begin{equation}
 E_\infty^\phi\ge v_{\rm comp}(0,z_0)
 =\sum_i v_i(0,z_{i,0})-\int_0^T\beta(s)\,ds.
 \label{eq:app-comp-value}
\end{equation}
For the stated continuous envelopes, if $\norm{\nabla v_i}_*\le a_i(t)$ and
$\norm{c_{i,j}}\le b_i(t)$, duality gives
$\nabla v_i^\top c_{i,j}\ge-a_i(t)b_i(t)$.  Summing over blocks verifies
\eqref{eq:coupling-budget-main} with
$\beta(t)=\sum_i a_i(t)b_i(t)$.
\end{proof}

\subsection{Independent executable checks}
\label{app:new-theory-verification}

The proofs above are analytic.  The release also contains two independent
programs whose role is defensive rather than inferential.  The attention
verifier enumerates all $2^D$ schedules for $2\le D\le12$, confirms the global
all-$+a$ optimum, checks the invariant-plane identities, and compares every
endpoint with \eqref{eq:attention-exact-main}.  The integer verifier checks
all residual, accumulator, state-range, conservation, and terminal-recovery
conditions over $2{,}991{,}904$ exhaustive scalar sequences and $500$
randomized multidimensional trials.  Their inputs, outputs, and hashes are
listed in Appendix~\ref{app:reproducibility} and the release
manifest.  Passing these tests cannot prove a theorem, but disagreement would
falsify either the derivation or the implementation and therefore blocks a
release.

\section{Soft-Threshold Analytic Details and Exact Certificates}
\label{app:soft-details}
\label{app:soft-threshold-proofs}

This appendix supplies the analytic details used in
Section~\ref{sec:soft-threshold-main}.  We restate the specialized results so
that every proof below has explicit hypotheses and quantifiers.  Throughout,
$\mathcal X=B_2^m(Y)$,
$\mathcal Z=C(\mathcal X;\mathbb R^d)$ with the uniform norm, and
\[
 T_j(z)(y)=\soft_{\lambda_j}(W_jy+S_jz(y))+A_jz(y),
 \qquad G_j(z)=T_j(z)-z.
\]
Assume the finite dictionary obeys
$\norm{S_j}_2+\norm{A_j}_2\le\rho<1$ and
$\norm{W_j}_2\le W_\star$.  Put
\begin{equation}
 \mu=1-\rho,
 \qquad R=\frac{W_\star Y}{\mu},
 \qquad L_G=1+\rho,
 \qquad B=2R.
 \label{soft:eq:constants}
\end{equation}

\begin{lemma}[Well-posedness, invariance, and input regularity]
\label[lemma]{soft:lem:wellposed}
Every reference, relaxed, and pure trajectory starting from zero remains in
$\norm z_\infty\le R$.  On this ball,
\[
 \norm{G_j(z)-G_j(\widetilde z)}_\infty
 \le L_G\norm{z-\widetilde z}_\infty,
 \qquad
 \norm{G_j(z)}_\infty\le B.
\]
Every measurable relaxed control induces a unique Carath\'eodory trajectory.
The corresponding terminal maps are Lipschitz in the input with constant at
most $W_\star/\mu$.
\end{lemma}

\subsection{Proofs of the Basic Analytic Properties}
\label{soft:app:wellposed}

\begin{proof}[Proof of \Cref{soft:lem:wellposed}]
Soft thresholding is nonexpansive, so
\[
  \norm{T_j(z)-T_j(\widetilde z)}_\infty
  \le (\norm{S_j}_2+\norm{A_j}_2)\norm{z-\widetilde z}_\infty
  \le \rho\norm{z-\widetilde z}_\infty.
\]
Thus $G_j=T_j-I$ is $(1+\rho)$-Lipschitz.  Moreover,
\[
  \norm{T_j(z)}_\infty\le W_\star Y+\rho\norm{z}_\infty.
\]
For an absolutely continuous relaxed trajectory, the upper Dini derivative obeys
\[
  D^+\norm{z(t)}_\infty
  \le W_\star Y-(1-\rho)\norm{z(t)}_\infty.
\]
Comparison with $r'=W_\star Y-\mu r$ gives $\norm{z(t)}_\infty\le R$.  The pure Euler update is
\[
  z^+=(1-h)z+hT_j(z),\qquad 0<h\le1,
\]
and maps the ball of radius $R$ into itself because
\[
  \norm{z^+}_\infty\le(1-h)R+h(W_\star Y+\rho R)=R.
\]
On this ball, $\norm{G_j(z)}_\infty\le\norm{T_j(z)}_\infty+\norm{z}_\infty\le2R=B$.  Standard Carath\'eodory theory in the Banach space $\Z$ gives existence and uniqueness for measurable relaxed controls because the vector field is uniformly Lipschitz and bounded.

For two inputs $y,\widetilde y$ under the same relaxed control, the difference $u(t)=z(t,y)-z(t,\widetilde y)$ satisfies
\[
  D^+\norm{u(t)}_2
  \le W_\star\norm{y-\widetilde y}_2-\mu\norm{u(t)}_2.
\]
Therefore
\[
  \norm{u(1)}_2
  \le \frac{W_\star(1-e^{-\mu})}{\mu}\norm{y-\widetilde y}_2
  \le\frac{W_\star}{\mu}\norm{y-\widetilde y}_2.
\]
For a pure or mixed Euler step with $0<h\le1$, the same comparison gives
\[
 \norm{u^+}_2\le(1-\mu h)\norm{u}_2
   +hW_\star\norm{y-\widetilde y}_2.
\]
Induction from $u_0=0$ yields the same $W_\star/\mu$ bound for every pure,
mixed, or reference Euler terminal map.
\end{proof}

\subsection{Signed-ray specialization}
Fix $B_W\in\mathbb R^{d\times m}$,
$B_S,A_0\in\mathbb R^{d\times d}$, $\bar\lambda>0$, and a finite nonnegative
gain alphabet $\mathcal C_b$.  For $c\in\mathcal C_b$, let
$W_c=cB_W$, $S_c=cB_S$, $A_c=A_0$, and
$\lambda_c=c\bar\lambda$.  Positive homogeneity gives
\begin{equation}
 G_c(z)=cT(z)+A_0z-z,
 \qquad
 T(z)(y)=\soft_{\bar\lambda}(B_Wy+B_Sz(y)).
 \label{soft:eq:ray-field}
\end{equation}
Define the scalar prefix-discrepancy radius
\begin{equation}
 \mathfrak d_b=
 \sup_{D\ge1}\sup_{c_0,\ldots,c_{D-1}\in\conv\mathcal C_b}
 \inf_{q_0,\ldots,q_{D-1}\in\mathcal C_b}
 \max_{0\le n\le D}
 \left|\sum_{k<n}(q_k-c_k)\right|.
 \label{soft:eq:scalar-discrepancy}
\end{equation}

\begin{theorem}[Signed-ray upper law and admissible-edge lower law]
\label[theorem]{soft:thm:signed-ray}
Let $c:[0,1]\to\conv\mathcal C_b$ be $K_c$-Lipschitz.  Put
\begin{align*}
 \tau&=\norm{B_S}_2,& a_0&=\norm{A_0}_2,&
 c_{\max}&=\max\mathcal C_b,\\
 R_c&=\frac{c_{\max}\norm{B_W}_2Y}
 {1-a_0-c_{\max}\tau},&
 B_T&=\norm{B_W}_2Y+\tau R_c,\\
 B_f&=2R_c,& L_f&=1+a_0+c_{\max}\tau,
\end{align*}
where $a_0+c_{\max}\tau<1$.  Then the continuous target generated by
$c(\cdot)$ has zero relaxed structural floor and a pure depth-$D$ schedule
satisfies
\begin{equation}
 E_b^{\mathrm{end}}(D;F_\infty^c)
 \le\frac{C_{\mathrm{disc}}+C_{\mathrm{quant}}\mathfrak d_b}{D},
 \label{soft:eq:ray-upper}
\end{equation}
where
\[
 C_{\mathrm{disc}}
 =\frac{(L_fB_f+K_cB_T)(e^{L_f}-1)}{2L_f},
 \qquad
 C_{\mathrm{quant}}=(B_T+\tau B_f)e^{L_f},
\]
with the continuous interpretation at $L_f=0$.

Assume additionally that $0\in\mathcal C_b$ and
$\Delta_0=\min(\mathcal C_b\cap(0,\infty))$ exists.  Suppose one scalar
channel is exposed: the corresponding rows of $B_S$ and $A_0$ vanish and,
for some input $y_0$,
$(B_Wy_0)_i-\bar\lambda=a>0$.  Then, for every $D\ge2$, there is an admissible
constant relaxed gain $c_D\in[0,\Delta_0]$ such that
\begin{equation}
 E_b^{\mathrm{end}}(D;F_\infty^{c_D})
 \ge\frac{a\Delta_0}{2D}\left(1-\frac1D\right)^{D-1}
 \ge\frac{a\Delta_0}{2eD}.
 \label{soft:eq:ray-lower}
\end{equation}
\end{theorem}

\begin{theorem}[Certified computation for rational PWA soft-threshold systems]
\label[theorem]{thm:computability}
Consider a soft-threshold residual system in fixed state and input dimensions
whose atom fields are supplied by finite rational piecewise-affine descriptions.
Assume the input domain is a rational polytope or rational Euclidean ball for
which certified rational nets can be constructed; the target path is supplied
by a finite rational piecewise-affine description or by a terminating
outward-rounded enclosure routine with a known rational modulus; and rational
bounds are supplied for a common state tube and for state, time, and input
regularity.  For every rational $\delta>0$, a terminating finite algorithm
returns
\[
 \underline E\le E_{\infty,b}^{\mathrm{end}}(F^\star)\le\overline E,
 \qquad
 0\le\overline E-\underline E\le\delta.
\]
The running time may be exponential in dimension, dictionary size, and
$1/\delta$.
\end{theorem}

\subsection{High-Resolution Euler Error}
\label{soft:app:euler}

Assume the weighted time regularity
\begin{equation}
  Y\norm{W(t)-W(s)}_2+R\norm{S(t)-S(s)}_2+R\norm{A(t)-A(s)}_2+
  \sqrt d\,|\lambda(t)-\lambda(s)|\le K_f|t-s|.
  \label{soft:eq:time-reg}
\end{equation}
Then $G_{\theta(t)}$ is $K_f$-Lipschitz in time on the invariant ball.  The residual Euler map
\[
  \Phi_{t,h}(z)=(1-h)z+hT_{\theta(t)}(z)
\]
is $(1-\mu h)$-Lipschitz for $0<h\le1$.  A local truncation estimate gives
\[
  \norm{z(t+h)-z(t)-hG_{\theta(t)}(z(t))}_\infty
  \le\frac{L_GB+K_f}{2}h^2.
\]
Therefore
\begin{equation}
  \norm{F_{D,\mathrm H}^\theta-F_\infty^\theta}_\infty
  \le \frac{L_GB+K_f}{2\mu D}.
  \label{soft:eq:euler-upper}
\end{equation}
For the scalar equation $x'=1-x$, Euler gives $x_D=1-(1-1/D)^D$ and the exact solution gives $1-e^{-1}$.  Using
\[
  D\log(1-1/D)\le -1-\frac{1}{2D}
\]
and $1-e^{-u}\ge u/2$ for $0\le u\le1$ yields, for every $D\ge2$,
\begin{equation}
  e^{-1}-\left(1-\frac1D\right)^D
  \ge \frac{1}{4eD}.
  \label{soft:eq:euler-lower}
\end{equation}

\subsection{Proof of the Signed-Ray Theorem}
\label{soft:app:signed-ray}

\begin{proof}[Proof of \cref{soft:thm:signed-ray}]
Because $c(t)\in\conv\mathcal C_b$, choose measurable simplex weights $p_j(t)$ satisfying $c(t)=\sum_jp_j(t)c_j$.  Equation~\eqref{soft:eq:ray-field} gives
\[
  G_{c(t)}(z)=\sum_jp_j(t)G_{c_j}(z),
\]
so the target flow is itself relaxed reachable and the infinite-depth floor is zero.

For the upper bound, compare the target Euler scheme with gains $c_k=c(k/D)$ to a pure scheme with gains $q_k\in\mathcal C_b$ whose prefix discrepancy is bounded by $\mathfrak d_b$.  On the invariant ball, $\norm{T(z)}\le B_T$, $\Lip(T)\le\tau$, and $\norm{G_c(z)}\le B_f$.  The high-resolution discretization term follows from the local truncation argument in \cref{soft:app:euler}, with $L_fB_f+K_cB_T$ in place of $L_GB+K_f$.

For the gain-quantization term, set $d_k=q_k-c_k$ and $R_n=\sum_{k<n}d_k$.  Abel summation gives
\begin{align*}
  \left\|\sum_{k<n}d_kT(z_k)\right\|_\infty
  &\le \mathfrak d_b B_T+
  \mathfrak d_b\sum_{k<n}\norm{T(z_{k+1})-T(z_k)}_\infty\\
  &\le \mathfrak d_b(B_T+\tau B_f).
\end{align*}
The transformed-error recursion from \cref{thm:floor-rate-main}, with Lipschitz constant $L_f$, gives the terminal contribution
\[
  \frac{(B_T+\tau B_f)e^{L_f}\mathfrak d_b}{D},
\]
which proves \eqref{soft:eq:ray-upper}.

For the lower bound, restrict to the exposed coordinate and input $y_0$.  Because the selected rows of $B_S$ and $A_0$ vanish, that coordinate obeys
\begin{equation}
  x_{k+1}=\left(1-\frac1D\right)x_k+\frac{a}{D}q_k,
  \qquad q_k\in\mathcal C_b.
  \label{soft:eq:exposed-recursion}
\end{equation}
Let $r=1-1/D$.  Since the gains are nonnegative, the zero schedule gives endpoint $0$.  The smallest positive pure endpoint is obtained by using the smallest positive gain $\Delta_0$ at the earliest step and zero elsewhere; it equals
\[
  g_D=\frac{a\Delta_0}{D}r^{D-1}.
\]
Choose the constant relaxed gain
\begin{equation}
  c_D=\frac{\Delta_0 r^{D-1}}{2D(1-e^{-1})}.
  \label{soft:eq:valid-hard-gain}
\end{equation}
Because $2D(1-e^{-1})\ge1$, one has $0\le c_D\le\Delta_0$, so this is an admissible target in $\conv\mathcal C_b$.  Its continuous endpoint is $a(1-e^{-1})c_D=g_D/2$.  Every pure endpoint is either zero or at least $g_D$, and therefore its error is at least $g_D/2$.  The inequality $r^{D-1}\ge e^{-1}$ proves \eqref{soft:eq:ray-lower}.
\end{proof}

\section{Certified Computation of the Infinite-Depth Low-Bit Floor}
\label{soft:app:computability}

\begin{proof}[Proof of \cref{thm:computability}]
The theorem is stated for finite rational piecewise-affine soft-threshold atom fields on a rational polytope or rational Euclidean ball with an effective certified-net construction.  The target path is either given by a finite rational piecewise-affine description or by a terminating outward-rounded enclosure routine with a known rational modulus.  A common rational state tube and rational state, time, and input regularity bounds are supplied, and every elementary operation is exactly rational or evaluated by outward-rounded interval arithmetic.

Fix the requested certificate width $\delta>0$.  Let $K_y$ be the maximum of the supplied target input-Lipschitz bound and the reachable-map bound $W_\star/\mu$ from \Cref{soft:lem:wellposed}.

\paragraph{1. Finite input reduction.}
Choose a rational $\eta$-net $Y_\eta=\{y_1,\ldots,y_N\}$ of the computably compact domain, where $\eta>0$ is the input-net radius.  For any two admissible terminal maps,
\begin{equation}
  \max_i\norm{F(y_i)-G(y_i)}_2
  \le \norm{F-G}_\infty
  \le \max_i\norm{F(y_i)-G(y_i)}_2+2K_y\eta.
  \label{soft:eq:net-bracket}
\end{equation}
For a rational Euclidean ball or rational polytope, invoke the certified rational-net construction required in the theorem statement.  Concretely, one may enumerate a sufficiently fine rational lattice in a rational outer box, retain interior points, and add rational interior approximants to boundary cells; exact rational distance bounds certify that every domain point lies within $\eta$ of a retained point.  This avoids assuming that an exact Euclidean projection of a rational point is rational.  The same control drives all $N$ inputs, so we propagate the joint state in $(\R^d)^N$ with the block maximum norm.

\paragraph{2. Finite approximation of relaxed controls.}
Partition $[0,1]$ into $N_t$ intervals and let $h=1/N_t$. Replace a measurable control by its interval averages. Define
\[
  C_E:=\frac{(L_GB+L_{t,G})(e^{L_G}-1)}{2L_G},
\]
with the continuous limit at $L_G=0$, where $L_{t,G}$ is the supplied time-Lipschitz bound for the specialized field family. The mixed Euler endpoint then differs from the relaxed endpoint by at most $C_E/N_t$.

Use the rational simplex grid
\[
  \Delta_{J,Q}=\left\{(n_1/Q,\ldots,n_J/Q):n_j\in\mathbb Z_{\ge0},\ \sum_jn_j=Q\right\}.
\]
Every $p\in\Delta_J$ has $q\in\Delta_{J,Q}$ with $\norm{p-q}_1\le2J/Q$.  If $z_k$ and $\widetilde z_k$ are mixed Euler trajectories driven by these controls, then
\[
  \norm{z_{k+1}-\widetilde z_{k+1}}
  \le(1+hL_G)\norm{z_k-\widetilde z_k}
  +hB\norm{p_k-q_k}_1,
\]
and hence
\begin{equation}
  \norm{z_{N_t}-\widetilde z_{N_t}}
  \le \frac{2BJ(e^{L_G}-1)}{L_GQ}.
  \label{soft:eq:simplex-grid-error}
\end{equation}
Let $E_{Y_\eta}$ denote evaluation of a terminal map on all points of the net $Y_\eta$. Thus the finite set $\mathcal S_{N_t,Q}$ of grid-controlled mixed Euler endpoints satisfies
\begin{equation}
  d_H\bigl(E_{Y_\eta}(\RR_b),\mathcal S_{N_t,Q}\bigr)
  \le \frac{C_E}{N_t}+\frac{2BJ(e^{L_G}-1)}{L_GQ}.
  \label{soft:eq:finite-relaxed-approx}
\end{equation}

\paragraph{3. Certified target enclosure.}
Evaluate the target terminal map jointly on $Y_\eta$ using its finite rational
piecewise-affine description or the supplied terminating outward-rounded
enclosure routine.  Denote the resulting radius by $\varepsilon_T(N_0)$,
where $N_0$ is the evaluator's refinement parameter.  A common special case is
a target generated by a rational time-dependent field with state Lipschitz
constant $L_T$, speed bound $B_T$, and time-Lipschitz constant $L_{t,T}$; then
an $N_0$-step rational Euler enclosure has radius
\[
  \varepsilon_T(N_0)=\frac{(L_TB_T+L_{t,T})(e^{L_T}-1)}{2L_TN_0},
\]
with the continuous limit when $L_T=0$.  In every allowed representation,
the assumed modulus guarantees $\varepsilon_T(N_0)\downarrow0$ effectively.

\paragraph{4. Finite exact distance and error bracket.}
Enumerate
\[
  \left|\Delta_{J,Q}\right|^{N_t}
  =\binom{Q+J-1}{J-1}^{N_t}
\]
grid-control sequences.  In the rational PWA setting their joint mixed Euler endpoints are rational.  Compute the minimum block-$\ell_\infty/\ell_2$ distance from the target enclosure to this finite set and enclose square roots by rational bisection.  If the finite computation returns $[d_-,d_+]$, set
\begin{align}
  \underline E&=\max\{0,d_- -\varepsilon_{\mathrm{tot}}\},\\
  \overline E&=d_+ +\varepsilon_{\mathrm{tot}},
\end{align}
where
\[
  \varepsilon_{\mathrm{tot}}
  =2K_y\eta+\frac{C_E}{N_t}
  +\frac{2BJ(e^{L_G}-1)}{L_GQ}
  +\varepsilon_T(N_0).
\]
Choose $\eta,N_t,Q,N_0$ and the square-root enclosure so that $2\varepsilon_{\mathrm{tot}}+(d_+-d_-)\le\delta$.  Since clipping the lower endpoint at zero can only reduce the interval width, this implies $\overline E-\underline E\le\delta$.  Every search is finite and all arithmetic is rational or outward rounded, so the procedure terminates.  A broader computable-analysis analogue is possible for computably compact domains and computable fields with known moduli, but is not needed for the certified claims in this paper.
\end{proof}

The six certified operations and their individual error contributions are summarized in \cref{soft:tab:certified-floor}.

\begin{table}[t]
\centering
\caption{Certified infinite-depth-floor approximation for the rational soft-threshold/PWA class.  The parameters are increased until the sum of the explicit error contributions is at most the requested interval width $\delta$.}
\label{soft:tab:certified-floor}
\small
\begin{tabular}{Y{0.07\linewidth}Y{0.42\linewidth}Y{0.40\linewidth}}
\toprule
Step & Certified operation & Error contribution\\
\midrule
1 & Construct a rational $\eta$-net of the computably compact input domain and propagate every net point under the same control. & Input extension: $2K_y\eta$.\\
2 & Partition time into $N_t$ intervals and replace each measurable control by its interval average. & Mixed-Euler defect: $C_E/N_t$.\\
3 & Quantize interval controls to the rational simplex grid $\Delta_{J,Q}$. & Control-grid defect: $2BJ(e^{L_G}-1)/(L_GQ)$.\\
4 & Enclose the target with $N_0$ rational or outward-rounded integration steps. & Target radius: $\varepsilon_T(N_0)$.\\
5 & Enumerate the finite grid-control family and compute its minimum joint distance to the target enclosure. & Rational square-root enclosure.\\
6 & Enlarge the finite minimum by the preceding terms and clip the lower endpoint at zero. & Final width bounded by $2\varepsilon_{\mathrm{tot}}+(d_+-d_-)\le\delta$.\\
\bottomrule
\end{tabular}
\end{table}

\section{Rational Certificate Methodology for the Matrix-Valued Example}
\label{soft:app:pwa}

The matrix certificate in \cref{sec:matrix-cert} combines exact schedule enumeration with a piecewise-affine domain bracket.

\paragraph{Witness lower bounds.}
For a fixed depth, all $2^D$ schedules are evaluated on a polar witness set fixed before schedule evaluation of $1281$ points.  The minimum of the resulting maxima is the exact global optimum on that set and therefore a lower bound on the full-disk optimum.

\paragraph{Full-disk upper and lower bounds.}
An exact rational $n$-gon with vertices on the unit circle gives an inscribed domain subset and hence a lower bound on a supremum.  Intersections of the corresponding tangent half-spaces give a rational circumscribed polygon and hence an upper bound.  For the student, high-resolution model, and reference, the algorithm propagates affine maps jointly through every soft-threshold activation split.  On each resulting polygonal cell, the difference is affine and its Euclidean norm is convex, so its maximum occurs at a cell vertex.  The $n=16$ bounds reported in \eqref{eq:D7-fail}--\eqref{eq:D8-win} involve $2432$ inner and $2504$ outer joint cells.

\paragraph{Exact rational data.}
The teacher matrices are
\[
W^{(0)}=\begin{bmatrix}13/20&0\\0&13/20\end{bmatrix},\quad
W^{(1)}=\begin{bmatrix}0&-13/20\\13/20&0\end{bmatrix},
\]
\[
S^{(0)}=\begin{bmatrix}1/10&0\\0&-1/10\end{bmatrix},\quad
S^{(1)}=\begin{bmatrix}0&1/10\\-1/10&0\end{bmatrix},\quad
(\lambda^{(0)},\lambda^{(1)})=(7/20,0).
\]
The deterministic lower-tie 4-bit quantizer produces
\[
\widehat W^{(0)}=\begin{bmatrix}13/20&-13/300\\-13/300&13/20\end{bmatrix},\quad
\widehat W^{(1)}=\begin{bmatrix}-13/300&-13/20\\13/20&-13/300\end{bmatrix},
\]
\[
\widehat S^{(0)}=\begin{bmatrix}1/10&-1/150\\-1/150&-1/10\end{bmatrix},\quad
\widehat S^{(1)}=\begin{bmatrix}-1/150&1/10\\-1/10&-1/150\end{bmatrix},\quad
(\widehat\lambda^{(0)},\widehat\lambda^{(1)})=(49/150,0).
\]
The successful schedule is $00000101$; the best depth-$7$ witness schedule is $0100001$.

\paragraph{Audit trail.}
The exact squared bounds are stored as rational numerators and denominators in \path{certificates/exact_matrix_certificate.json}.  The script \path{code/verify_all_certificates.py} checks every $D\le7$ witness inequality and the decisive depth-$8$ outer-versus-inner inequality; its \path{--recompute-pwa} option regenerates the rational polygonal decomposition from the matrices above.  Dense heatmaps in \cref{fig:matrix-cert} are visualizations only and are not used in the proof.

\section{Additional Experimental and Search Details}
\label{app:experiments}

This appendix records the parameter choices, finite search domains, numerical
references, and evidence boundaries behind Section~\ref{sec:certified-evidence-main}.
Every plotted curve is generated from a fixed input family or a fixed
pretrained checkpoint declared before the depth sweep.  Synthetic first-order
claims are supported by exact formulas, exhaustive search, or plateaus in
$DE_D$ across depth doublings rather than by an unconstrained regression alone.

\subsection{Small attention and MLP systems}
The small attention experiments use deterministic low-dimensional token
systems, one or two complete residual atoms, regularized normalization, and a
small tanh MLP.  Exact matrices, input families, and seeds are recorded in the
released experiment scripts.  The input family combines structured sequences
with seeded random matrices projected to the declared norm ball.  Continuous
references use a substantially finer RK4 grid than any displayed student
depth.  Unless a caption states an analytic invariant-subspace floor, plotted
errors are maxima over the fixed witness family.

\subsection{Codebook resource accounting}
The attention-output codebook study charges every complete atom by its stored
quantized scalars and charges a depth-$D$ pure program by its schedule bits.
This intentionally simple storage model isolates dictionary geometry and
schedule length.  It does not claim a latency or energy model: alignment,
packing, activation traffic, kernel launch cost, and platform-specific
metadata are excluded unless explicitly named.  The resource conclusions are
therefore comparisons under the declared accounting, not universal hardware
rankings.

\subsection{Exhaustive noncommuting search}
For two complete noncommuting attention--MLP atoms, every one of the $2^D$
schedules is evaluated through $D=16$ on a witness set fixed before schedule
evaluation.  The target is a high-resolution integration of the equal relaxed
mixture.  Coordinate descent flips one schedule bit at a time until no
improving flip remains and is restarted from seeded random schedules.  The
balanced schedule is generated by the online rounding rule of
Lemma~\ref{lem:online-rounding-main}.  Global optima, balanced schedules, and
local-search distributions are retained separately.

\subsection{Numerical precision and visual generation}
Synthetic simulations use double-precision NumPy unless the released manifest
states otherwise; the medium-scale attention experiment uses FP32 PyTorch with
a high-resolution reference.  Random seeds are fixed.  Source
tables are released alongside the figure scripts.

\section{Attention, Transformer, and Routed-System Settings}
\label{app:transformer-experiments}

The attention and routing studies are deterministic and generated by the
released scripts.  The normalized state domains, atom matrices, gain
alphabets, router margins, schedule searches, and random seeds are recorded in
the data manifest.  The principal settings are summarized below.

\begin{table}[ht]
\centering
\caption{Transformer and routed-system experiment settings.  All precision
codebooks are finite and fixed before each depth sweep.}
\label{tab:transformer-settings}
\small
\begin{tabular}{p{0.22\linewidth}p{0.24\linewidth}p{0.45\linewidth}}
\toprule
Experiment & State size & Design\\
\midrule
Attention structural split & two tokens, width two & Two state-dependent attention atoms; one compatible target and one missing-output-direction target\\
Gain-alphabet rate & two tokens, width two & Binary, ternary, uniform 2-bit, and uniform 3-bit scalar gain alphabets\\
Stability sweep & two tokens, width two & Normalization regularizer in $\{1/16,1/4,1\}$ and residual-block scale in $\{0.5,1,1.5\}$\\
Codebook phase diagram & two tokens, width two & Binary through 4-bit output codebooks with exact infinite-depth floors\\
Hard routing diagnostic & two tokens, two experts & Fixed-route margin sweep plus the separate transversal-event illustration of Figure~\ref{fig:hybrid-route-window-main}\\
Medium scale & $n=8$, width $16$, two heads & 96 held-out random input sequences\\
Scale sweep & $(4,8,1)$ to $(16,32,4)$ & First-order compatible curves and orthogonal infinite-depth floors\\
Noncommuting blocks & two tokens, width two & Exhaustive schedules through $D=16$ and seeded local-search restarts\\
\bottomrule
\end{tabular}
\end{table}

\begin{table}[ht]
\centering
\caption{Scale-sweep summary.  $E_D$ is the maximum terminal error at depth
$D$; the compatible case is summarized through a representative large-depth
product $DE_D$, and the orthogonal column is the invariant-subspace floor.}
\label{tab:scale-sweep-summary}
\small
\begin{tabular}{cccc}
\toprule
Tokens $n$ & Width $d$ & Representative $DE_D$ & Orthogonal floor\\
\midrule
4 & 8 & $3.45\times10^{-3}$ & $0.1828$\\
8 & 16 & $3.77\times10^{-3}$ & $0.2589$\\
16 & 32 & $6.55\times10^{-3}$ & $0.3389$\\
\bottomrule
\end{tabular}
\end{table}

\section{Pretrained DistilBERT Study}
\label{app:pretrained-experiment}

The pretrained experiment uses the public
\path{distilbert-base-uncased-finetuned-sst-2-english} checkpoint at the
resolved commit recorded in the release manifest
\citep{SanhEtAl2019DistilBERT,HuggingFaceDistilBERTSST2}, through the
Transformers software stack \citep{WolfEtAl2020Transformers}.  The evaluation
task is SST-2 from GLUE \citep{WangEtAl2019GLUE,SocherEtAl2013SST}.  We study
the feed-forward residual branches before output LayerNorm in Transformer
layers 1 and 3.  The direct target is the original one-step branch.  The
coherent target is a depth-32 high-precision Euler refinement of the same
pretrained residual field over horizon $T=1$.

For each layer, student depths are $D\in\{4,8,16\}$.  A learned dictionary
contains two complete 4-bit feed-forward residual atoms.  Quantization-aware
training uses 512 training examples, 128 calibration examples, 350 steps, and
three fixed seeds.  The primary pure program is obtained by hardening the
learned time-dependent relaxed controls and applying two local schedule
refinement passes.  All 872 SST-2 validation examples are used for full-model
evaluation.  Student-$t$ intervals summarize seed variation; the release also
contains a hierarchical bootstrap over seeds and examples.

\begin{table}[ht]
\centering
\caption{Pretrained-model experiment configuration.}
\label{tab:pretrained-settings}
\small
\begin{tabularx}{\linewidth}{@{}Y{.23\linewidth}X@{}}
\toprule
Item & Setting\\
\midrule
Model & Public DistilBERT SST-2 checkpoint; resolved revision recorded in the release\\
Residual branches & Feed-forward branches in layers 1 and 3\\
Reference construction & Horizon $T=1$, high-precision depth 32\\
Student depths & $D\in\{4,8,16\}$\\
Dictionary & Two learned 4-bit feed-forward residual atoms\\
QAT seeds & Three fixed seeds listed in the manifest\\
QAT data and steps & 512 train, 128 calibration, 350 steps\\
Evaluation & All 872 SST-2 validation examples\\
Primary pure schedule & Hardening plus two local-refinement passes\\
Stored-parameter ratio & Approximately $3.97\times$ under the declared accounting\\
\bottomrule
\end{tabularx}
\end{table}

The first-depth direct errors are small but nonzero because hidden states and
model operations are evaluated in finite floating-point arithmetic.  This does
not affect the causal contrast, which spans several orders of magnitude.  The
pure-schedule compilation gaps relative to the optimized relaxed programs are
shown in Figure~\ref{fig:pretrained-compilation}.  At layer 3 they remain below
approximately $0.22\%$ at every tested depth; at layer 1 they fall below $1\%$
by $D=16$.

\begin{figure}[ht]
\centering
\includegraphics[width=.72\linewidth]{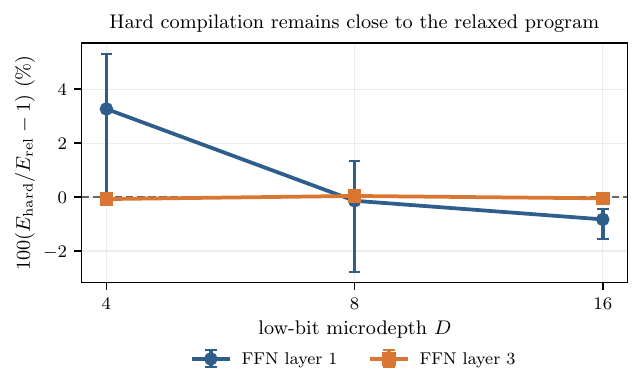}
\caption{Relaxed-to-pure compilation gap in the pretrained experiment.  Bars
show $(E_{\rm hard}/E_{\rm relaxed}-1)\times100\%$ for locally refined pure
schedules.  The small gaps demonstrate that the principal positive result is
realized by pure programs rather than only by convex mixtures.}
\label{fig:pretrained-compilation}
\end{figure}

The compact release contains the complete seed-level trajectories, candidate
metrics, confidence intervals, depth contrasts, per-example arrays, environment
record, and resolved model/data manifest.  Large model caches and resumable QAT
checkpoints are deliberately not duplicated in the compact archive.  The study
quantizes the feed-forward branch dictionary while hidden states, nonlinear
operations, and accumulation use floating point.  It is therefore evidence for
target coherence, dictionary learning, and pure compilation---not a claim of
integer-only latency or energy improvement.

\section{Supplementary Multi-Teacher Matrix Study}
\label{app:multiteacher}

To test whether the exact matrix certificate is an isolated construction, we
also evaluate eight independently specified $2\times2$ recurrent teachers,
five precision models, and nine depths.  The campaign contains 360 globally
optimized schedule problems and 255 whole-input interval-certificate rows.  It
finds exact matches, certified feasible-depth upper bounds, all-depth
floor-impossibility cases, and unresolved cases at the tested certification
cap.  The study is deliberately secondary: its median matching-depth
prediction factor is $4.0$, so it demonstrates breadth of the mechanisms rather
than a universally calibrated predictor.

\begin{figure}[ht]
\centering
\includegraphics[width=\linewidth]{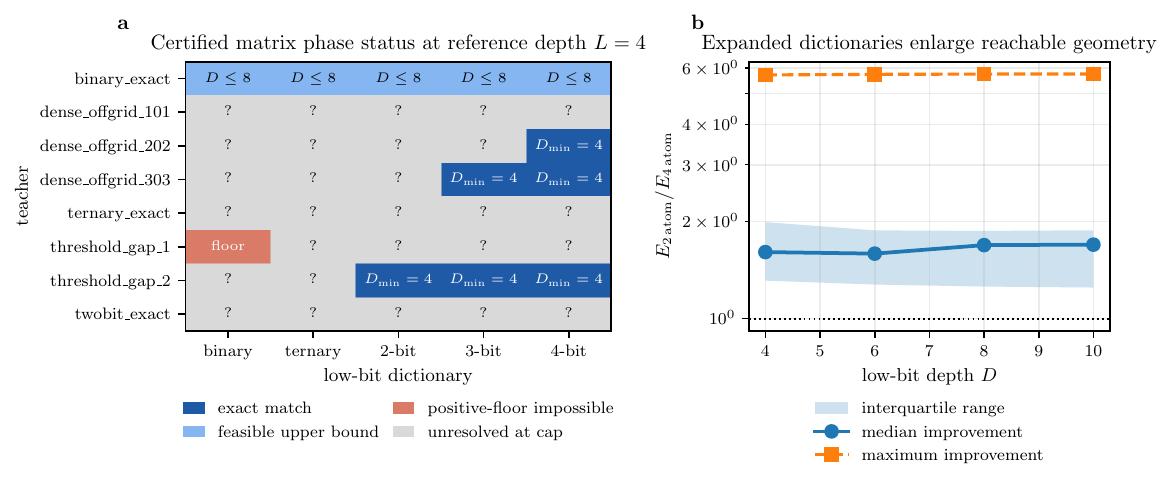}
\caption{Supplementary breadth evidence.  \textnormal{(a)} The eight-teacher
matrix campaign contains exact matches, certified feasible upper bounds,
positive-floor impossibilities, and unresolved cases.  \textnormal{(b)}
Expanding the complete dictionary from two to four atoms improves the globally
optimized error by a median factor of roughly $1.7$ and by more than $5$ in the
strongest cases.  Dictionary geometry, not nominal bit count alone, controls
both relaxed reachability and the finite-depth constant.}
\label{fig:broad-matrix}
\end{figure}

\section{Reproducibility Checklist and Artifact Map}
\label{app:reproducibility}

Every headline theorem is tied to an analytic proof, a rational or symbolic
certificate, or a declared finite-grid experiment.  Table~\ref{tab:repro} gives
the shortest route from each numerical claim to the released artifact.  Paths
are relative to the source-bundle root.

\begin{table}[ht]
\centering
\caption{Mapping from claims to released artifacts and verification entry
points.}
\label{tab:repro}
\small
\begin{tabular}{p{0.28\linewidth}p{0.62\linewidth}}
\toprule
Claim or figure & Released artifact and verification entry point\\
\midrule
Resource arithmetic phase & \path{data/arithmetic_resource_phase_v8.csv} and \path{code/make_resource_figures_v8.py}\\
Fixed-teacher necessity and metadata surface & \path{data/fixed_teacher_exact_v8.csv}, \path{data/metadata_depth_surface_v7.csv}, and the same figure script\\
Hybrid event-window illustration & \path{data/hybrid_route_window_v8.csv} and the same figure script\\
Soft-threshold structural split & \path{data/structural_split.csv} and \path{code/make_figures.py}\\
Signed-ray rate & \path{data/minimax_edge_law.csv}; exact formula checked by the theory sanity suite\\
Matrix $D_{\rm match}=2L$ & \path{certificates/exact_matrix_certificate.json}, \path{certificates/exact_witness_D1_D7.json}, and \path{code/verify_exact_matrix.py}\\
Attention and routing studies & Files beginning \path{data/transformer_}; regenerate with the transformer scripts\\
Medium-scale attention & \path{data/transformer_medium_scale_transformer.csv}\\
Pretrained DistilBERT causal study & \path{data/pretrained/distilbert_depth_precision_robustness_paper_5471683e8f56.zip}, \path{data/pretrained/README_RESULTS.md}, and \path{data/pretrained/model_and_data_manifest.json}\\
Broad matrix stress test & \path{data/matrix_frontier_raw.csv} and \path{data/certified_phase_diagram.csv}\\
Proof and reference audits & \path{audits/When_Can_Depth_Replace_Precision_Proof_Audit_v8.csv}, \path{audits/When_Can_Depth_Replace_Precision_Reference_Audit_v8.csv}, and \path{audits/When_Can_Depth_Replace_Precision_Preflight_v8.md}\\
\bottomrule
\end{tabular}
\end{table}

\section{QReplace: an auditable replacement recommender}
\label{app:qreplace-software}

This appendix specifies the practitioner-facing software companion.  QReplace
is not a learned oracle and does not infer a structural floor from an
architecture name.  It is an auditable decision engine for the resource laws
proved in the paper.  A user supplies the target, operation library, finite
resource grid, execution semantics, routing model, cost model, and the actual
evidence available for each bound.  The software returns a configuration, an
error interval, an evidence level, and the assumptions that prevent or permit
a certified conclusion.

\subsection{Input contract}
\label{app:qreplace-input}

A replacement study is represented by a versioned YAML or JSON document.  The
principal fields are summarized in \cref{tab:qreplace-schema}; the complete
JSON schema is distributed with the package and validated in continuous
integration.

\begin{table}[ht]
\centering
\caption{QReplace input contract.  Model metadata and theorem evidence are
kept separate so that descriptive checkpoint information cannot silently
become a certified constant.}
\label{tab:qreplace-schema}
\small
\begin{tabularx}{\linewidth}{@{}Y{.18\linewidth}Y{.31\linewidth}X@{}}
\toprule
Block & Required content & Trust boundary\\
\midrule
Target & map or block name, architecture, domain, endpoint/path objective,
metric, and tolerance & a checkpoint identifier does not determine a
structural floor or a common tube\\
Dictionary & atom count, bit format, learned/fixed status, family dimension,
metadata storage, and shared-schedule convention & input-dependent schedules
must not be used under a shared-schedule theorem\\
Structural floor & lower/upper bracket and evidence provenance & finite-witness
claims remain scoped to the hash-locked witness lift\\
Synthesis & BV or H\"older regularity and constants & constants must be uniform
on the declared tube and independent of the tested depth\\
Execution & write-back, increment, error-feedback, or custom semantics;
quantizer ranges; carry; saturation/wraparound & wrapping is uncertified
without a no-overflow invariant; saturation requires an implemented tube\\
Routing & no routing, prescribed events, or state-dependent event data & a
finite hybrid radius requires the small-gain and event-isolation premises\\
Costs and search & depth/metadata grid and storage, latency, activation, carry,
and schedule weights & cost proxies are not hardware measurements unless
measured values are supplied\\
Certificates & paths, claim identifiers, and required status & a formal proof
checks the encoded implication, not checkpoint fidelity\\
\bottomrule
\end{tabularx}
\end{table}

Every numerical premise carries an evidence label:
\emph{analytic certified}, \emph{interval certified}, \emph{finite-witness
certified}, \emph{calibrated estimate}, \emph{diagnostic}, or
\emph{unavailable}.  The ordering is intentional: the weakest nonzero premise
used by a candidate determines its evidence level.  A configuration is called
certified only when the structural bracket and every nonzero finite-resource
term are certified and their assumptions are marked verified.

\subsection{Decision algorithm}
\label{app:qreplace-algorithm}

For each candidate depth $D$ and metadata budget $s$, the software evaluates
\begin{equation}
 R_{D,s}=R_{\mathrm{syn}}(D)+R_{\mathrm{meta}}(s)
          +R_{\mathrm{arith}}(D)+R_{\mathrm{route}}(D).
 \label{eq:qreplace-radius}
\end{equation}
The choices implemented in the release are exactly the formulas stated in the
paper:
\begin{align*}
 R_{\mathrm{syn}}(D)
   &= C_{\mathrm{syn}}D^{-1}
      &&\text{(BV)},\\
 R_{\mathrm{syn}}(D)
   &= C_{\vartheta}D^{-\vartheta}+C_1D^{-1}
      &&\text{(H\"older)},\\
 R_{\mathrm{meta}}(s)
   &= C_{\mathrm{meta}}2^{-s/m},\\
 R_{\mathrm{writeback}}(D)
   &=D\rho_zS+\eta_G\Phi_{L_z}(T),\\
 R_{\mathrm{EF}}(D)
   &=\rho_{\mathrm{carry}}e^{L_zT}
     +\eta_G\Phi_{L_z}(T)
     +D\rho_{\mathrm{reg}}e^{L_zT}.
\end{align*}
Prescribed routing uses the declared H\"older event term, while state-dependent
routing uses $b_D/(1-\chi)$ only when $\chi<1$.  If a certified floor bracket
satisfies
\[
 L_\infty\le E_{\Omega,\infty}(F^\star)\le U_\infty,
\]
then QReplace reports
\begin{equation}
 \max\{0,L_\infty-R_{D,s}\}
 \le E_{D,s}^{\mathrm{impl}}(F^\star)
 \le U_\infty+R_{D,s}.
 \label{eq:qreplace-interval}
\end{equation}
The decision logic is consequently proof transparent:
\begin{align*}
 U_\infty+R_{D,s}\le\eps
   &\Longrightarrow \text{certified feasible},\\
 L_\infty-R_{D,s}>\eps
   &\Longrightarrow \text{certified impossible at }D,\\
 L_\infty>\eps
   &\Longrightarrow \text{certified asymptotically impossible}.
\end{align*}
A closing noncertified upper envelope is reported as conditionally feasible;
a near-closing diagnostic envelope is reported as diagnostically promising;
all remaining cases are unresolved.  A certificate reference marked
\emph{required} is a separate trust gate: a missing claim or a status other
than \texttt{kernel\_checked} automatically prevents a certified output even
when all numerical inequalities close.  The software never converts a failed
optimizer into an impossibility result.

Candidates are ranked by decision class and declared scalar cost.  The report
also returns the nondominated cost--upper-error Pareto frontier, so deployment
teams may apply a different operational objective without recomputing the
mathematical bounds.

\subsection{Notebook runtime and installation}
\label{app:qreplace-runtime}

The four notebooks use a CPU runtime with no hardware accelerator.  Each
notebook installs and validates QReplace before importing it and accepts
either the source ZIP or the prebuilt wheel; this avoids reliance on notebook
working-directory conventions.  Typical fresh Colab runtimes are 2--4 minutes
for the quick-start and arithmetic notebooks, 3--6 minutes for guided intake,
and 3--8 minutes for the PyTorch model-card notebook.  Larger user-supplied
checkpoints or wider candidate grids can take longer.  The notebooks display
the imported package version and location before running an analysis, and the
quick-start notebook exports a complete report ZIP.

\subsection{Interfaces and generated evidence}
\label{app:qreplace-interfaces}

The release provides:
\begin{enumerate}[leftmargin=1.75em,itemsep=.25em]
\item a Python API and typed data model;
\item a command-line interface for initialization, validation, search, schema
      export, certificate inspection, and arithmetic phase demonstrations;
\item a browser interface built on the same decision engine;
\item four Colab notebooks covering quick-start replacement, write-back versus
      error feedback, guided certificate intake, and PyTorch model cards;
\item JSON, CSV, Markdown, HTML, and figure outputs for every recommendation;
\item SHA-256 certificate provenance and an interoperable QReplaceLean claim
      envelope.
\end{enumerate}
The packaged model-card helper records parameter counts, module types, data
types, and checkpoint hashes, but deliberately refuses to infer Lipschitz,
tube, reachability, or routing constants from those descriptors.

The arithmetic notebook contains a transparent scalar simulator for ideal,
full-state write-back, independent increment, and error-feedback semantics.  It
is a controlled mechanism check rather than a representation certificate.  The
example reproduces the qualitative phase transition in
\cref{fig:arithmetic-resource}: once an ideal increment falls below the state
half-cell, state write-back freezes, whereas the error-feedback endpoint plus
carry continues to satisfy the conservation identity.

\subsection{Software tests and reproducibility}
\label{app:qreplace-tests}

The release test suite contains 20 regression tests covering schema validation,
all decision classes, structural-impossibility precedence, write-back's
interior optimum, wraparound downgrading, hybrid small-gain failure, Pareto
construction, report generation, portable relative certificate paths,
required-certificate fail-safe downgrading, starter-spec generation, scalar
error-feedback conservation, and the PyTorch model-card hash path.  All four
notebooks are executed in the package
preflight.  A clean example run produces a certified finite-witness Transformer
recommendation, a conditionally feasible routed-MoE recommendation, and the
arithmetic phase data and figure.  These are software regression tests and
demonstrations; their constants are not substituted for checkpoint-specific
claims elsewhere in the paper.

\section{QReplaceLean: staged formal verification}
\label{app:lean-verification}

QReplaceLean is a focused Lean~4 project for seven exact result groups that are
central to the manuscript and directly connected to the QReplace decision
interface.  The accompanying evidence archive marks all twelve mapped
claim-level statements \texttt{kernel\_checked}.  Every theorem module and the
integrated project build succeeded, and every mapped declaration passed its own
\texttt{\#print axioms} audit.  The distributed project contains only modules
that belong to the current theorem crosswalk; unfinished auxiliary
formalizations are not used to support claims.

\subsection{Runtime, environment, and verification protocol}
\label{app:lean-status}

The project pins Lean~4 version \texttt{v4.30.0}, a specific Mathlib commit,
and the complete Lake dependency lock.  The Colab notebook uses a CPU runtime;
a GPU or TPU does not accelerate Lean compilation.  A fresh run typically
takes 10--25 minutes and uses approximately 4--8 GB of local runtime disk,
primarily for the toolchain and Mathlib artifacts.  A repeated run with cached
dependencies typically takes 4--10 minutes.

The notebook verifies the release in eight visible stages:
\begin{enumerate}[leftmargin=1.75em,itemsep=.2em]
\item source-policy and integrity audit;
\item fixed-grid freezing;
\item error-feedback conservation and terminal bounds;
\item common-lattice recovery and signed-register range;
\item feasible/impossible decision implications;
\item metadata--depth packing;
\item fixed-teacher recursion, schedule bound, and optimality closing step;
\item two-token attention modes.
\end{enumerate}
Each theorem stage prints the manuscript result, Lean target, elapsed time,
completed count, and remaining stages, and writes a separate build log and
claim-specific \texttt{\#print axioms} log.  The notebook then performs an
integrated consistency build, generates the verification manifest, and
downloads a ZIP containing the manifest, exact sources, stage logs, axiom
audits, theorem crosswalk, and faithfulness audit.

The source-policy audit rejects \texttt{sorry}, \texttt{admit}, custom theorem
axioms, unsafe declarations, and missing source files.  Claim status is either \texttt{kernel\_checked} or
\texttt{not\_kernel\_checked}.  The artifact-level summary takes one of three
machine-readable values:
\begin{center}
\path{source_audited_not_kernel_checked},\quad
\path{partially_kernel_checked},\quad
\path{kernel_checked_all}.
\end{center}
The archived manifest supplied with this paper has artifact status
\texttt{kernel\_checked\_all}: twelve of twelve mapped claims are marked
\texttt{kernel\_checked}.  The theorem-source digest is
\texttt{cc4358adbc21...f92f521}.
The project-pinned build logs invoke Lean~4 \texttt{v4.30.0} and Mathlib commit
\texttt{c5ea00351c28...b1188f}.  The axiom audits report only
\texttt{propext}, \texttt{Quot.sound}, and, where classical reasoning is used,
\texttt{Classical.choice}.  No custom theorem axioms, \texttt{sorry},
\texttt{admit}, or \texttt{unsafe} declaration occurs in the audited theorem
sources.

\subsection{Theorem crosswalk}
\label{app:lean-crosswalk}

\Cref{tab:lean-crosswalk} summarizes the formalized scope.  The release also
contains a declaration-level crosswalk and a statement-by-statement
faithfulness audit.

\begin{table}[ht]
\centering
\caption{QReplaceLean crosswalk.  The archived
\texttt{verification\_results.json} marks every listed group kernel checked;
the table states what each formal group actually proves.}
\label{tab:lean-crosswalk}
\footnotesize
\begin{tabularx}{\linewidth}{@{}Y{.25\linewidth}Y{.30\linewidth}X@{}}
\toprule
Lean declaration group & Paper claim & Formalized scope\\
\midrule
\texttt{fixedGridFreeze} & fixed activation-grid no-go & half-cell quantizer
property and small-step premise\\
\texttt{errorFeedback.*} & conservation and terminal carry bound & exact
additive telescoping identity and scalar terminal bound\\
\texttt{commonLattice.*}, \texttt{signedBitWidth.*} & exact lattice recovery
and sufficient range & integer-code conservation; sufficient, not minimal,
capacity premise\\
\texttt{finiteRadius.*}, decision rules & feasible/impossible interface &
order-theoretic closing implications used by QReplace\\
\texttt{metadataDepth.*} & metadata and schedule capacity law & finite
injective encoding count\\
\texttt{fixedTeacher.*}, Voronoi closing & exact fixed-teacher result & scalar
recurrence algebra, maximal-schedule bound, and optimality implication under
the manuscript's depth-specific Voronoi premise\\
\texttt{TokenPair.*} & exact two-token attention reduction & symmetric
row-stochastic scalar-channel mean/disagreement algebra and coordinate
reconstruction\\
\bottomrule
\end{tabularx}
\end{table}

The factorization is conservative.  The common-lattice theorem
is stated in integer codes, with multiplication by the physical lattice unit
external.  The fixed-teacher formalization verifies the recurrence, closed
form, endpoint difference, maximal-schedule inequality, and Voronoi closing
lemma; the manuscript proves the analytic inequality establishing the
particular depth-dependent Voronoi premise.  The attention formalization
verifies the exact two-token mode algebra, not LayerNorm, multihead coupling,
or a complete Transformer block.

\subsection{Archived evidence identity and integrity}
\label{app:lean-evidence}

The verification run 
records a successful
integrated build in addition to the claim-specific stages.  The normalized
evidence ZIP contains the exact theorem sources, per-stage build logs,
per-stage axiom logs, theorem crosswalk, faithfulness audit, machine-readable
claim manifest, and an evidence-specific SHA-256 manifest.  The original Colab
global shell reported a different default Lean version, but every Lake build
trace invokes the project-pinned Lean~4 \texttt{v4.30.0} executable; the
evidence manifest records this distinction explicitly.  The source and evidence
archives included with the supplement reproduce the theorem-source digest and
preserve the unmodified kernel logs.

\subsection{Interoperability with QReplace}
\label{app:lean-qreplace-bridge}

The result ZIP contains a machine-readable manifest with the pinned toolchain,
source digest, claim identifiers, Lean declarations, status, and axiom-log
path.  QReplace may attach this file to a study specification and checks its
hash and requested claim.  A required claim that is absent or not marked
\texttt{kernel\_checked} prevents a certified recommendation.  Formal status
cannot upgrade unrelated numerical premises such as checkpoint fidelity,
common-tube validity, overflow safety, or a witness set's coverage of the full
input domain.

\subsection{Trust boundary}
\label{app:lean-trust}

Kernel acceptance establishes that an encoded conclusion follows from its
encoded assumptions and imported foundations.  It does not verify that a
deployed quantizer follows the declared model, that checkpoint-derived
constants hold on the deployment domain, or that a finite witness represents
a full-map claim.  The package does not claim Lean verification of the full
Banach-space relaxed-control synthesis theorem, the state-dependent hybrid
small-gain theorem, HJB verification for measurable relaxed trajectories,
occupation-measure equivalence, or moment/SOS convergence.  Those results
retain the conventional proofs and independent proof audit supplied with the
paper.

\FloatBarrier
\bibliography{references}
\end{document}